\documentclass{article}

\usepackage[preprint]{neurips_2023}

\usepackage{times}
\usepackage[utf8]{inputenc} 
\usepackage[T1]{fontenc}    
\usepackage{hyperref}       
\usepackage{url}            
\usepackage{booktabs}       
\usepackage{amsfonts}       
\usepackage{nicefrac}       
\usepackage{microtype}      
\usepackage{xcolor}         
\usepackage{lipsum}
\usepackage{marvosym}
\usepackage{dirtytalk}

\hypersetup{
    colorlinks,
    linkcolor=blue,
    citecolor=blue,
    urlcolor=blue
}

\usepackage{tikz}
\usetikzlibrary{arrows,automata}
\usetikzlibrary{positioning}
\usetikzlibrary{shapes,snakes}
\usepackage{adjustbox}

\definecolor{goodgoalcolor}{HTML}{85c0f9}
\definecolor{badgoalcolor}{HTML}{f5793a}

\usepackage{xspace}
\newcommand{\lava}{\textsc{Lava Gridworld}\xspace}
\newcommand{\pointgoal}{\textsc{PointGoal1}\xspace}
\newcommand{\pointgoalhard}{\textsc{PointGoal1-Hard}\xspace}

\newcommand{\pillar}{\textsc{Pillar}\xspace}

\usepackage{floatrow}
\newfloatcommand{capbtabbox}{table}[][\FBwidth]

\usepackage{amssymb}
\usepackage{dblfloatfix}    
\usepackage{todonotes}
\usepackage{subcaption}
\usepackage[ruled]{algorithm2e}

\usepackage{algorithmic}

\usepackage{mathrsfs}
\usepackage[inline]{enumitem}
\usepackage{calc}
\usepackage{nccmath}

\setuptodonotes{inline}

\usepackage{amsmath,amsfonts,bm}

\def\eqref#1{equation~\ref{#1}}

\def\1{\bm{1}}

\DeclareMathAlphabet{\mathsfit}{\encodingdefault}{\sfdefault}{m}{sl}
\SetMathAlphabet{\mathsfit}{bold}{\encodingdefault}{\sfdefault}{bx}{n}

\newcommand{\E}{\mathbb{E}}

\DeclareMathOperator*{\argmin}{arg\,min}

\makeatletter
\providecommand*{\barvee}{%
  \mathbin{%
    \mathpalette\@barvee{}%
  }%
}
\newcommand*{\@barvee}[2]{%
  \sbox0{$#1\veebar\m@th$}%
  \sbox2{%
    \hbox to \wd0{%
      \hss
      \resizebox{1.05\wd0}{\height}{$#1-\m@th$}%
      \hss
    }%
  }%
  \sbox4{%
    \resizebox{\wd0}{.7\ht0}{$#1\vee\m@th$}%
  }%
  \sbox6{$#1\vcenter{}$}
  \ht2=\ht6 %
  \vbox to \ht0{%
    \copy2 %
    \vss
    \copy4 %
  }%
}
\makeatother

\newcommand{\state}{\mathcal{S}}
\newcommand{\action}{\mathcal{A}}

\newcommand{\reward}{R}
\newcommand{\rmax}{\reward_{\text{MAX}}}
\newcommand{\rmin}{\reward_{\text{MIN}}}
\newcommand{\goals}{\mathcal{G}}
\newcommand{\unsafegoals}{\mathcal{G^!}}
\renewcommand{\v}{V}

\newcommand{\pistar}{\pi^*}

\usepackage{mathtools}

\newcommand{\rbar}{\bar{\reward}}
\newcommand{\rbarmin}{\rbar_{\text{MIN}}}

\newcommand{\BlackBox}{\rule{1.5ex}{1.5ex}}  
\newenvironment{proof}{\par\noindent{\bf Proof\ }}{\hfill\BlackBox\\[2mm]}
 
\newtheorem{theorem}{Theorem}

\newtheorem{definition}{Definition}

\newlength{\strutheight}
\newcommand{\vmin}{\v_{\text{MIN}}}
\newcommand{\vmax}{\v_{\text{MAX}}}

\newcommand{\circled}[2]{\tikz[baseline=(char.base)]{
            \node[shape=circle,draw=#2,fill=#2,inner sep=1pt] (char) {#1};}}

\newcommand{\sgood}{\circled{$s_3$}{goodgoalcolor}}
\newcommand{\sbad}{\circled{$s_1$}{badgoalcolor}}

\title{ROSARL: Reward-Only Safe Reinforcement Learning}

\author{
Geraud Nangue Tasse$^1$,
Tamlin Love$^1$,
Mark Nemecek$^2$,
Steven James$^1$,
Benjamin Rosman$^1$
\\
$^1$University of the Witwatersrand, Johannesburg, South Africa
\\
$^2$Department of Computer Science, Duke University, Durham,
North Carolina, USA
\\
geraudnt@gmail.com,
tamlinlollislove@gmail.com,
markn@cs.duke.edu,
\\
\{Steven.James,benjamin.rosman1\}@wits.ac.za
}

\begin{document}

\maketitle

\begin{abstract}
\looseness=-1
An important problem in reinforcement learning is designing agents that learn to solve tasks safely in an environment. A common solution is for a human expert to define either a penalty in the reward function or a cost to be minimised when reaching unsafe states. However, this is non-trivial, since too small a penalty may lead to agents that reach unsafe states, while too large a penalty increases the time to convergence. Additionally, the difficulty in designing reward or cost functions can increase with the complexity of the problem. Hence, for a given environment with a given set of unsafe states, we are interested in finding the upper bound of rewards at unsafe states whose optimal policies minimises the probability of reaching those unsafe states, irrespective of task rewards. We refer to this exact upper bound as the \textit{Minmax penalty}, and show that it can be obtained by taking into account both the controllability and diameter of an environment. We provide a simple practical model-free algorithm for an agent to learn this Minmax penalty while learning the task policy, and demonstrate that using it leads to agents that learn safe policies in high-dimensional continuous control environments. 
[
\href{https://github.com/geraudnt/rosarl}{github.com/geraudnt/rosarl}
]

\begin{figure}[h!]
    \centering
    \includegraphics[width=\linewidth]{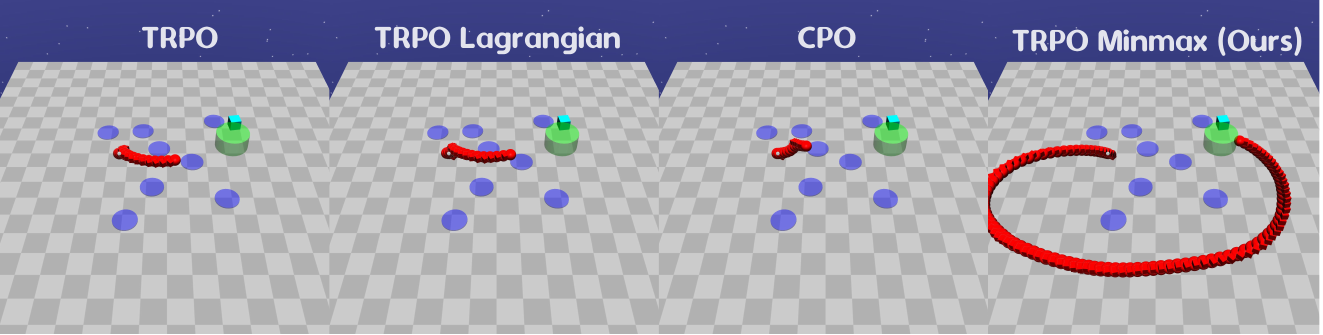}
    \caption{Example trajectories of prior work---TRPO \citep{schulman2015trust} (left-most), TRPO-Lagrangian \citep{Ray2019} (midle-left), CPO \citep{achiam2017constrained} (middle-right)---compared to ours (right-most). For each, a point mass agent learns to reach a goal location (green cylinder) while avoiding unsafe regions (blue circles). The cyan block is a randomly placed movable obstacle.}
    \label{fig:trajectory_safe}
\end{figure}%

\end{abstract}

\section{Introduction}

\looseness=-1
Reinforcement learning (RL) has recently achieved success across a variety of domains, such as video games \citep{shao2019survey}, robotics  \citep{kalashnikov2018scalable,kahn2018self} and autonomous driving \citep{kiran2021deep}. 
However, if we hope to deploy RL in the real world, agents must be capable of completing tasks while avoiding unsafe or costly behaviour. 
For example, a navigating robot must avoid colliding with objects and actors around it, while simultaneously learning to solve the required task. An example of this desired behaviour is illustrated in Figure~\ref{fig:trajectory_safe}.

Many approaches in RL deal with this problem by allocating arbitrary penalties to unsafe states when hand-crafting a reward function. However, the problem of specifying a reward function for desirable, safe behaviour is notoriously difficult \citep{amodei2016concrete}. \emph{Importantly, penalties that are too small will result in unsafe behaviour, while penalties that are too large may result in increased learning times. }
Furthermore, these rewards must be specified by an expert for each new environment an agent faces.
If our aim is to design truly autonomous, general agents, it is then simply impractical to require that a human designer specify penalties to guarantee optimal but safe behaviours for every task. 

When safety is an explicit goal, a common approach is to constrain policy learning according to some threshold on cumulative cost \citep{schulman2015trust,Ray2019,achiam2017constrained}. While effective, these approaches require the design of a cost function whose specification can be as challenging as designing a reward function. Additionally, these methods may still result in unacceptably frequent constraint violations in practice due to the large cost threshold typically used.

\looseness=-1
Rather than attempting to both maximise a reward function and minimise a cost function, which requires specifying both rewards and costs and a new learning objective, we should simply aim to have a better reward function---since we then do not have to specify yet another scalar signal nor change the learning objective. 
This approach is consistent with the \textit{reward hypothesis} \citep{sutton2018reinforcement} which states:
\textcolor{red}{
\say{
All of what we mean by goals and purposes can be well thought of as maximisation of the expected value of the cumulative sum of a received scalar signal (reward).
}
}
Therefore, the question we examine in this work is how to determine the smallest penalty assigned to unsafe states such that the probability of reaching them is minimised by an optimal policy.
Rather than requiring an expert's input, we show that this penalty can be defined by the \textit{controllability} and \textit{diameter} of an environment, and in practice can be autonomously estimated by an agent using its current value estimates. 
More specifically, we make the following contributions:
\begin{enumerate*}[label=(\roman*)]
    \item \textbf{The Minmax penalty}: We obtain the analytical form of the penalty and prove that its use results in learned behaviours that minimise the probability of visiting unsafe states;
    \item \textbf{Estimating the Minmax penalty}: We propose a simple model-free algorithm for estimating this penalty online, which can be integrated into any RL pipeline that learns value functions; 
    \item \textbf{Experiments}: We empirically demonstrate that our method outperforms existing baselines in high-dimensional, continuous control tasks, learning to complete tasks while minimising unsafe behaviour.
\end{enumerate*}
    Our results demonstrate that, while prior methods often violate safety constraints, our approach results in agents capable of learning to solve tasks while avoiding unsafe states.

\section{Background}

We consider the typical RL setting where the task faced by an agent is modelled by a Markov Decision Process (MDP). An MDP is defined as a tuple $\langle \state, \action, P, \reward \rangle$, where $\state$ is the set of states, $\action$ is the set of actions, $P:\state\times\action\times\state\to[0~1]$ is the transition function, and $\reward:\state\times\action\times\state\to\mathbb{R}$ is a reward function bounded by $[\rmin ~ \rmax]$. Our particular focus is on undiscounted MDPs that model stochastic shortest path problems in which an agent must reach some goal states in the non-empty set of absorbing states $\goals \subset \state$ \citep{bertsekas1991analysis} while avoiding unsafe absorbing states $\unsafegoals \subset \goals$. The non-empty set of non-absorbing states $\state \setminus \goals$ are referred to as \textit{internal states}. 

A \textit{policy} $\pi:\state\to\action$ is a mapping from states to actions. The \textit{value function} $V^{\pi}(s) = \mathbb{E}[\sum_{t=0}^{\infty}\reward(s_t,a_t,s_{t+1})]$ associated with a policy specifies the expected return under that policy starting from state $s$. The goal of an agent in this setting is to learn an optimal policy $\pi^{*}$ that maximises the value function $V^{\pi^{*}}(s) = V^{*}(s) = \max_\pi V^{\pi}(s) $ for each $s \in \state$. Since tasks are undiscounted, an optimal policy is guaranteed to exist by assuming that the value function of \textit{improper policies} is unbounded from below---where \textit{proper policies} are those that are guaranteed to reach an absorbing state \citep{van2019composing}. Since there always exists a deterministic optimal policy \citep{sutton1998introduction}, and all optimal policies are proper, we will focus our attention to the set of all deterministic proper policies $\Pi$.


\section{Avoiding Unsafe Absorbing States} \label{sec:idea}


\begin{figure*}[b!]
    \centering
    \begin{subfigure}[t]{0.35\textwidth}
        \centering
        \begin{adjustbox}{width=\linewidth}
        \begin{tikzpicture}[>=stealth',shorten >=1pt,auto,node distance=4cm, line width=0.4mm, font = \Large]
  \node[state,accepting,minimum size=1.2cm,fill=goodgoalcolor] (s3)      {$s_3$};
  \node[state, minimum size=1.2cm]         (s2) [right of=s3]  {$s_2$};
  \node[state,minimum size=1.2cm]         (s0) [below=2.3cm of s2] {$s_0$};
  \node[state,accepting,minimum size=1.2cm,fill=badgoalcolor]         (s1) [left of=s0] {$s_1$};

  \path[->]          (s3)  edge [loop left] node {$a_1 / a_2$} (s3);

  \draw[->] (s2) to node[above] {$a_1 / a_2$} node[below] {w.p. $(1-p)$} (s3);

  \path[->]          (s2)  edge [loop right] node {w.p. $p$} (s2);
  
  \draw[->] (s0) to[align=center, bend left] node[left] {$a_1$ w.p. $(1-p)$} (s2);
  \draw[->] (s0) to[align=center, bend right] node[right] {$a_2$ \\ w.p. $p$} (s2);

  \draw[->] (s0) to[align=center,bend right] node[above] {$a_1$} node[below] {w.p. $p$} (s1);
  \draw[->] (s0) to[align=center,bend left] node[above] {$a_2$} node[below] {w.p. $(1-p)$} (s1);

  \path[->]          (s1)  edge [loop left] node {$a_1 / a_2$} (s1);

\end{tikzpicture}
        \end{adjustbox}
        \caption{Chain-walk MDP}
        \label{fig:chainwalkmdp}
    \end{subfigure}%
    ~
    \begin{subfigure}[t]{0.21\textwidth}
        \centering
        \includegraphics[width=\linewidth]{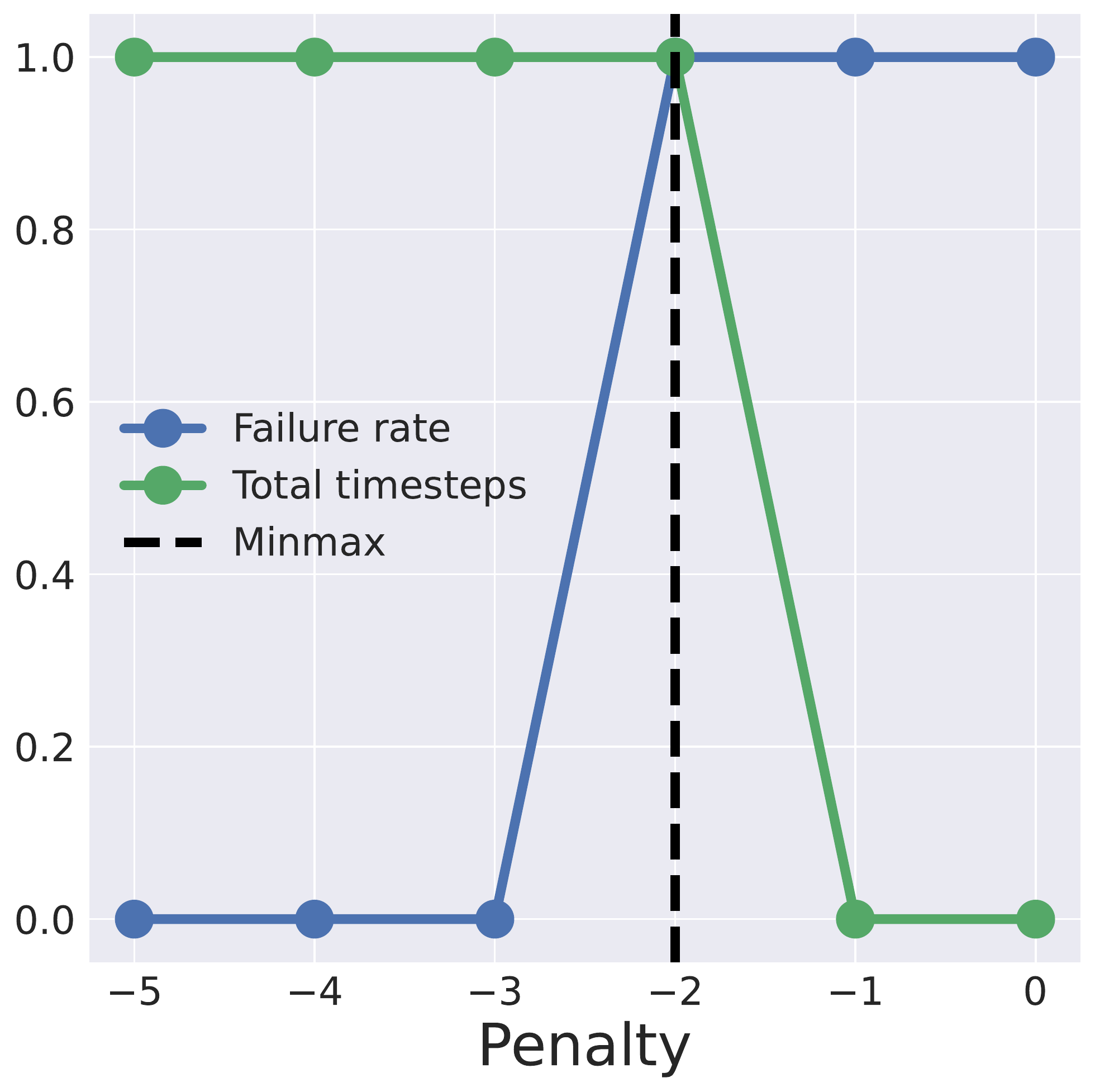}
        \caption{$p=0$}
        \label{fig:p0}
    \end{subfigure}%
    \centering
    \begin{subfigure}[t]{0.21\textwidth}
        \centering
        \includegraphics[width=\linewidth]{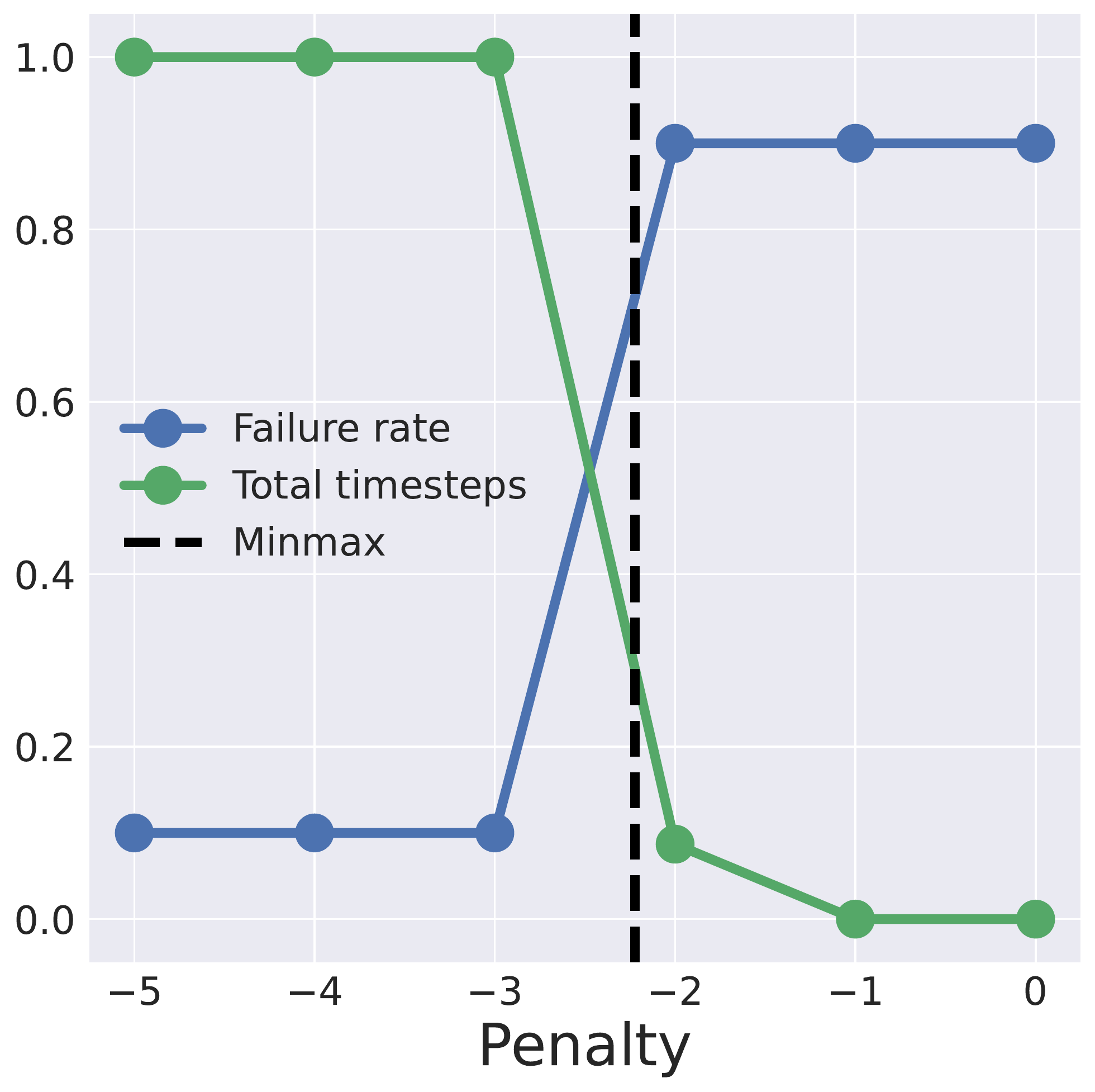}
        \caption{$p=0.1$}
        \label{fig:p0.1}
    \end{subfigure}%
    \centering
    \begin{subfigure}[t]{0.21\textwidth}
        \centering
        \includegraphics[width=\linewidth]{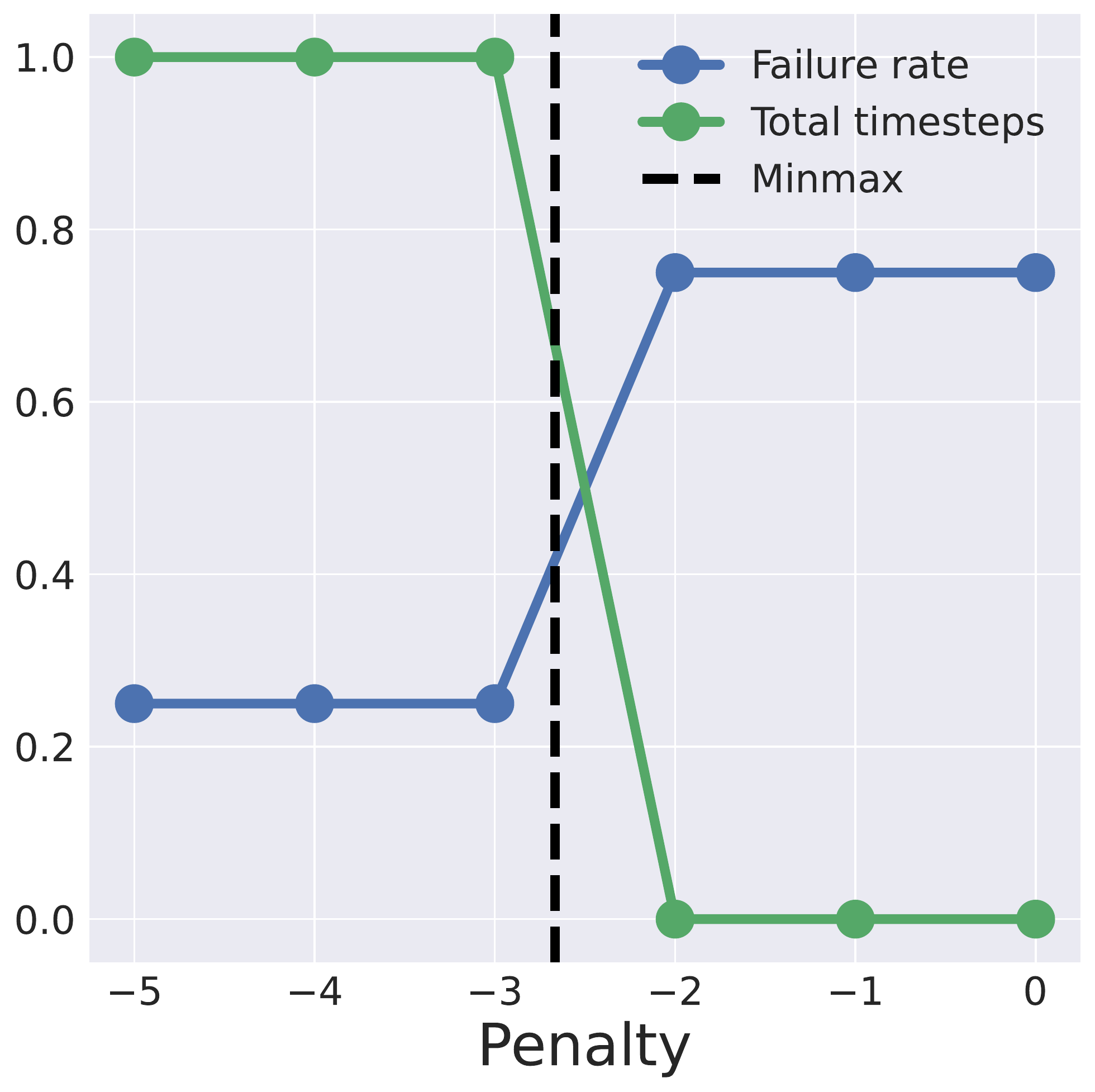}
        \caption{ $p=0.25$}
        \label{fig:p0.25}
    \end{subfigure}%
    \caption{A simple chain-walk MDP (a), and the effect of different choices of penalties (b--d) on the failure rate of optimal policies and the total timesteps needed to learn them (using value iteration \protect\citep{sutton1998introduction}) in it. The total-timesteps are normalised to lie in $[0~1]$, and the Minmax penalty is as per Definition \protect\ref{def:penalty}.}
    \label{fig:chainwalk}
\end{figure*}
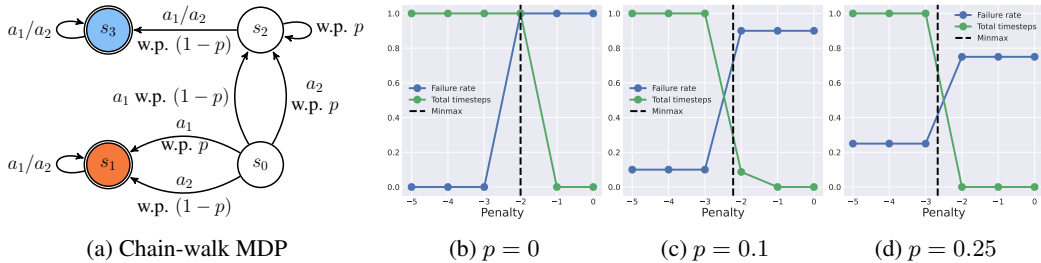%

Given an MDP, our goal is to determine the smallest penalty (and hence the largest reward) to use as feedback in unsafe states to obtain safe policies.
We formally define a safe policy as one that minimises the probability of reaching any unsafe terminal state from any internal state:

\begin{definition}[Safe Policy]
Where $s_{T}$ is the final state of a trajectory and $\unsafegoals \subset \goals$ is the non-empty set of unsafe absorbing states, let $P^{\pi}_{s}(s_{T} \in \unsafegoals)$ be the probability of reaching $\unsafegoals$ from $s$ under a proper policy $\pi \in \Pi$.
Then $\pi$ is called \textit{safe} if
\[
\pi \in \argmin\limits_{\pi' \in \Pi} P^{\pi'}_{s}(s_{T} \in \unsafegoals) \quad \textit{for all } s \in \state.
\]
\label{def:mastery}
\end{definition}

\looseness=-1
To illustrate the difficulty in designing reward functions for safe behaviour, consider the simple \textit{chain-walk} MDP in Figure \ref{fig:chainwalkmdp}. 
The MDP consists of four states $s_0, \ldots, s_3$, the agent has two actions $a_1, a_2$, the initial state is $s_0$, and the diagram denotes the transitions and associated probabilities when executing actions. The absorbing transitions have a reward of $0$ while all other transitions have a reward of $-1$ and the agent must reach the goal state $\sgood$, but not the unsafe state $\sbad$. 

Since the transitions per action are stochastic, controlled by $p\in[0~1]$, and $\sgood$ is further from the start state $s_0$ than $\sbad$, the agent may not always be able to avoid $\sbad$. In fact, for $p = 0 $ and $-1$ reward for transitions into $\sbad$, the optimal policy is to always pick $a_2$ which always reaches $\sbad$. However, for a sufficiently high penalty for reaching $\sbad$ (any penalty higher than $-2$), the optimal policy is to always pick action $a_1$, which always reaches $\sgood$.
Finally, observe that as $p$, the probability that the agent can transition from $s_2$ to $\sgood$ decreases. Therefore, the penalty for $\sbad$ must be sufficiently large to ensure an optimal policy will not simply reach $\sbad$ to avoid self-transitions in $s_2$.
To capture this relationship between the stochasticity of an MDP and the required penalty to obtain safe policies, 
we introduce the notion of \textit{controllability}, which measures how much of an effect actions have on the transition dynamics.

\begin{definition}[Controllability]
\label{def:controllability}
Define the degree of controllability of the environment with
\[
C = \min_{\substack{\pi_1\not=\pi_2 \in \Pi}} \max_{\substack{s \in \state}} |\Delta P_{s}(\pi_1,\pi_2)|,
\]
where $\Delta P_{s}(\pi_1,\pi_2) \coloneqq P^{\pi_1}_{s}(s_{T} \not\in \unsafegoals) - P^{\pi_2}_{s}(s_{T} \not\in \unsafegoals) $.
\end{definition}

\looseness=-1
$C$ measures the degree of controllability of the environment by simply taking the smallest difference between the probability of reaching a safe goal state by following any two policies. For example, using deterministic policies, a deterministic task will have $C=1$ , and $C$ will tend to $0$ as the task's stochasticity increases.
At $C=0$, the MDP becomes a Markov reward process since the agent's actions no longer affect its transition probabilities.
This can be seen in Figure \ref{fig:controllability} for the chain-walk MDP with different choices for $p$. Since actions in $s_2$ do not affect the transition probability, there are only $2$ relevant deterministic proper policies $\pi_1(s)=a_1$ and $\pi_2(s)=a_2$. Here, $C=1$ when $p=0$ because the task is deterministic. $C=0$ when $p=0.5$ because the dynamics are uniformly random, and $C=0$ when $p=1$ because the agent can never reach $\sgood$ from $s_2$ even though the dynamics are deterministic.



\begin{figure*}[b!]
    \centering
    \begin{subfigure}[t]{0.3\textwidth}
        \centering
        \includegraphics[width=\linewidth]{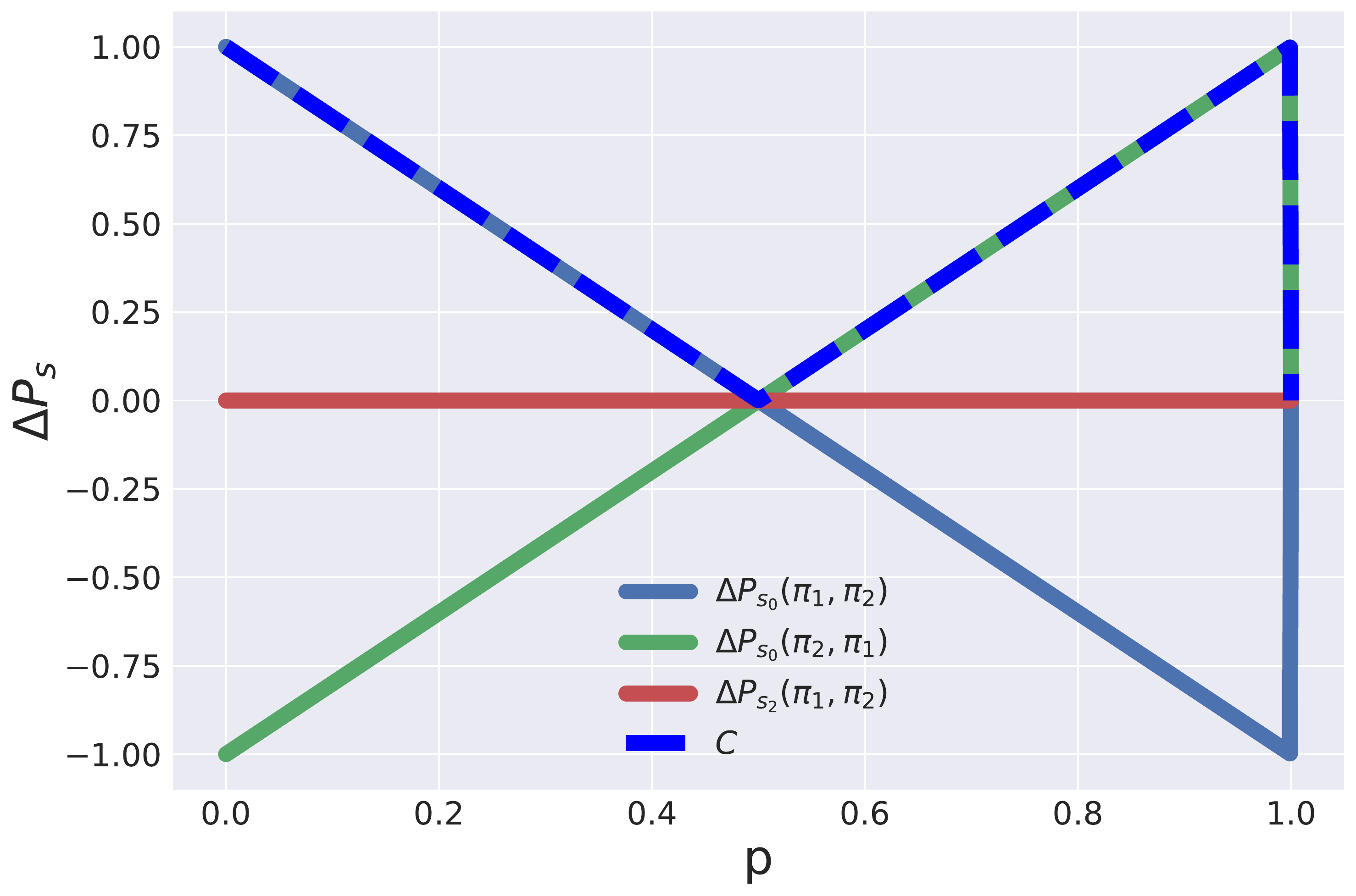}
        \caption{Controllability of the chain-walk MDP. Here,  $\pi_1(s)=a_1$ and $\pi_2(s)=a_2$ giving $\Delta P_{s_0}(\pi_1,\pi_2)$ $= 1-2p$, $\Delta P_{s_0}(\pi_2,\pi_1)$ $= 2p-1$ and $\Delta P_{s_2}(\pi_1,\pi_2)$ $= \Delta P_{s_2}(\pi_2,\pi_1) = 0$.}
        \label{fig:controllability}
    \end{subfigure}%
    \quad
    \centering
    \begin{subfigure}[t]{0.3\textwidth}
        \centering
        \includegraphics[width=\linewidth]{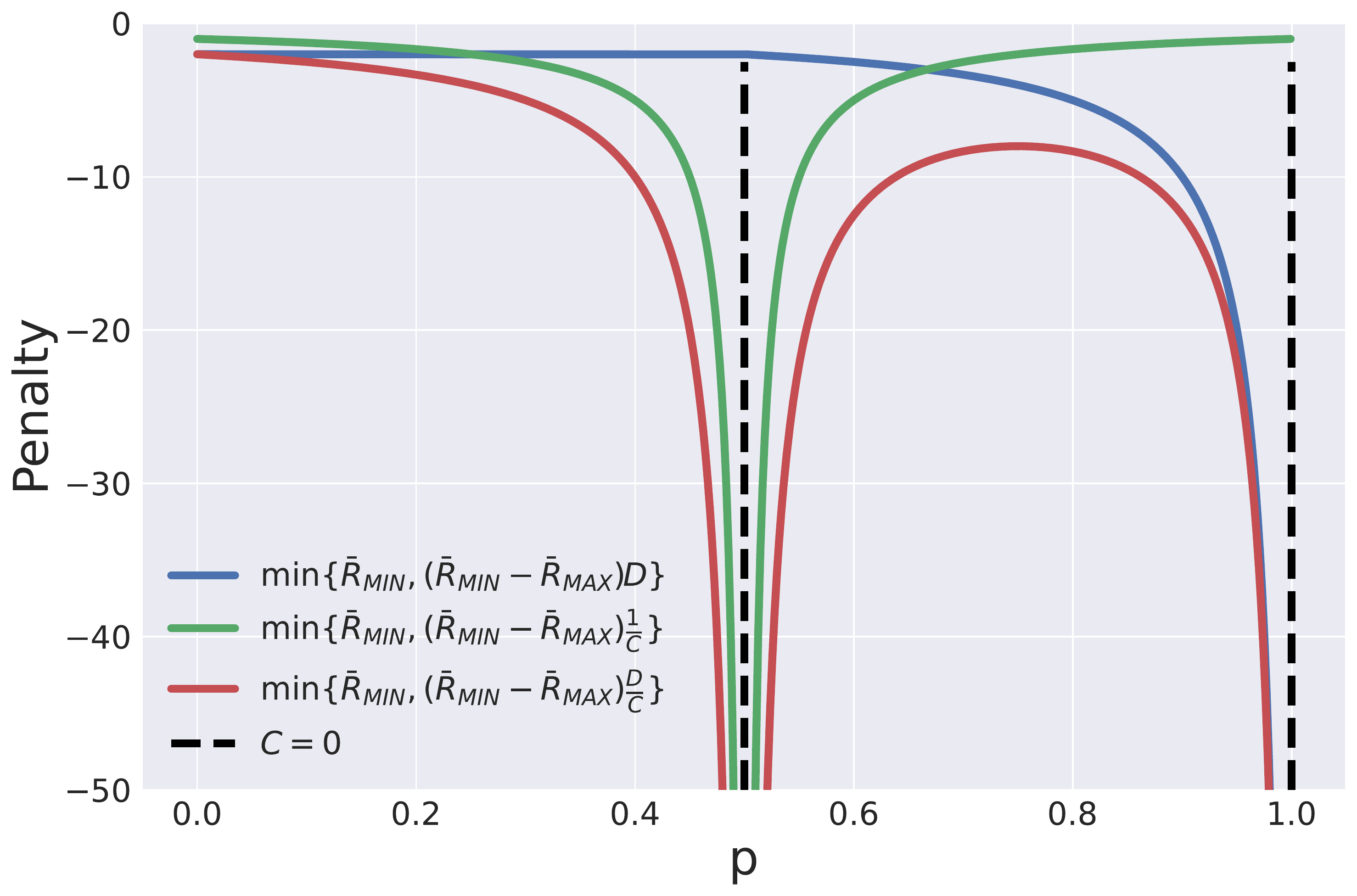}
        \caption{The Minmax penalty (red) compared to accounting for the diameter $D$ and controllability $C$ only (blue and green respectively) in the chain-walk MDP for varying $p$. $p=0.5$ and $p=1$ are excluded since their corresponding controllability is $0$. 
        }
        \label{fig:penalty}
    \end{subfigure}%
    \quad
    \centering
    \begin{subfigure}[t]{0.3\textwidth}
        \centering
        \includegraphics[width=0.9\linewidth]{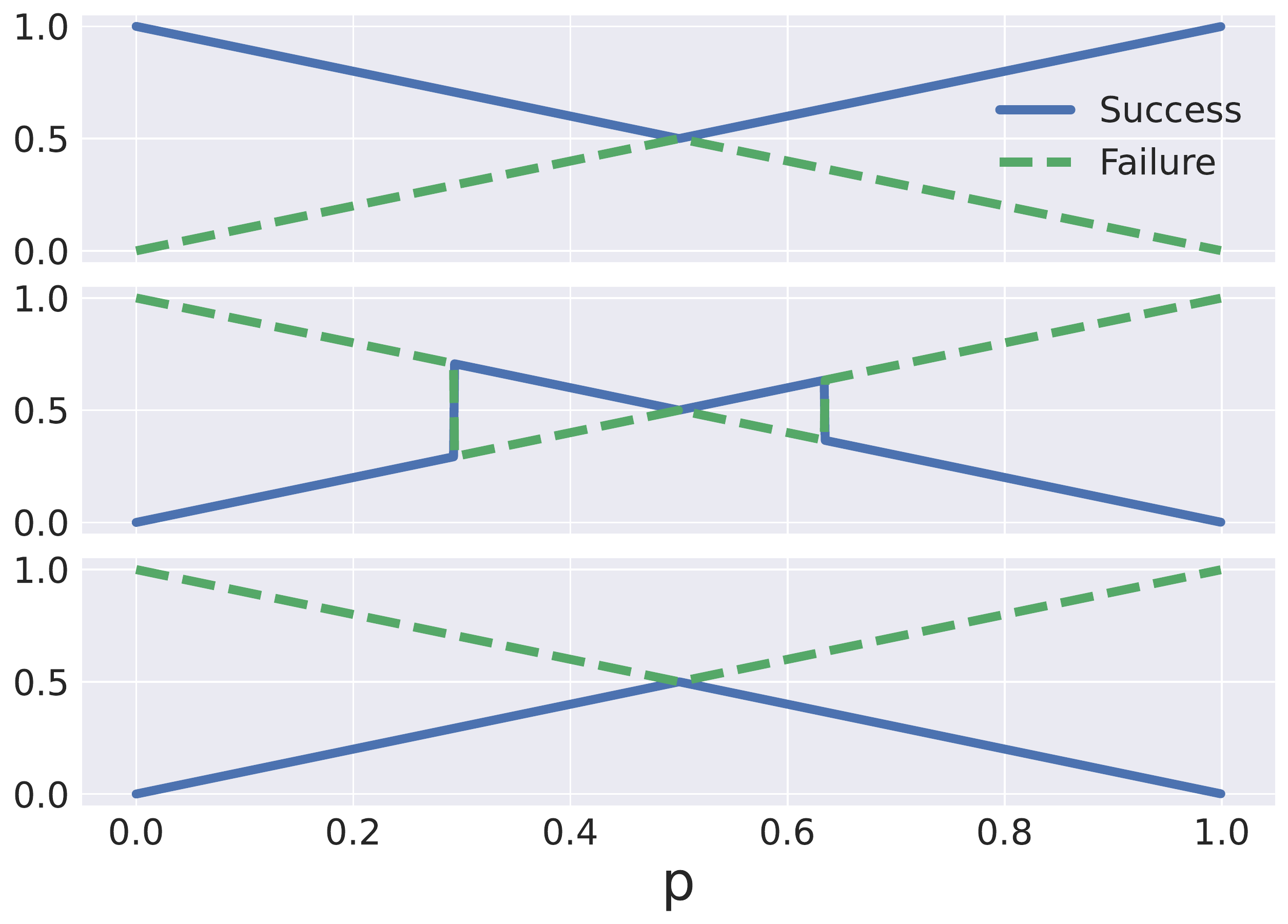}
        \caption{ Success rate of optimal policies in the chain-walk MDP for different unsafe state penalties: \textbf{(top)} $\rbarmin - \epsilon$, \textbf{(middle)} $\min \{\rmin, (\rmin - $$\rmax)\frac{1}{C} \}$, \textbf{(bottom)} $\min \{\rmin, (\rmin - \rmax)D \}$. 
        }
        \label{fig:success}
    \end{subfigure}%
    \caption{Effect of stochasticity ($p$) on the controllability ($C$), Minmax penalty ($\rbar_{MIN}$), and success rate of optimal policies in the chain-walk MDP. The diameter and optimal policies for varying $p$ are obtained via value iteration \protect\citep{sutton1998introduction}. We use $\epsilon=10^{-10}$ in (c).  }
    \label{fig:stochasticity}
\end{figure*}%

Clearly, the size of the penalty that needs to be given for unsafe states also depends on the \textit{size} of the MDP. We define this size as the \textit{diameter} of the MDP, which is the highest number of expected timesteps required to reach an absorbing state from an internal state:

\begin{definition}[Diameter]
\label{def:diameter}
Define the diameter of an MDP as 
\[
D = \max_{s \in \state} \max_{\pi \in \Pi} \E \left[ T(s_T \in \goals | \pi) \right],
\]
where $T$ is the number of timesteps taken to first reach $\goals$ from $s$ under a proper policy $\pi$.
\end{definition}

Given the controllability and diameter of an MDP, we can now obtain the upper bound of rewards which minimise the probability of reaching unsafe states. We call this upper bound the \textit{Minmax penalty}:

\begin{definition}[Minmax Penalty]
\label{def:penalty}
Define the Minmax penalty for a controllable MDP ($C>0$) as
$$\rbarmin = \min \{\rmin, (\rmin - \rmax)\frac{D}{C} \}.$$
\end{definition}

This choice of $\rbarmin$ says that since stochastic shortest path tasks require an agent to learn to achieve desired terminal states, if the agent enters an unsafe terminal state, it should receive the largest penalty possible by a proper policy.
For example, Figure \ref{fig:penalty} shows the Minmax penalty for varying choices of $p$.
We now show in Theorem \ref{theorem:safety} that $\rbarmin$ is indeed the upper bound of rewards that minimises the probability of reaching unsafe terminal states.


\begin{theorem}[Safety Bounds]
\label{theorem:safety}
Consider a controllable MDP $\langle \state, \action, P, \reward \rangle$ with a non-empty set of unsafe goal states $\unsafegoals \subset \goals$. Let $\pistar$ be an optimal policy for the modified MDP with possibly different rewards in $\unsafegoals$: $\langle \state, \action, P, \rbar \rangle$ where $\rbar(s,a,s')=\reward(s,a,s')$ for all $s'\not\in\unsafegoals$.
\begin{enumerate}[label=(\roman*)]
  \item If  $\rbar(s,a,s') < \rbarmin$ ~ for all $s'\in\unsafegoals$, then $\pistar$ is safe for all $R$;
  \item If  $\rbar(s,a,s') > \rbarmin$ ~ for all $s'\in\unsafegoals$, then $\pistar$ is unsafe for all $R$.
\end{enumerate}
  
\end{theorem}

\looseness=-1
Theorem \ref{theorem:safety} says that for any MDP whose rewards at unsafe goal states are bounded above by $\rbarmin$, the optimal policy minimises the probability of reaching those states.
Interestingly, this is a strict upper bound, and any reward higher than it will have optimal policies that do not minimise the probability of reaching unsafe states.
Hence any penalty $\rbarmin - \epsilon$, where $\epsilon>0$ can be arbitrarily small, will guarantee safe optimal policies.
This is demonstrated in Figure~\ref{fig:success} which shows that the Minmax penalty (the top plot) maximises the success rate of the resulting optimal policies. 
It also demonstrates why taking into account both the diameter and controllability of an MDP in calculating the Minmax penalty is necessary, because either one alone does not always maximise the success rate.
In fact, using the diameter alone is the worst choice for the chain-walk MDP (the bottom plot).
Interestingly, and perhaps unexpectedly, we notice that 
the small difference between the inverse controllability ($1/C$) and the Minmax penalty when $p$ is close to $0.5$ is sufficient to learn safe optimal policies (the middle plot); however, the even smaller difference between the diameter and the Minmax penalty when $p$ is close to $1$ results in unsafe policies. This suggests that the controllability of an MDP may be the most critical factor to consider when designing safe agents, at least compared to the diameter.

\section{Estimating the Minmax Penalty}\label{LearningMinMax}

The Minmax penalty of an MDP can be inferred by estimating its controllability and diameter. However, estimating the controllability requires knowledge of the termination probabilities of all proper policies over all states. Similarly, estimating the diameter requires involves the expected timesteps to termination of all proper policies over all states. These are clearly difficult quantities to estimate from experience. While they can be determined using dynamic programming (in the case of known dynamics), it is still impractical to do so for larger state and policy spaces.
This is further complicated by the need to also learn the true optimal policy for the estimated Minmax penalty.


\looseness=-1
Given the above challenges, we require a practical method for estimating the Minmax penalty. Ideally, this method should require no knowledge of the environment dynamics, and easily integrate with existing RL approaches.
We therefore propose an approach that makes use of the value function being learned by any off-the-shelf model-free RL algorithm. 
To achieve this, we first note that $(\rmin - \rmax)\frac{D}{C} = (D \rmin - D \rmax)\frac{1}{C} = (V_{MIN} - V_{MAX})\frac{1}{C}$. 
Hence, in practice, we need only estimate $V_{MIN}$ and $V_{MAX}$, since $C=1$ if safe goals are reachable and the environment is deterministic. If the environment is not deterministic, then the agent will still slowly increase its penalty every time it accidentally enters an unsafe state (as the value function estimate gets more negative or less positive). We demonstrate this in the \lava experiments in the next section (Figure \ref{fig:exp_2_penalty}).

\looseness=-1
The Minmax penalty can easily be estimated using observations of the reward and an agent's estimate of the value function, as shown in Algorithm \ref{alg:minmax}. This method requires initial estimates of $\rmin$ and $\rmax$, which in this work are initialised to $0$. 
The agent receives a reward after each environment interaction and updates the estimate of the Minmax penalty $\rbarmin$ as $\rbarmin=\min(\rmin,\vmin-\vmax)$. Whenever an agent encounters an unsafe state, the reward can be replaced by $\rbarmin$ so as to disincentivise unsafe behaviour. 
Since $\vmax$ is estimated using $\vmax \gets  \max(\vmax,\rmax,V(s_t))$ where $V(s_t)$ is the learned value function at time step $t$, it leads to an optimistic estimation of $\rbarmin$. Hence, in practice, we find no need to add $\epsilon>0$ to $\rbarmin$.

\SetAlgoNoLine
\begin{algorithm}
\caption{RL while learning Minmax penalty}
 \label{alg:minmax}
\DontPrintSemicolon
    \SetKwInOut{Input}{Input}
    \SetKwInOut{Initialise}{Initialise}
 \Input{RL algorithm $\mathbf{A}$, max timesteps $T$ \;}
 \Initialise{ $\rmin=0,\rmax=0,\vmin=\rmin,\vmax=\rmax$, $\pi$ and $V$ as per $\mathbf{A}$ \;}
 \begin{algorithmic}
    \STATE \textbf{observe} $s_0$
    \FOR{t in T}
        \STATE \textbf{select action} $a_t$ using $\pi$ as per $\mathbf{A}$ 
        \STATE \textbf{observe} $s_{t+1},r_t$
        \STATE $\rmin \gets  \min(\rmin,r_t)$
        \STATE $\rmax \gets  \max(\rmax,r_t)$
        \STATE $\vmin \gets  \min(\vmin,\rmin,V(s_t))$
        \STATE $\vmax \gets  \max(\vmax,\rmax,V(s_t))$
        \STATE $\rbarmin=\min(\rmin,\vmin-\vmax)$
        \IF{$s_{t+1}$ is unsafe}
            \STATE $r_t \gets \rbarmin$
        \ENDIF
        \STATE \textbf{update} $\pi$ and $V$ with $(s_t,a_t,s_{t+1},r_t)$ as per $\mathbf{A}$
    \ENDFOR
 \end{algorithmic}
\end{algorithm}


\section{Experiments}
To demonstrate that an agent estimating the Minmax penalty for unsafe states is able to learn an optimally safe policy, we evaluate our approach across three domains, as illustrated by Figure~\ref{fig:domains}.


\begin{figure*}[t!]
    \centering
    \begin{subfigure}[t]{0.3\textwidth}
        \centering
        \includegraphics[width=\linewidth]{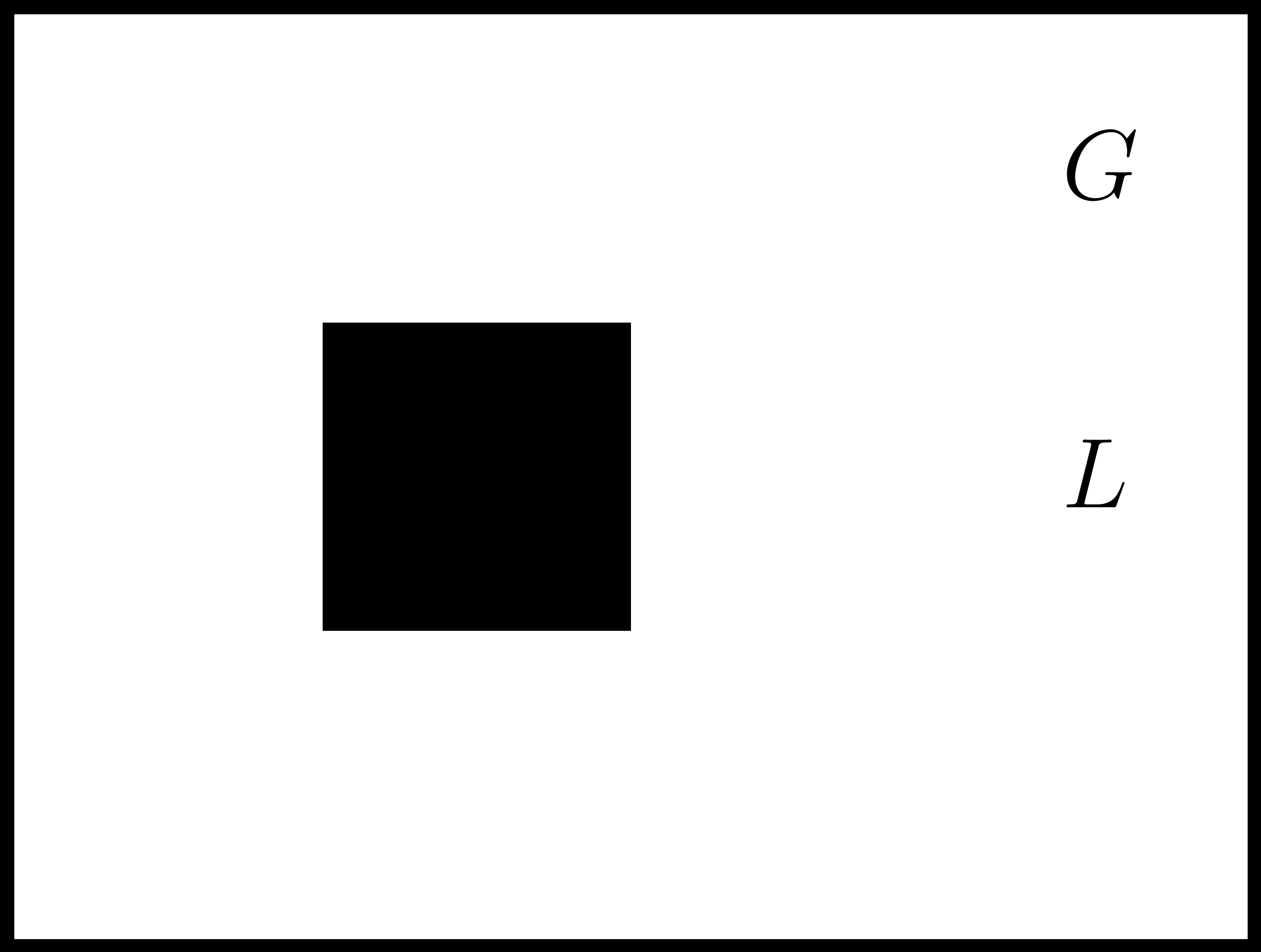}
        \caption{\lava}
        \label{fig:lava}
    \end{subfigure}%
    \quad
    \centering
    \begin{subfigure}[t]{0.3\textwidth}
        \centering
        \includegraphics[width=\linewidth]{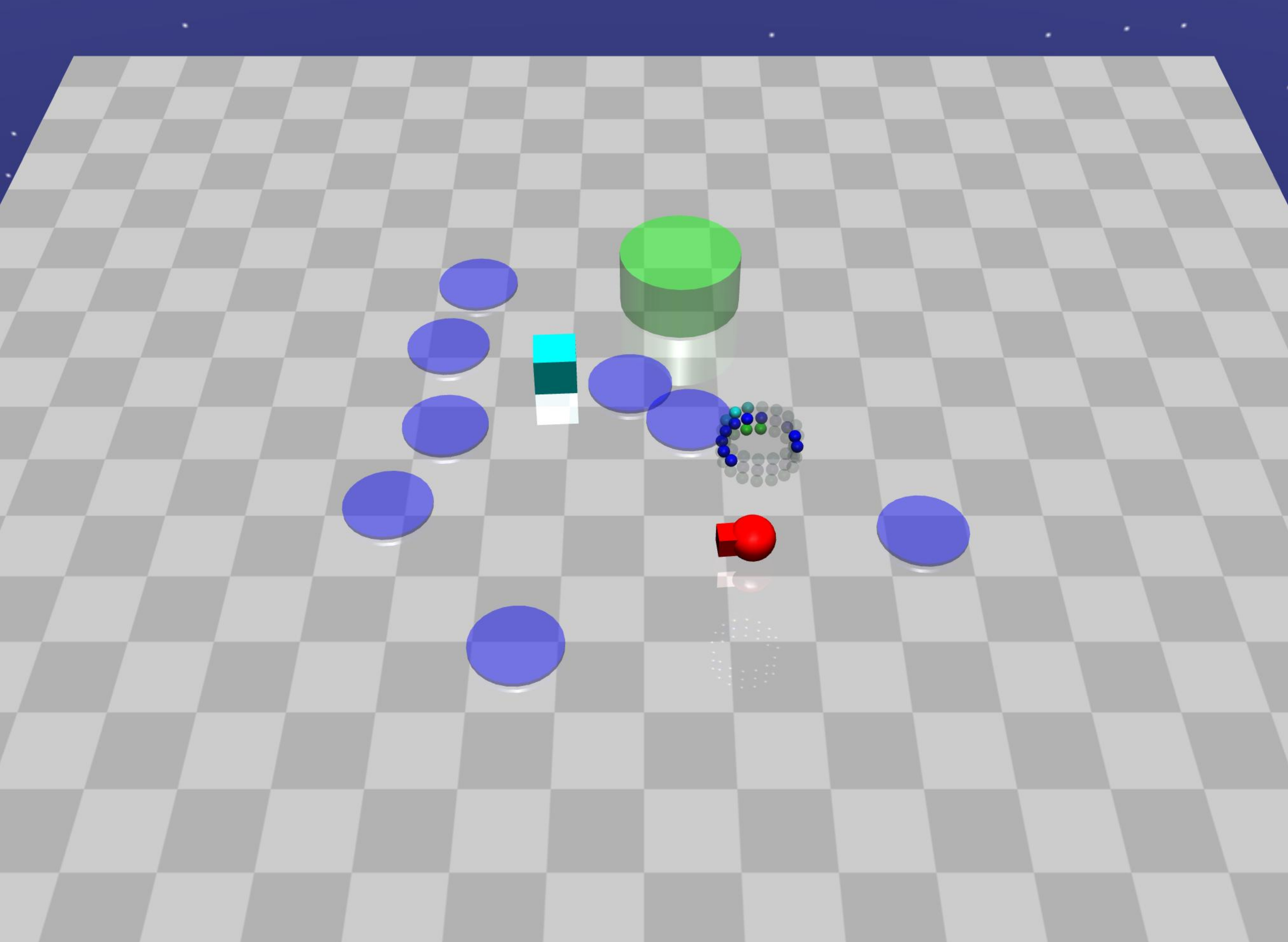}
        \caption{Safety Gym \pointgoalhard}
        \label{fig:goal1}
    \end{subfigure}%
    \quad
    \centering
    \begin{subfigure}[t]{0.3\textwidth}
        \centering
        \includegraphics[width=\linewidth]{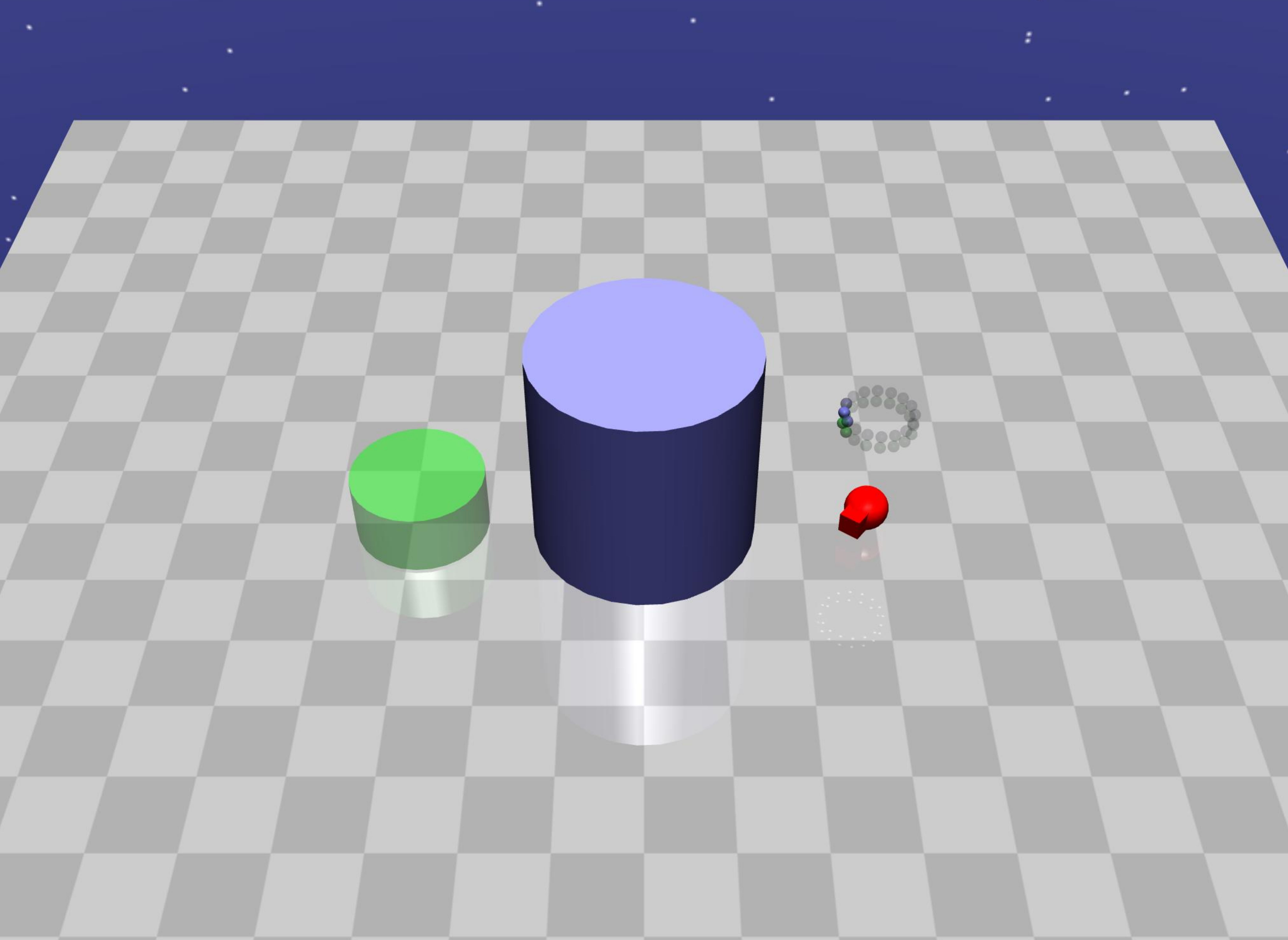}
        \caption{ Safety Gym Pillar}
        \label{fig:pillar}
    \end{subfigure}%
    \caption{Experiments domains. (a) An agent must navigate a gridworld to goal $G$ while avoiding the lava, denoted $L$. The black region represents an impassable wall that must be navigated around. (b) and (c) The agent in red must navigate to the green region, while avoiding the blue hazards.}
    \label{fig:domains}
\end{figure*}%

\paragraph{Lava Gridworld}

The first domain is a simple tabular environment, the \lava domain, in which an agent must reach a goal location while avoiding a lava location. A wall is also present in the environment and, while not unsafe, must be navigated around. The environment has a \textit{slip probability (sp)} of 0.25, so that with probability 0.25 the agent's action is overridden with a random action. The agent receives a positive reward (+1) for reaching the goal, as well as negative rewards (-0.1) at each timestep to incentivise taking the shortest path to the goal.
To test our approach, we modify Q-learning \citep{watkins89} with $\epsilon$-greedy exploration such that the agent updates its estimate of the Minmax penalty as learning progresses and uses it as the reward whenever the lava state is reached, following the procedure outlined in Section \ref{LearningMinMax}. The action-value function is initialised to $0$ for all states and actions, $\epsilon = 0.1$ and the learning rate $\alpha=0.1$. The experiments are run over 10,000 episodes and averaged over 70 runs.

\paragraph{Safety Gym PointGoal1-Hard}\label{SafetyGym}

The second domain is a modified version of the \pointgoal task from OpenAI's Safety Gym environments \citep{Ray2019}, which represents a complex, high-dimensional, continuous control task. Here a simple robot must navigate to a goal location across a 2D plane while avoiding a number of obstacles. The agent uses \textit{pseudo-lidar} to observe the distance to objects around it, and the action space is continuous over two actuators controlling the direction and forward velocity. The goal's location is randomly reset when the agent reaches it, while the locations of the obstacles remain unchanged. The agent is rewarded for reaching the goal, as well as for moving towards it. We modify the environment so that any collision with a hazard results in episode termination with a reward of $-1$, thereby making the problem much harder.

\looseness=-1
As a baseline representative of typical RL approaches, we use Trust Region Policy Optimisation (TRPO) \citep{schulman2015trust}. Our approach modifies TRPO (denoted TRPO-Minmax) to use the estimate of the Minmax penalty as described in Algorithm \ref{alg:minmax}. To represent constraint-based approaches, we compare against Constrained Policy Optimisation (CPO) \citep{achiam2017constrained} and TRPO with Lagrangian constraints (TRPO-Lagrangian) \citep{Ray2019}. All baselines use the implementations provided by \citet{Ray2019}.
As in \citet{Ray2019}, all approaches use feed-forward MLPs, value networks of size ($256$,$256$), and $tanh$ activation functions. Results are averaged over 10 runs with different seeds, where episode lengths are capped at 1000 interactions. The cost threshold for the constrained algorithms is set to 0, 
the best we found
(see the Appendix for more experiment variations). The initial values of $\rmin$ and $\rmax$ for our approach are set to 0.

\paragraph{Safety Gym Pillar}
The final domain we consider is a custom Safety Gym \pillar environment, in which the simple robot must navigate to a goal location around a large hazard. Aside from the locations of the goal and hazard, the environment is identical to the \pointgoalhard environment. Results are averaged over 5 runs with different seeds.









\subsection{What is the tradeoff between safety and sample efficiency with the Minmax penalty?}
To answer this question, we examine the average episode length (converged timesteps) and failure rate (converged failure rate) of an agent that has converged to a policy in the Lava Gridworld domain, as well as the number of steps to converge (total timesteps). These metrics are compared against the size of the penalty for entering the lava, ranging from -5 to 0. Normalised results are plotted in Figure \ref{fig:tradeoff}.
These results clearly show the tradeoff between choices of penalty size, with large penalties resulting in longer convergence times and small penalties resulting in unsafe policies that cause the episodes to terminate early. Furthermore, these results show that the learned estimate of the Minmax penalty provides a good balance between the failure rate and the convergence time.

\begin{figure*}[b!]
    \centering
    \begin{subfigure}[t]{0.24\textwidth}
        \centering
        \includegraphics[width=\linewidth]{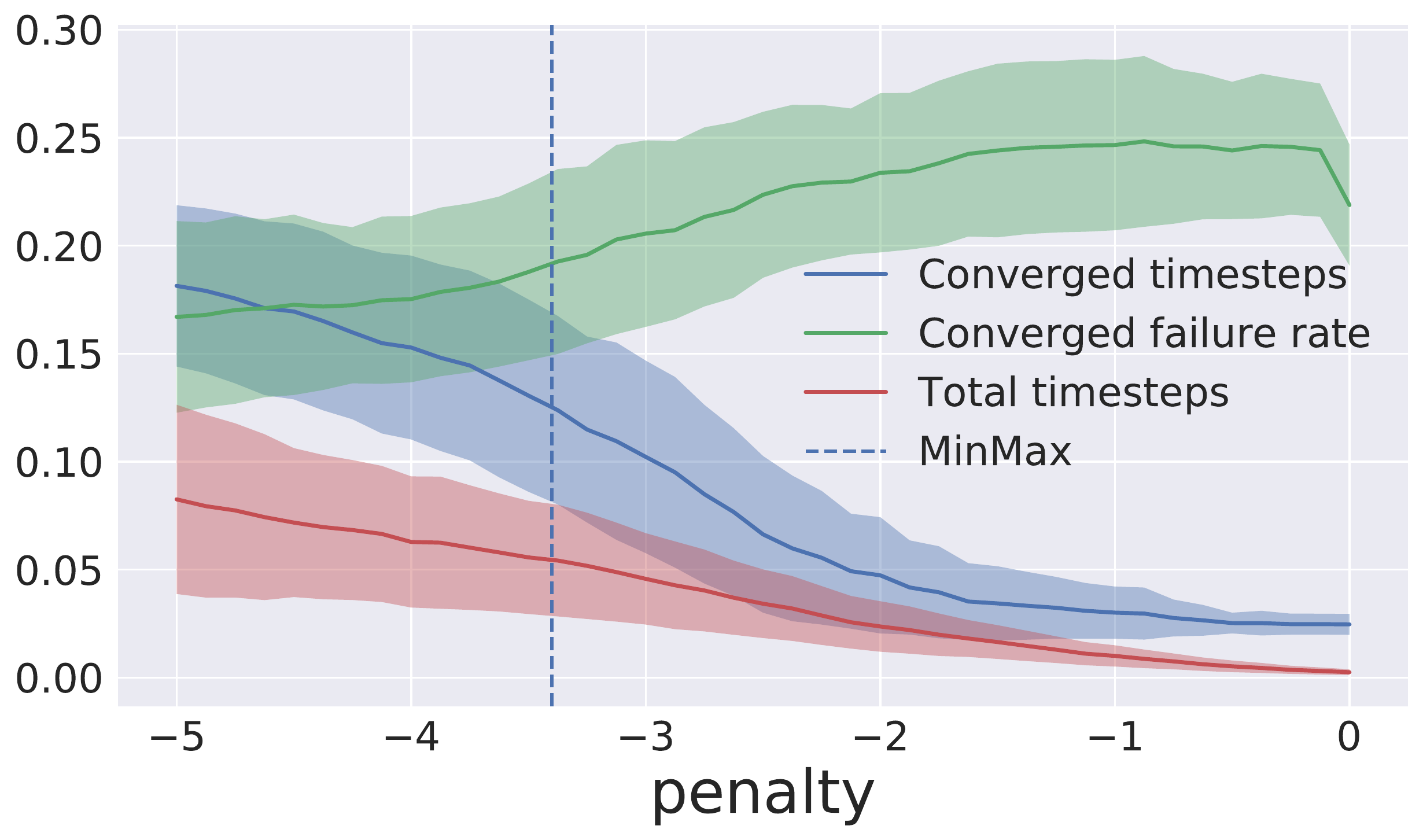}
        \caption{$sp=0.25$}
        \label{fig:tradeoff}
    \end{subfigure}%
    ~
    \begin{subfigure}[t]{0.24\textwidth}
        \centering
        \includegraphics[width=\linewidth]{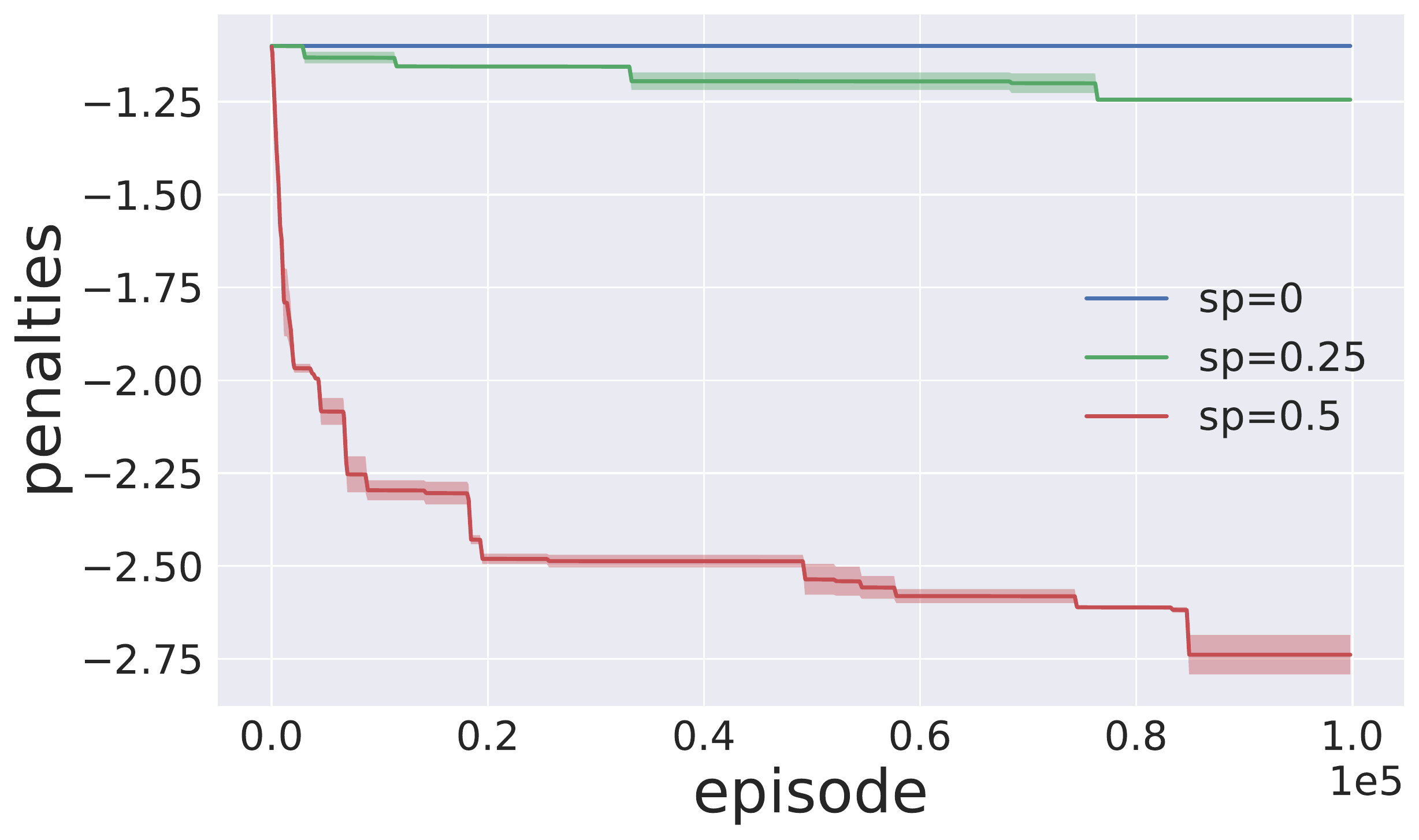}
        \caption{Learned penalty}
        \label{fig:exp_2_penalty}
    \end{subfigure}%
    ~
    \begin{subfigure}[t]{0.24\textwidth}
        \centering
        \includegraphics[width=\linewidth]{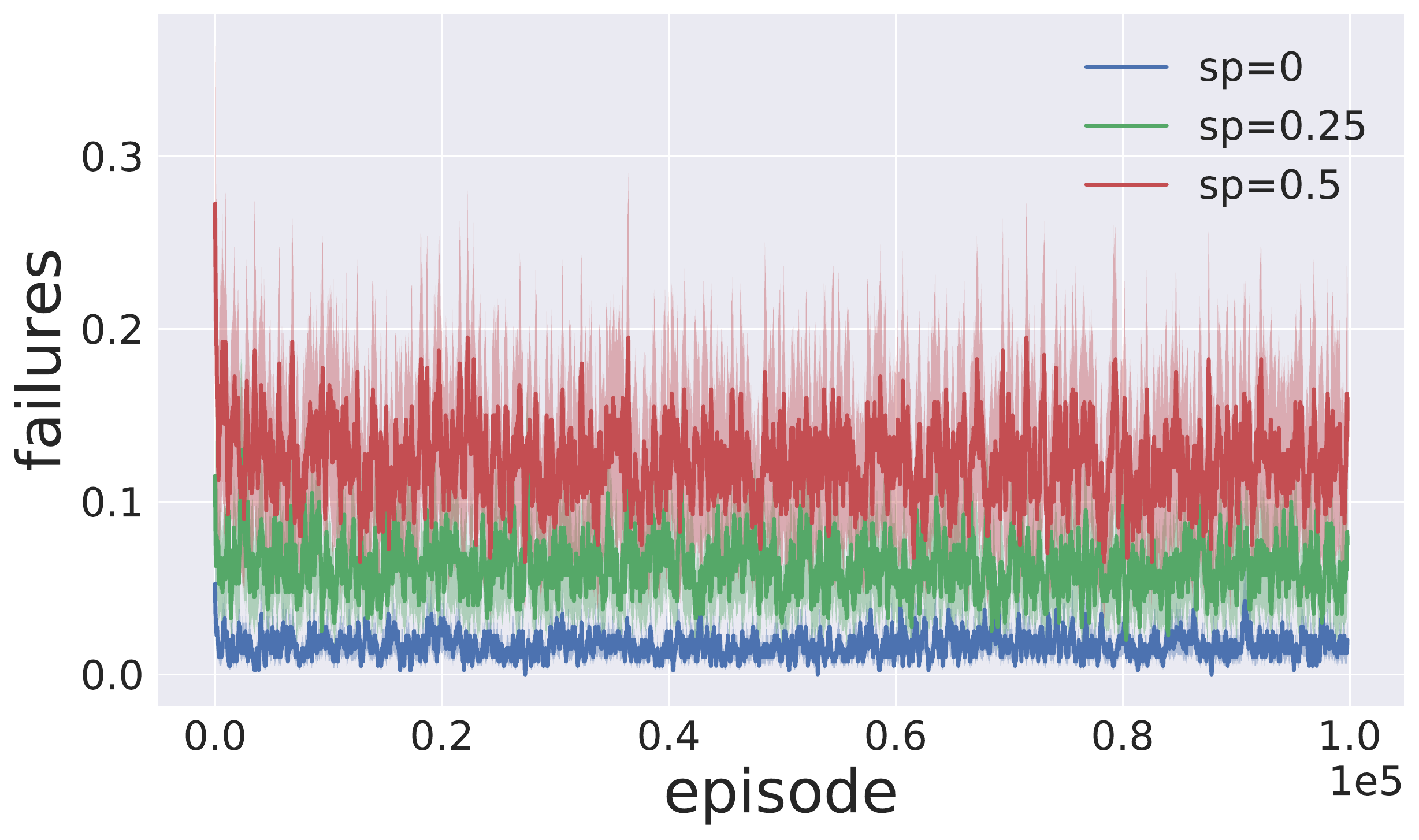}
        \caption{Failure rate}
        \label{fig:exp_2_failures}
    \end{subfigure}%
    ~
    \begin{subfigure}[t]{0.24\textwidth}
        \centering
        \includegraphics[width=\linewidth]{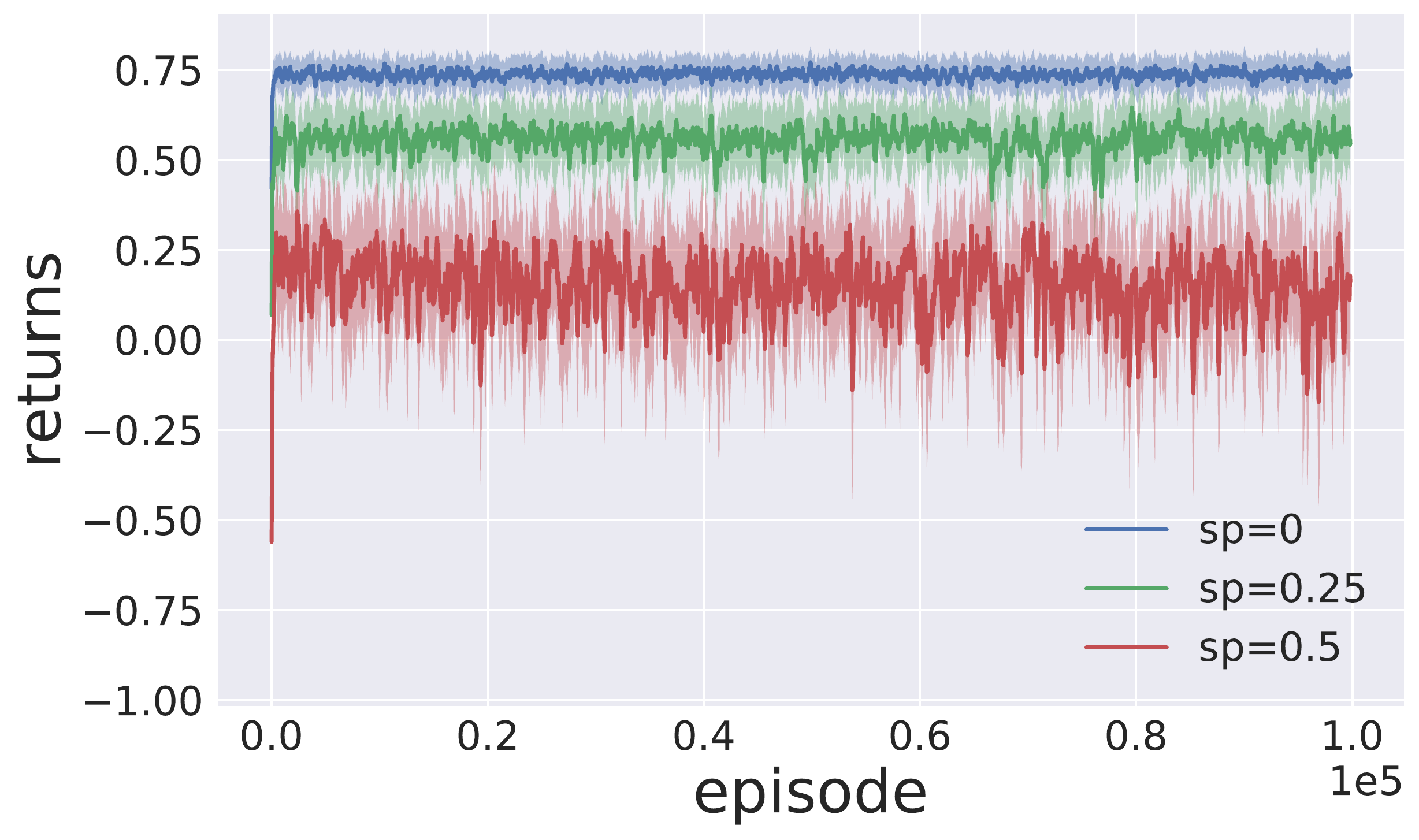}
        \caption{Average returns}
        \label{fig:exp_2_returns}
    \end{subfigure}%
    \caption{Effect of increase in the slip probability of the \lava on the learned Minmax penalty and corresponding failure rate and returns. The timesteps in (a) are normalised for clarity.}
    \label{fig:slip probability}
\end{figure*}%

\begin{figure*}[b!]
    \centering
    \includegraphics[width=0.7\linewidth]{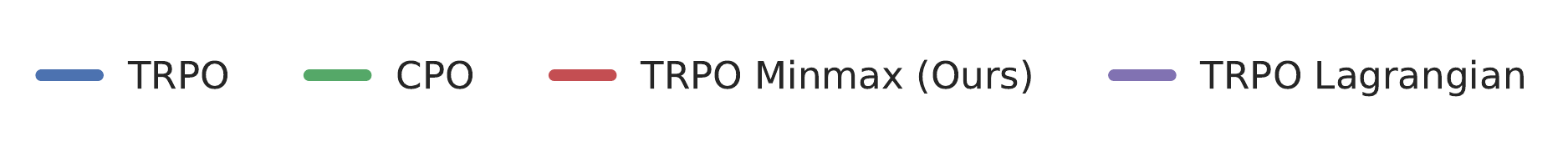}
    \\
    \begin{subfigure}[t]{0.43\textwidth}
        \centering
        \includegraphics[width=\linewidth]{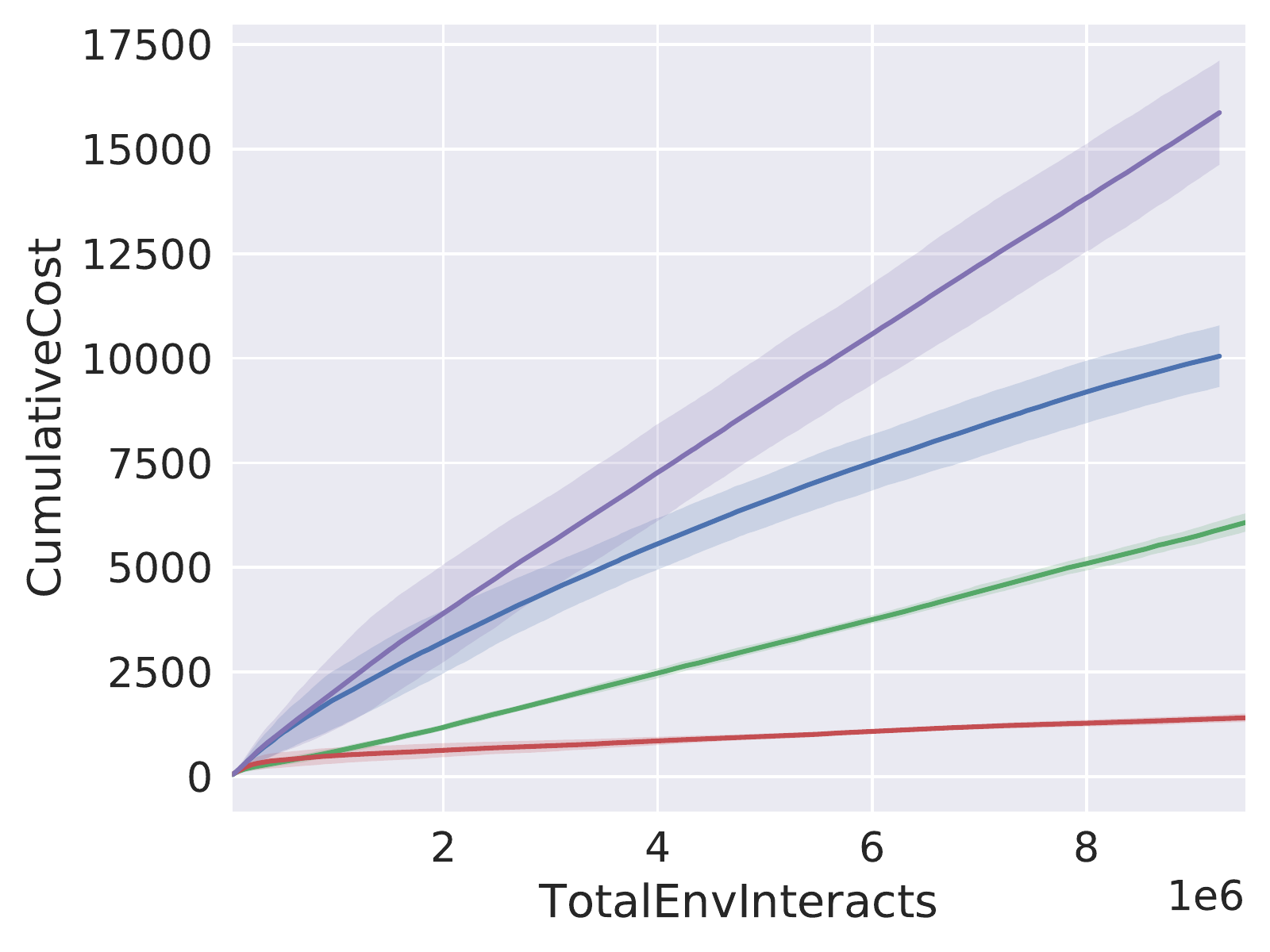}
        \caption{The cumulative cost.}
        \label{fig:cumulative_cost_hard_cost0}
    \end{subfigure}%
    \quad
    \begin{subfigure}[t]{0.43\textwidth}
        \centering
        \includegraphics[width=\linewidth]{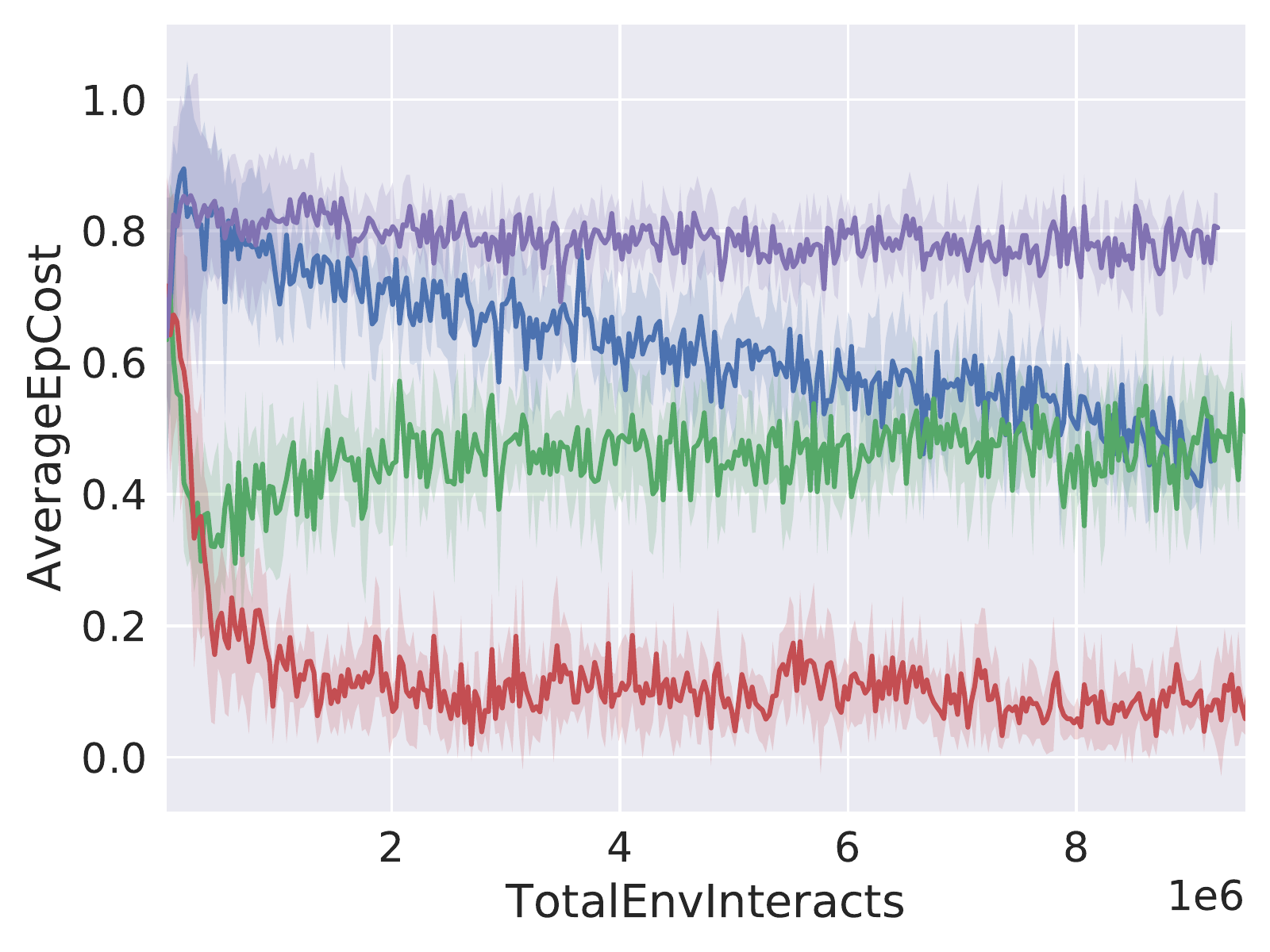}
        \caption{Failure rate}
        \label{fig:baseline_cost_hard_cost0}
    \end{subfigure}%
    \\
    \begin{subfigure}[t]{0.43\textwidth}
        \centering
        \includegraphics[width=\linewidth]{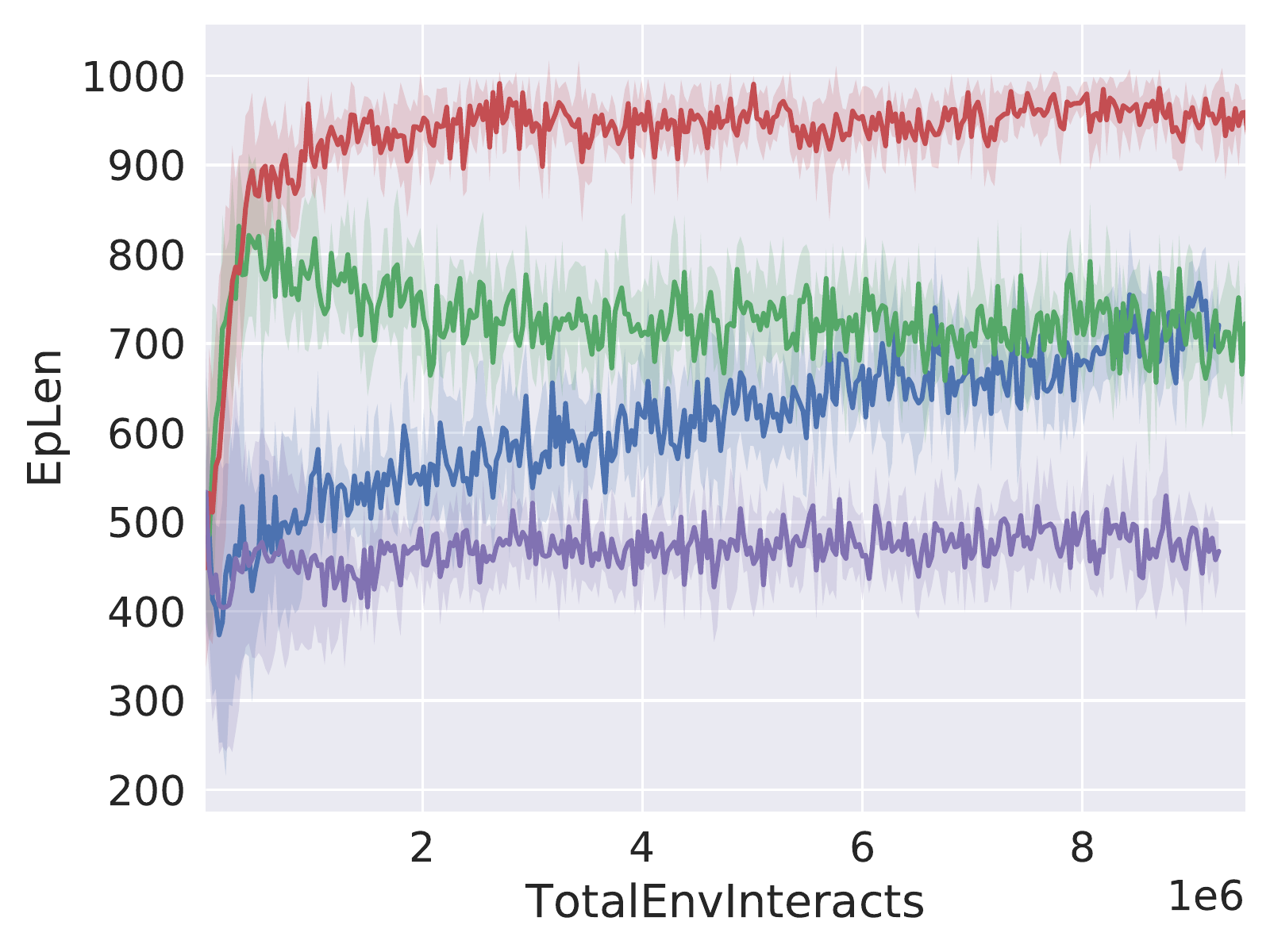}
        \caption{Average episode length}
        \label{fig:baseline_length_hard_cost0}
    \end{subfigure}%
    \quad
    \begin{subfigure}[t]{0.43\textwidth}
        \centering
        \includegraphics[width=\linewidth]{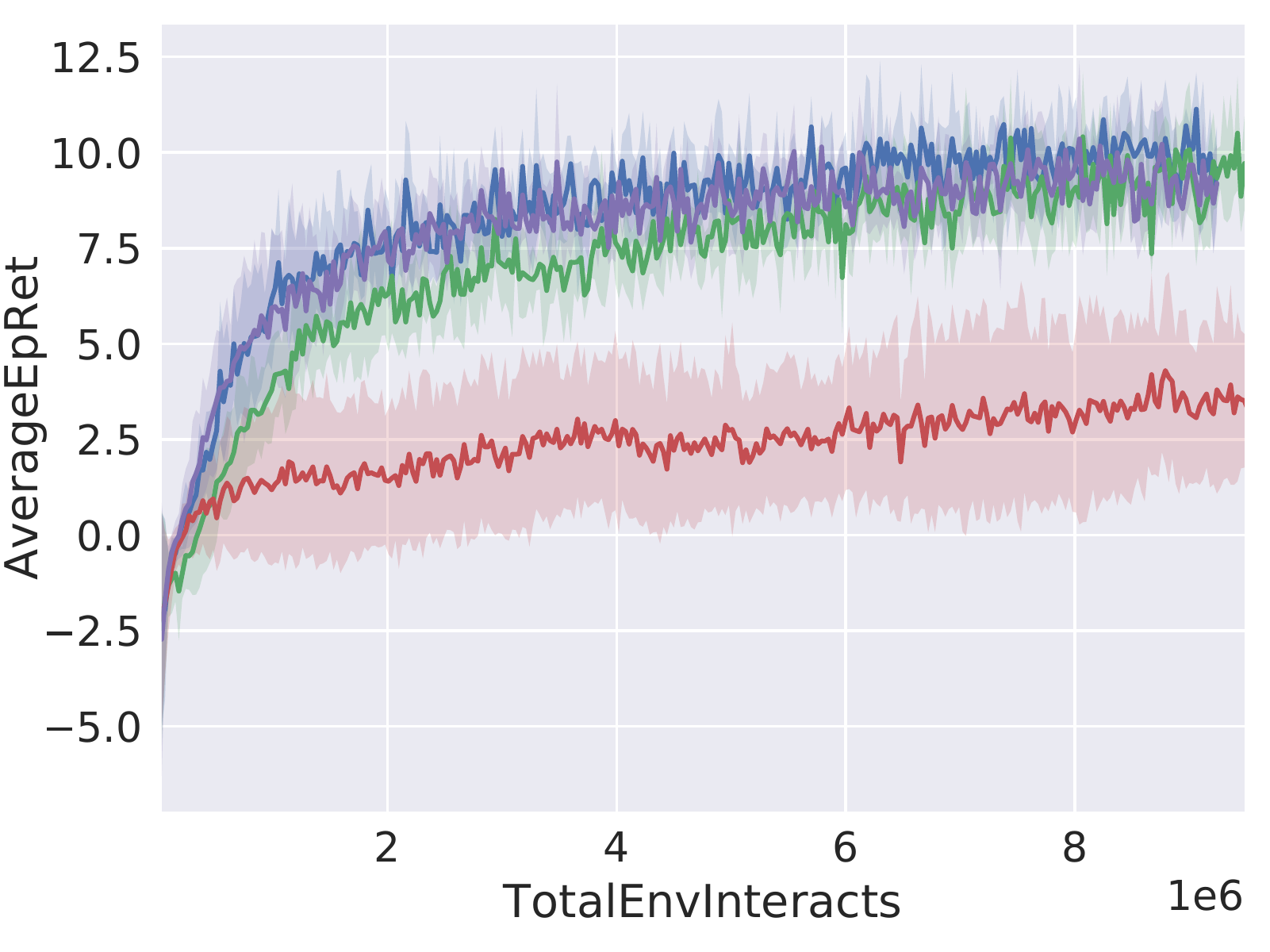}
        \caption{Average returns}
        \label{fig:baseline_return_hard_cost0}
    \end{subfigure}%
    \caption{Comparison with baselines in the Safety Gym \pointgoalhard environment. The cost threshold for TRPO lagrangian and CPO is set to $0$.}
    \label{fig:safety_baselines_hard_cost0}
\end{figure*}%

\subsection{How does stochasticity affect the Minmax penalty?}\label{LavaStochastic}
To answer this question, we compare the performance of our modified Q-learning approach across three values of the slip probability of the \lava. A slip probability of 0 represents a fully deterministic environment, while a slip probability of 0.5 represents a more stochastic environment. Results are plotted in Figure \ref{fig:exp_2_penalty}.
In the case of the fully deterministic environment, the agent is able to use a relatively small penalty to consistently minimise failure rate and maximise returns (Figures \ref{fig:exp_2_failures} and \ref{fig:exp_2_returns}). As the stochasticity of the environment increases, a larger penalty is learned to incentivise longer, safer policies. Given the starting position of the agent next to the lava, the failure rate inevitably increases with increased stochasticity.
We can therefore conclude that increased stochasticity in an environment, which can increase the risk of entering unsafe states, necessitates learning a larger Minmax penalty in order to incentivise more risk-averse behaviour.

\subsection{Does the learned Minmax penalty improve safety in high-dimensional domains requiring function approximation?}\label{sec:exp:PointGoal1Hard}

\looseness=-1
To answer this question we turn to our major results in the \pointgoalhard environment, plotted in Figure \ref{fig:safety_baselines_hard_cost0}.
The baselines all achieve similar performance, maximising returns but maintaining a relatively high failure rate  (Figures \ref{fig:baseline_return_hard_cost0} and \ref{fig:baseline_cost_hard_cost0}). By examining the average episode length (Figure \ref{fig:baseline_length_hard_cost0}), we can conclude that the baselines have learned risky policies that maximise rewards over short trajectories that are highly likely to result in collisions.
These results are also consistent with the benchmarks of \cite{Ray2019} where the cumulative cost of TRPO is much greater than that of TRPO-Lagrangian, which is greater than that of CPO (Figure \ref{fig:cumulative_cost_hard_cost0}).
They are however unable to avoid unsafe states because of the difficulty of maximising rewards while minimising the cost.
By comparison, TRPO-Minmax uses the estimated Minmax penalty to dramatically reduce the failure rate (and thus cumulative failures). 
The average returns achieved, as well as the trajectories observed (see, for example, Figure \ref{fig:trajectory_safe}), indicate that the improved safety does not compromise the agent's ability to learn to complete the task, although returns are lower due to the dense reward function that incentivises moving towards the goal. 
%

\subsection{Is the learned Minmax penalty robust to noisy, continuous dynamics?}
To examine the effect of noisy continuous dynamics on the performance of TRPO-Minmax, we vary the levels of stochasticity in the \pillar environment, similarly to the experiments in Section \ref{LavaStochastic}. Once again, the value of the slip probability denotes the probability of overriding the agent's action with a random action. Results are plotted in Figure \ref{fig:pillar fig}.
The value of the slip probability appears to have little effect on both the failure rate of and average returns obtained by the agent, indicating that in an environment with a relatively simple hazard configuration such as the \pillar environment, the learned Minmax penalty is indeed robust to noisy continuous dynamics. Figure \ref{fig:pil_trajectory} shows a sample trajectory representative of the trajectories observed for all 3 noise levels.

\begin{figure}[b!]
    \centering
    \begin{subfigure}[t]{0.3\textwidth}
        \centering
        \includegraphics[width=\linewidth]{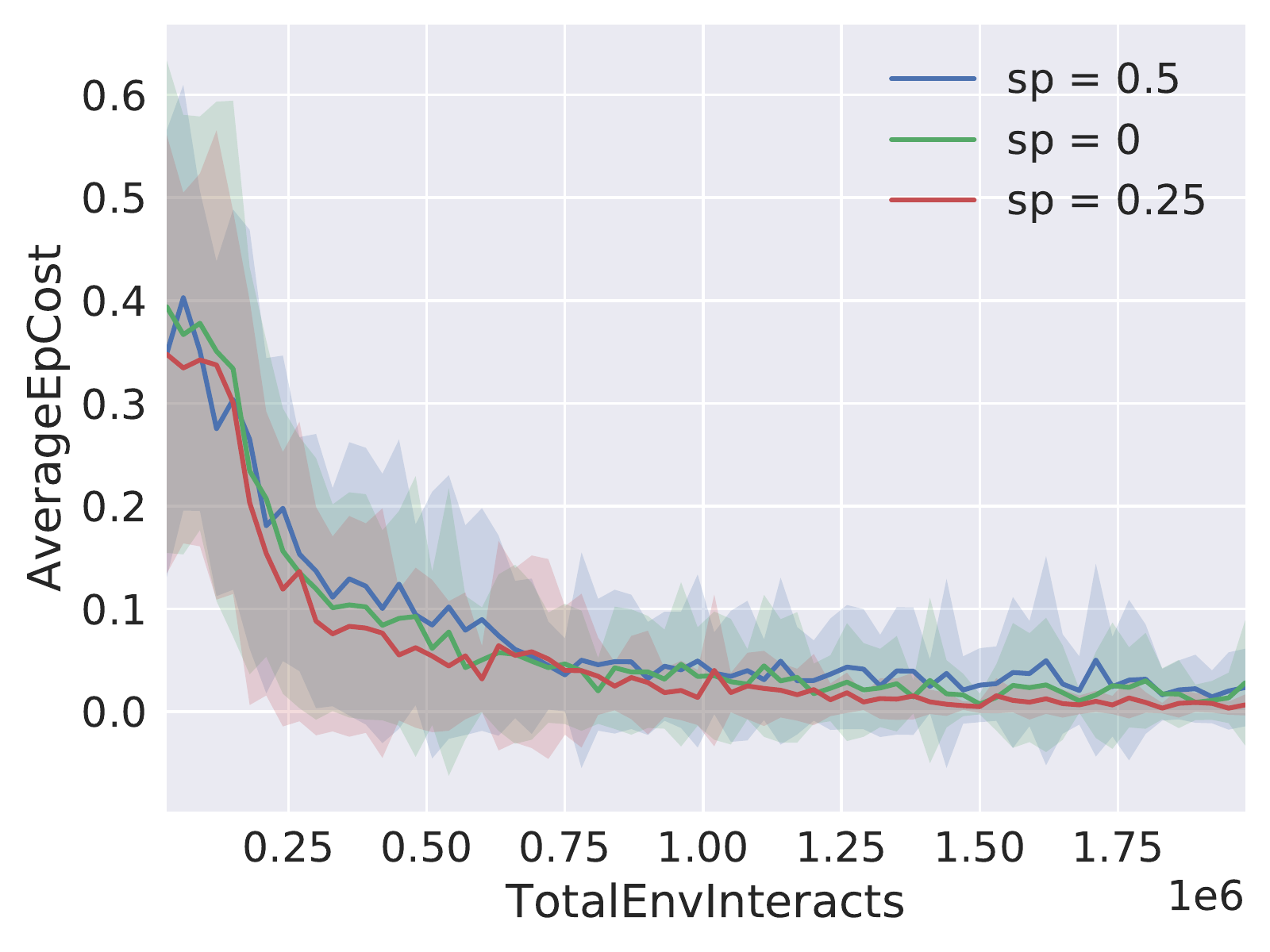}
        \caption{Failure rate}
        \label{fig:pil_fail_rate}
    \end{subfigure}%
    ~
    \begin{subfigure}[t]{0.3\textwidth}
        \centering
        \includegraphics[width=\linewidth]{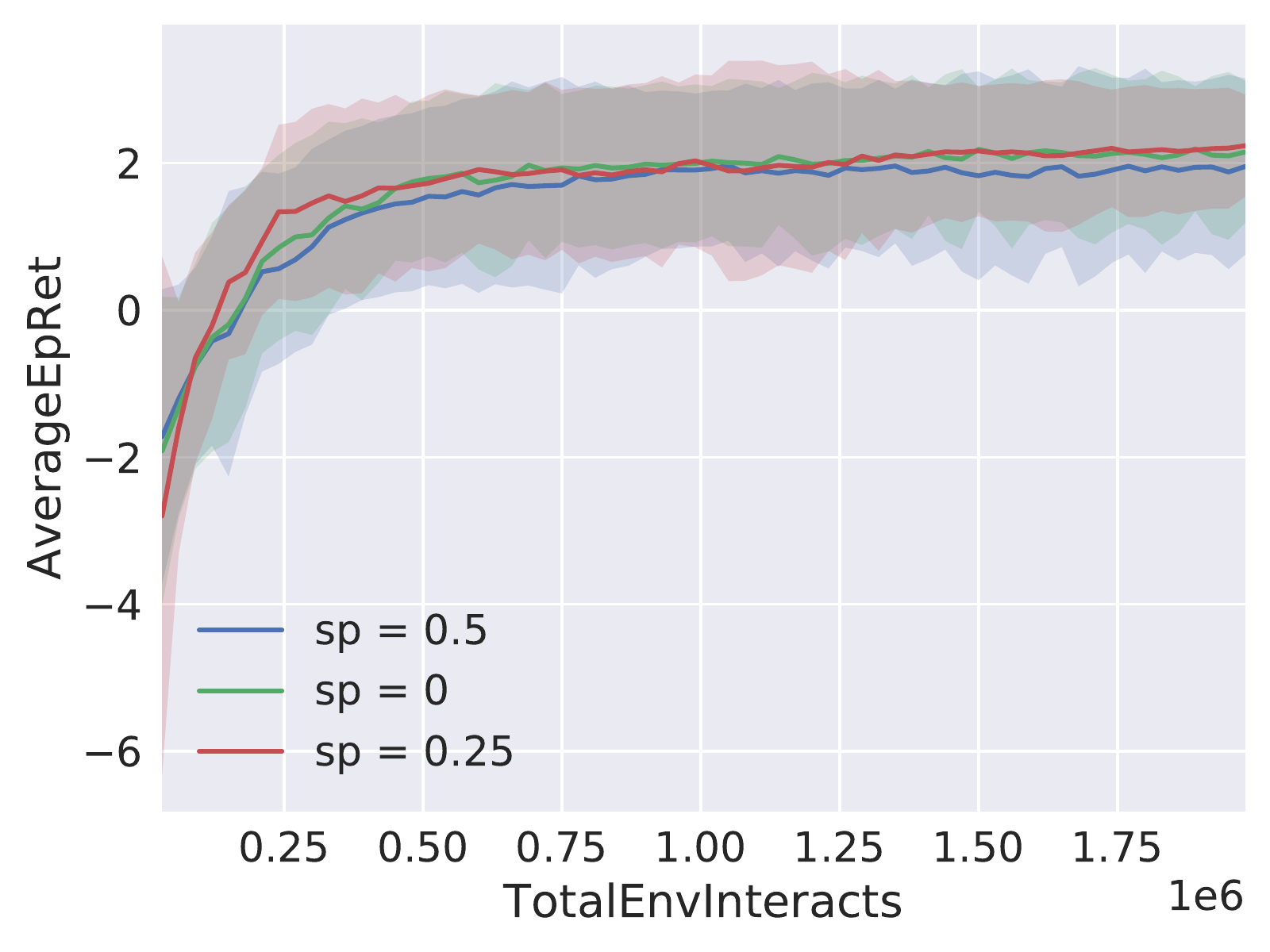}
        \caption{Average returns}
        \label{fig:pil_average_returns}
    \end{subfigure}%
    ~~
    \begin{subfigure}[t]{0.3\textwidth}
        \centering
        \includegraphics[trim=20 300 20 300, clip, width=\linewidth]{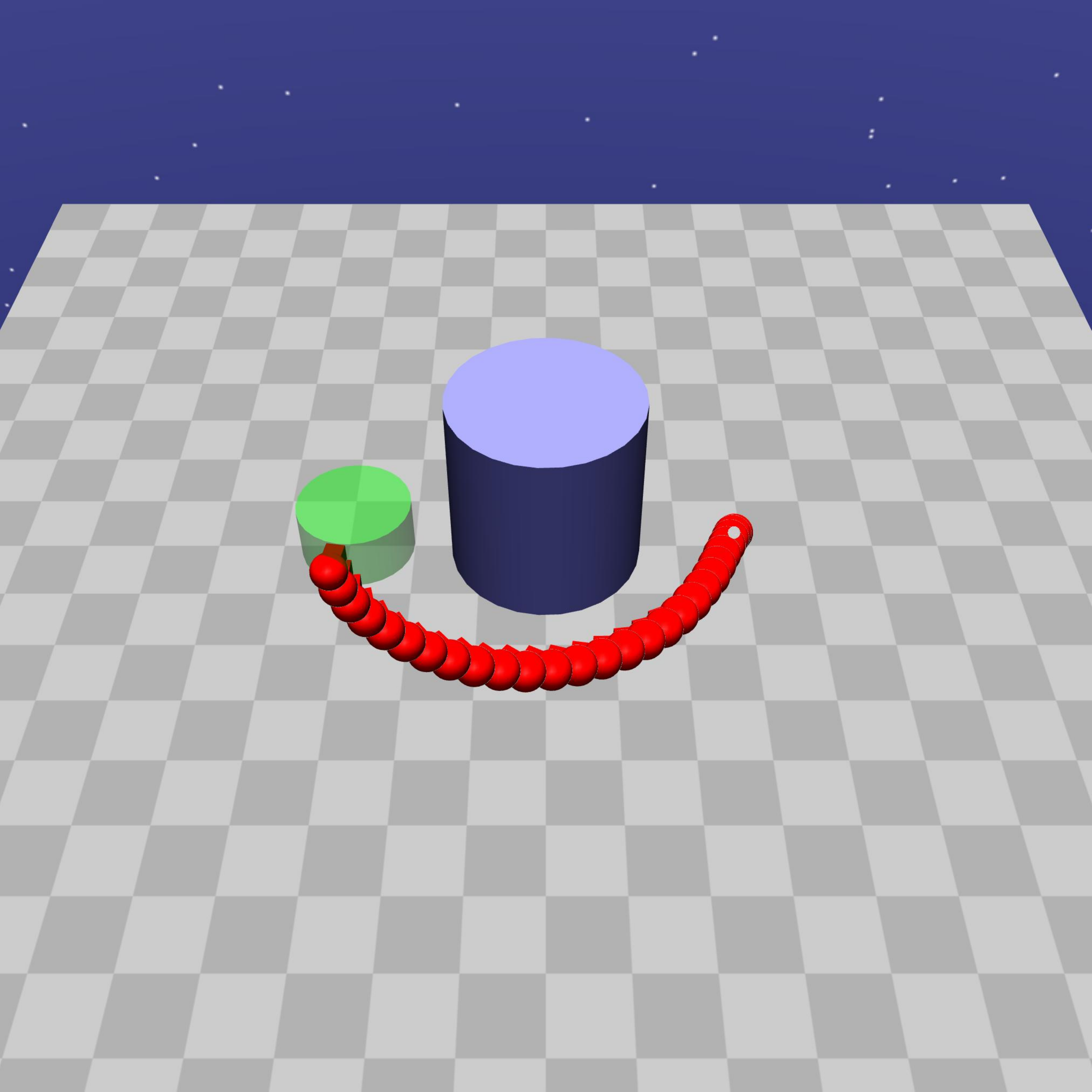}
        \caption{Sample trajectory}
        \label{fig:pil_trajectory}
    \end{subfigure}%
    \caption{Comparisons of TRPO-Minmax with varying slip probabilities in the \pillar environment.}
    \label{fig:pillar fig}
\end{figure}%

\section{Limitations}

\textbf{Handling unsafe non-terminal states:} 
This paper investigates a new approach towards safe-RL by asking the question: \emph{Is a scalar reward enough to solve tasks safely?} While we show that it is indeed enough, both theoretically and experimentally, the current approach is only applicable to unsafe terminal states---which only covers tasks that can be naturally represented by stochastic-shortest path MDPs. Hence this may not be applicable to all conceivable safe-RL settings. 
However, since our approach relies on the concept of MDP diameter and controllability which are properties of the entire MDP and not just the terminal states, our approach may also be extendable in future work to other settings.
Finally, given that other popular RL settings like discounted MDPs can be converted to stochastic shortest path MDPs \citep{bertsekas1987dynamic,sutton1998introduction}, a promising future direction could be to find the dual of our results for other theoretically equivalent settings.

\textbf{Handling degrees of safety:} In this work, we only consider terminal states that are maximally unsafe. This leads to very risk-averse policies, as shown by the trajectories produced by our TRPO-Minmax agent (see appendix).
Additionally, it may be desirable for an agent to accommodate different degrees of safety---for example "breaking a vase" is less unsafe than "hitting a baby". Our focus on scalar rewards at unsafe states leads to a natural future extension to the case of different degrees of safety. Since the Minmax reward is the least negative reward that guarantees safety ($\bar{R}_{MIN} - \epsilon$), we may assign weights to it corresponding to different levels of unsafety. Here, smaller weights would lead to policies that pass near the lava and hazards, and large weights would lead to the safest policy that chooses the longer path to the goal. 

\section{Related Work}

\looseness=-1
\textbf{Reward shaping}: The problem of designing reward functions to produce desired policies in RL settings is well-studied by \cite{singh2009rewards}. Particular focus has been placed on the practice of \textit{reward shaping}, in which an initial reward function provided by an MDP is augmented in order to improve the rate at which an agent learns the same optimal policy \citep{ng1999policy,devidze2021explicable}. 
While sacrificing some optimality, other approaches like \cite{lipton2016combating} propose shaping rewards using an idea of intrinsic fear. Here, the agent trains a supervised fear model representing the probability of reaching unsafe states in a fixed horizon, scales said probabilities by a fear factor, and then subtracts the scaled probabilities from Q-learning targets.
These approaches differ from ours in that they seek to find reward functions that improve convergence while preserving the optimality from an initial reward function. In contrast, ours seeks to determine the optimal rewards for terminal states in order to minimise undesirable behaviours irrespective of the original reward function and optimal policy.

\textbf{Constrained RL}: Disincentivising or preventing undesirable behaviours is core to the field of safe RL. A popular approach is to define constraints on the behaviour of an agent \citep{altman1999constrained,achiam2017constrained,chow2018lyapunov,Ray2019, hasanzadezonuzy2021learning}. Agents in these settings are tasked with limiting the accumulation of costs associated with violating safety constraints while simultaneously maximising reward. Examples of these approaches include the CPO \citep{achiam2017constrained} and TRPO-Lagrangian \citep{Ray2019} algorithms outlined in Section \ref{SafetyGym}.

\looseness=-1
\textbf{Shielding}: Another important line of work involves relying on interventions from a model \citep{dalal2018safe,wagener2021safe} or human \citep{tennenholtz2022reinforcement} to prevent unsafe actions from being considered by the agent (shielding the agent) or prevent the environment from executing those unsafe actions by correcting them (shielding the environment). Other approaches here also look at using temporal logics to define or enforce safety constraints on the actions considered or selected by the agent \citep{alshiekh2018safe}. These approaches fit seamlessly into our proposed reward-only framework since they are primarily about modifications on the transition dynamics and not the reward function---for example, unsafe actions here can simply lead to unsafe goal states.



\section{Conclusion}
In this work, we present the Minmax penalty, which takes into account the diameter and controllability of an environment in order to minimise the probability of encountering unsafe states. We prove that the penalty does indeed minimise this probability, and present a method that uses an agent's value estimates to learn an estimate of the penalty.
Our results in tabular and high-dimensional continuous settings have demonstrated that, by encoding the safe behaviour directly in the reward function via the Minmax penalty, our method is able to solve tasks while prioritising safety, learning safer policies than popular constraint-based approaches.
Our method is also easy to incorporate with any off-the-shelf RL algorithms that maintain value estimates, requiring no changes to the algorithms themselves. By autonomously learning the penalty, our method also alleviates the need for a human designer to manually tweak rewards or cost functions to elicit safe behaviour.
While it may be feasible to handcraft reward or cost functions to induce safe behaviour for individual tasks, our ultimate aim is to have general agents capable of operating safely in a variety of environments, and thus we cannot rely on human-crafted reward or cost functions. We see this as a step towards truly autonomous agents capable of independently learning to solve tasks safely.

\newpage
\looseness=-1

\bibliographystyle{named}
\bibliography{references}

\appendix

\section{Proofs of Theoretical Results}

\setcounter{theorem}{0}

\begin{theorem}[Safety Bounds]
Consider a controllable MDP $\langle \state, \action, P, \reward \rangle$ with a non-empty set of unsafe goal states $\unsafegoals \subset \goals$. Let $\pistar$ be an optimal policy for the modified MDP with possibly different rewards in $\unsafegoals$: $\langle \state, \action, P, \rbar \rangle$ where $\rbar(s,a,s')=\reward(s,a,s')$ for all $s'\not\in\unsafegoals$.
\begin{enumerate}[label=(\roman*)]
  \item If  $\rbar(s,a,s') < \rbarmin$ ~ for all $s'\in\unsafegoals$, then $\pistar$ is safe for all $R$;
  \item If  $\rbar(s,a,s') > \rbarmin$ ~ for all $s'\in\unsafegoals$, then $\pistar$ is unsafe for all $R$.
\end{enumerate}
  
\end{theorem}
\begin{proof}
Since $\pistar$ is optimal, it is also proper and hence must reach $\goals$. 

\textit{(i)} Assume $\pistar$ is unsafe. Then there exists another proper policy $\pi$ such that 
\[
P^{\pistar}_{s}(s_{T} \not\in \unsafegoals) < P^{\pi}_{s}(s_{T} \not\in \unsafegoals) \quad \text{for some } s\in\state.
\]

Without loss of generality, let $s$ be the state that maximises $\Delta P_s(\pistar,\pi)$. Then,
\begingroup
\allowdisplaybreaks
\begin{align*}
    &V^{\pistar}(s) \geq V^{\pi}(s) \\
    \implies 
    &\E_s^{\pistar} \left[ \sum_{t=0}^{\infty} \gamma^{t} \rbar(s_t,a_t,s_{t+1}) \right]
    \geq \E_s^{\pi} \left[ \sum_{t=0}^{\infty} \gamma^{t} \rbar(s_t,a_t,s_{t+1}) \right] \\
    \implies 
    &\E_s^{\pistar} \left[ G^{T-1} + \rbar(s_T,a_T,s_{T+1}) \right] \geq \E_s^{\pi} \left[ G^{T-1} + \rbar(s_T,a_T,s_{T+1}) \right], \\
    &\text{ where } G^{T-1} = \sum_{t=0}^{T-1} \gamma^{t} \rbar(s_t,a_t,s_{t+1}) \text{ and } T \text{ is a random variable denoting when } s_{T+1} \in \goals. \\
    \implies 
    &\E_s^{\pistar} \left[ G^{T-1} \right] + \left( P^{\pistar}_{s}(s_{T}  \not\in \unsafegoals) \rbar(s_T, a_T, s_{T+1}) + P^{\pistar}_{s}(s_{T}  \in \unsafegoals) \rbar_\text{unsafe}(s_T, a_T, s_{T+1}) \right) \\
    &\geq \E_s^{\pi} \left[ G^{T-1} \right] + \left( P^{\pi}_{s}(s_{T}  \not\in \unsafegoals) \rbar(s_T, a_T, s_{T+1}) + P^{\pi}_{s}(s_{T}  \in \unsafegoals) \rbar_\text{unsafe}(s_T, a_T, s_{T+1}) \right), \\
    &\text{ where } \rbar_\text{unsafe}(s,a) \text{ denotes the rewards in $\unsafegoals$ and } a_T = \pistar(g). \\
    \implies 
    &\E_s^{\pistar} \left[ G^{T-1} \right] + \left( P^{\pistar}_{s}(s_{T}  \in \unsafegoals) - P^{\pi}_{s}(s_{T}  \in \unsafegoals) \right) \rbar_\text{unsafe}(s_T, a_T, s_{T+1}) \\
    &\geq \E_s^{\pi} \left[ G^{T-1} \right] + \left( P^{\pi}_{s}(s_{T}  \not\in \unsafegoals) - P^{\pistar}_{s}(s_{T}  \not\in \unsafegoals) \right) \rbar(s_T, a_T, s_{T+1}) \\
    \implies 
    &\E_s^{\pistar} \left[ G^{T-1} \right] + \left( P^{\pistar}_{s}(s_{T}  \in \unsafegoals) - P^{\pi}_{s}(s_{T}  \in \unsafegoals) \right) \rbarmin \\
    &> \E_s^{\pi} \left[ G^{T-1} \right] + \left( P^{\pi}_{s}(s_{T}  \not\in \unsafegoals) - P^{\pistar}_{s}(s_{T}  \not\in \unsafegoals) \right) \rbar(s_T, a_T, s_{T+1}), \text{ since } \rbar_\text{unsafe}(s,a) < \rbarmin.\\
    \implies 
    &\E_s^{\pistar} \left[ G^{T-1} \right] + \left( P^{\pistar}_{s}(s_{T}  \in \unsafegoals) - P^{\pi}_{s}(s_{T}  \in \unsafegoals) \right) (\rmin - \rmax)\frac{D}{C} \\
    &> \E_s^{\pi} \left[ G^{T-1} \right] + \left( P^{\pi}_{s}(s_{T}  \not\in \unsafegoals) - P^{\pistar}_{s}(s_{T}  \not\in \unsafegoals) \right) \rbar(s_T, a_T, s_{T+1}), \text{ using definition of } \rbarmin.  \\
    \implies 
    &\E_s^{\pistar} \left[ G^{T-1} \right] + (\rmin - \rmax)D \\
    &> \E_s^{\pi} \left[ G^{T-1} \right] + \left( P^{\pi}_{s}(s_{T}  \not\in \unsafegoals) - P^{\pistar}_{s}(s_{T}  \not\in \unsafegoals) \right) \rbar(s_T, a_T, s_{T+1}), \text{ using definition of } C.  \\
    \implies 
    &\E_s^{\pistar} \left[ G^{T-1} \right] - \rmax D \\
    &> \E_s^{\pi} \left[ G^{T-1} \right] + \left( P^{\pi}_{s}(s_{T}  \not\in \unsafegoals) - P^{\pistar}_{s}(s_{T}  \not\in \unsafegoals) \right) \rbar(s_T, a_T, s_{T+1}) - \rmin D  \\
    \implies 
    &\E_s^{\pistar} \left[ G^{T-1} \right] - \rmax D > 0, \\
    &\text{ since } \E_s^{\pi} \left[ G^{T-1} \right] + \left( P^{\pi}_{s}(s_{T}  \not\in \unsafegoals) - P^{\pistar}_{s}(s_{T}  \not\in \unsafegoals) \right) \rbar(s_T, a_T, s_{T+1}) \geq \rmin D \\
    \implies 
    &\E_s^{\pistar} \left[ G^{T-1} \right] > \rmax D.
\end{align*}
\endgroup
But this is a contradiction since the expected return of following an optimal policy up to a terminal state without the reward for entering the terminal state must be less than receiving $\rmax$ for every step of the longest possible trajectory to $\goals$. 
Hence we must have 
$
\pistar \in \argmin\limits_{\pi} P^{\pi}_{s}(s_{T} \in \unsafegoals).
$

\textit{(ii)} Assume $\pistar$ is safe. Then, 
$
P^{\pistar}_{s}(s_{T} \not\in \unsafegoals) \geq P^{\pi}_{s}(s_{T} \not\in \unsafegoals)
$ 
for all $ s\in\state$, $\pi \in \Pi$.
Similarly to (i), 

\begingroup
\allowdisplaybreaks
\begin{align*}
    &V^{\pistar}(s) \geq V^{\pi}(s) \\
    \implies 
    &\E_s^{\pistar} \left[ G^{T-1} \right] + \left( P^{\pistar}_{s}(s_{T}  \in \unsafegoals) - P^{\pi}_{s}(s_{T}  \in \unsafegoals) \right) \rbar_\text{unsafe}(s_T, a_T, s_{T+1}) \\
    &\geq \E_s^{\pi} \left[ G^{T-1} \right] + \left( P^{\pi}_{s}(s_{T}  \not\in \unsafegoals) - P^{\pistar}_{s}(s_{T}  \not\in \unsafegoals) \right) \rbar(s_T, a_T, s_{T+1}) \\
    \implies 
    &\E_s^{\pi} \left[ G^{T-1} \right] + \left( P^{\pi}_{s}(s_{T}  \in \unsafegoals) - P^{\pistar}_{s}(s_{T}  \in \unsafegoals) \right) \rbar_\text{unsafe}(s_T, a_T, s_{T+1}) \\
    &\leq \E_s^{\pistar} \left[ G^{T-1} \right] + \left( P^{\pistar}_{s}(s_{T}  \not\in \unsafegoals) - P^{\pi}_{s}(s_{T}  \not\in \unsafegoals) \right) \rbar(s_T, a_T, s_{T+1}) \\
    \implies 
    &\E_s^{\pi} \left[ G^{T-1} \right] + \left( P^{\pi}_{s}(s_{T}  \in \unsafegoals) - P^{\pistar}_{s}(s_{T}  \in \unsafegoals) \right) \rbarmin \\
    &< \E_s^{\pistar} \left[ G^{T-1} \right] + \left( P^{\pistar}_{s}(s_{T}  \not\in \unsafegoals) - P^{\pi}_{s}(s_{T}  \not\in \unsafegoals) \right) \rbar(s_T, a_T, s_{T+1}), \text{ since } \rbar_\text{unsafe} > \rbarmin.\\
    \implies 
    &\E_s^{\pi} \left[ G^{T-1} \right] + \left( P^{\pi}_{s}(s_{T}  \in \unsafegoals) - P^{\pistar}_{s}(s_{T}  \in \unsafegoals) \right) (\rmin - \rmax)\frac{D}{C} \\
    &< \E_s^{\pistar} \left[ G^{T-1} \right] + \left( P^{\pistar}_{s}(s_{T}  \not\in \unsafegoals) - P^{\pi}_{s}(s_{T}  \not\in \unsafegoals) \right) \rbar(s_T, a_T, s_{T+1}), \text{ by definition of } \rbarmin.  \\
    \implies 
    &\E_s^{\pi} \left[ G^{T-1} \right] + \left( P^{\pistar}_{s}(s_{T}  \in \unsafegoals) - P^{\pi}_{s}(s_{T}  \in \unsafegoals) \right) (\rmax-\rmin)\frac{D}{C} \\
    &< \E_s^{\pistar} \left[ G^{T-1} \right] + \left( P^{\pistar}_{s}(s_{T}  \not\in \unsafegoals) - P^{\pi}_{s}(s_{T}  \not\in \unsafegoals) \right) \rbar(s_T, a_T, s_{T+1})  \\
    \implies 
    &\E_s^{\pi} \left[ G^{T-1} \right] + (\rmax-\rmin)D \\
    &< \E_s^{\pistar} \left[ G^{T-1} \right] + \left( P^{\pistar}_{s}(s_{T}  \not\in \unsafegoals) - P^{\pi}_{s}(s_{T}  \not\in \unsafegoals) \right) \rbar(s_T, a_T, s_{T+1}), \text{ by definition of } C.  \\
    \implies 
    &\E_s^{\pi} \left[ G^{T-1} \right] - \rmin D < 0, \\
    &\text{ since } \E_s^{\pistar} \left[ G^{T-1} \right] + \left( P^{\pistar}_{s}(s_{T}  \not\in \unsafegoals) - P^{\pi}_{s}(s_{T}  \not\in \unsafegoals) \right) \rbar(s_T, a_T, s_{T+1}) \leq \rmax D \\
    \implies 
    &\E_s^{\pi} \left[ G^{T-1} \right] < \rmin D.
\end{align*}
\endgroup
But this is a contradiction since $\rmin D$ is the smallest expected return under a proper policy. 

\end{proof}

\newpage
\section{Additional Figures}





%

%

%
\begin{figure*}[b!]
    \centering
    \begin{subfigure}[t]{\textwidth}
        \centering
        \includegraphics[trim=0 300 0 500, clip, width=0.24\linewidth]{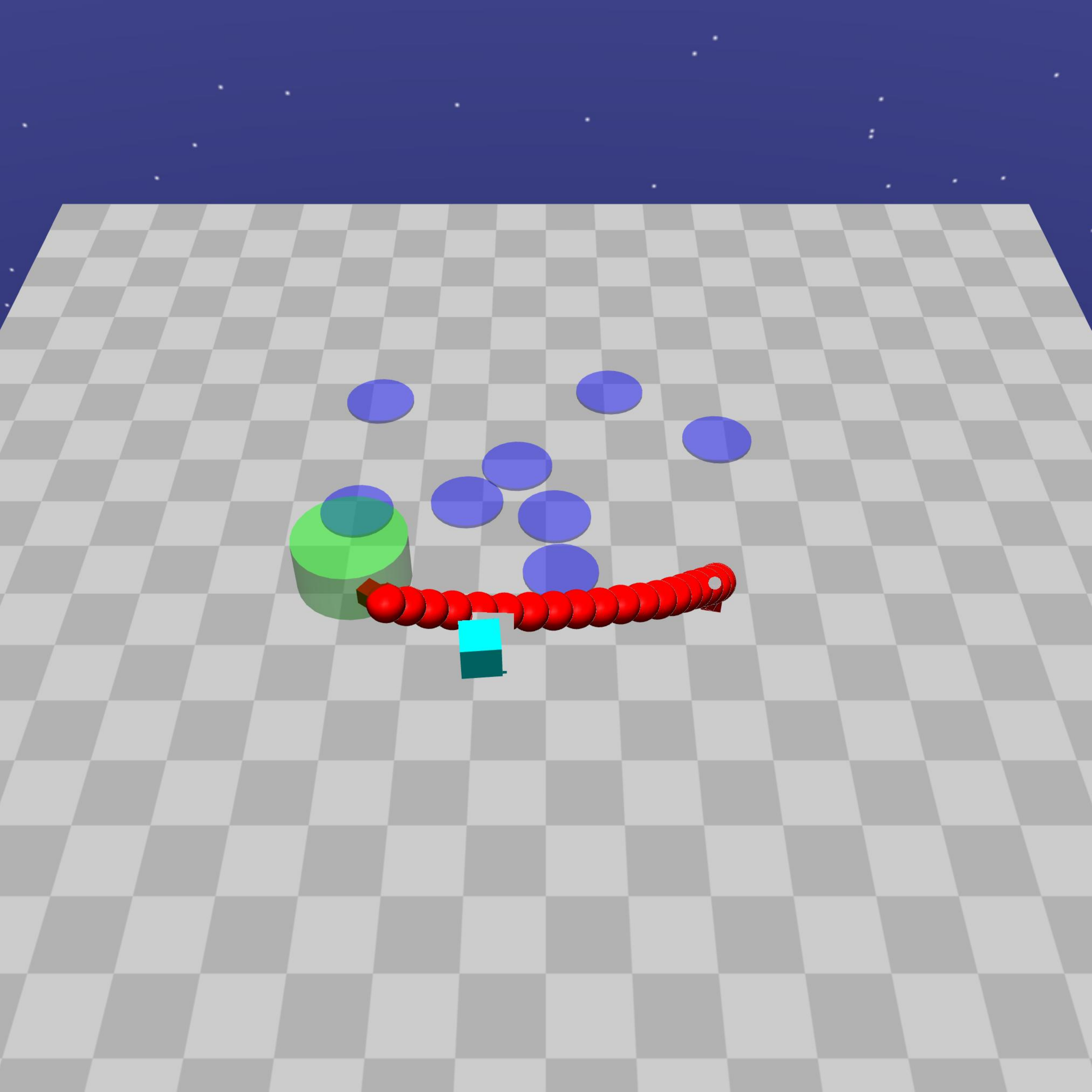}
        \includegraphics[trim=0 300 0 500, clip, width=0.24\linewidth]{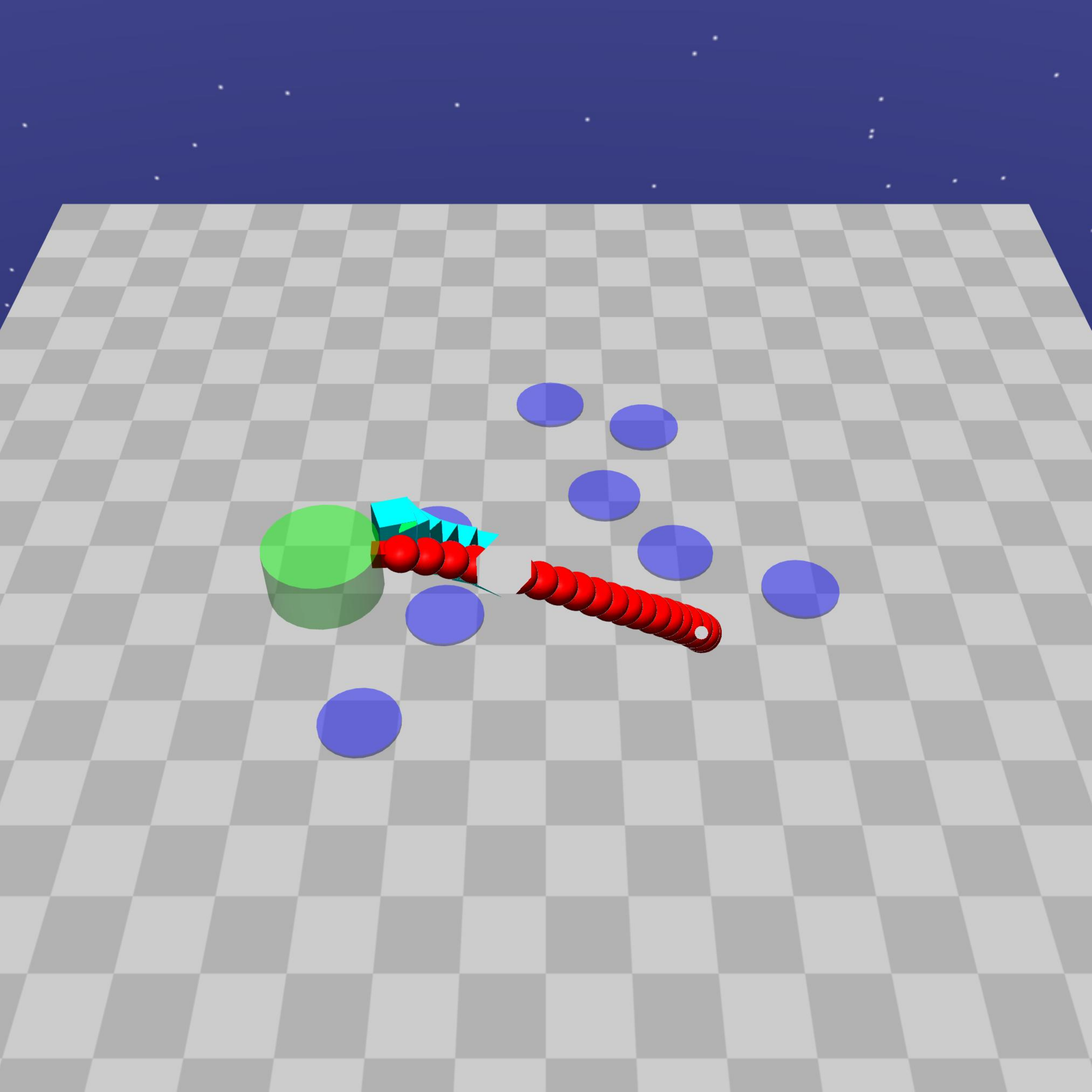}
        \includegraphics[trim=0 300 0 500, clip, width=0.24\linewidth]{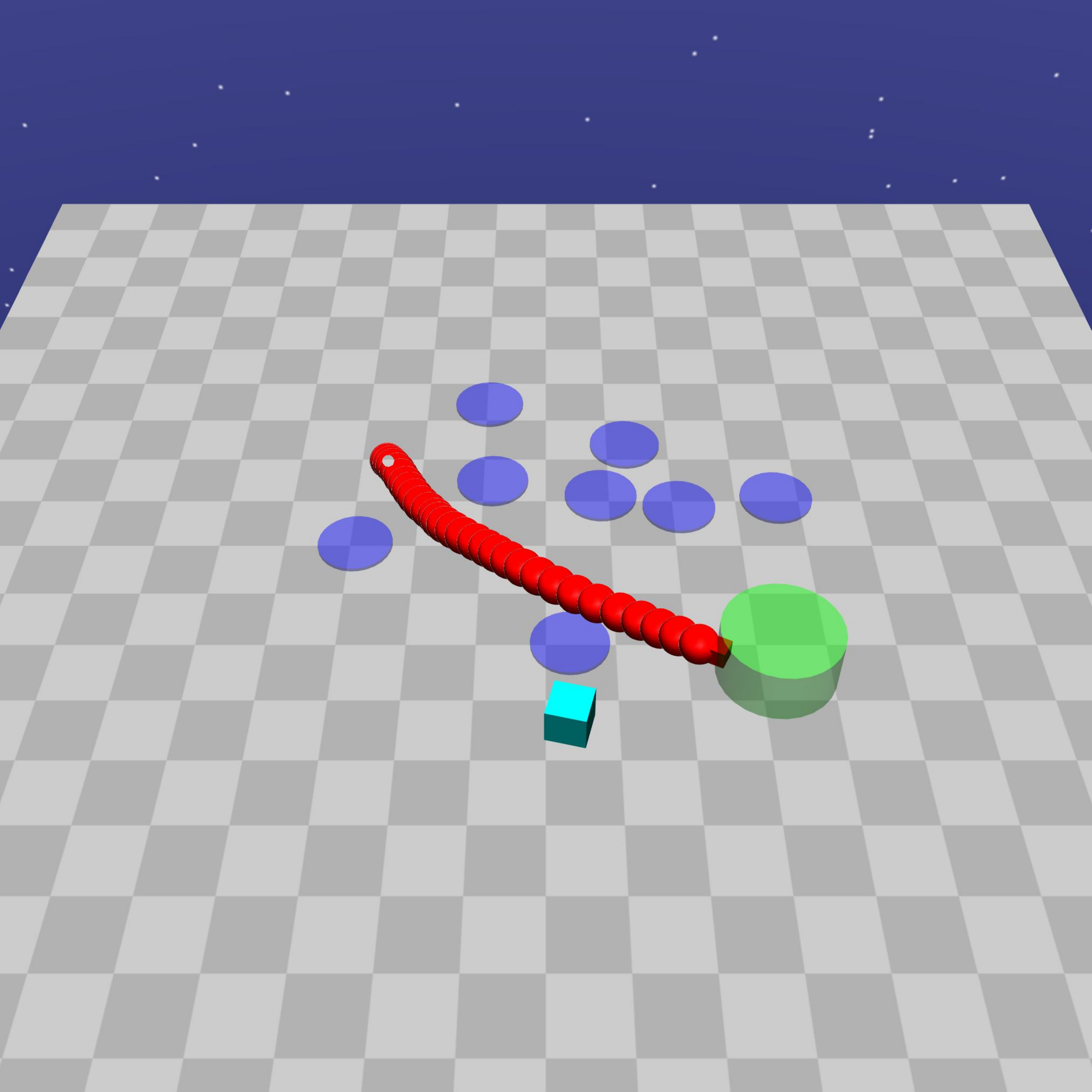}
        \includegraphics[trim=0 300 0 500, clip, width=0.24\linewidth]{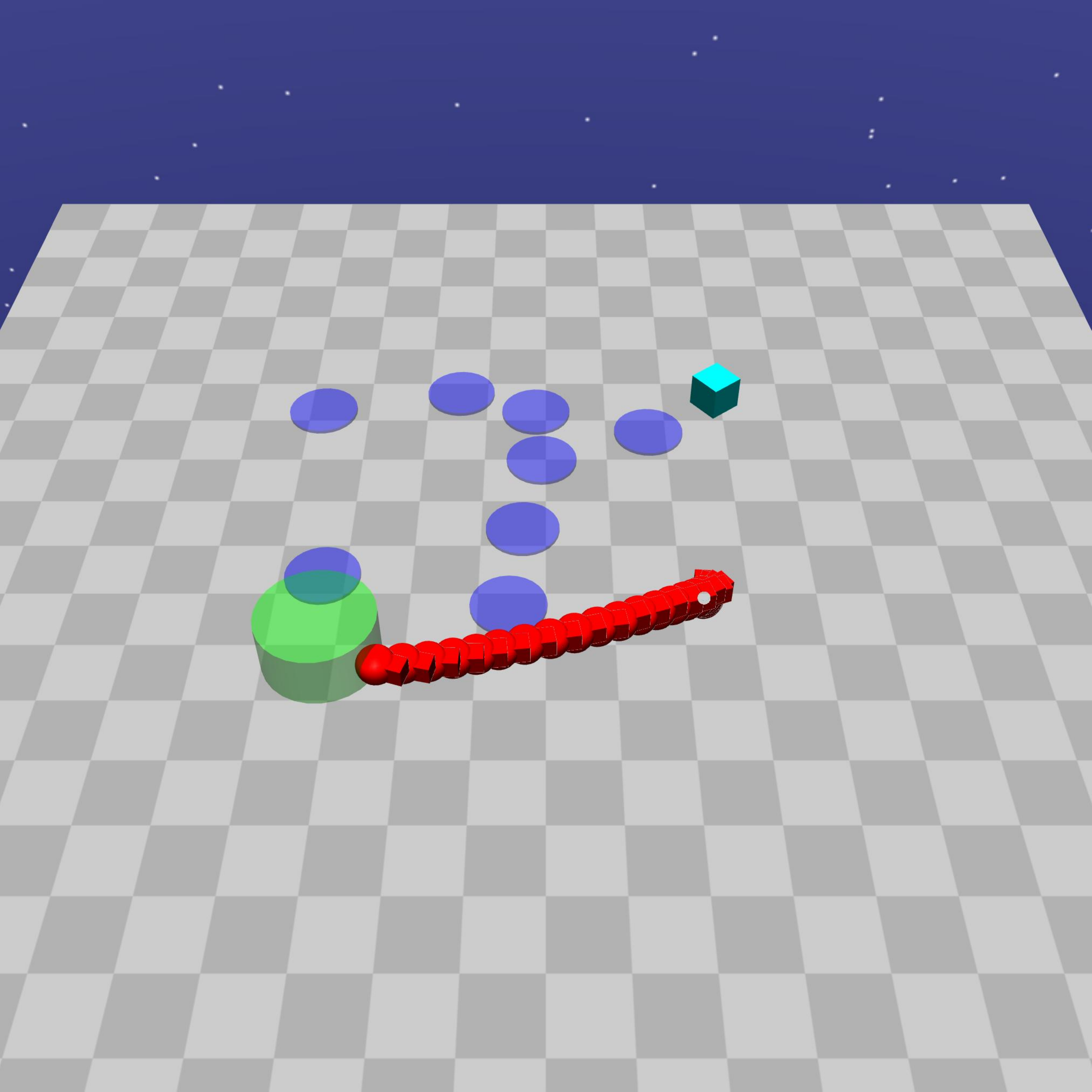}
        \\
        \includegraphics[trim=0 300 0 500, clip, width=0.24\linewidth]{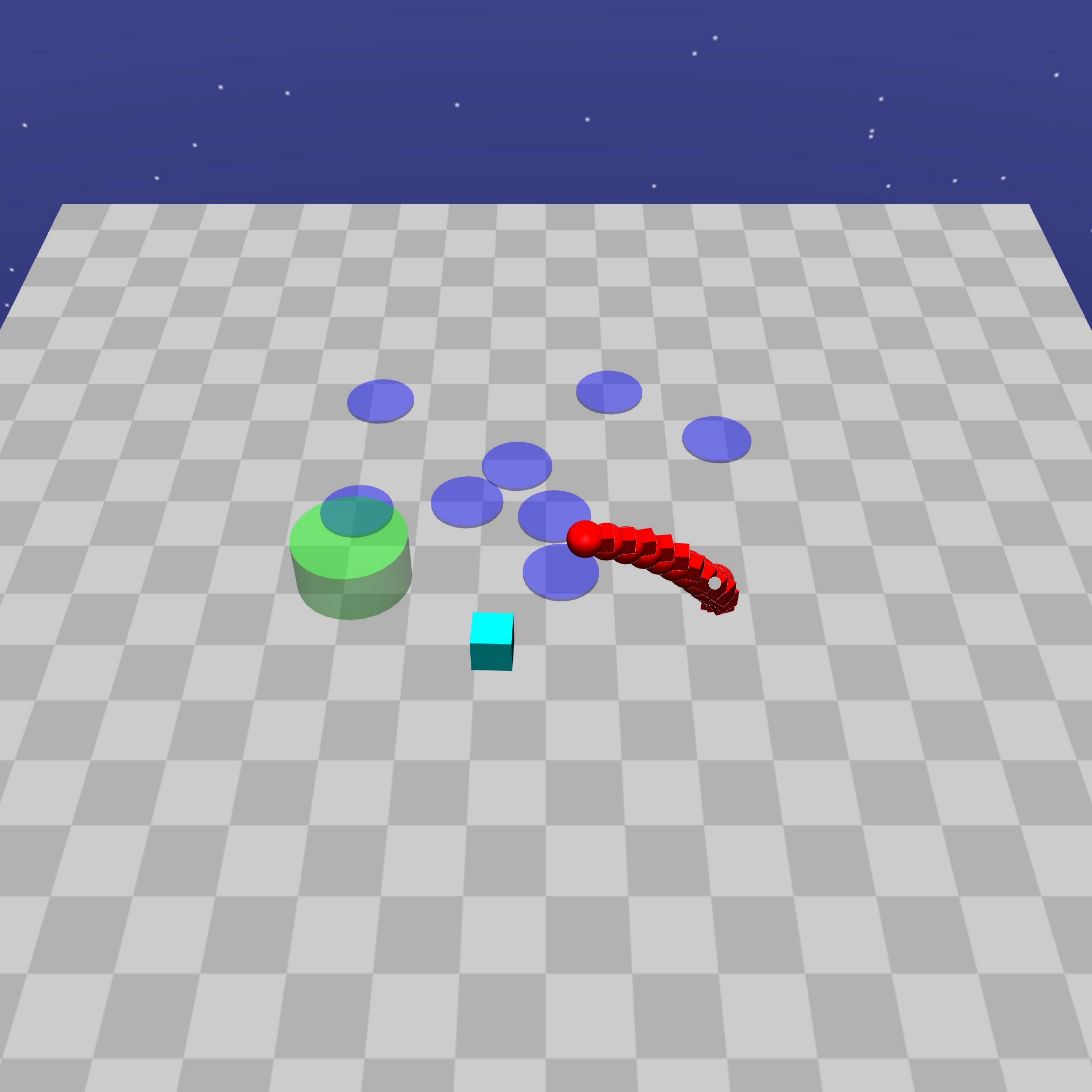}
        \includegraphics[trim=0 300 0 500, clip, width=0.24\linewidth]{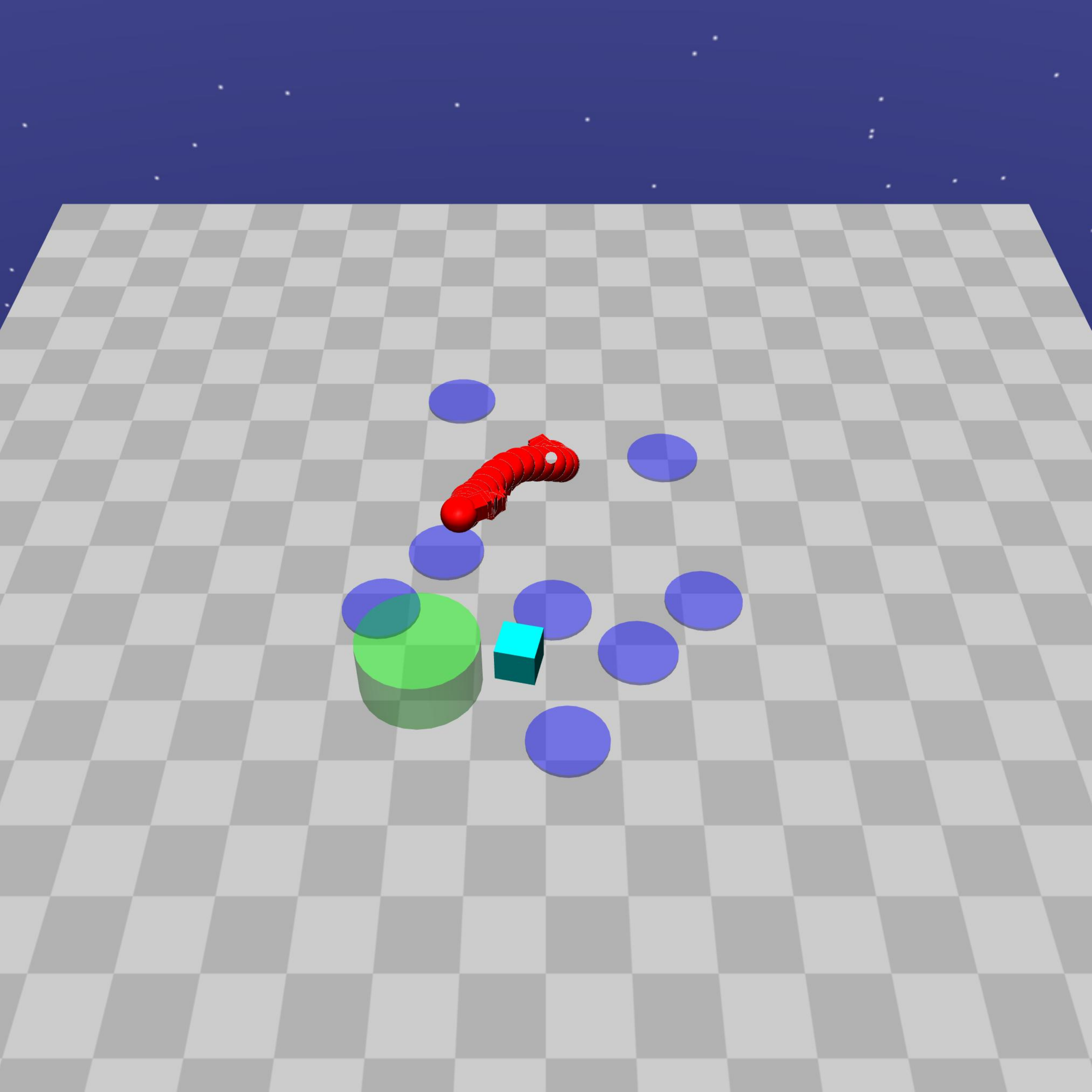}
        \includegraphics[trim=0 300 0 500, clip, width=0.24\linewidth]{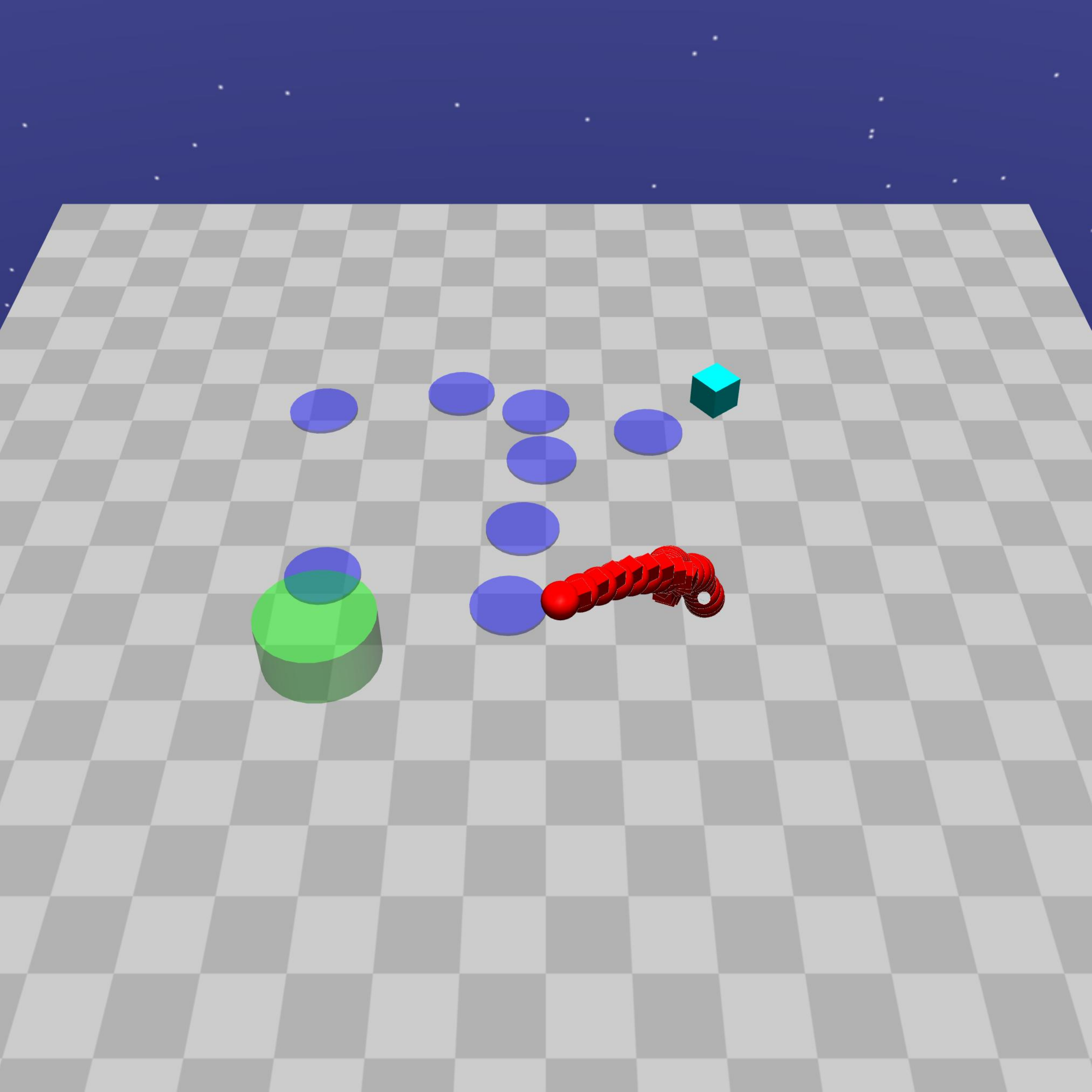}
        \includegraphics[trim=0 300 0 500, clip, width=0.24\linewidth]{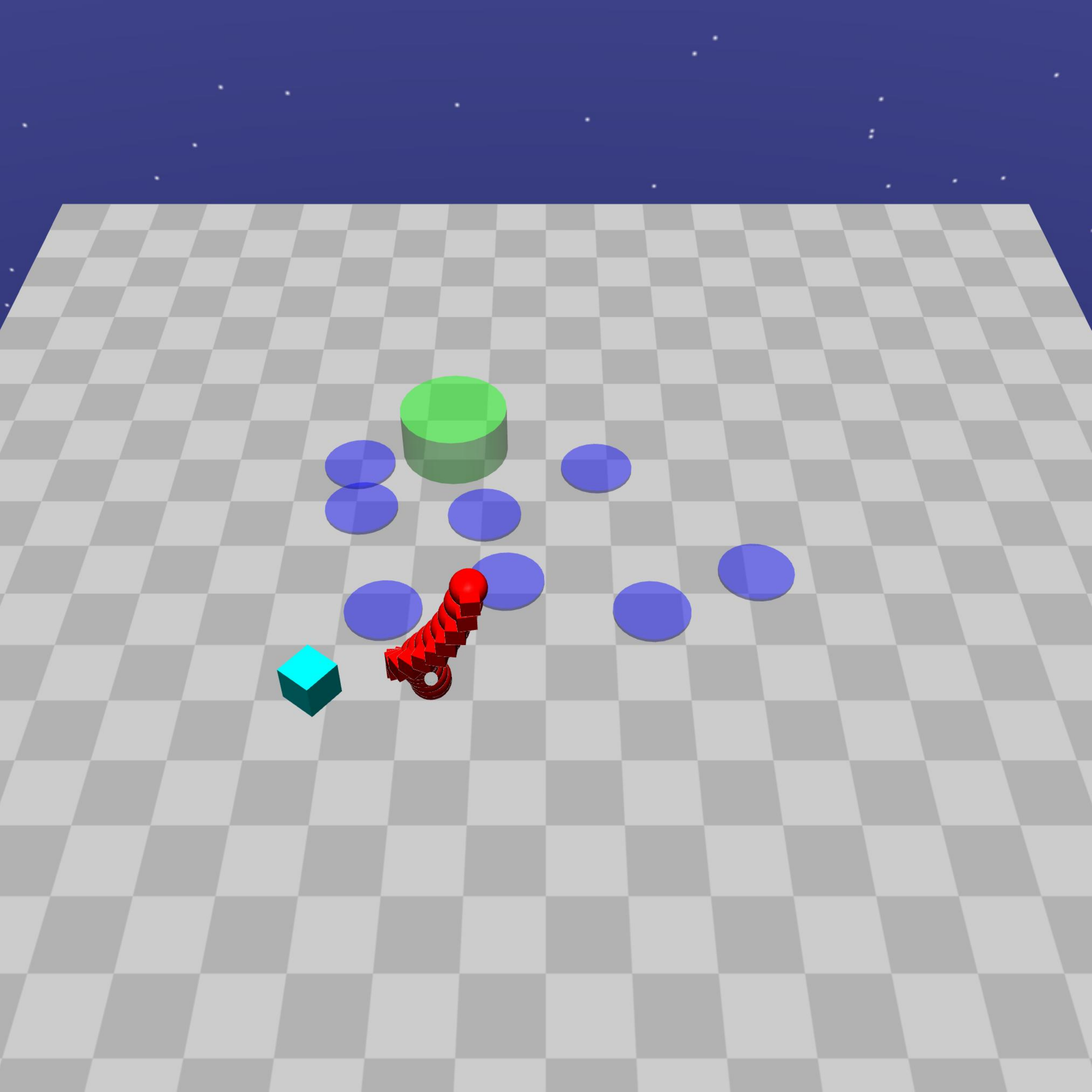}
        \caption{TRPO successes (top) and failures (bottom)}
    \end{subfigure}%
    \\
    \begin{subfigure}[t]{\textwidth}
        \centering
        \includegraphics[trim=0 300 0 500, clip, width=0.24\linewidth]{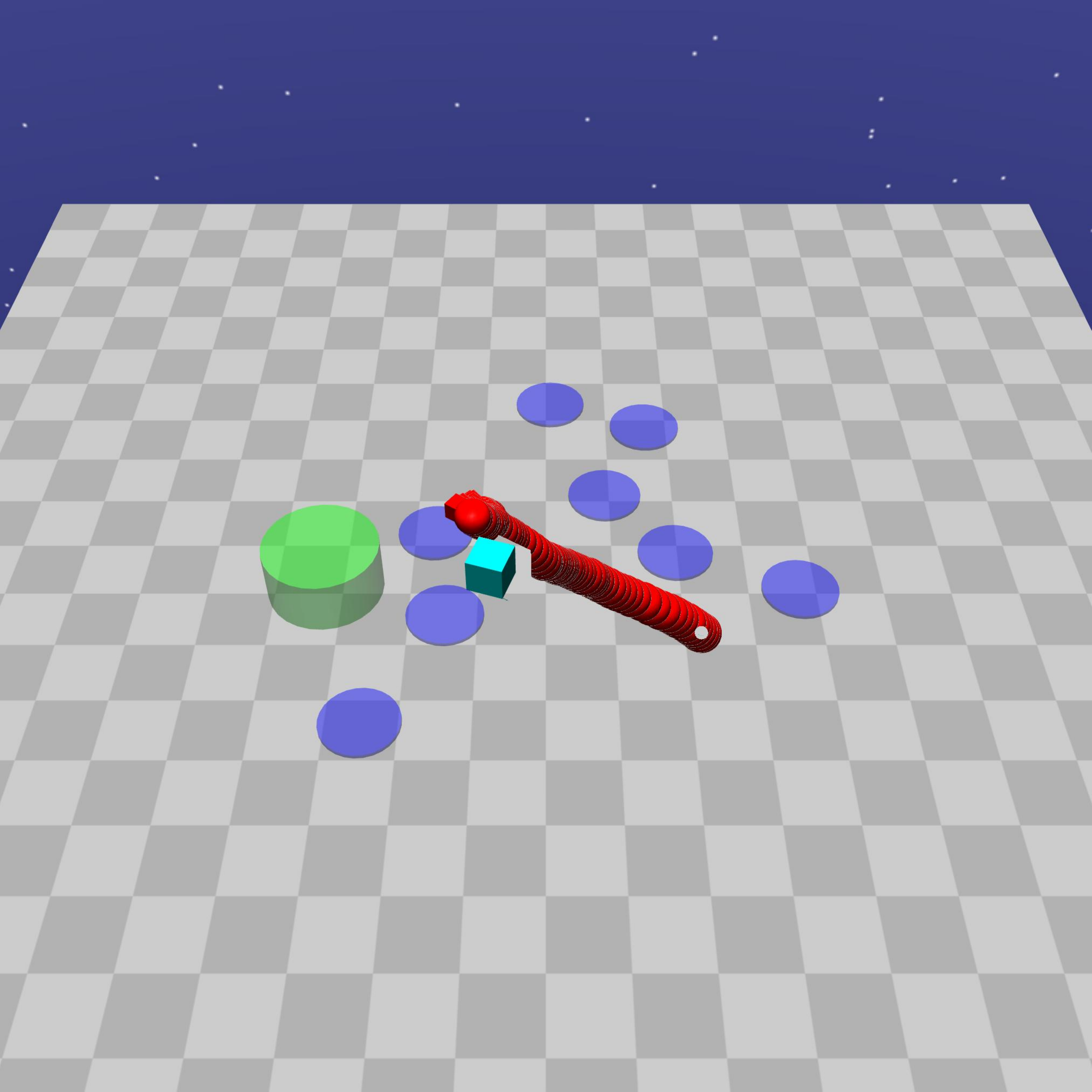}
        \includegraphics[trim=0 300 0 500, clip, width=0.24\linewidth]{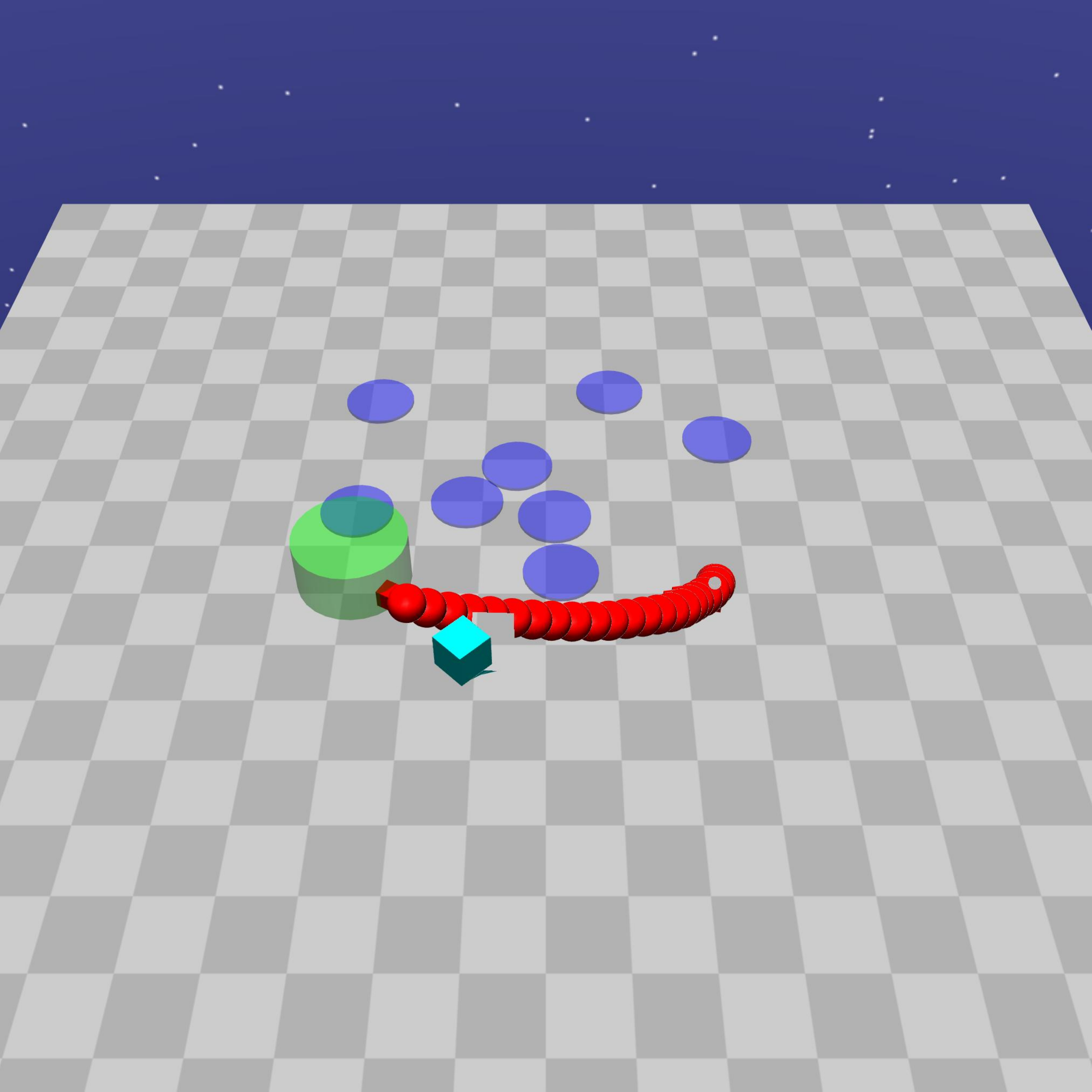}
        \includegraphics[trim=0 300 0 500, clip, width=0.24\linewidth]{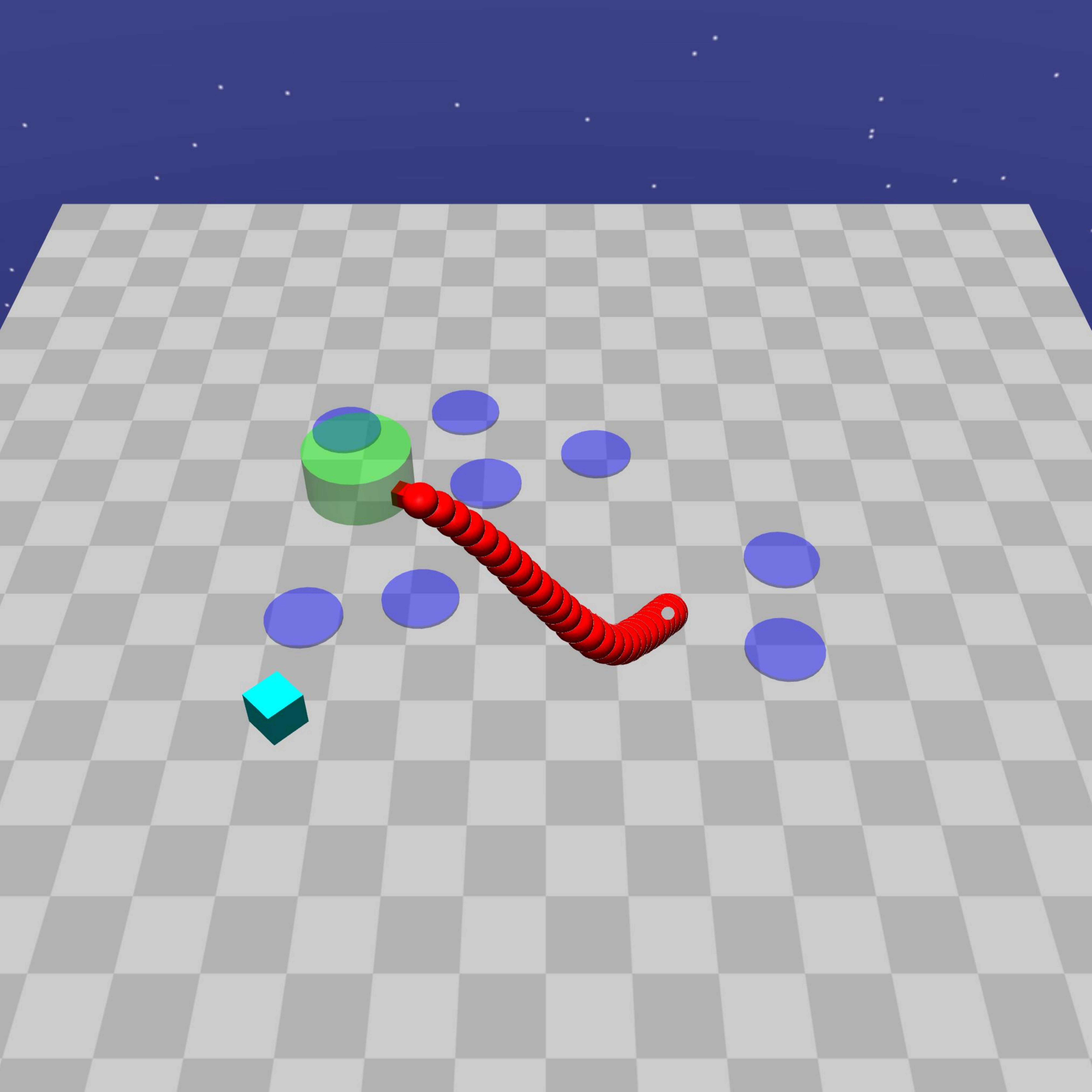}
        \includegraphics[trim=0 300 0 500, clip, width=0.24\linewidth]{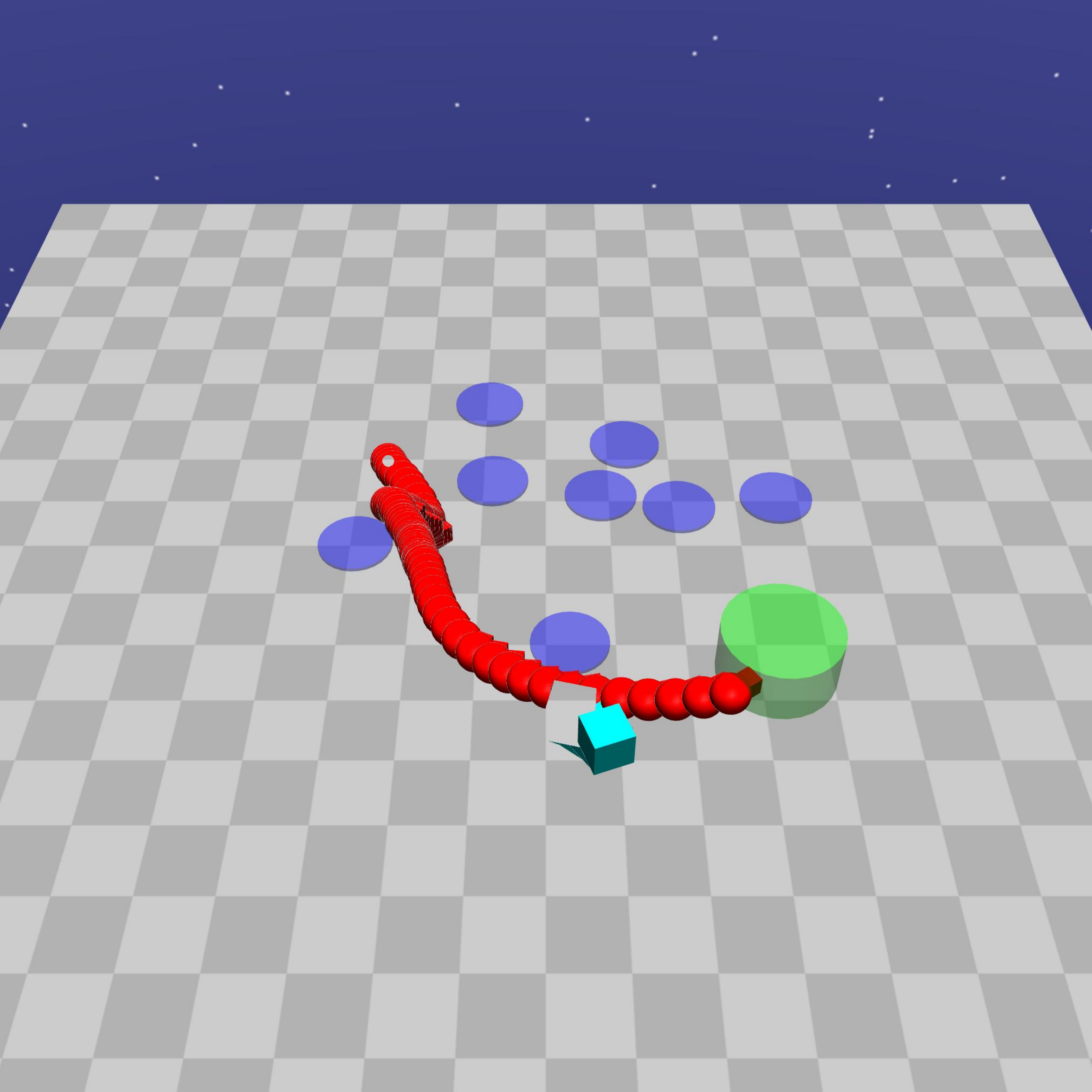}
        \\
        \includegraphics[trim=0 300 0 500, clip, width=0.24\linewidth]{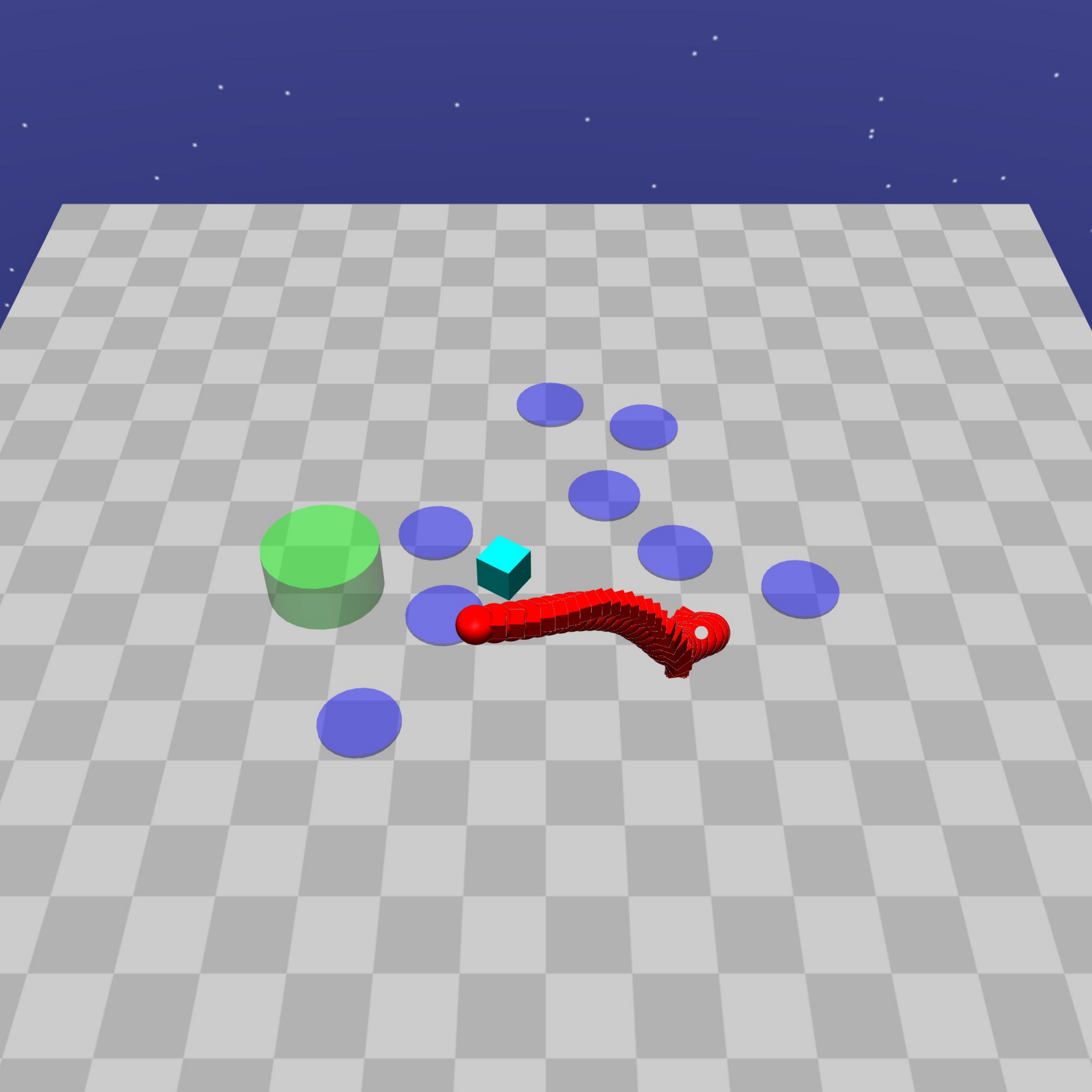}
        \includegraphics[trim=0 300 0 500, clip, width=0.24\linewidth]{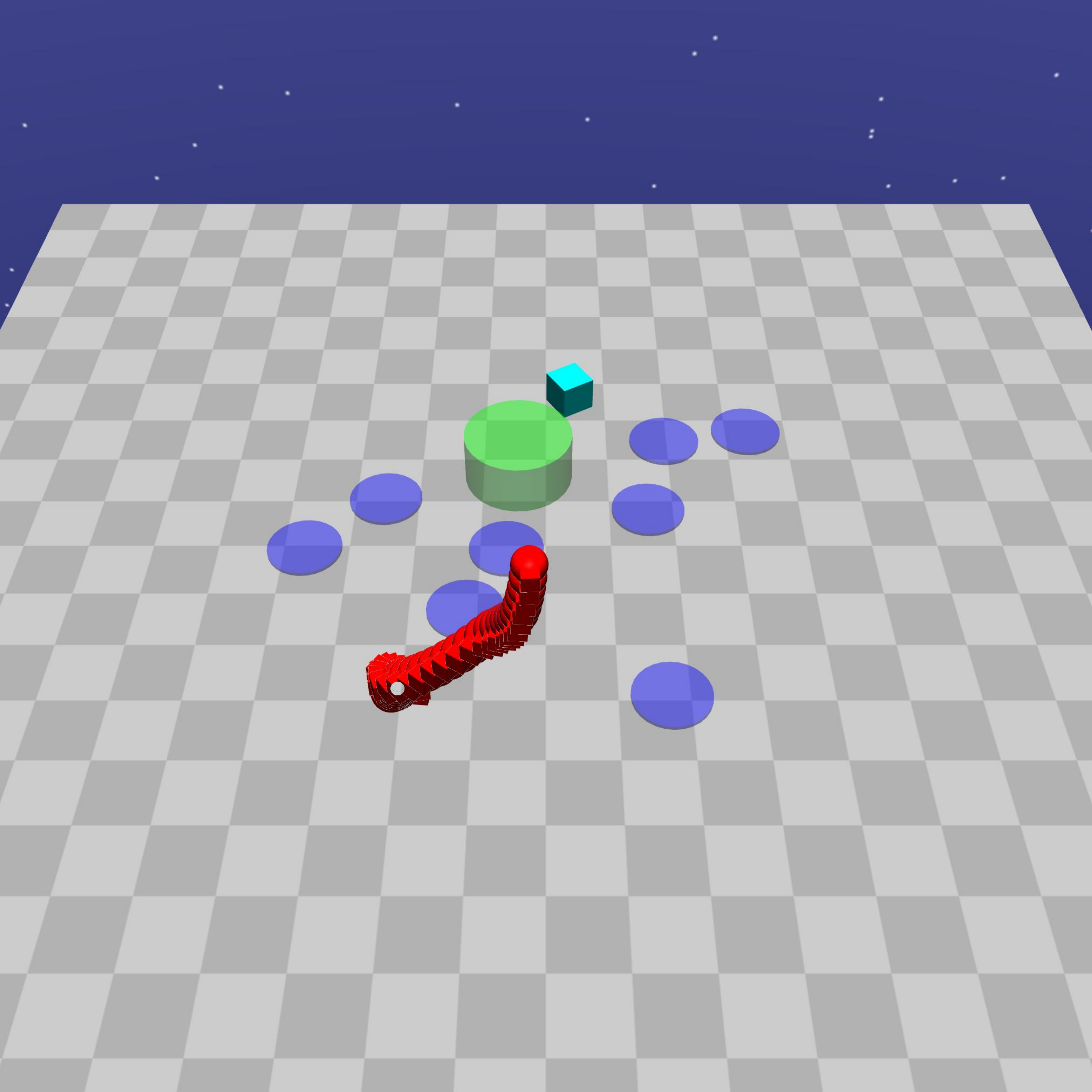}
        \includegraphics[trim=0 300 0 500, clip, width=0.24\linewidth]{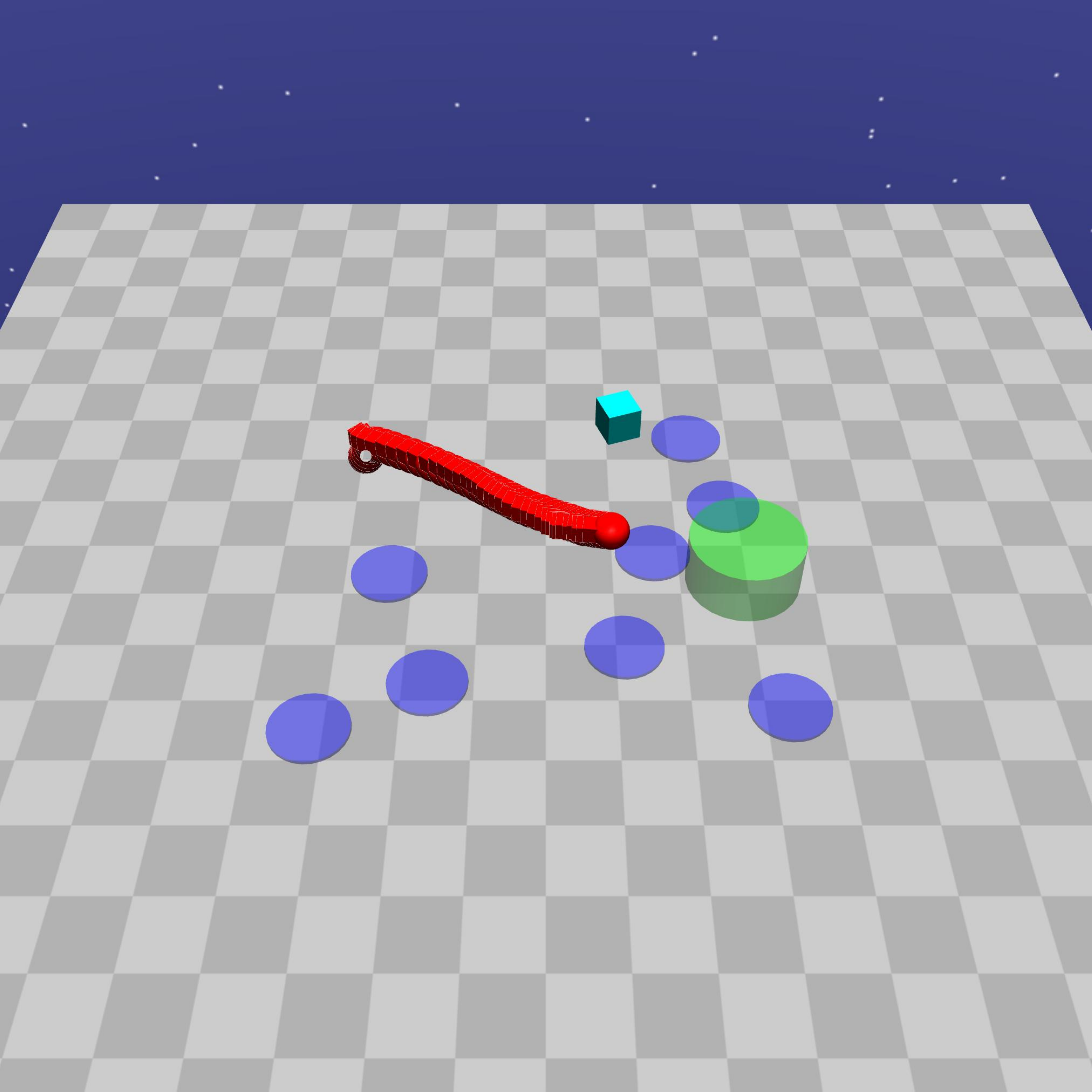}
        \includegraphics[trim=0 300 0 500, clip, width=0.24\linewidth]{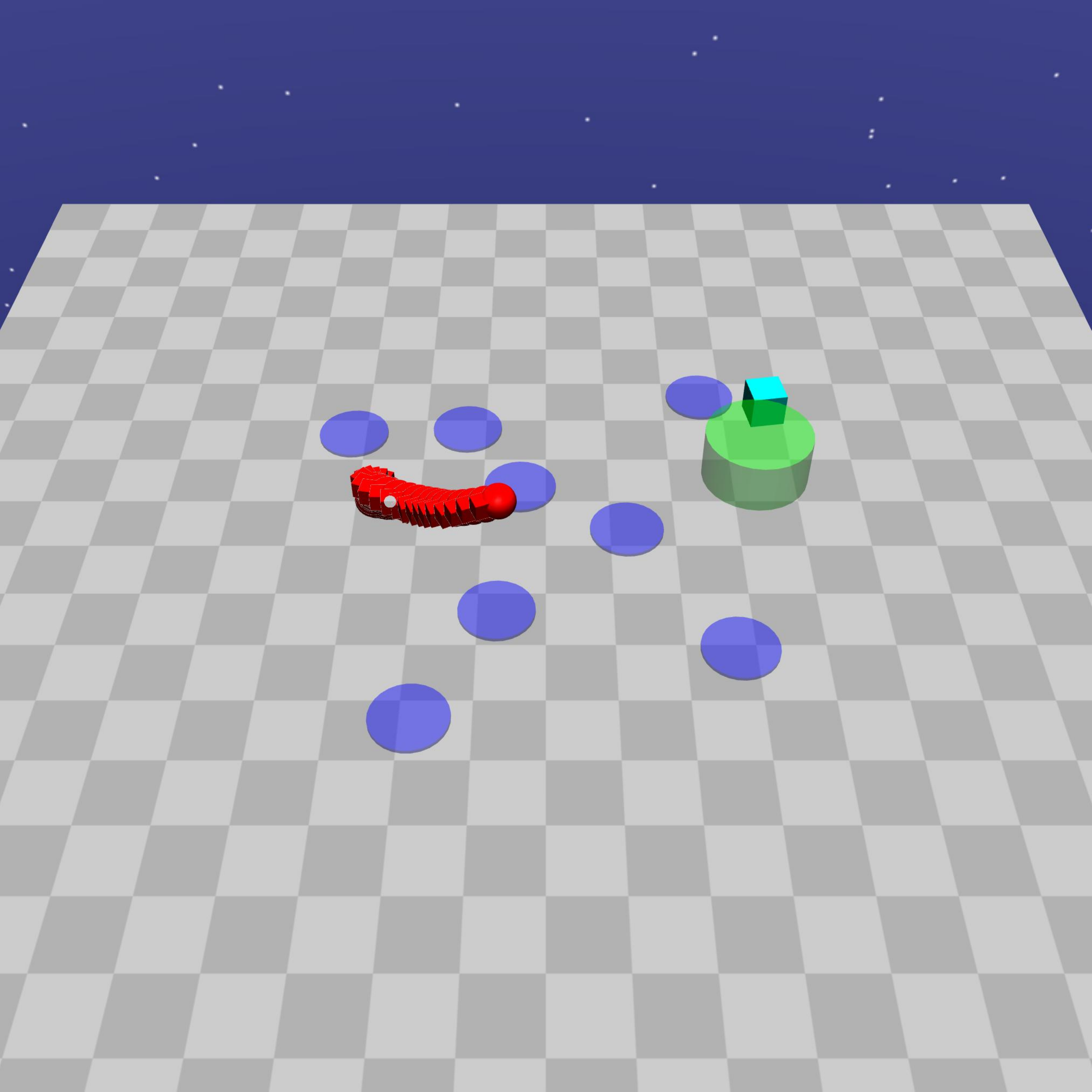}
        \caption{TRPO-Lagrangian successes (top) and failures (bottom)}
    \end{subfigure}%
    \\
    \begin{subfigure}[t]{\textwidth}
        \centering
        \includegraphics[trim=0 300 0 500, clip, width=0.24\linewidth]{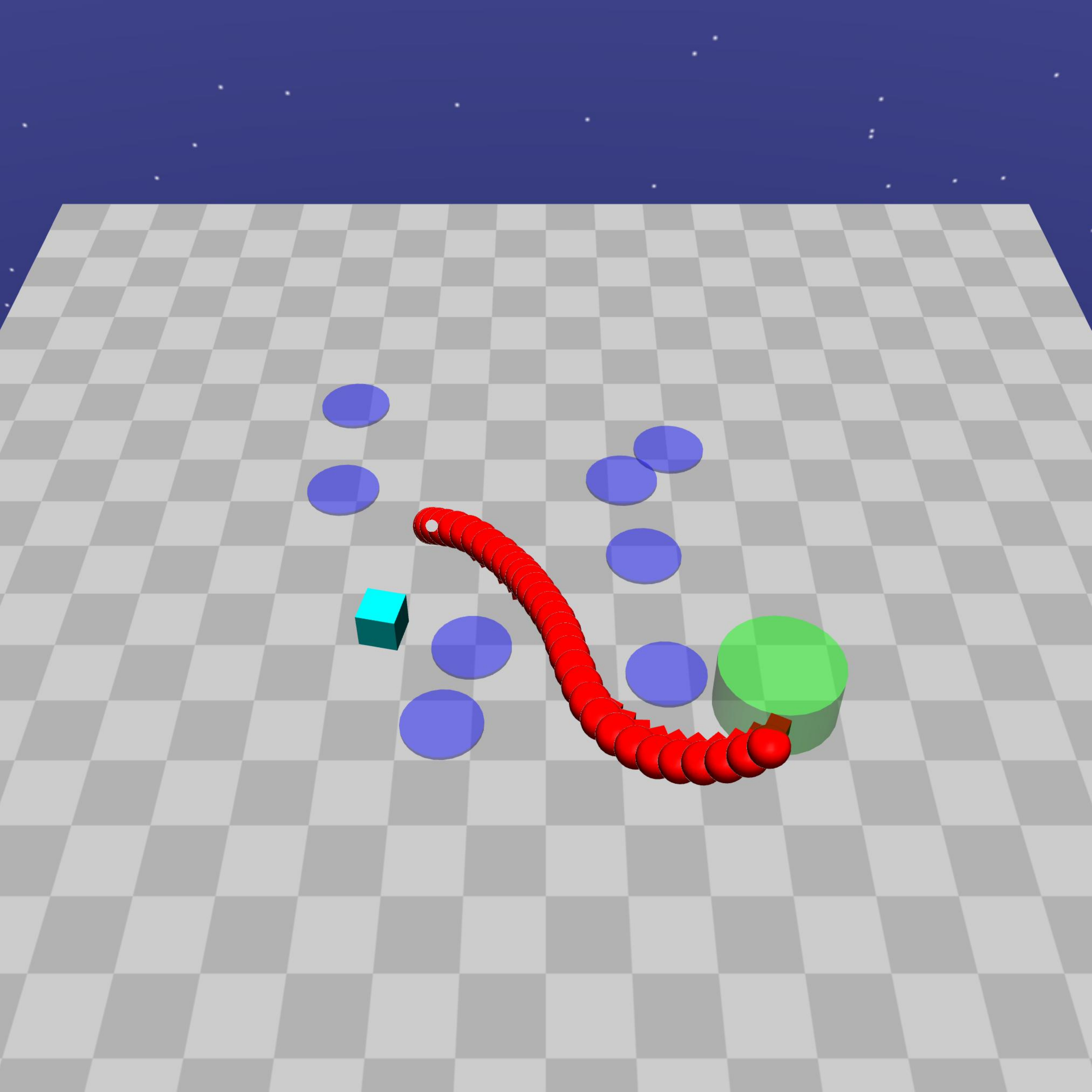}
        \includegraphics[trim=0 300 0 500, clip, width=0.24\linewidth]{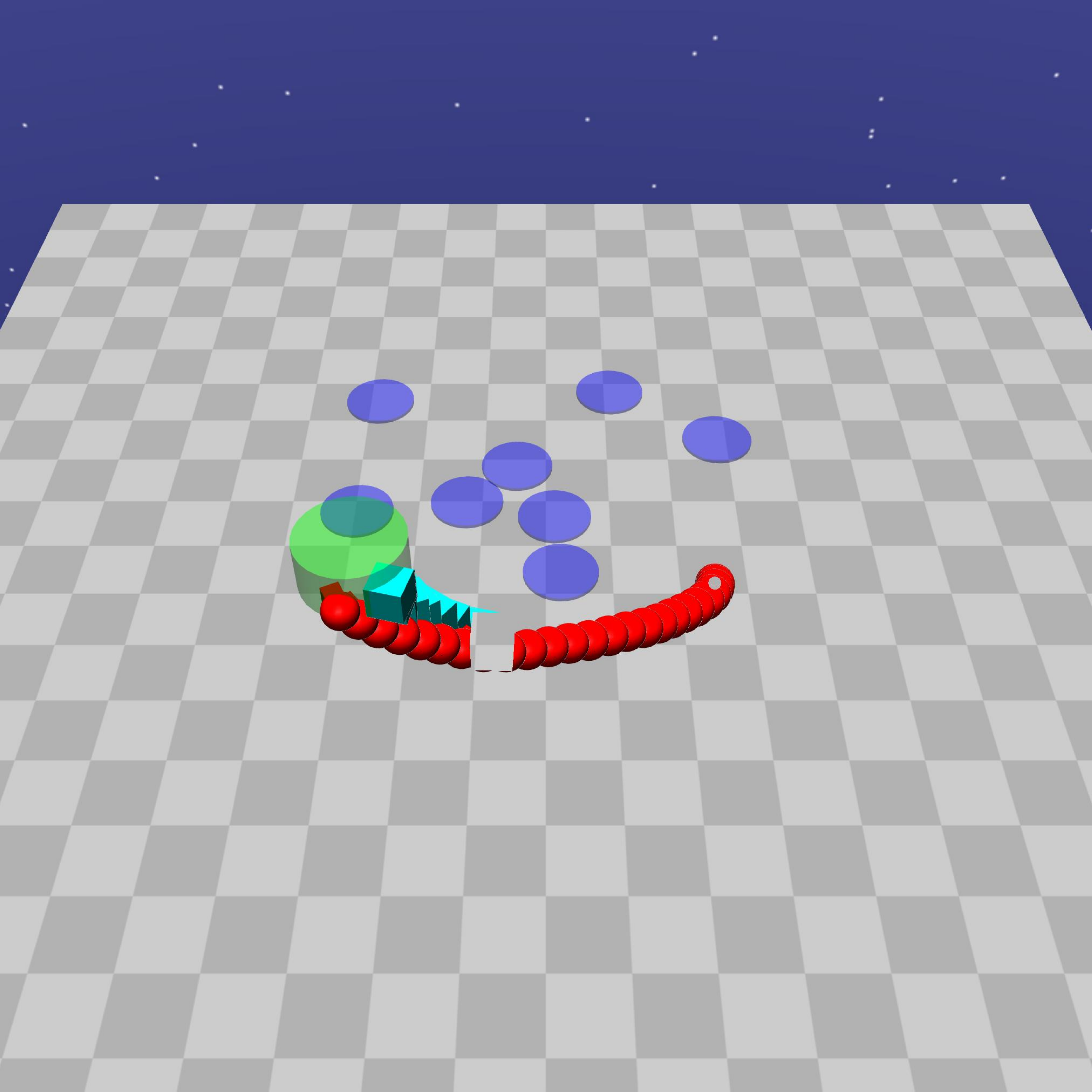}
        \includegraphics[trim=0 300 0 500, clip, width=0.24\linewidth]{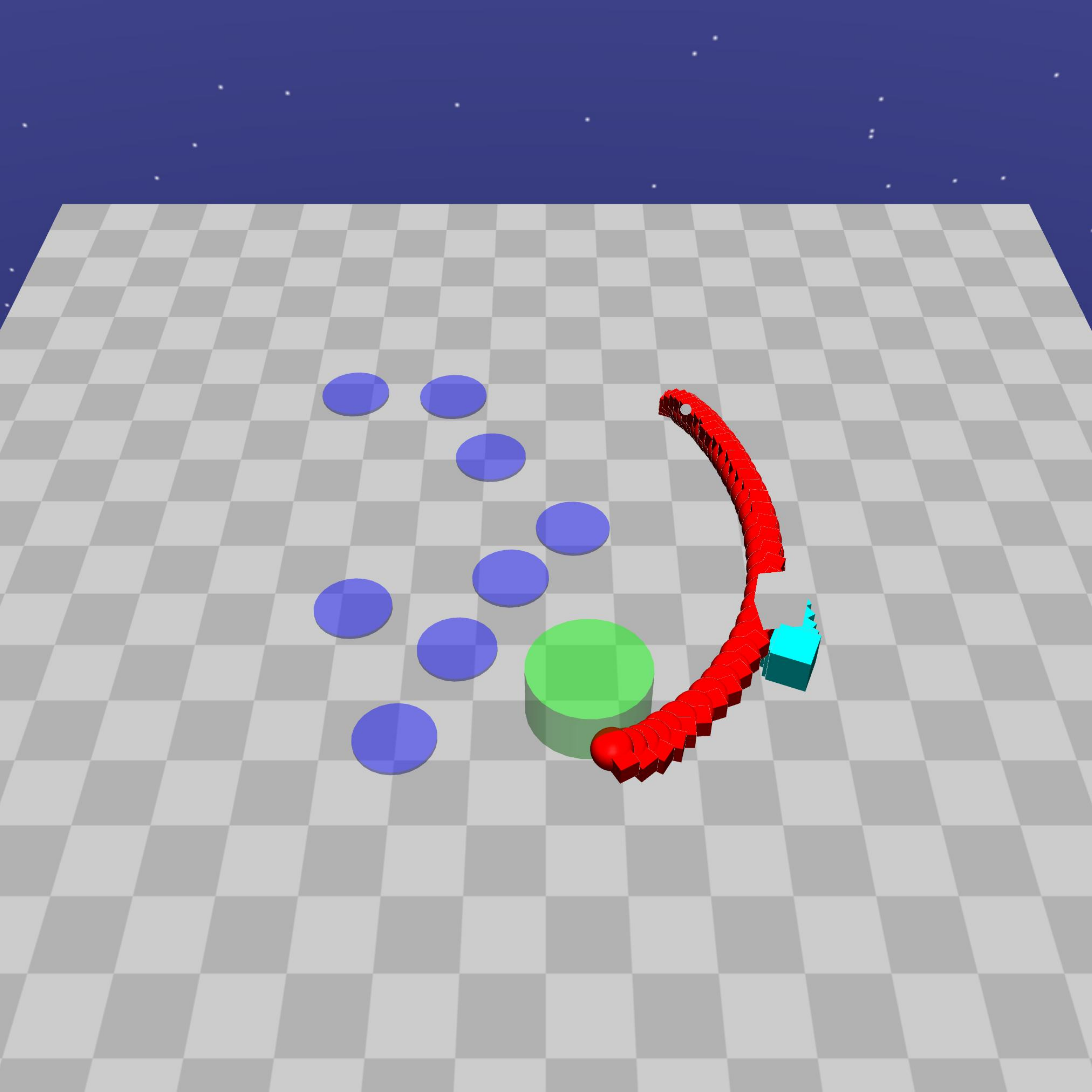}
        \includegraphics[trim=0 300 0 500, clip, width=0.24\linewidth]{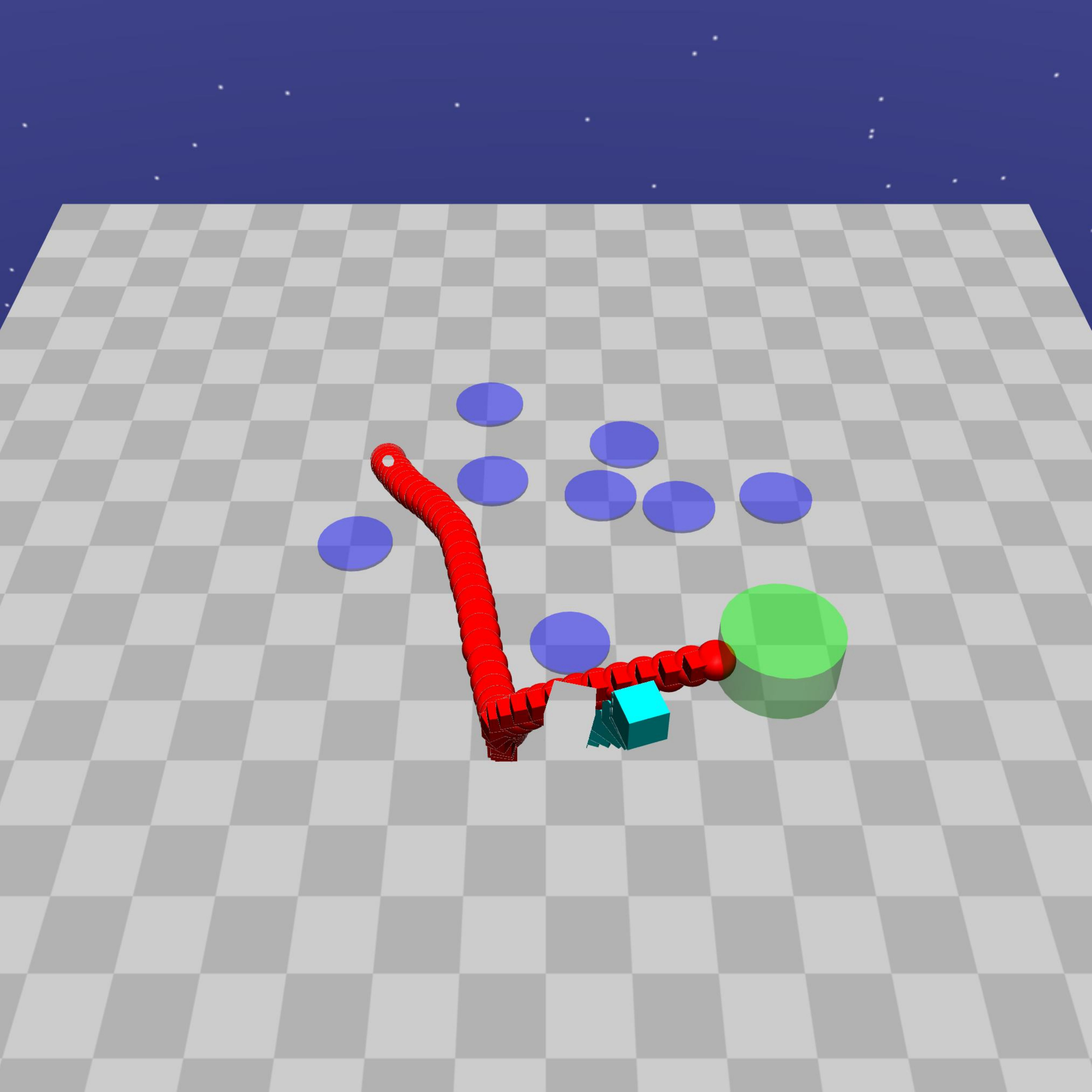}
        \\
        \includegraphics[trim=0 300 0 500, clip, width=0.24\linewidth]{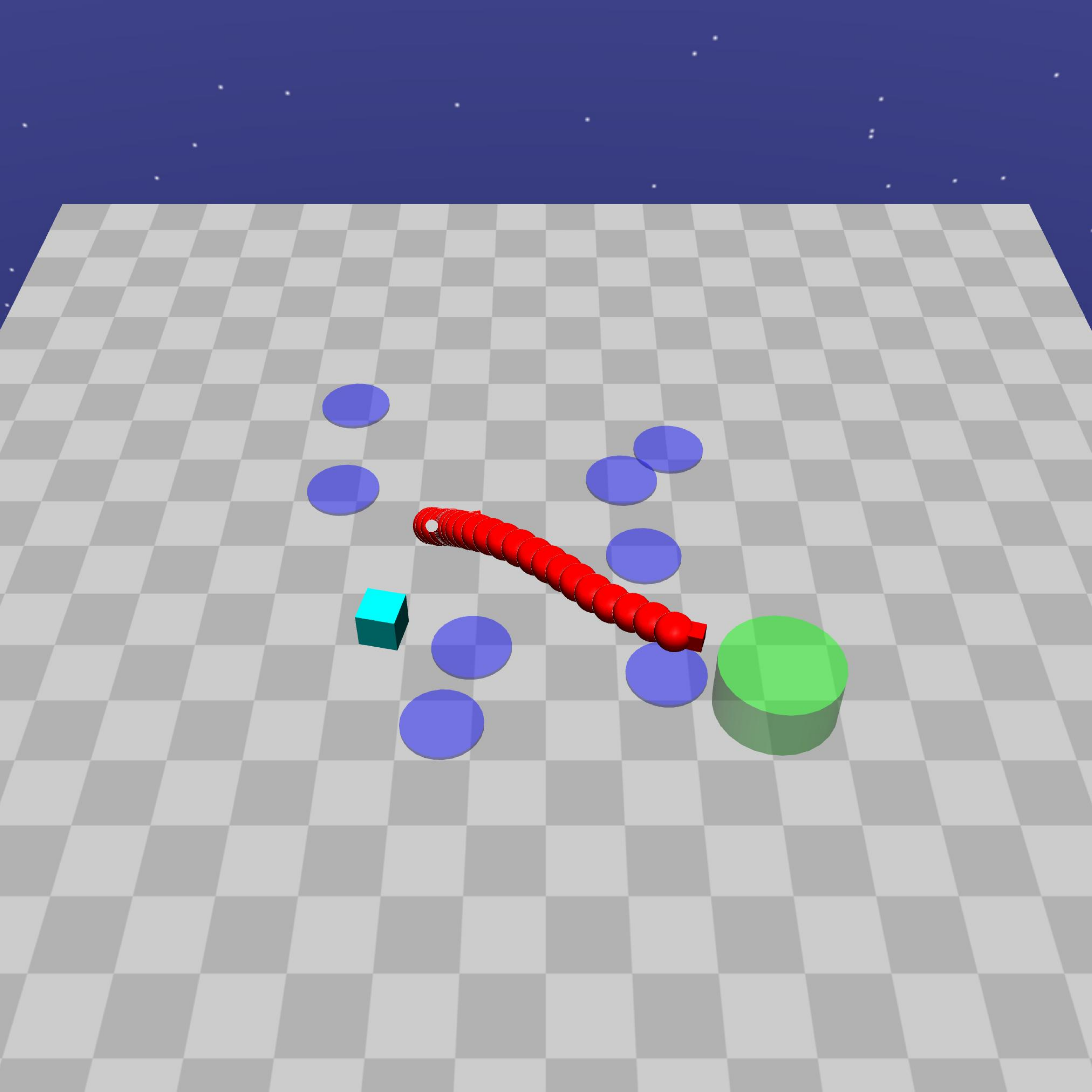}
        \includegraphics[trim=0 300 0 500, clip, width=0.24\linewidth]{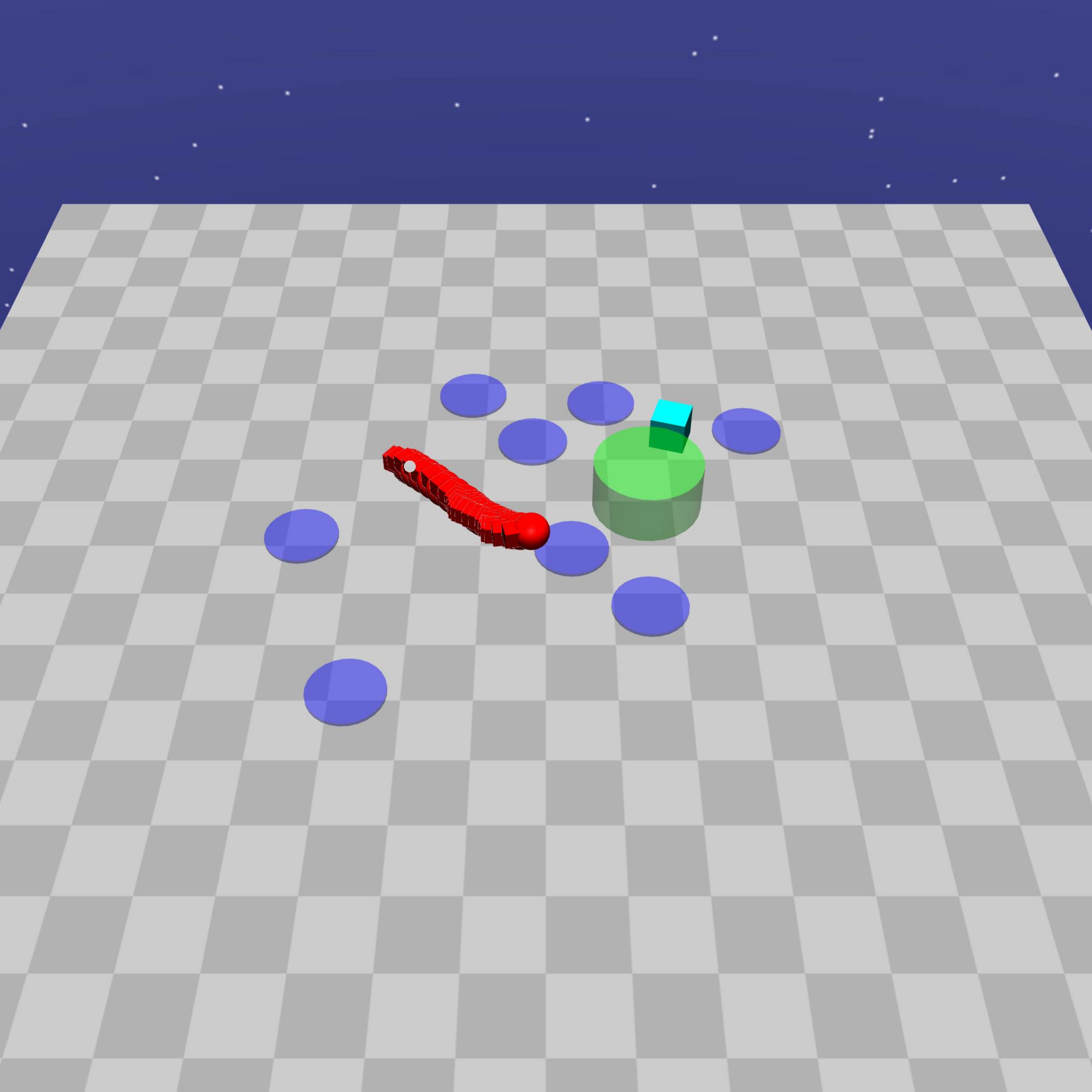}
        \includegraphics[trim=0 300 0 500, clip, width=0.24\linewidth]{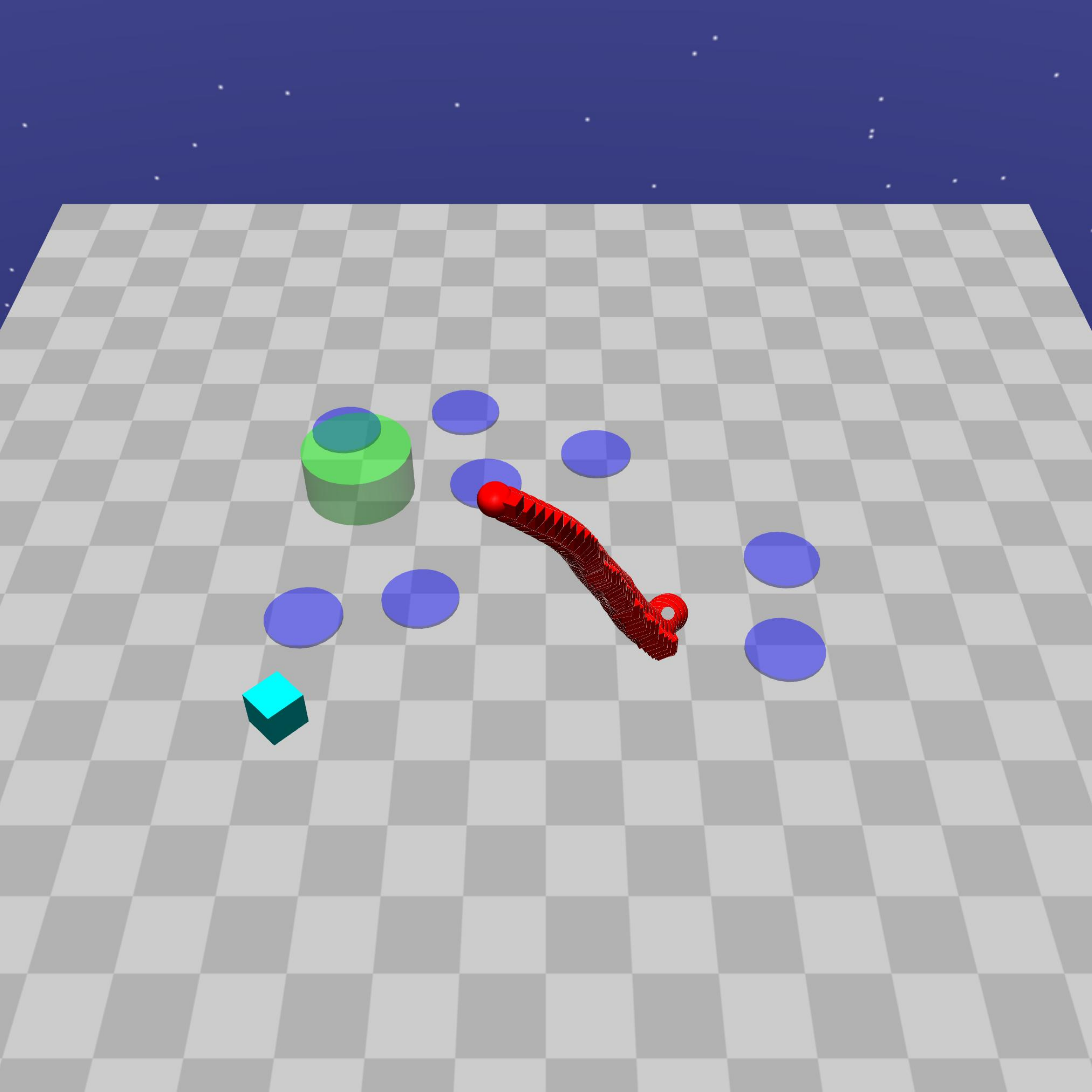}
        \includegraphics[trim=0 300 0 500, clip, width=0.24\linewidth]{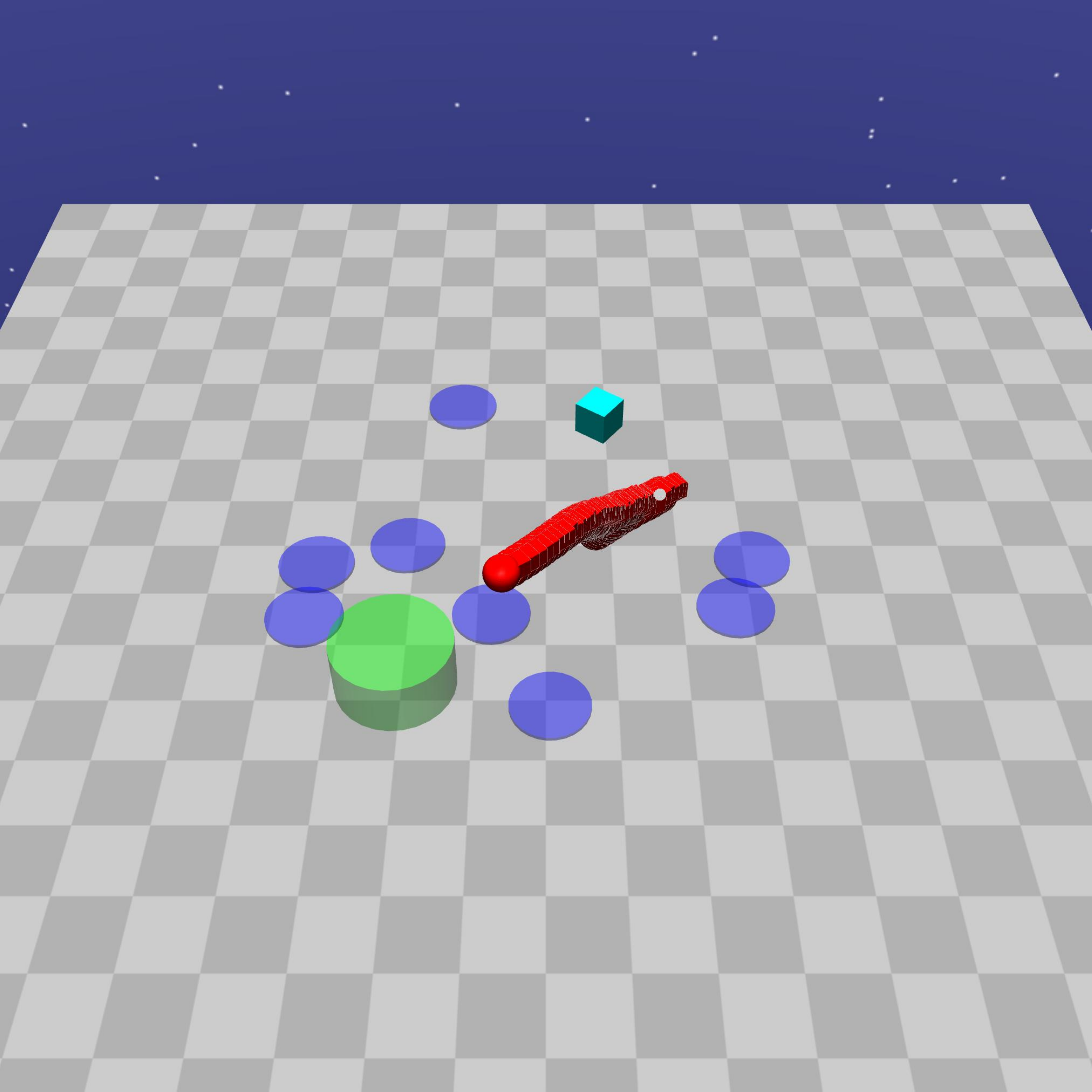}
        \caption{CPO successes (top) and failures (bottom)}
    \end{subfigure}%
    \\
    \begin{subfigure}[t]{\textwidth}
        \centering
        \includegraphics[trim=0 300 0 300, clip, width=0.24\linewidth]{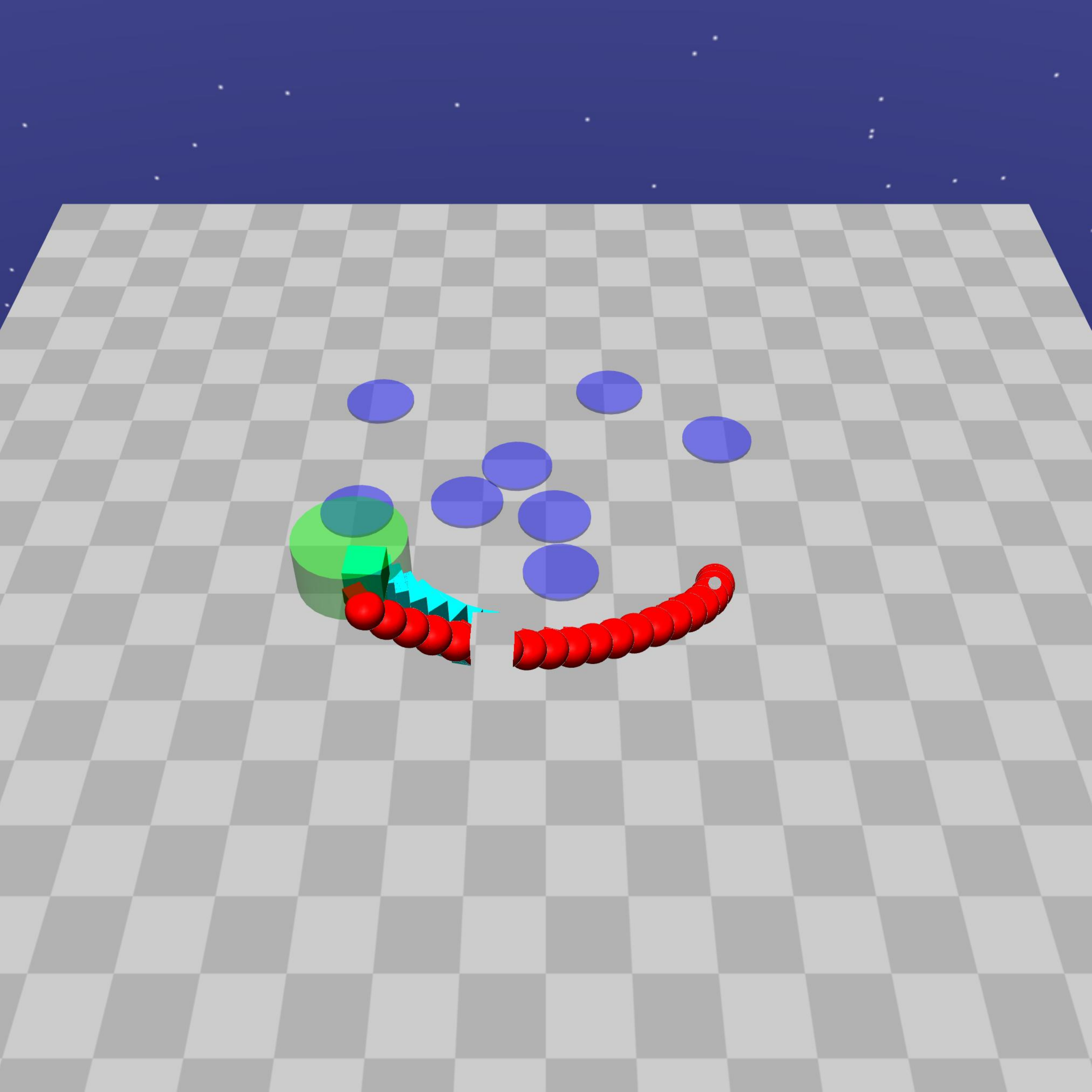}
        \includegraphics[trim=0 300 0 300, clip, width=0.24\linewidth]{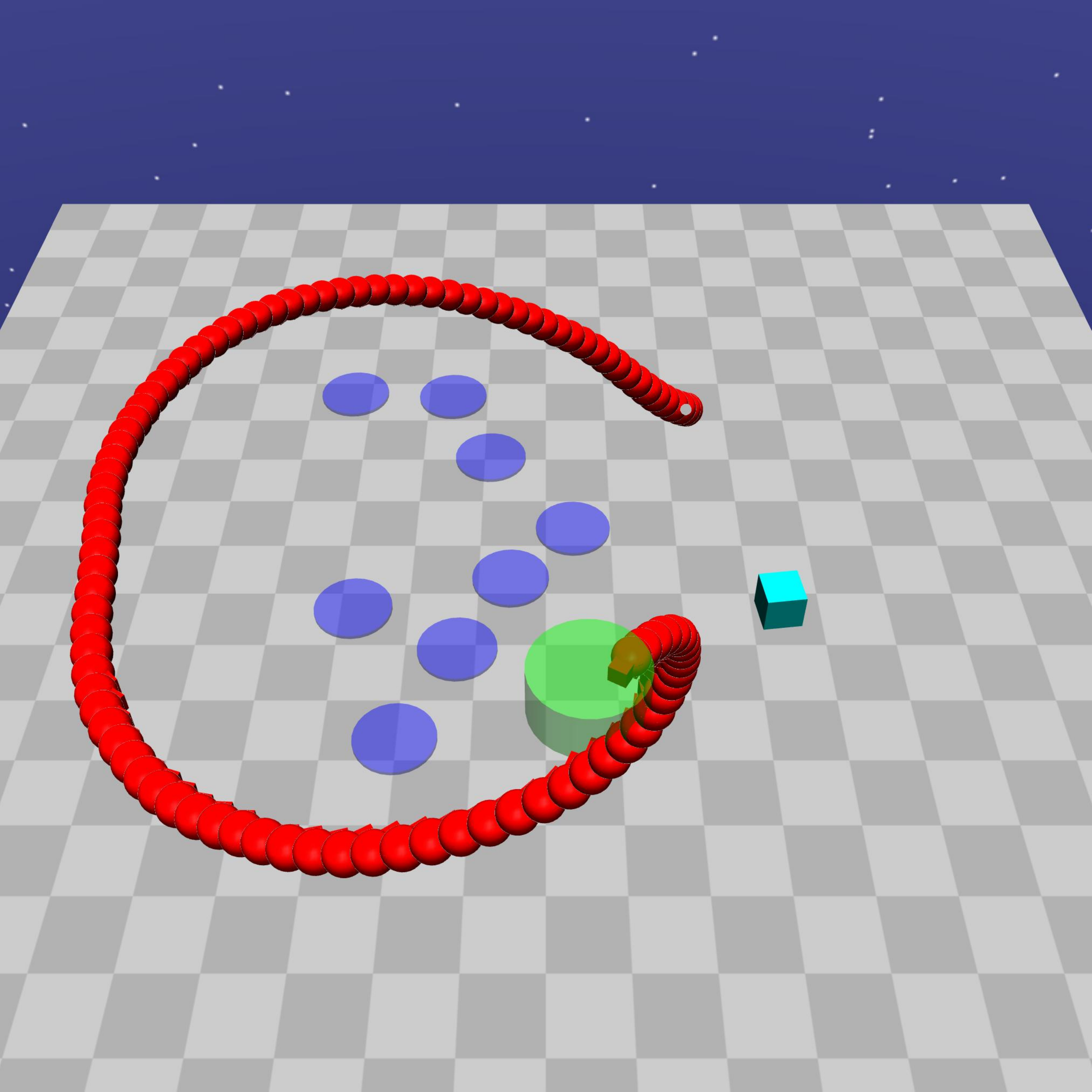}
        \includegraphics[trim=0 300 0 300, clip, width=0.24\linewidth]{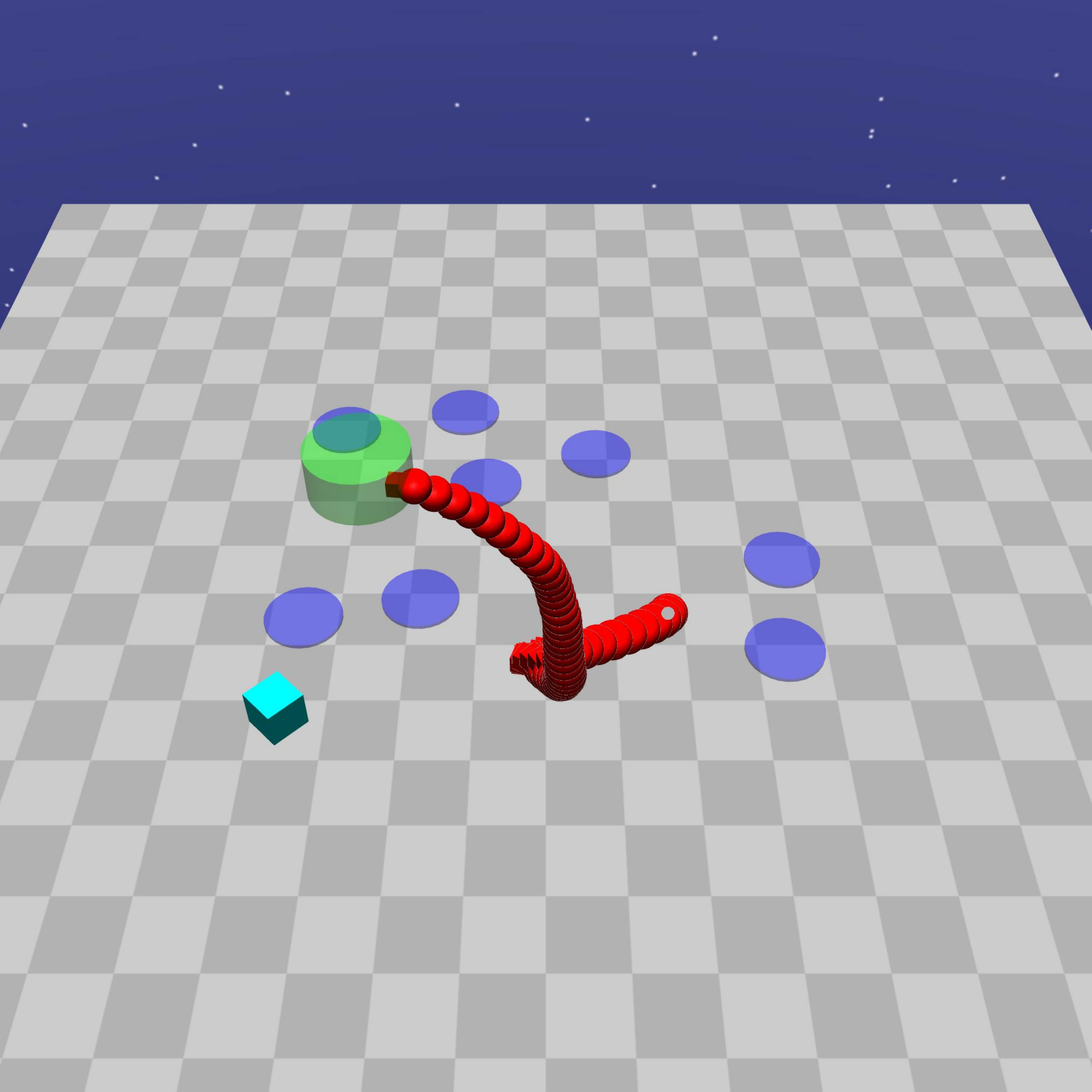}
        \includegraphics[trim=0 300 0 300, clip, width=0.24\linewidth]{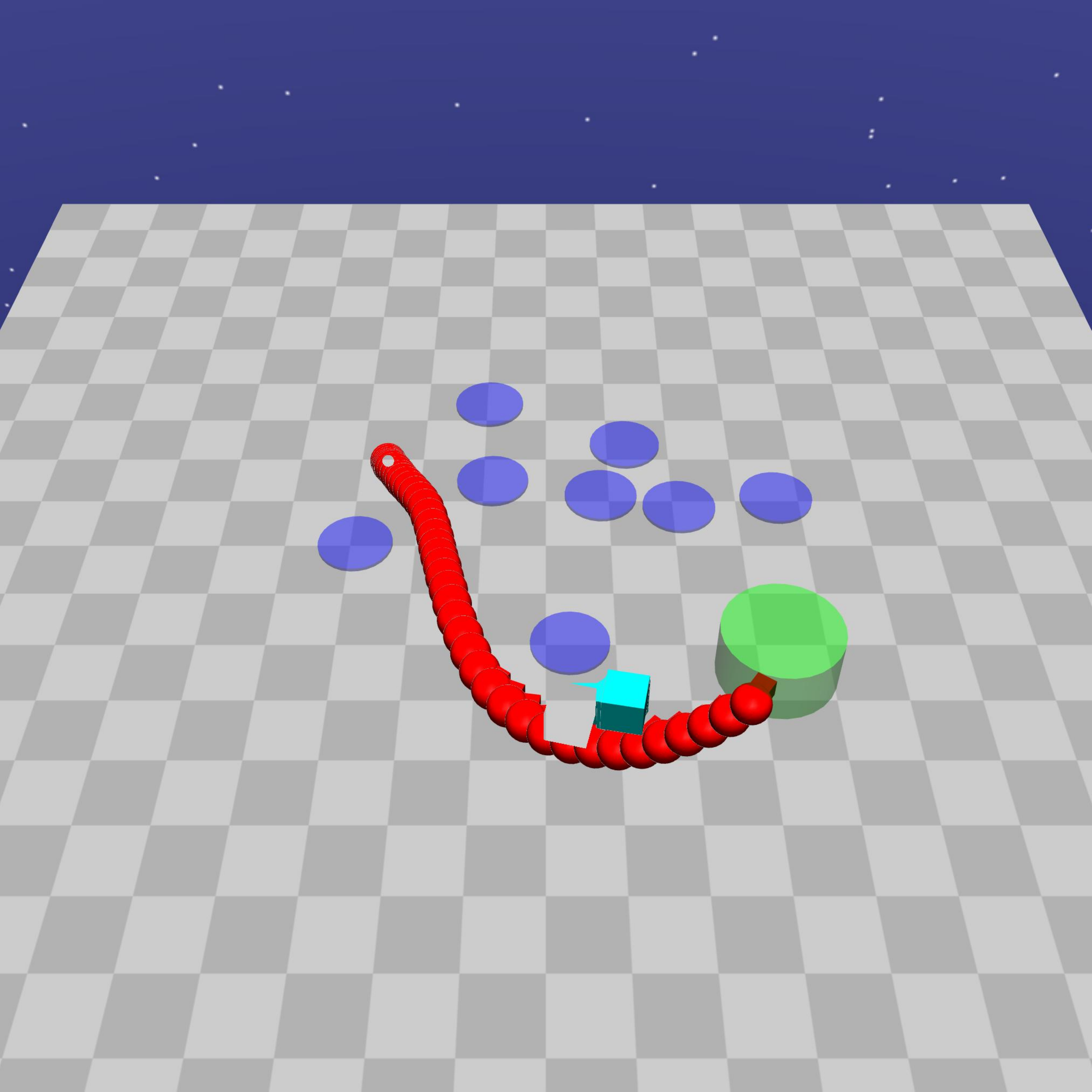}
        \\
        \includegraphics[trim=20 0 20 300, clip, width=0.24\linewidth]{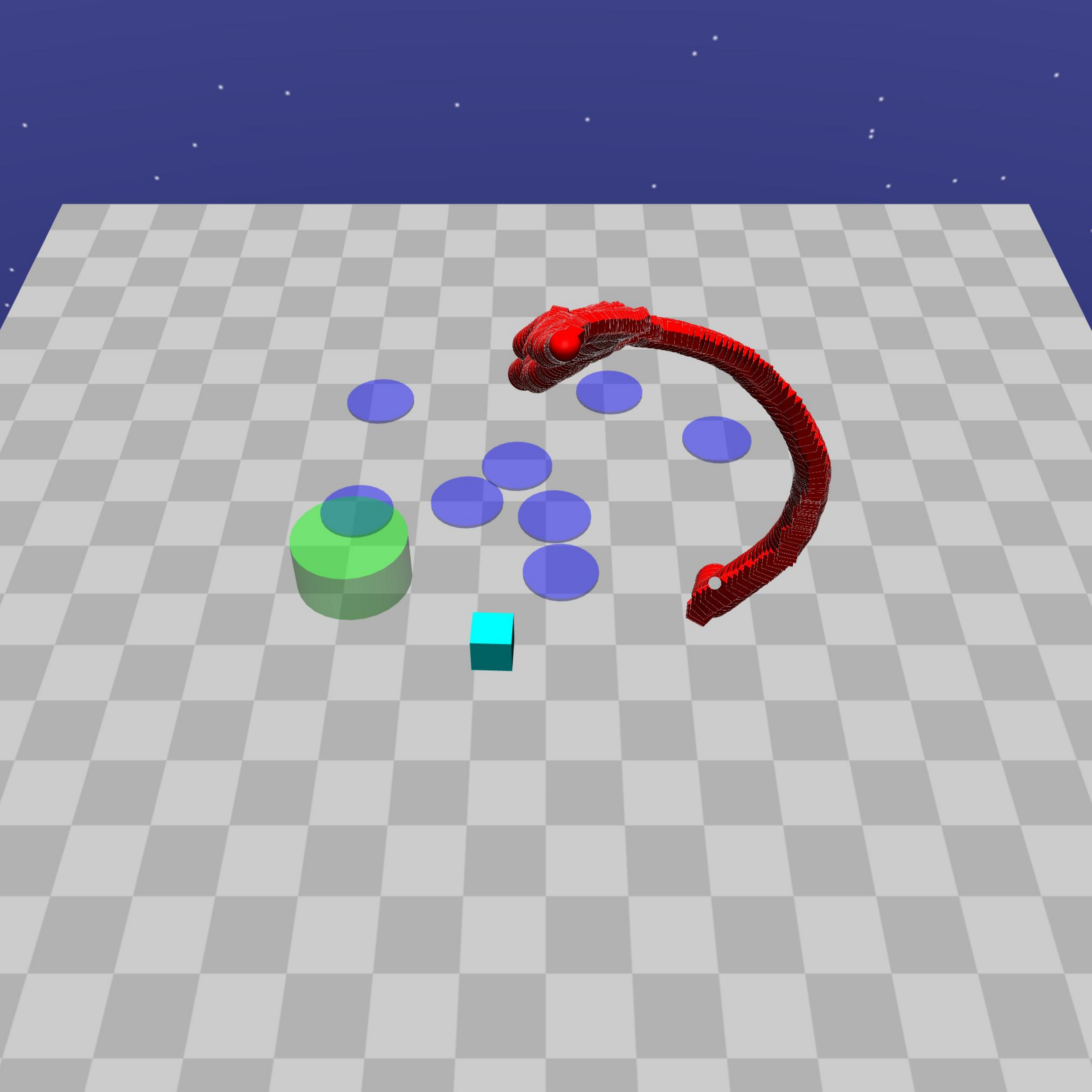}
        \includegraphics[trim=20 0 20 300, clip, width=0.24\linewidth]{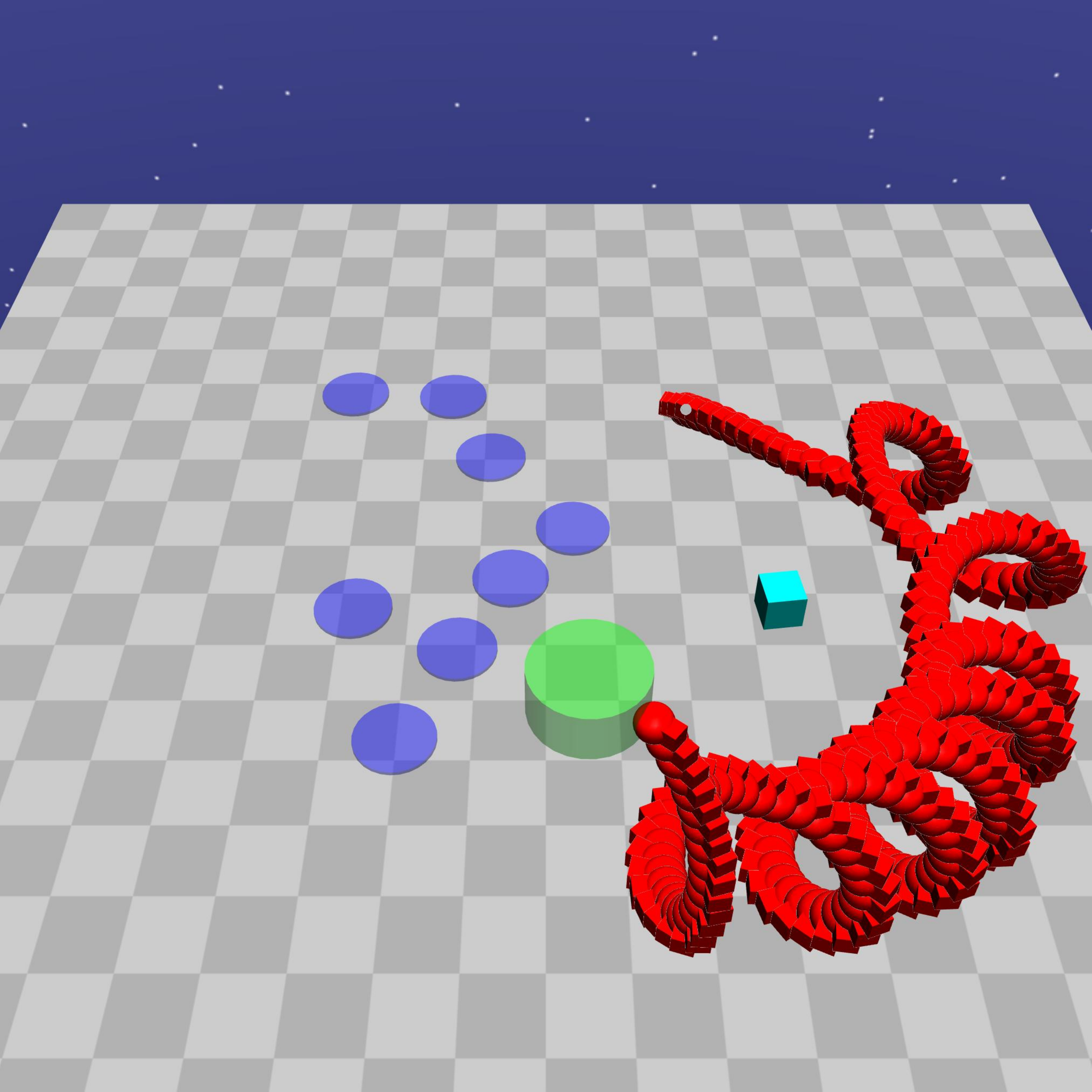}
        \includegraphics[trim=20 0 20 300, clip, width=0.24\linewidth]{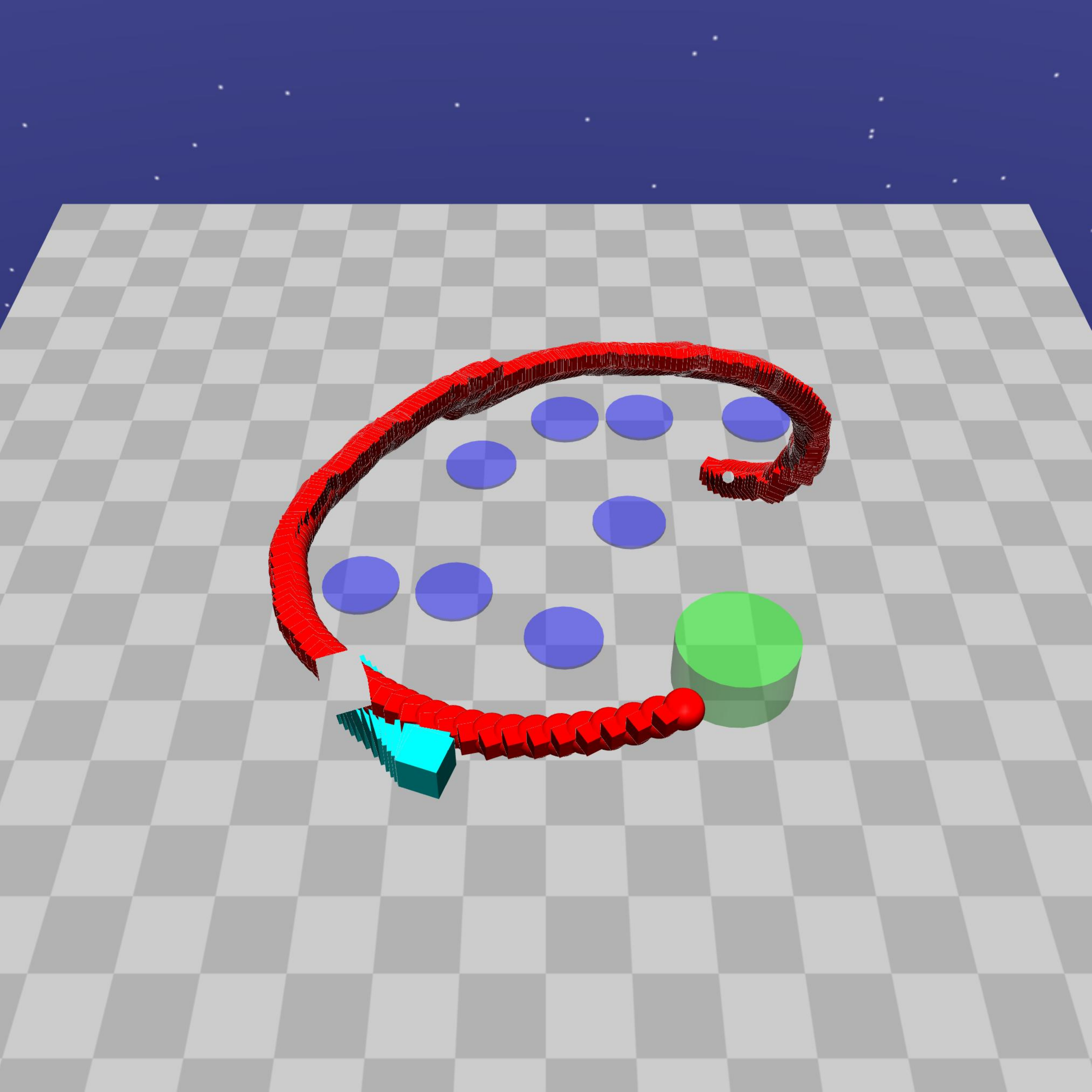}
        \includegraphics[trim=20 0 20 300, clip, width=0.24\linewidth]{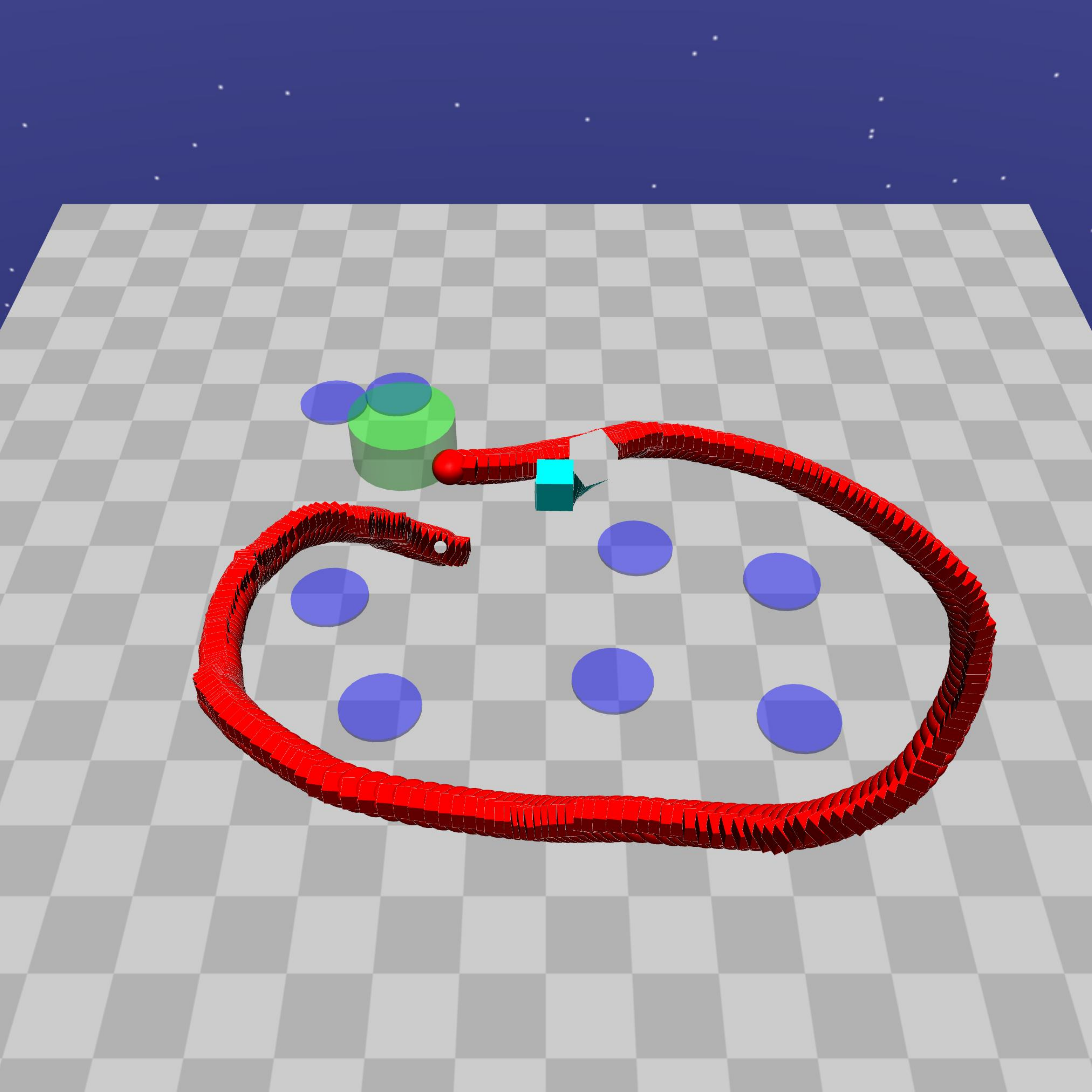}
        \caption{TRPO-Minmax successes (top) and failures (bottom)}
    \end{subfigure}%
    \caption{Sample trajectories of policies learned by each baseline and our Minmax approach in the Safety Gym \pointgoalhard domain, in the experiments of Figure \ref{fig:safety_baselines_hard_cost0}.
    Trajectories that hit hazards or take more than 1000 timesteps to reach the goal location are considered failures.   
    }
    \label{fig:all_trajectories_hard_cost0}
\end{figure*}%

\begin{figure*}[b!]
    \centering
    \includegraphics[width=0.9\linewidth]{images/_legend.pdf}
    \\
    \begin{subfigure}[t]{0.47\textwidth}
        \centering
        \includegraphics[width=\linewidth]{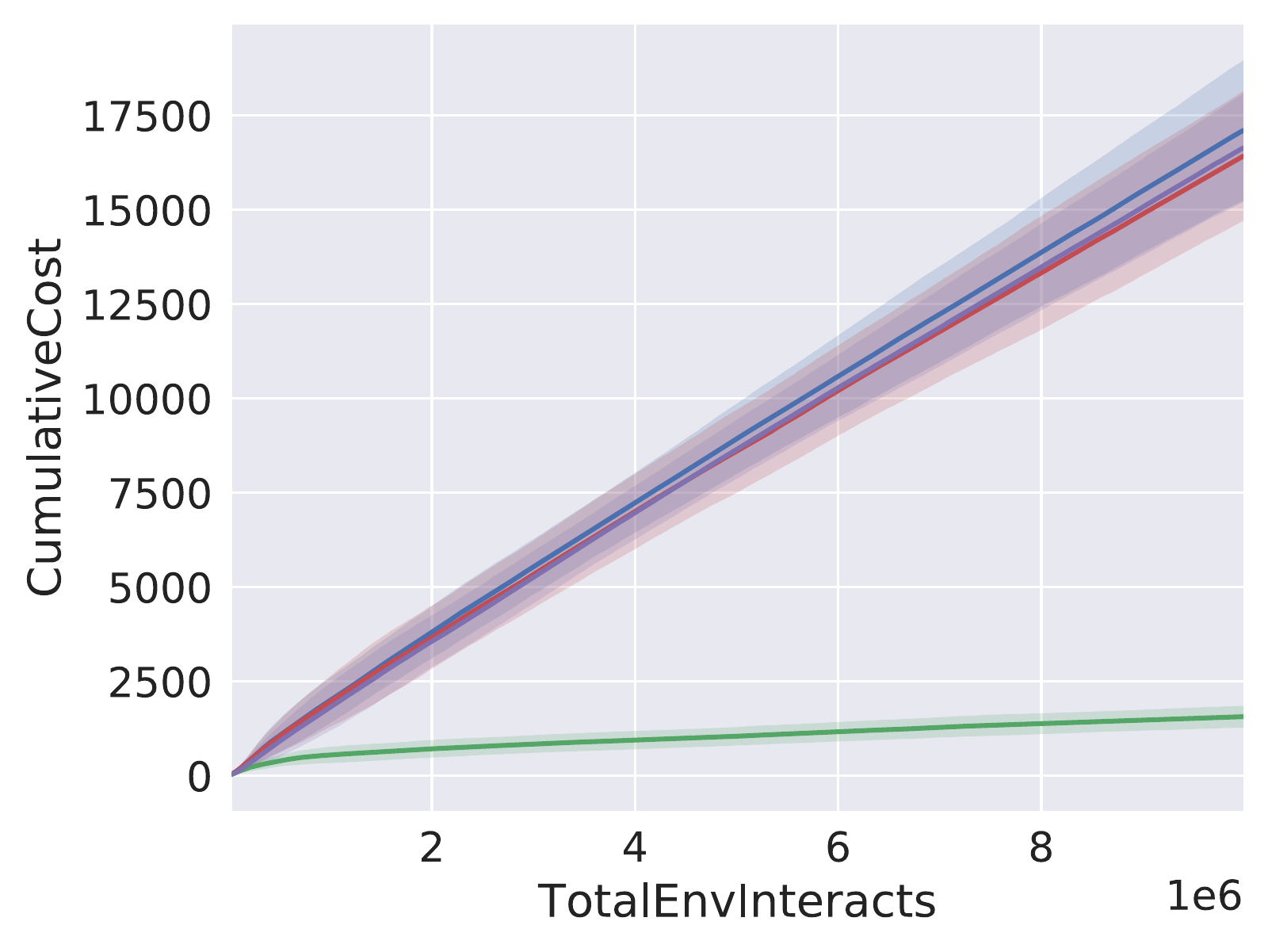}
        \caption{The cumulative cost.}
        \label{fig:cumulative_cost}
    \end{subfigure}%
    \quad
    \begin{subfigure}[t]{0.47\textwidth}
        \centering
        \includegraphics[width=\linewidth]{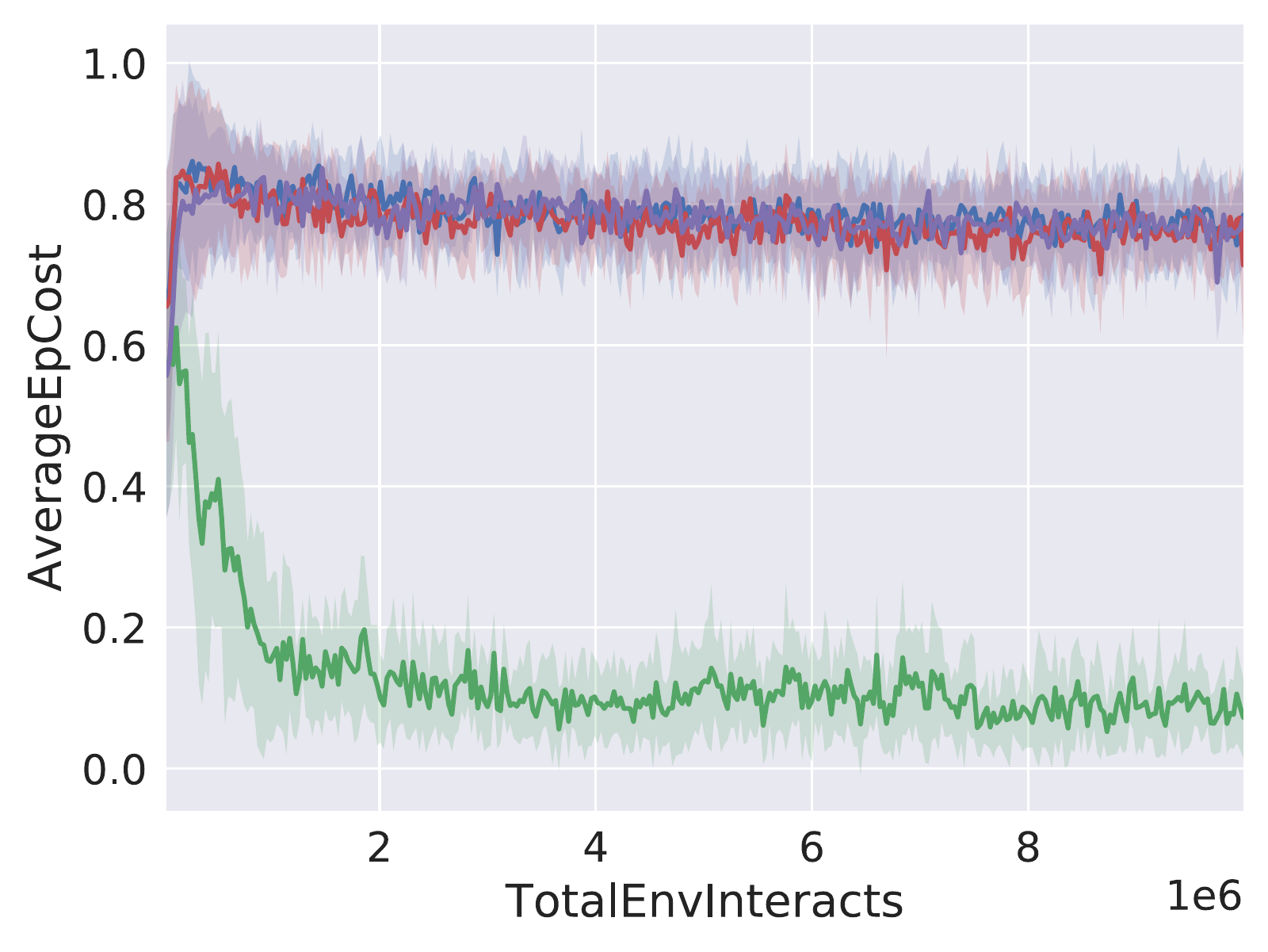}
        \caption{Failure rate}
        \label{fig:baseline_cost}
    \end{subfigure}%
    \\
    \begin{subfigure}[t]{0.47\textwidth}
        \centering
        \includegraphics[width=\linewidth]{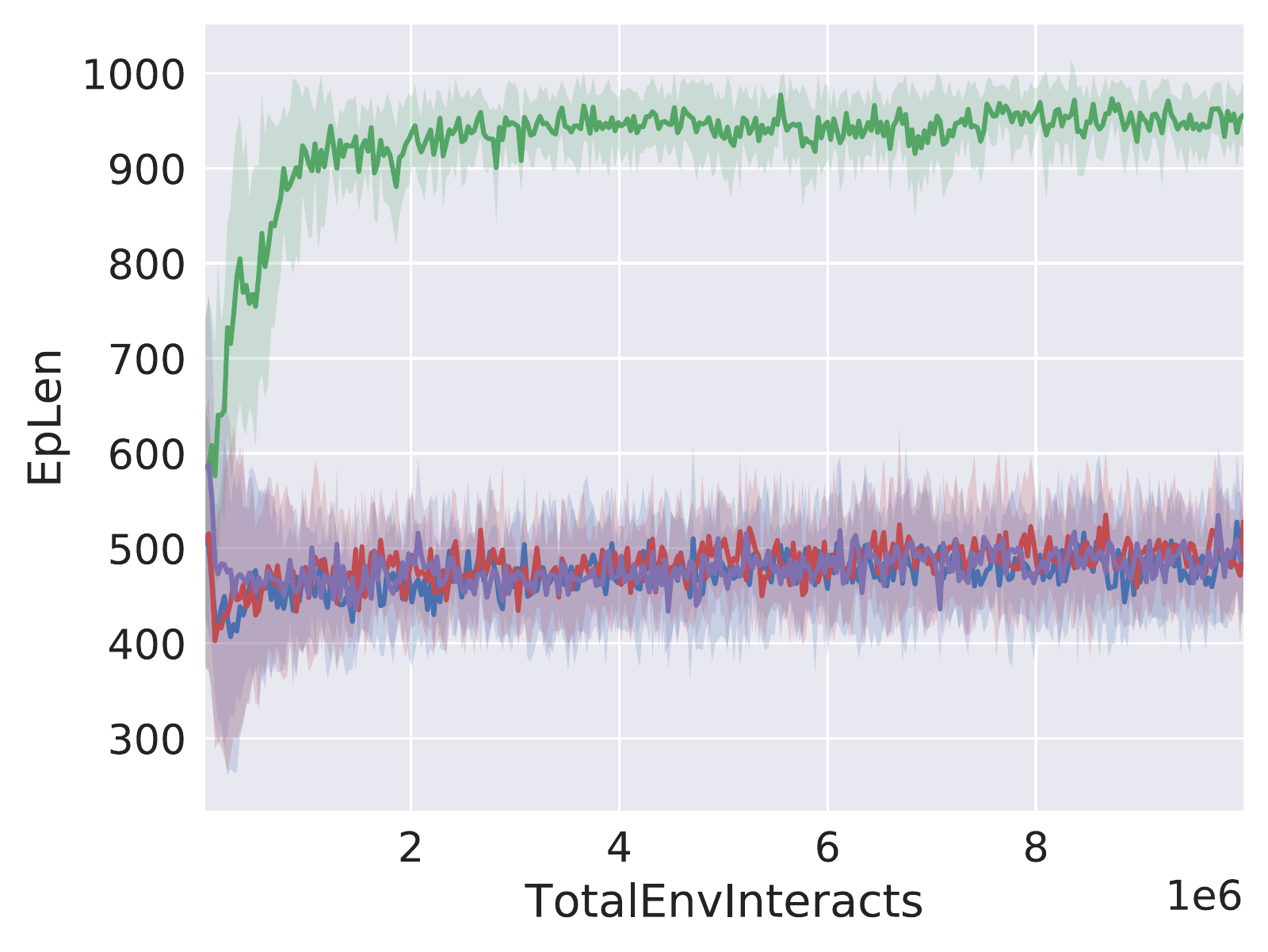}
        \caption{Average episode length}
        \label{fig:baseline_length}
    \end{subfigure}%
    \quad
    \begin{subfigure}[t]{0.47\textwidth}
        \centering
        \includegraphics[width=\linewidth]{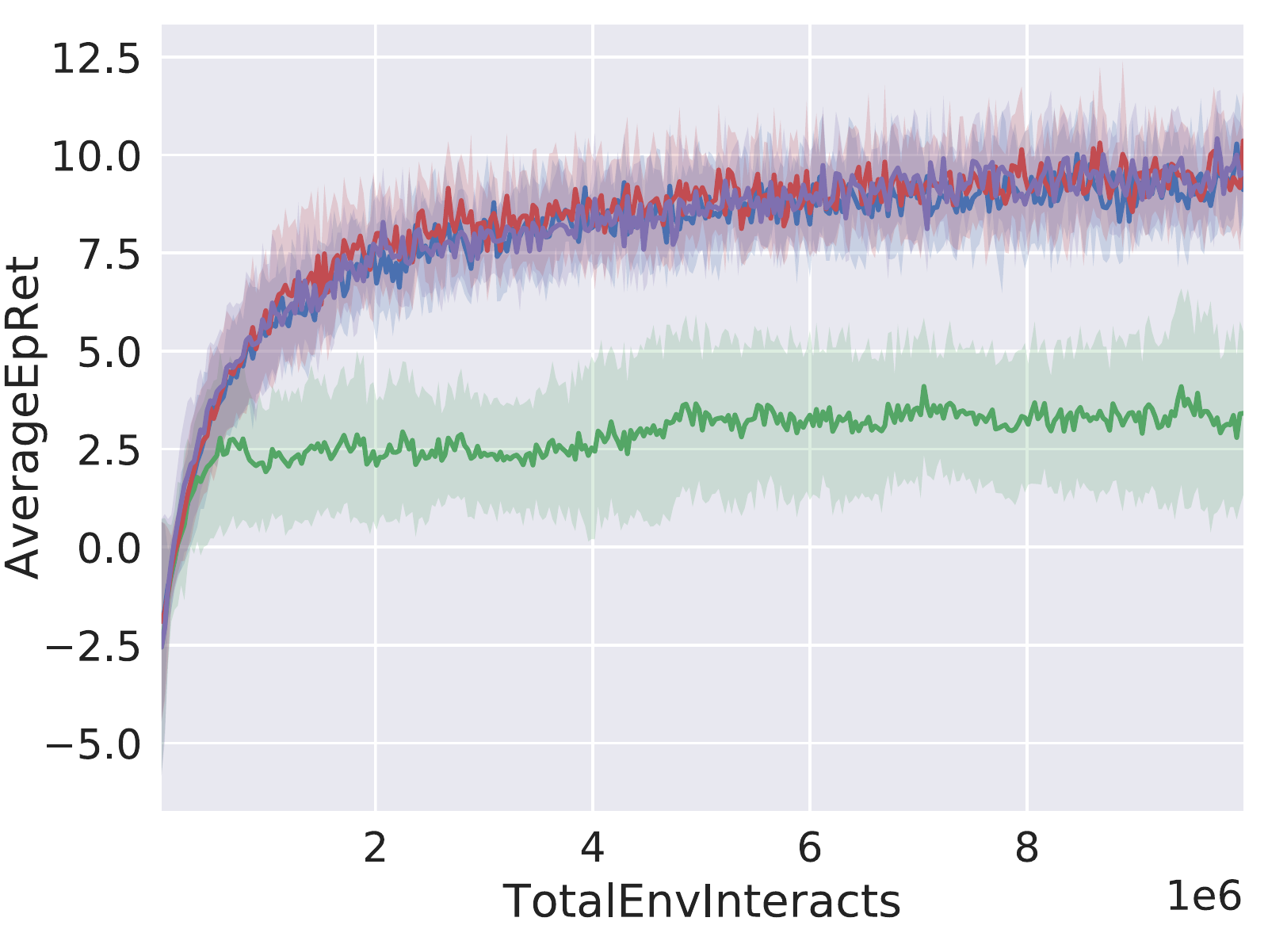}
        \caption{Average returns}
        \label{fig:baseline_return}
    \end{subfigure}%
    \caption{Comparison with baselines in the \pointgoalhard environment. The cost threshold for TRPO lagrangian and CPO is set to $25$ as in \citet{Ray2019}.}
    \label{fig:safety_baselines}
\end{figure*}%

\begin{figure*}[b!]
    \centering
    \begin{subfigure}[t]{\textwidth}
        \centering
        \includegraphics[trim=0 300 0 500, clip, width=0.24\linewidth]{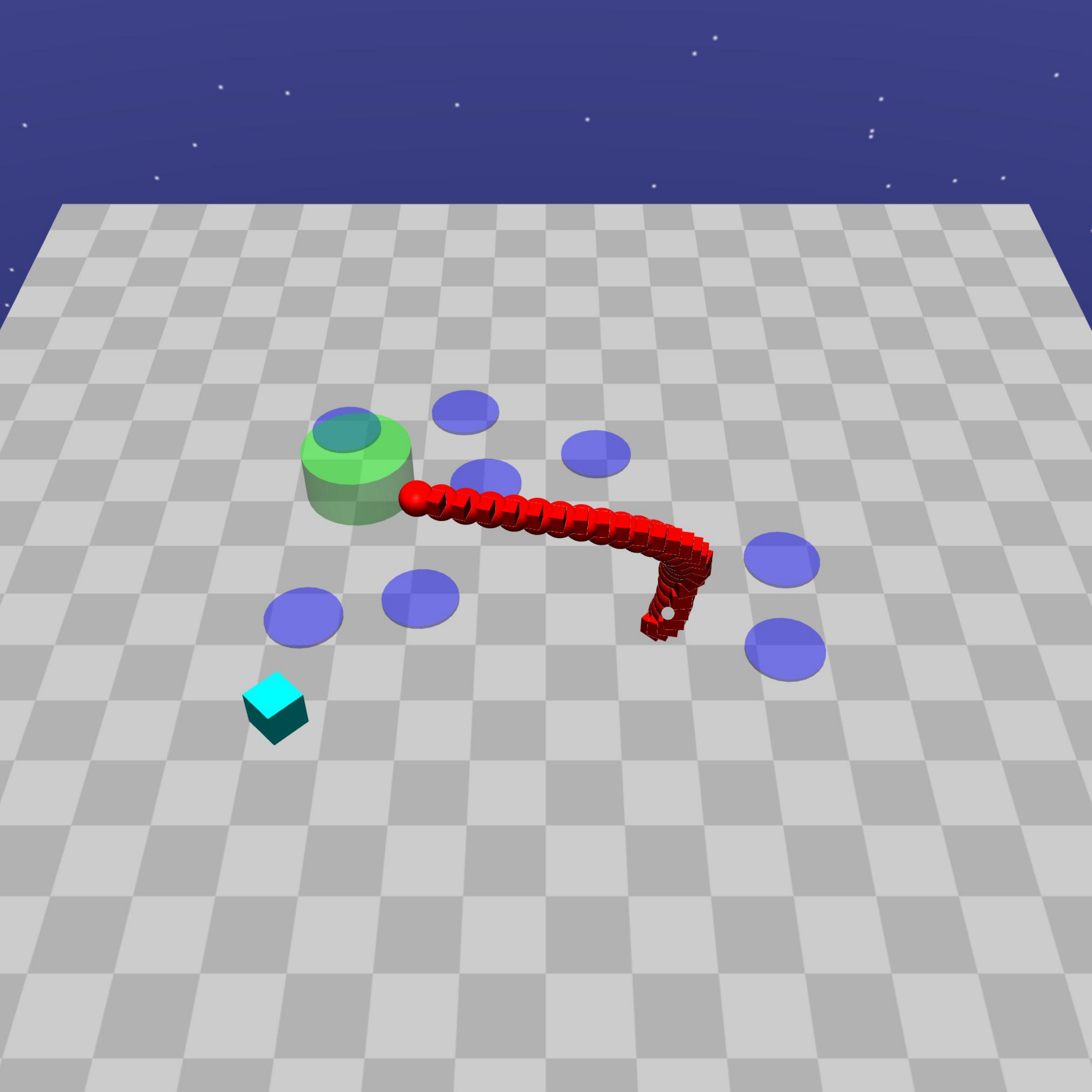}
        \includegraphics[trim=0 300 0 500, clip, width=0.24\linewidth]{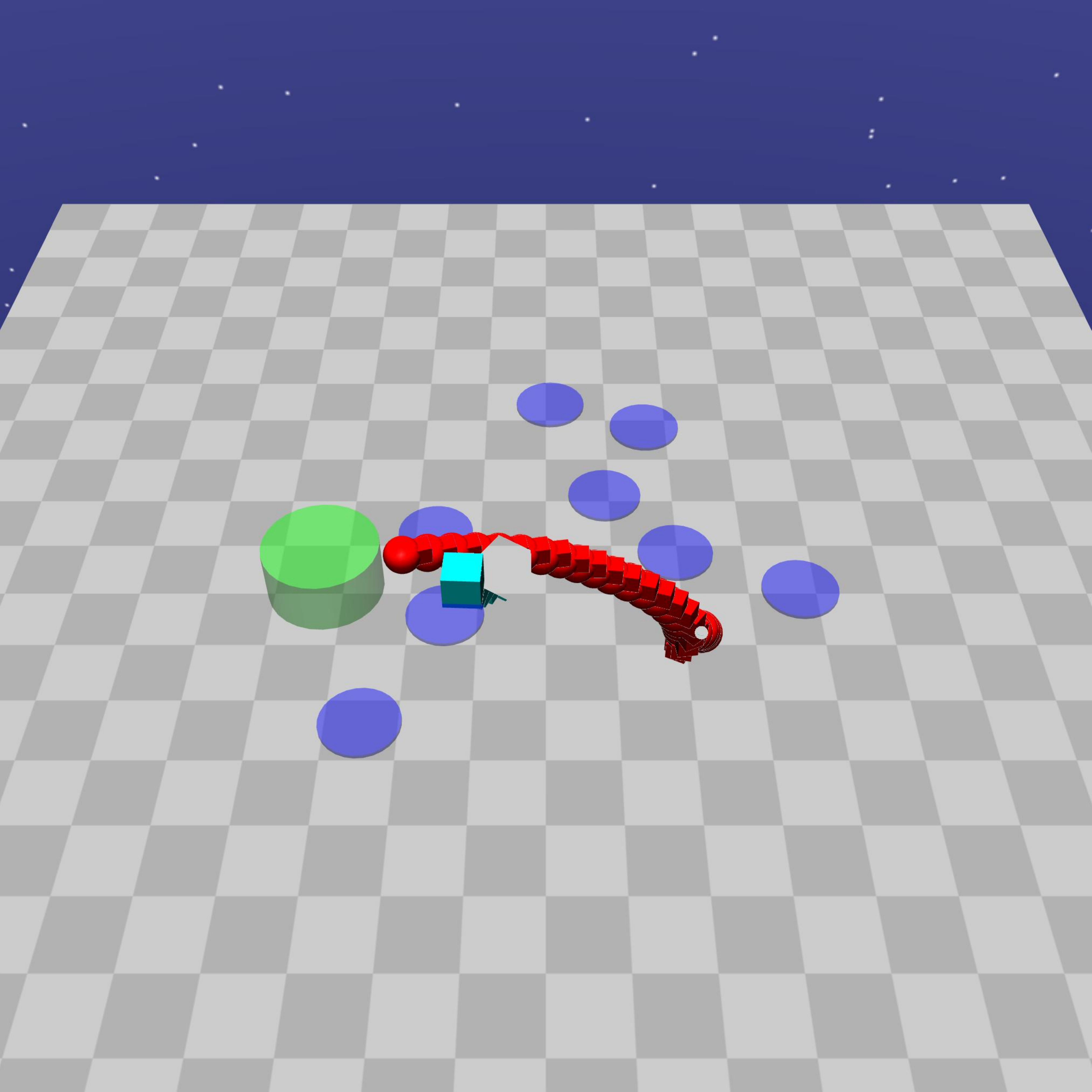}
        \includegraphics[trim=0 300 0 500, clip, width=0.24\linewidth]{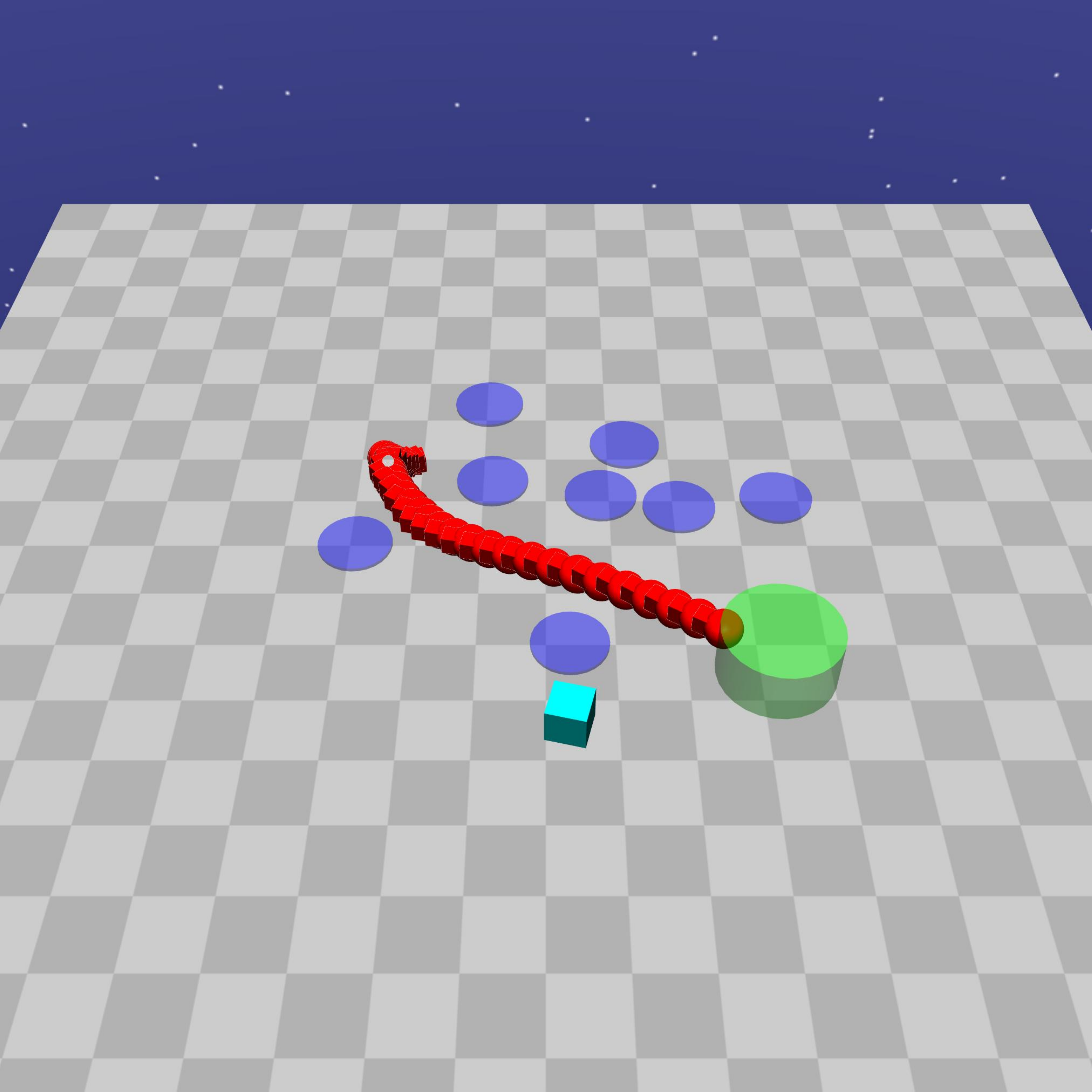}
        \includegraphics[trim=0 300 0 500, clip, width=0.24\linewidth]{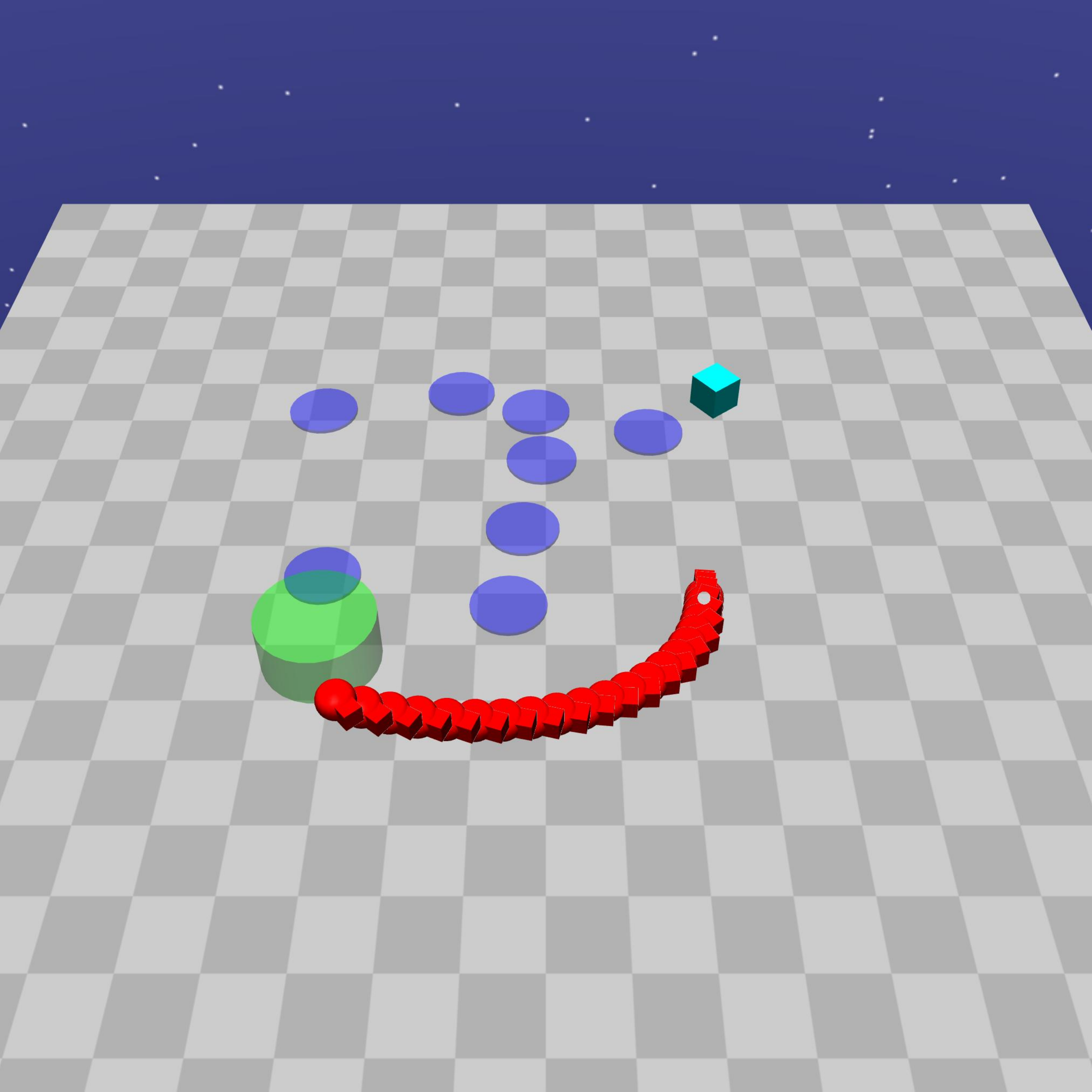}
        \\
        \includegraphics[trim=0 300 0 500, clip, width=0.24\linewidth]{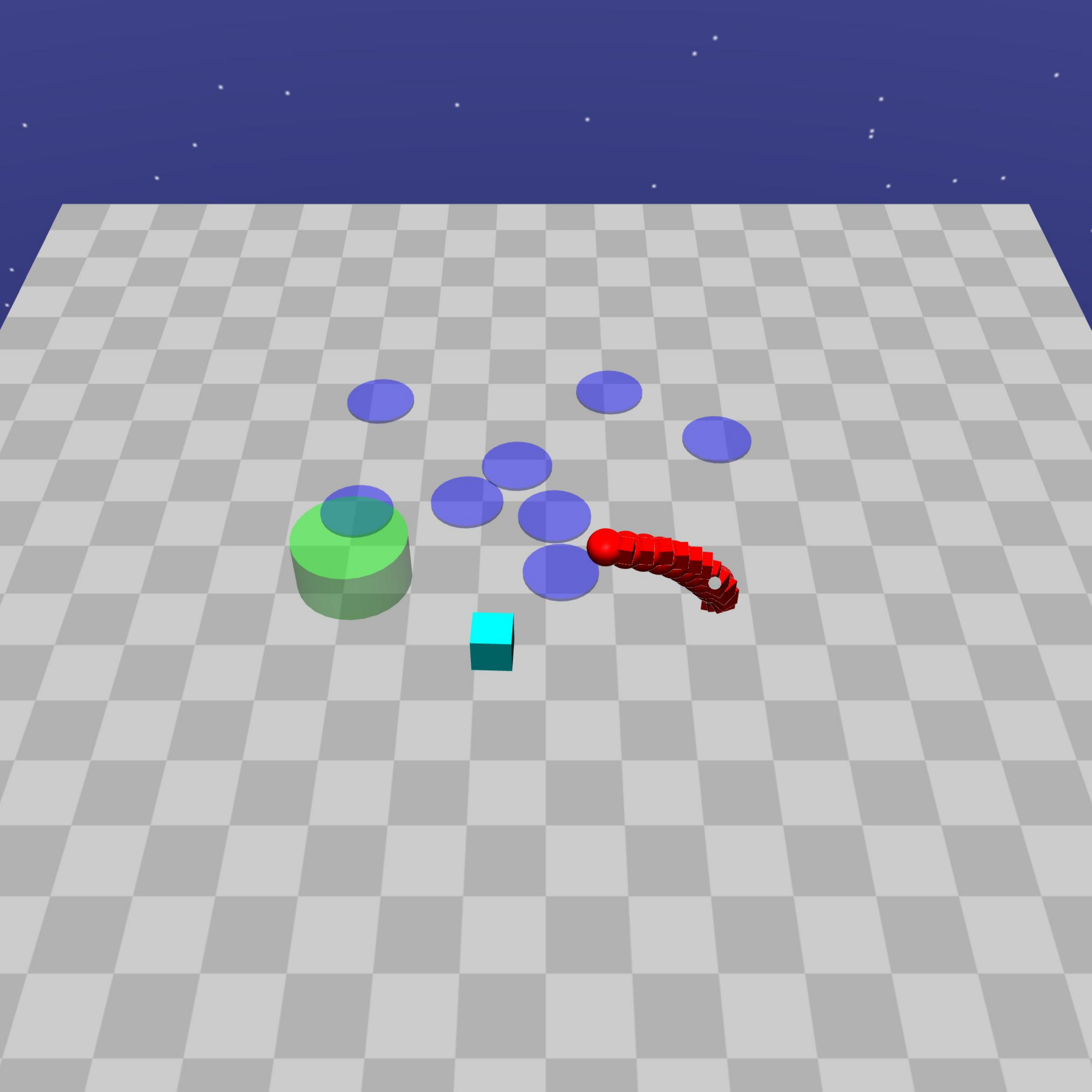}
        \includegraphics[trim=0 300 0 500, clip, width=0.24\linewidth]{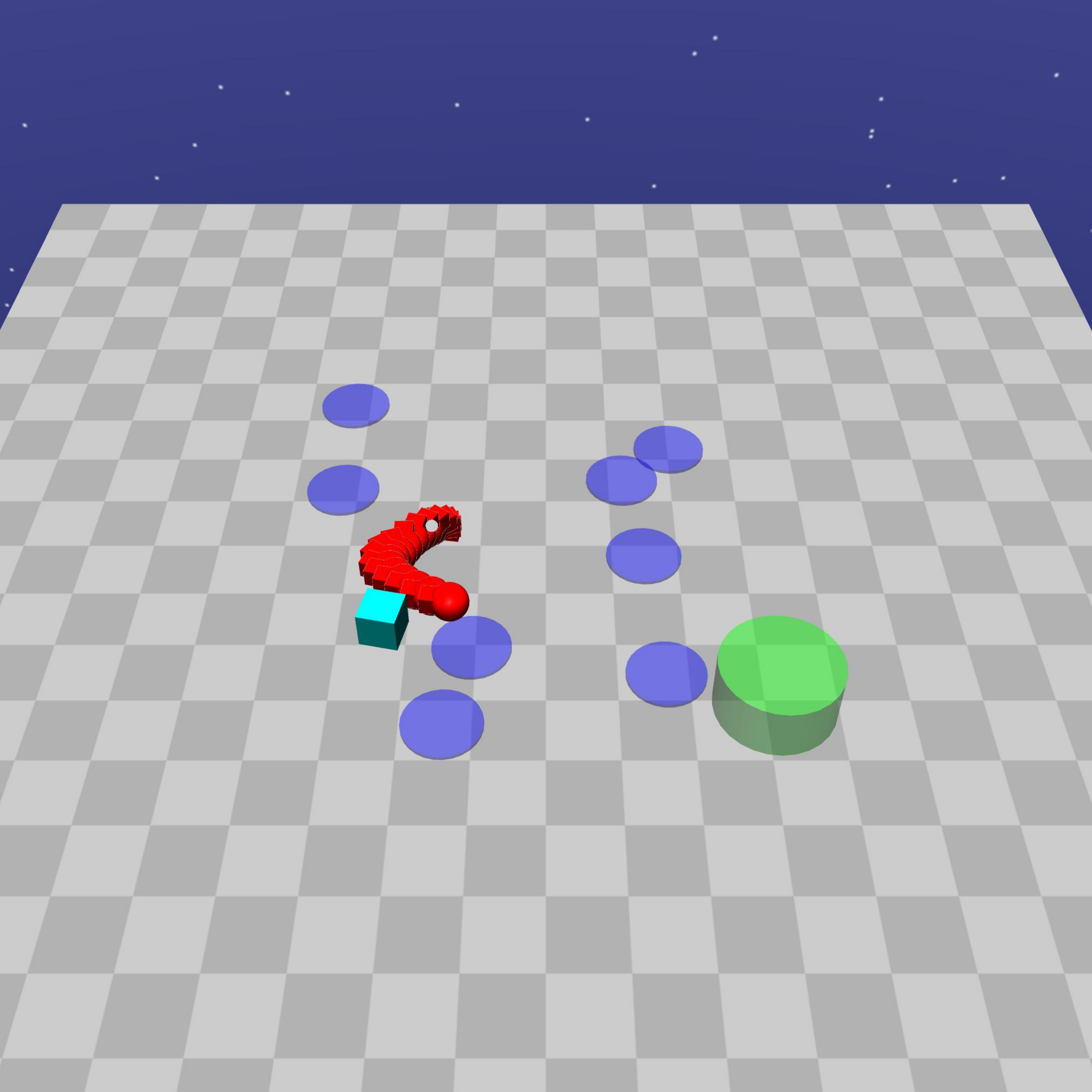}
        \includegraphics[trim=0 300 0 500, clip, width=0.24\linewidth]{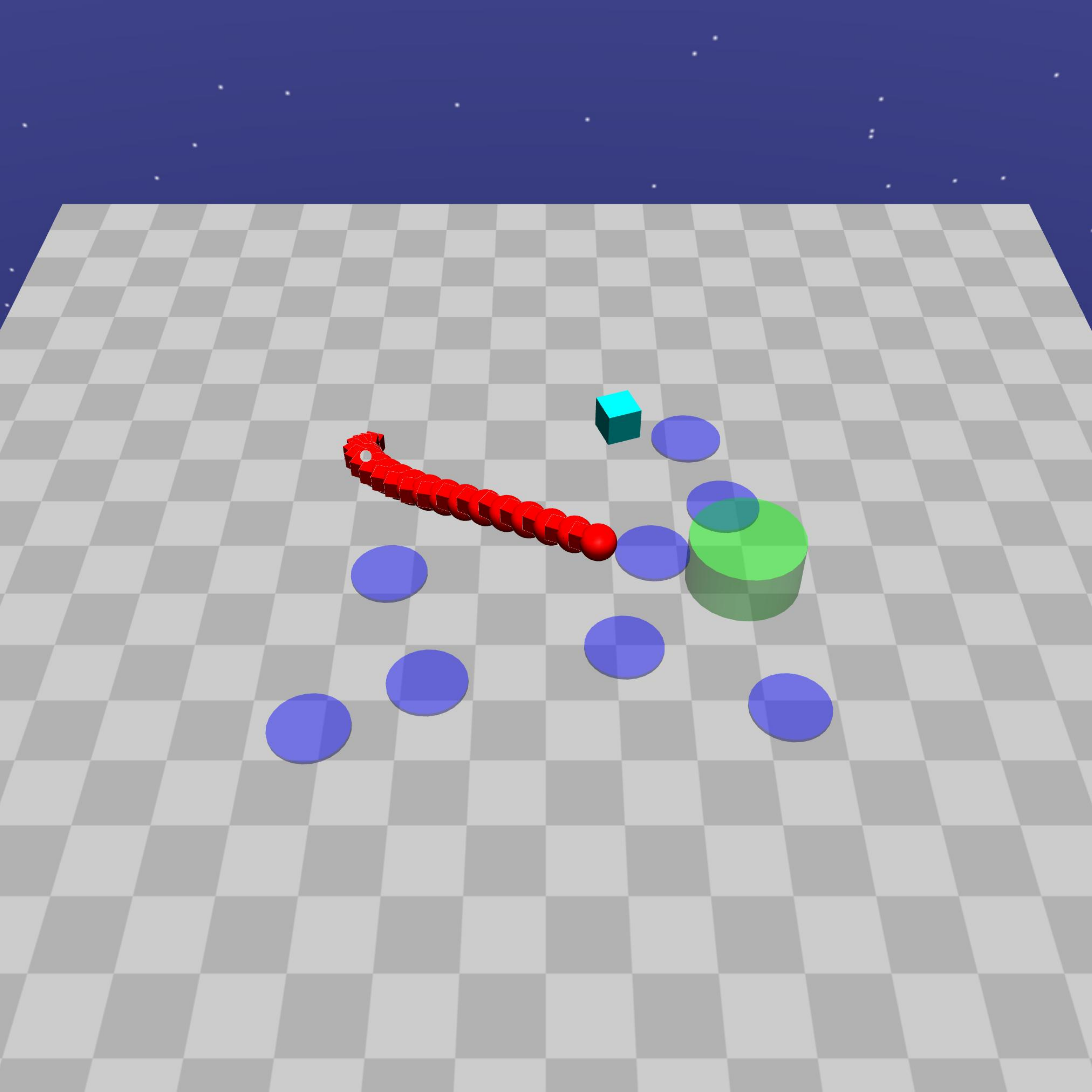}
        \includegraphics[trim=0 300 0 500, clip, width=0.24\linewidth]{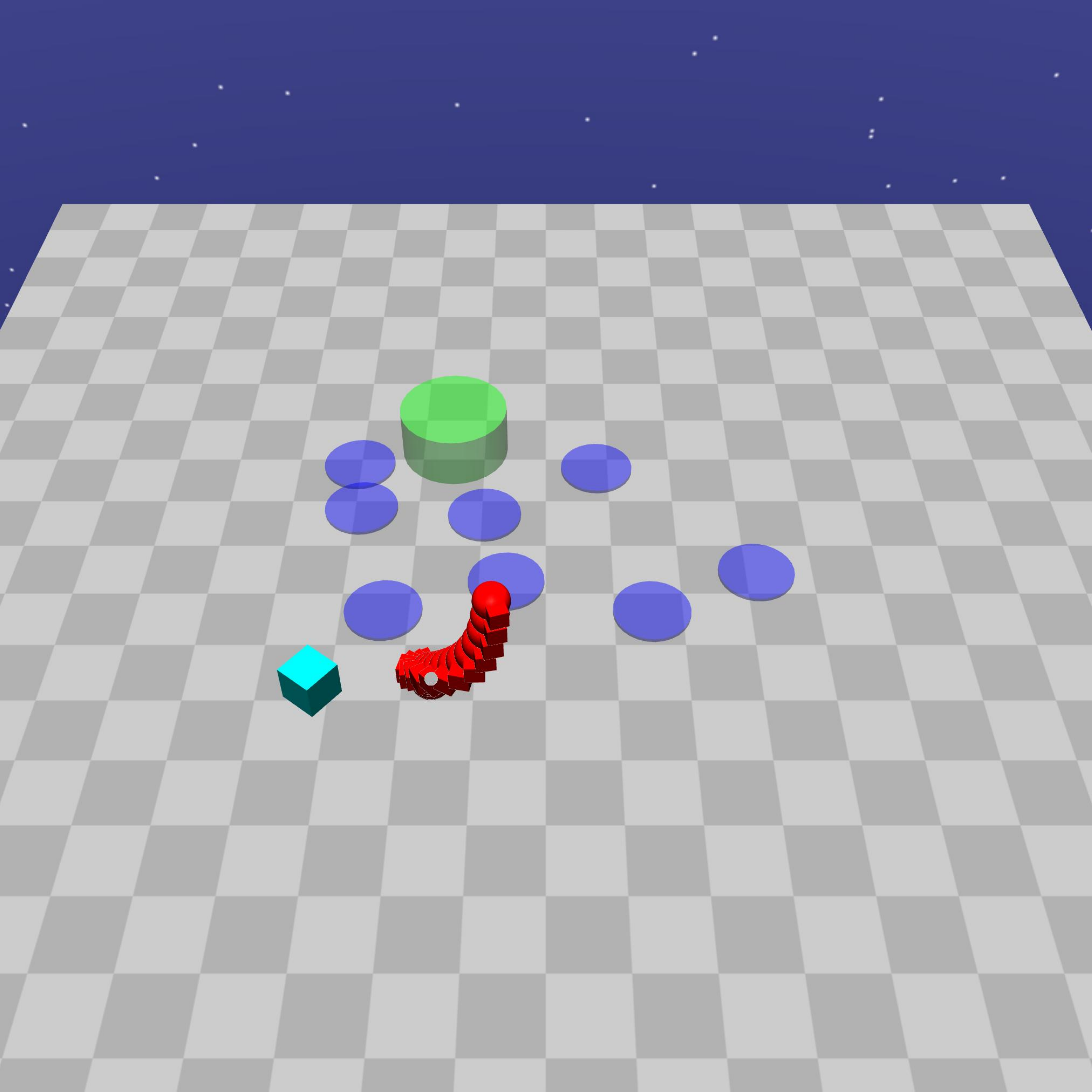}
        \caption{TRPO successes (top) and failures (bottom)}
    \end{subfigure}%
    \\
    \begin{subfigure}[t]{\textwidth}
        \centering
        \includegraphics[trim=0 300 0 500, clip, width=0.24\linewidth]{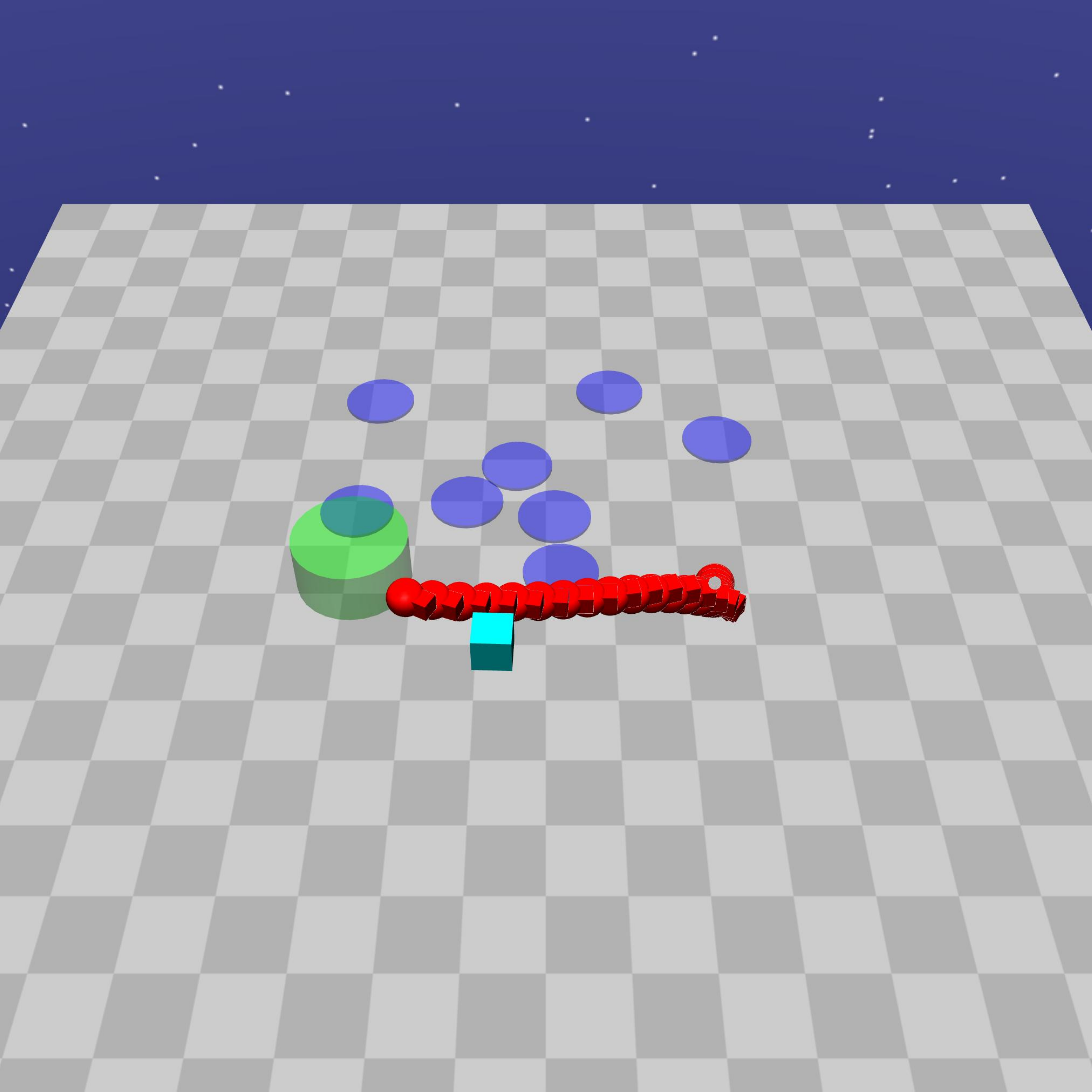}
        \includegraphics[trim=0 300 0 500, clip, width=0.24\linewidth]{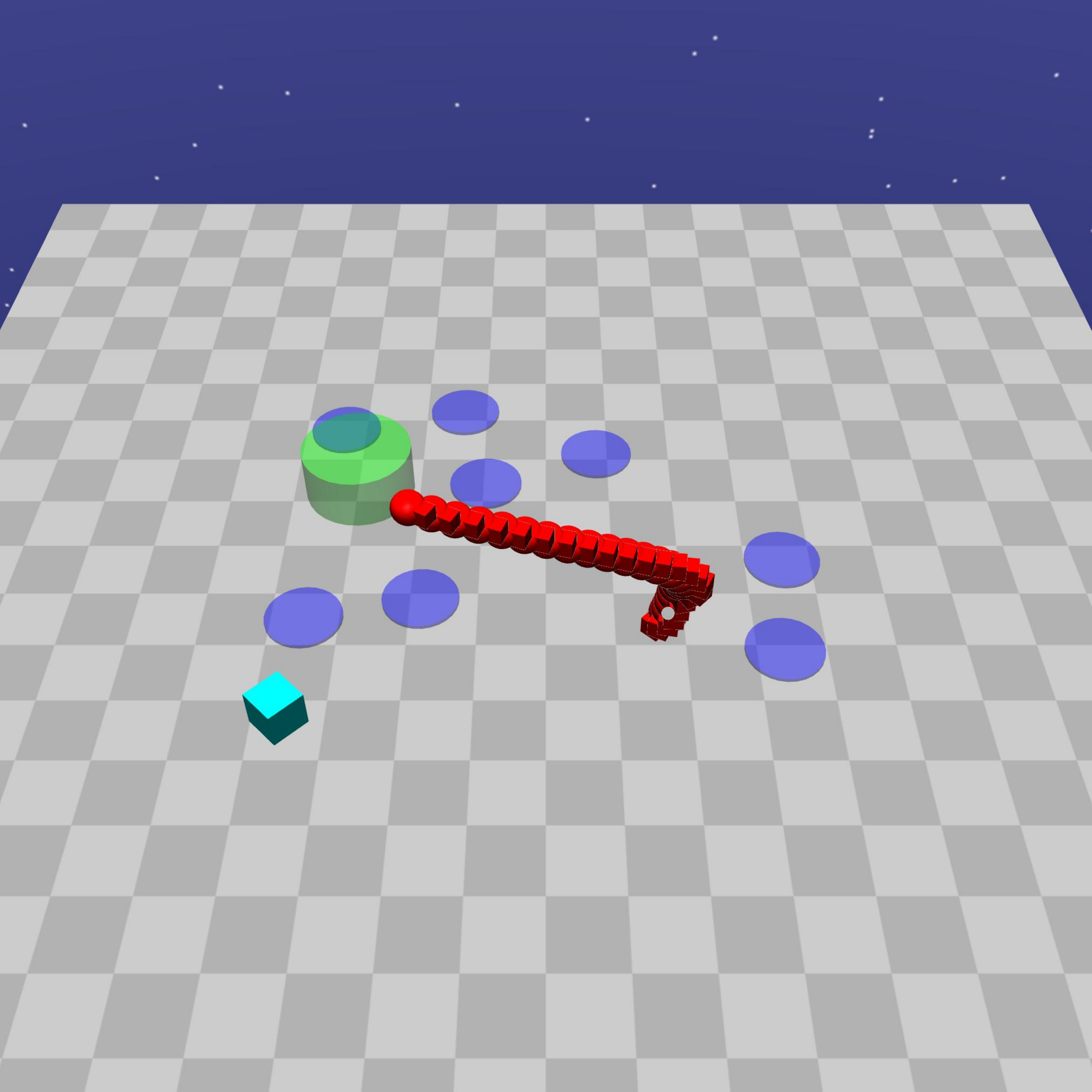}
        \includegraphics[trim=0 300 0 500, clip, width=0.24\linewidth]{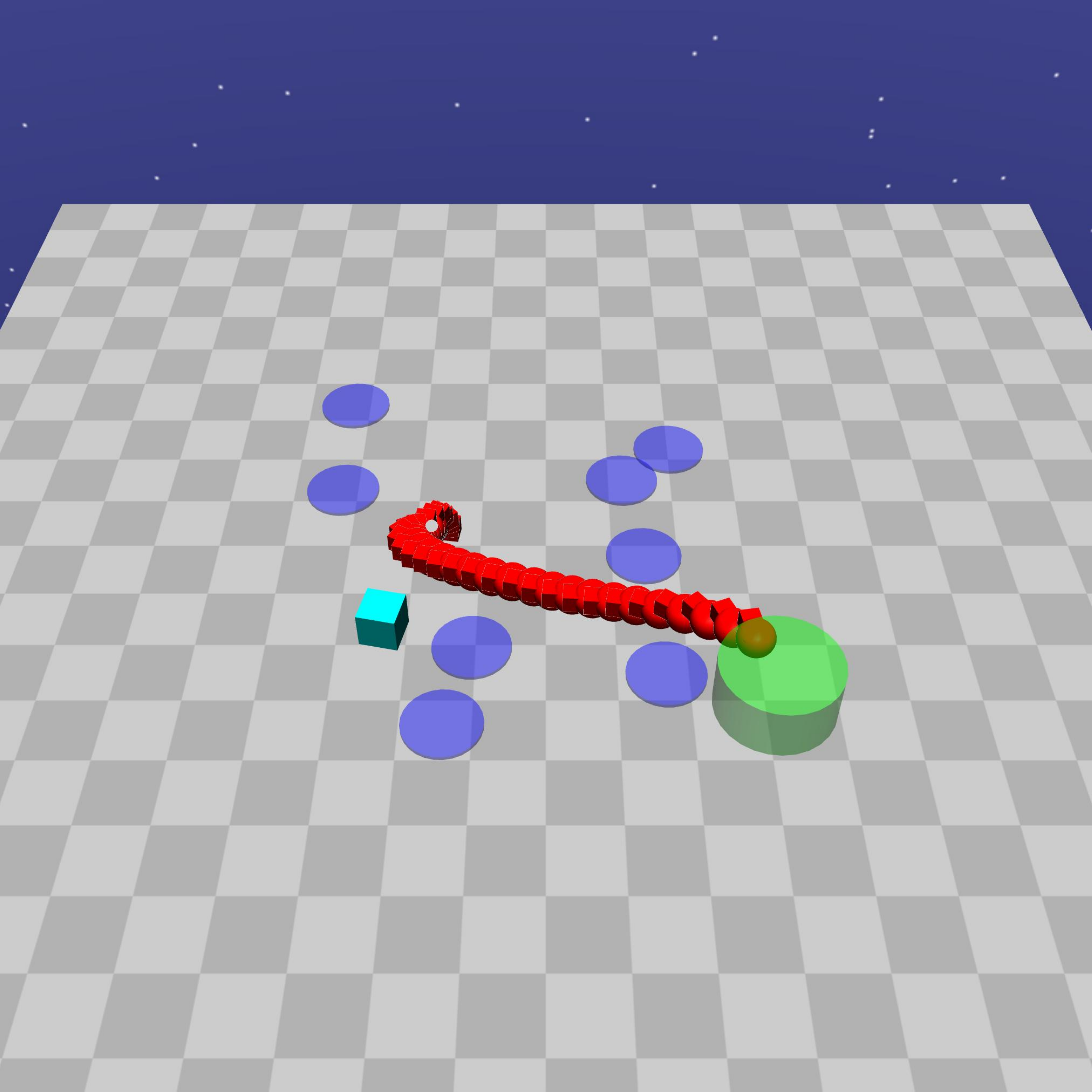}
        \includegraphics[trim=0 300 0 500, clip, width=0.24\linewidth]{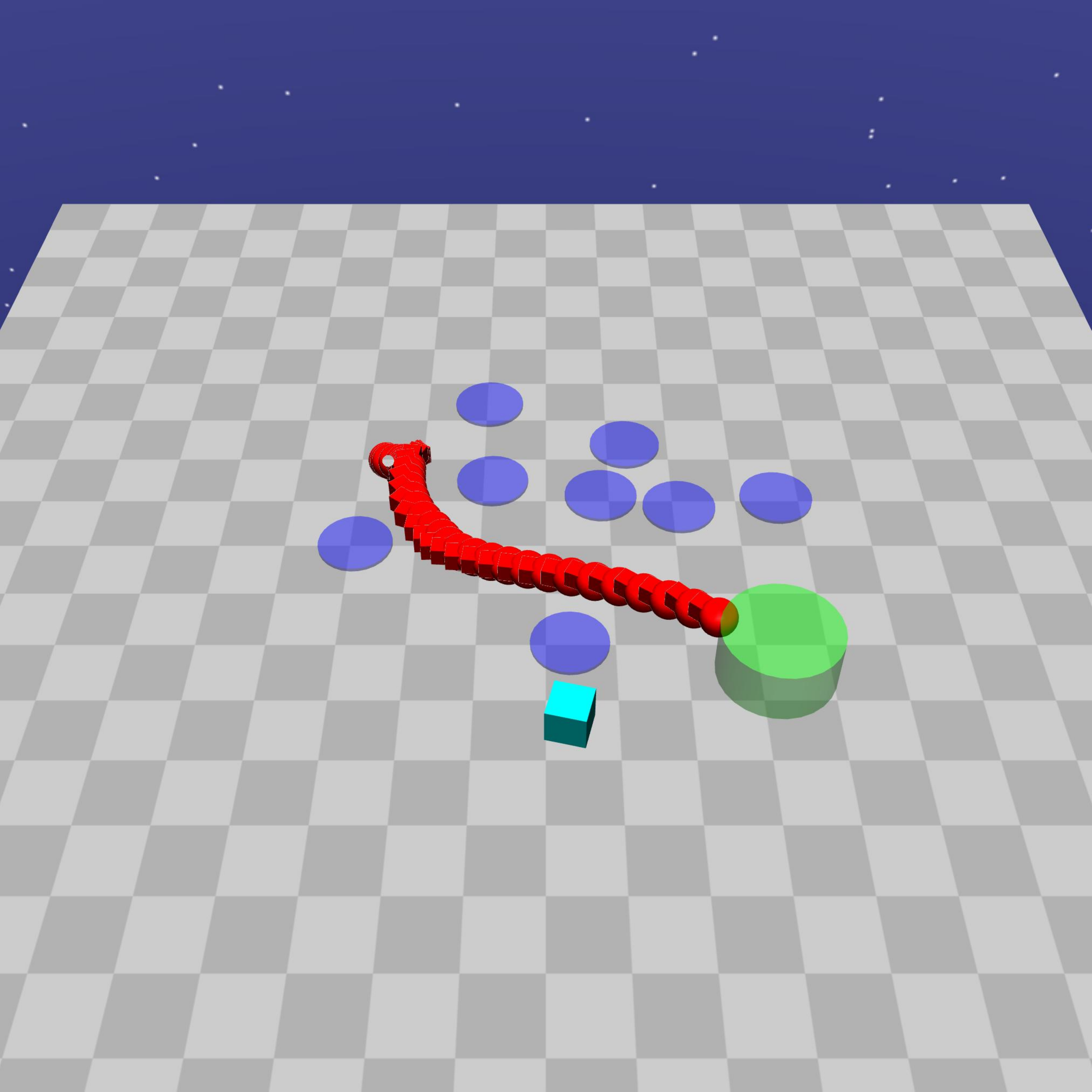}
        \\
        \includegraphics[trim=0 300 0 500, clip, width=0.24\linewidth]{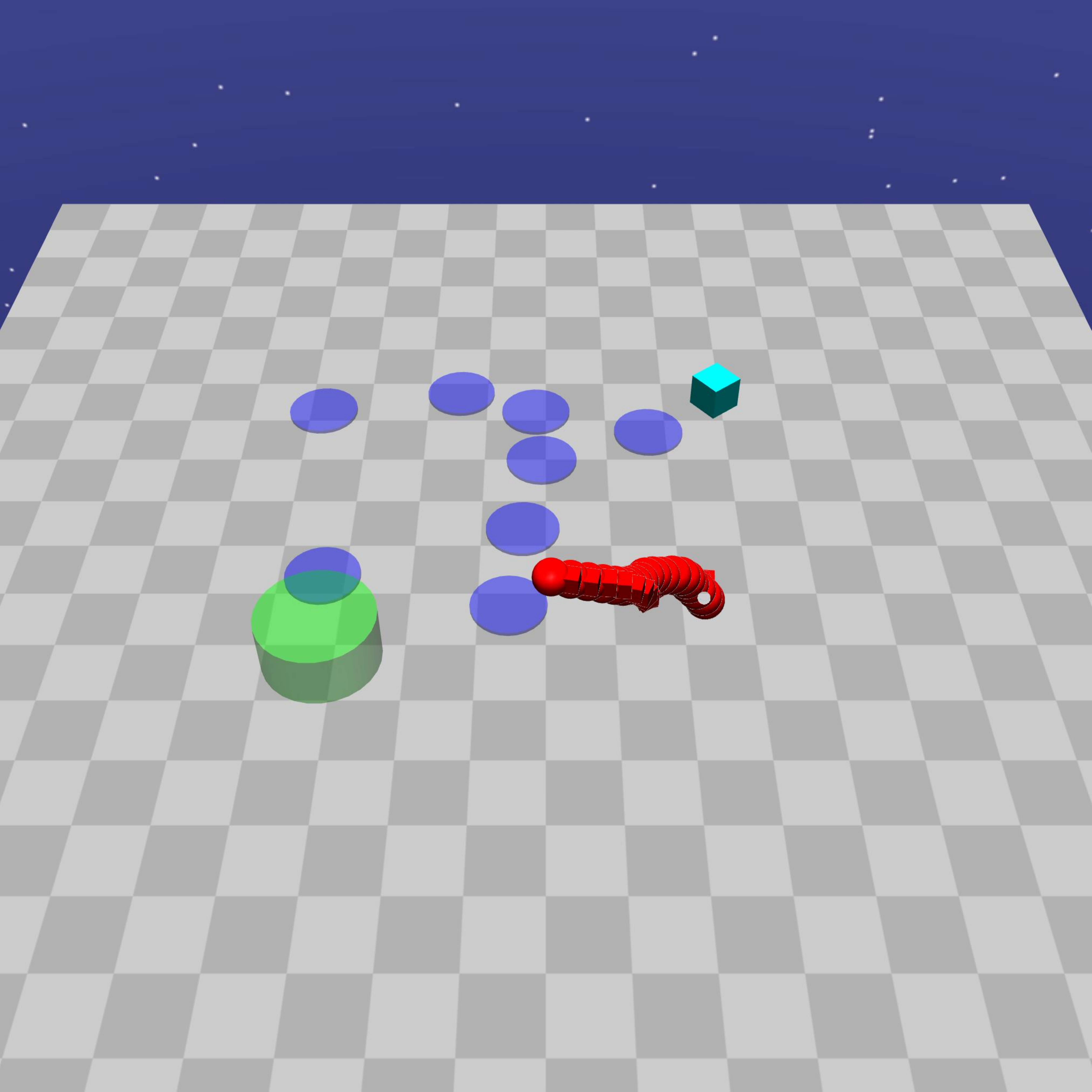}
        \includegraphics[trim=0 300 0 500, clip, width=0.24\linewidth]{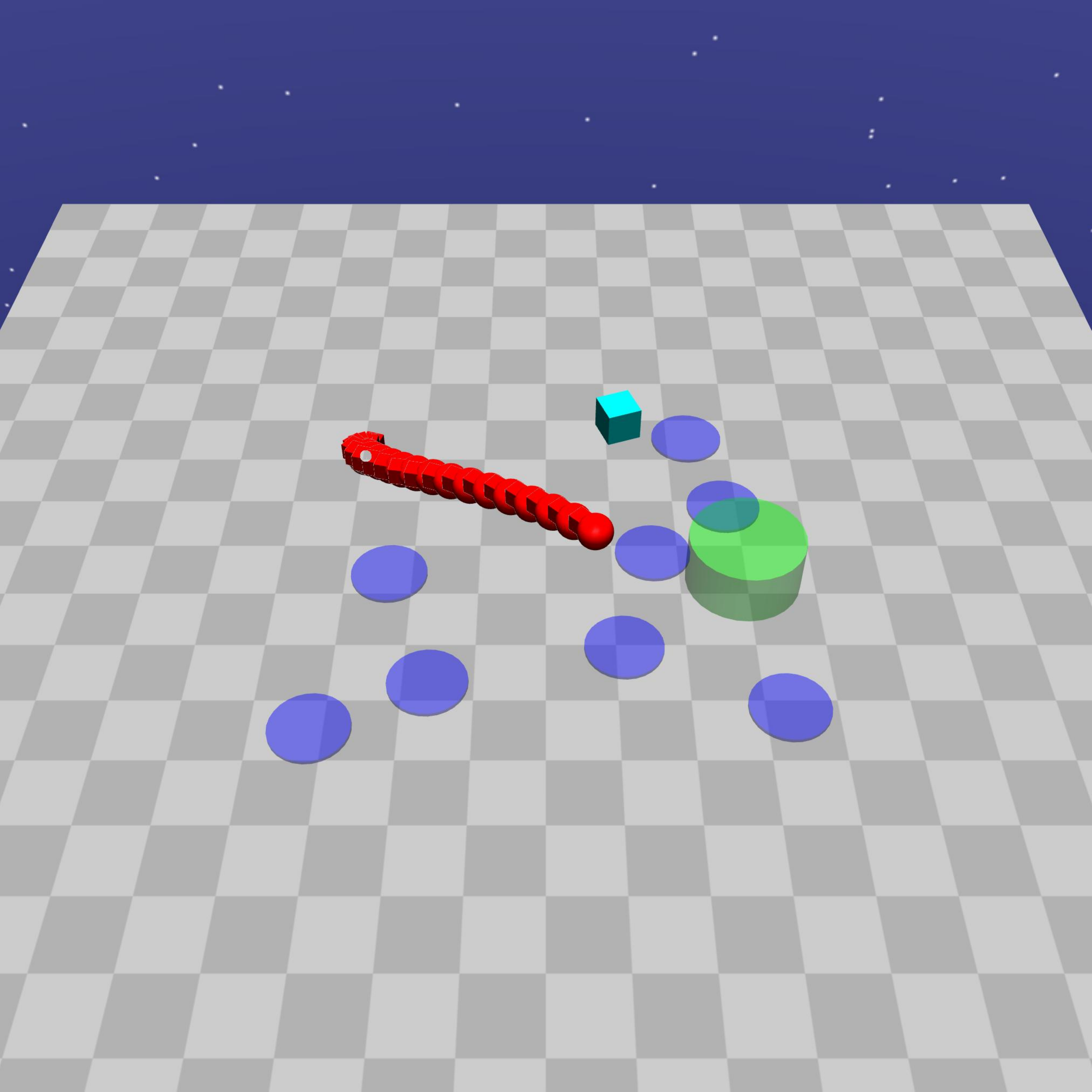}
        \includegraphics[trim=0 300 0 500, clip, width=0.24\linewidth]{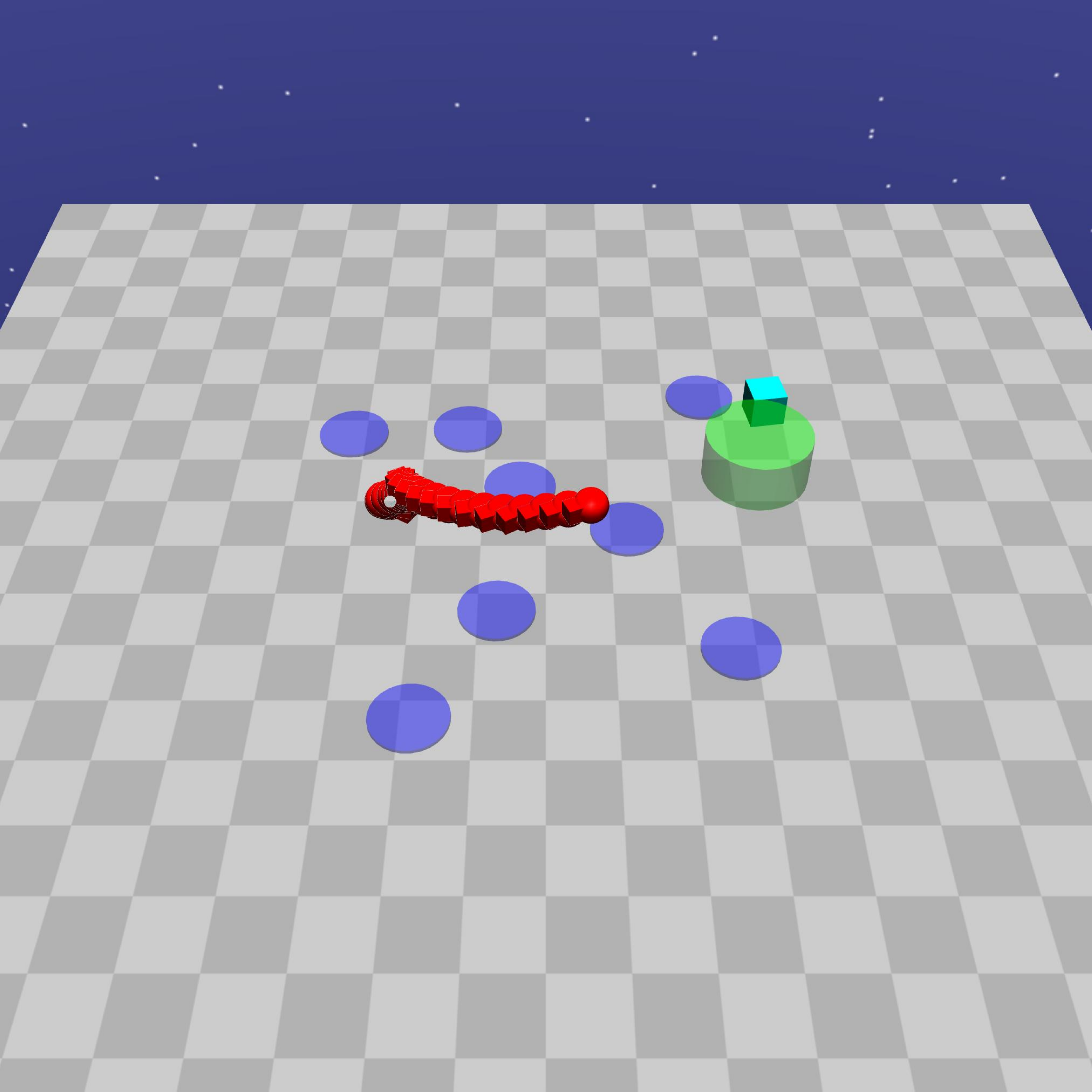}
        \includegraphics[trim=0 300 0 500, clip, width=0.24\linewidth]{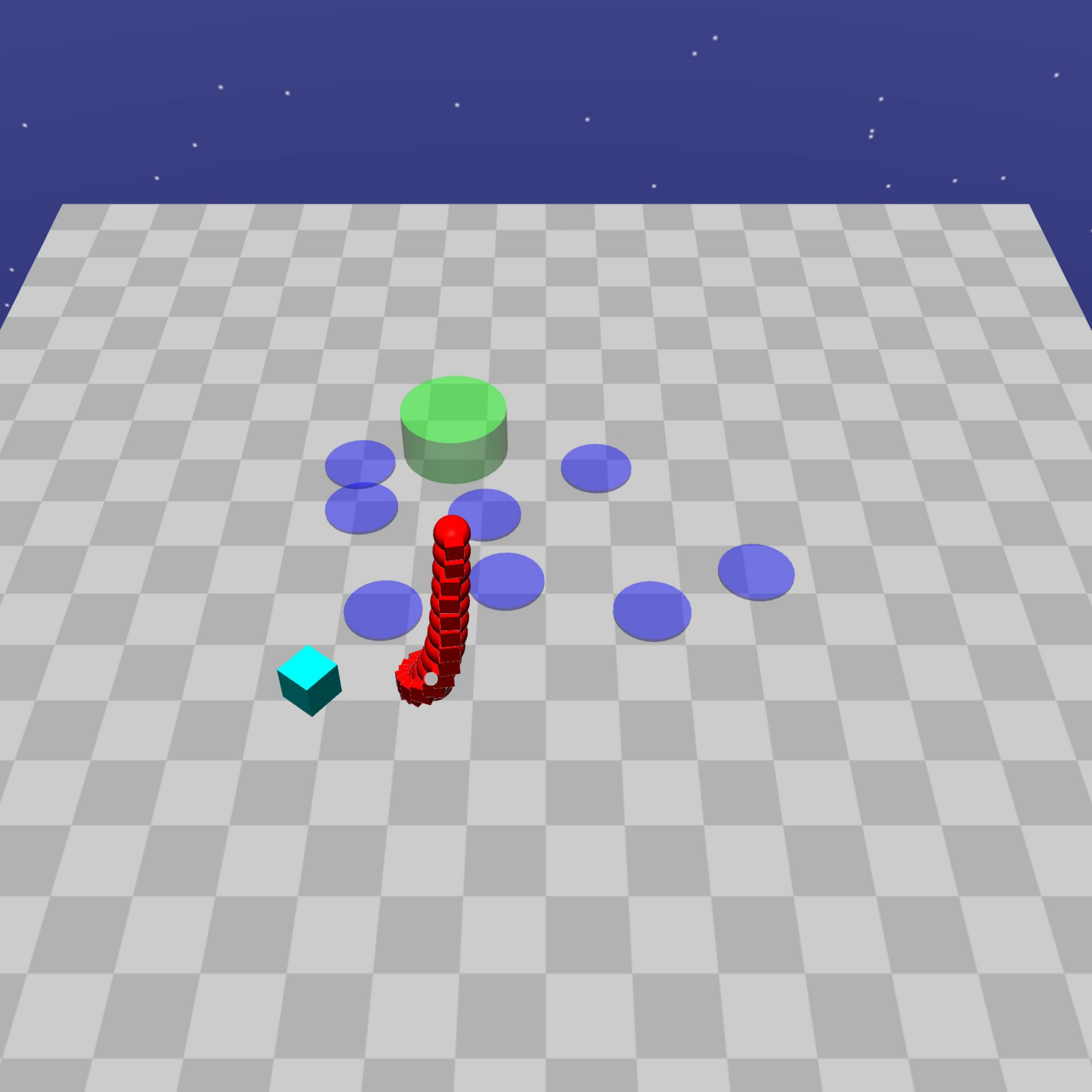}
        \caption{TRPO-Lagrangian successes (top) and failures (bottom)}
    \end{subfigure}%
    \\
    \begin{subfigure}[t]{\textwidth}
        \centering
        \includegraphics[trim=0 300 0 500, clip, width=0.24\linewidth]{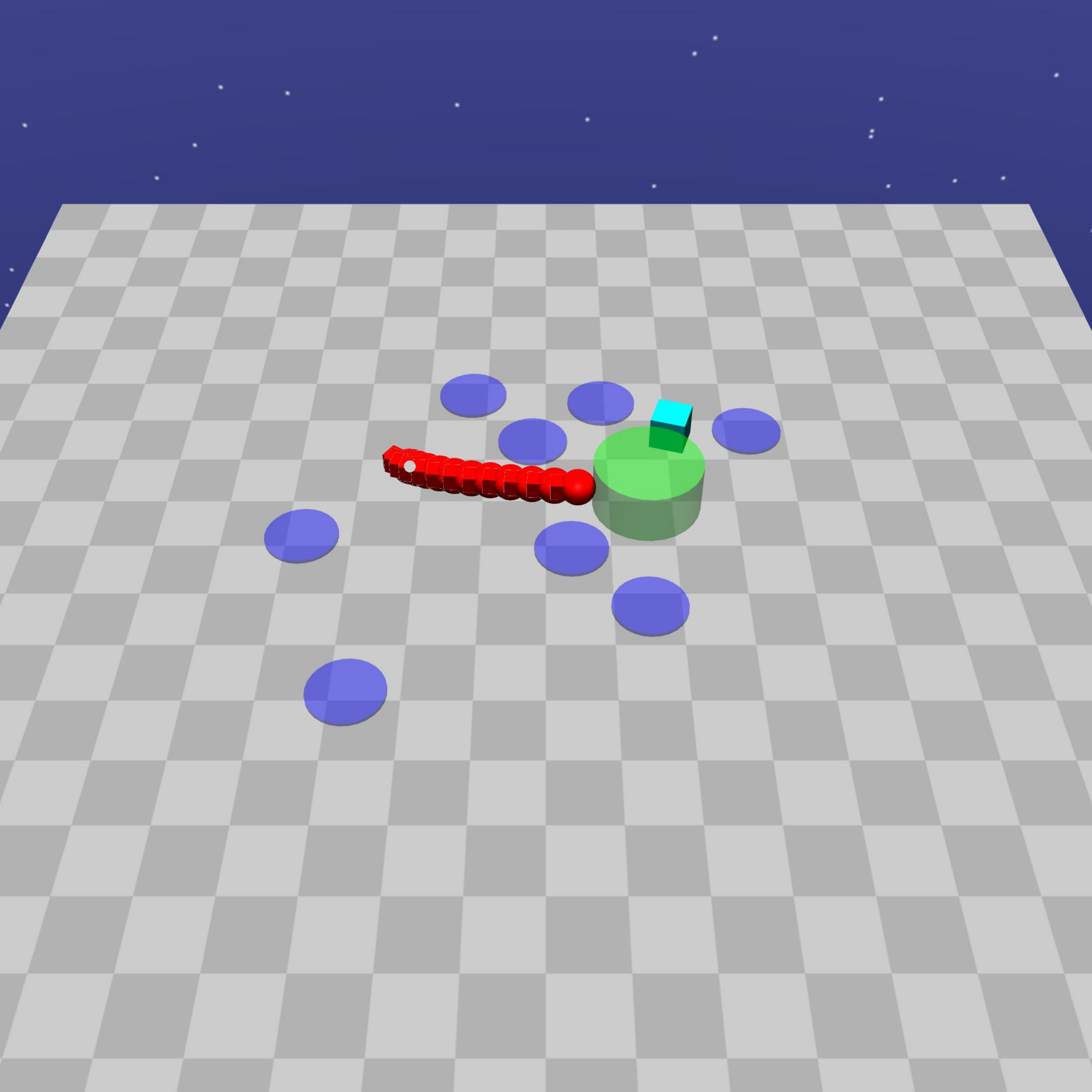}
        \includegraphics[trim=0 300 0 500, clip, width=0.24\linewidth]{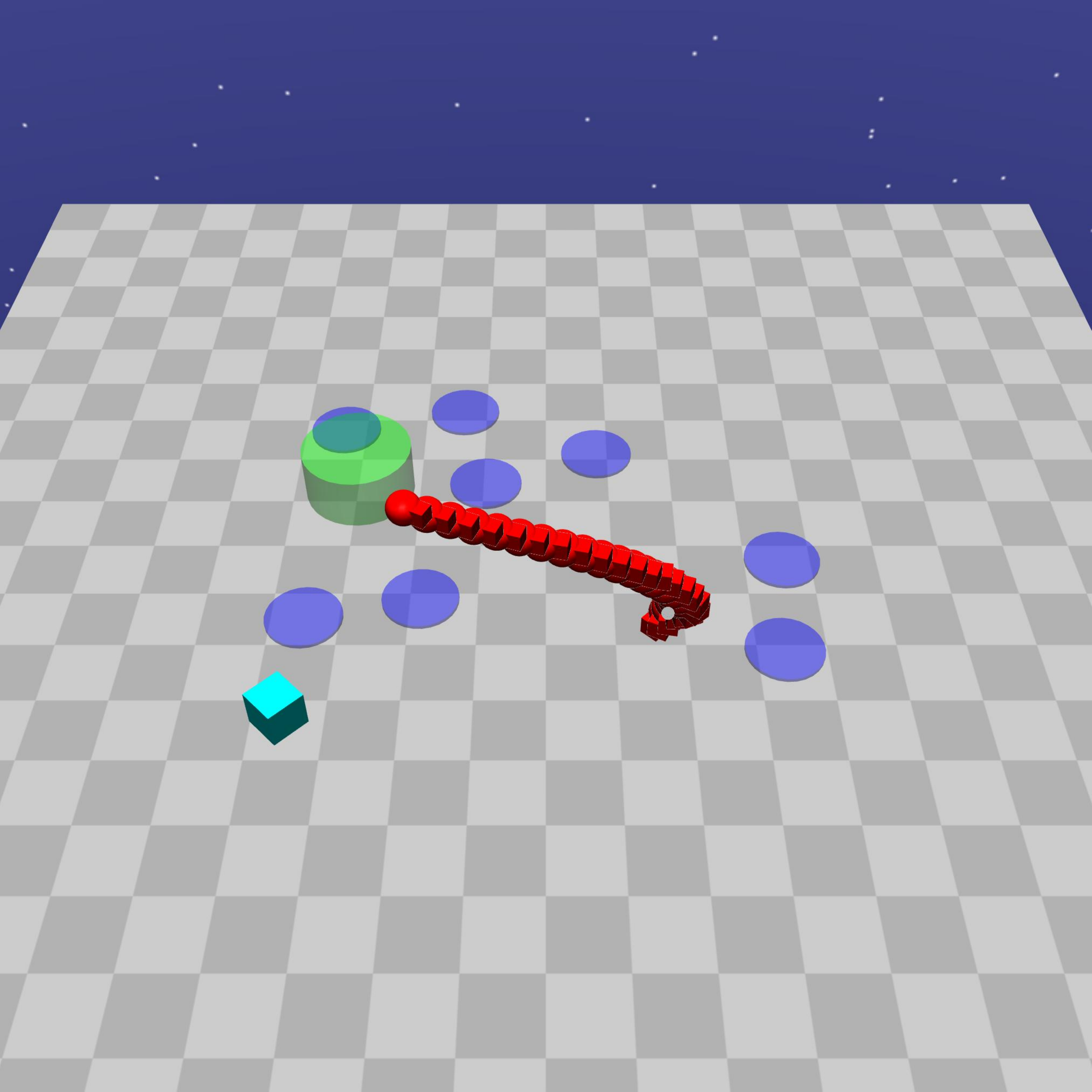}
        \includegraphics[trim=0 300 0 500, clip, width=0.24\linewidth]{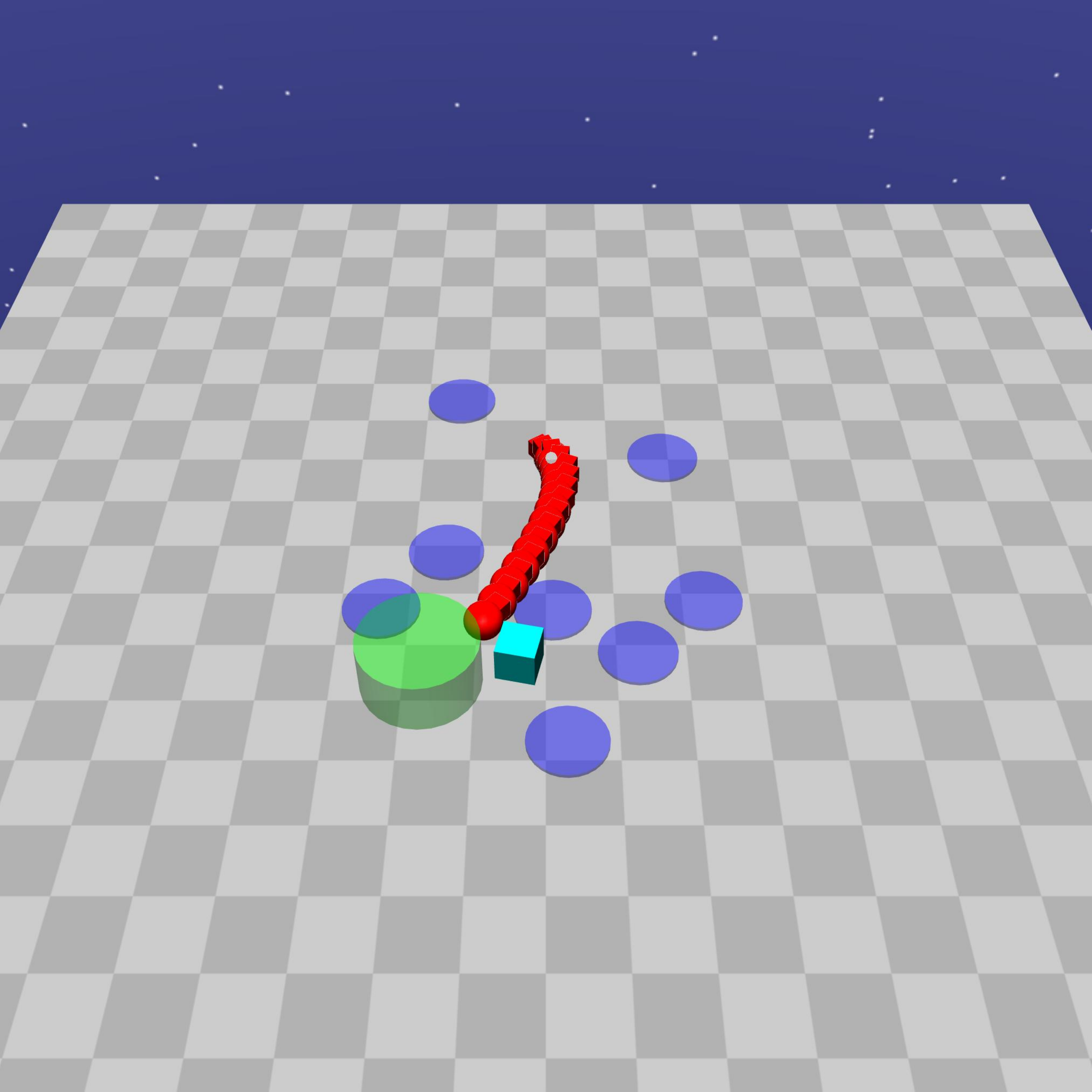}
        \includegraphics[trim=0 300 0 500, clip, width=0.24\linewidth]{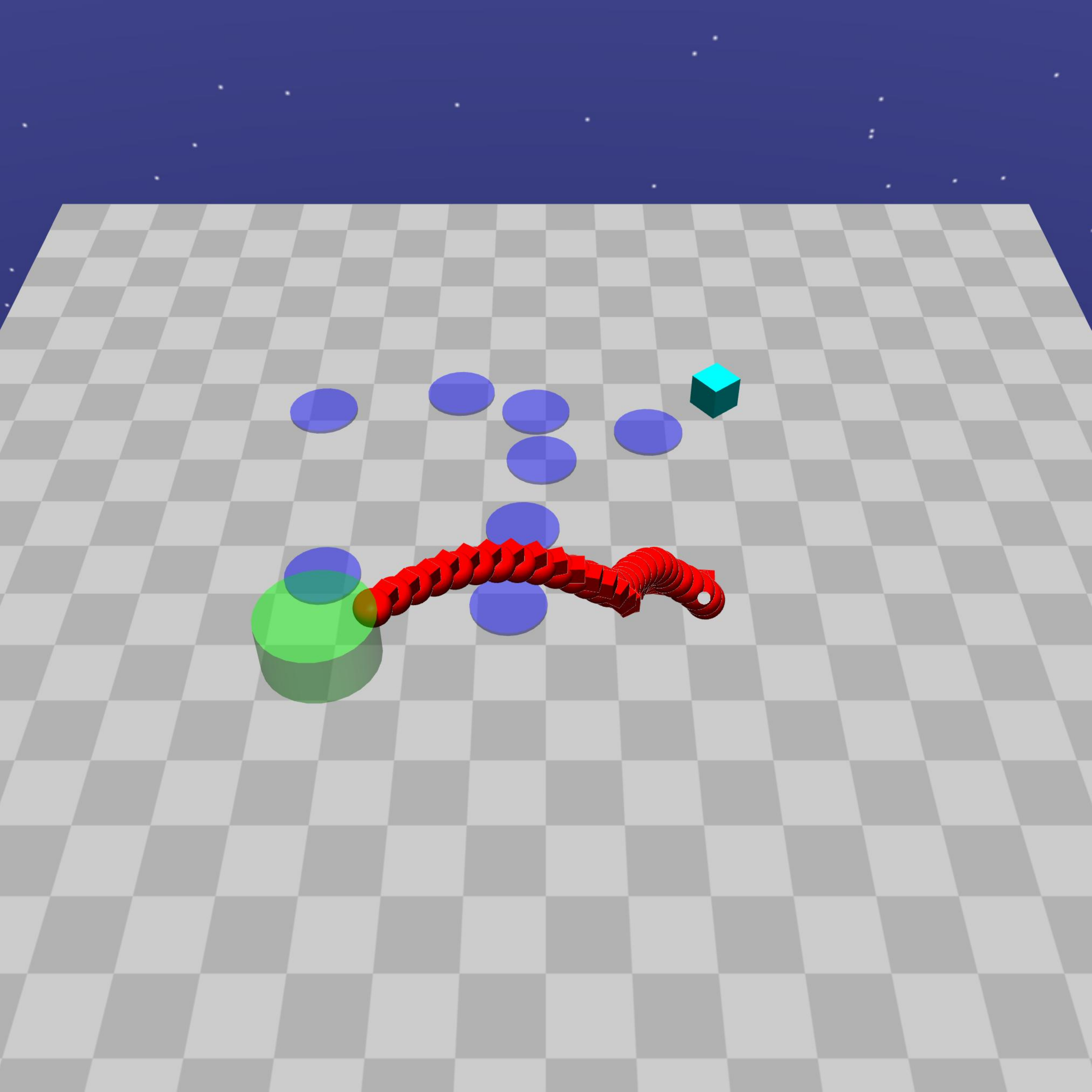}
        \\
        \includegraphics[trim=0 300 0 500, clip, width=0.24\linewidth]{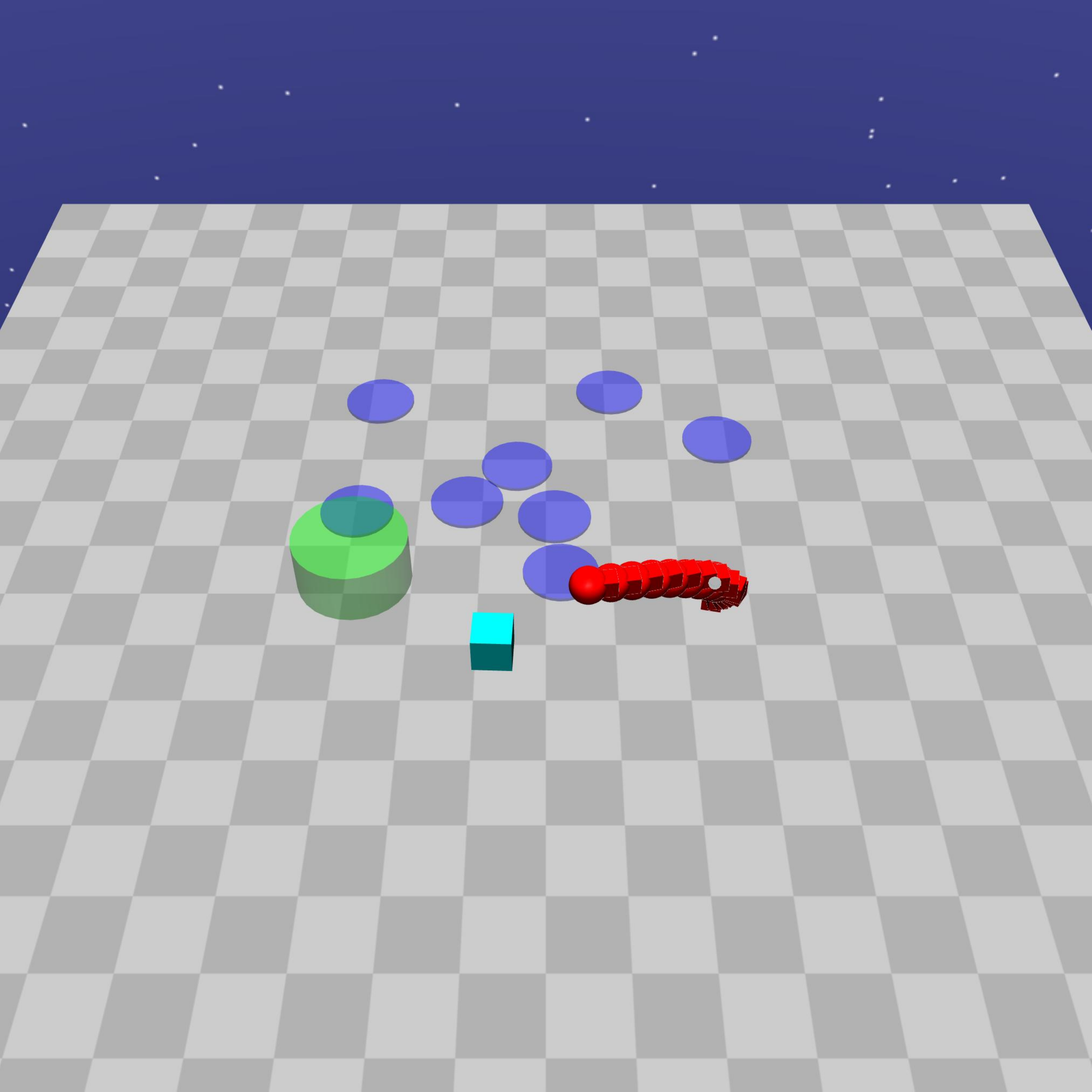}
        \includegraphics[trim=0 300 0 500, clip, width=0.24\linewidth]{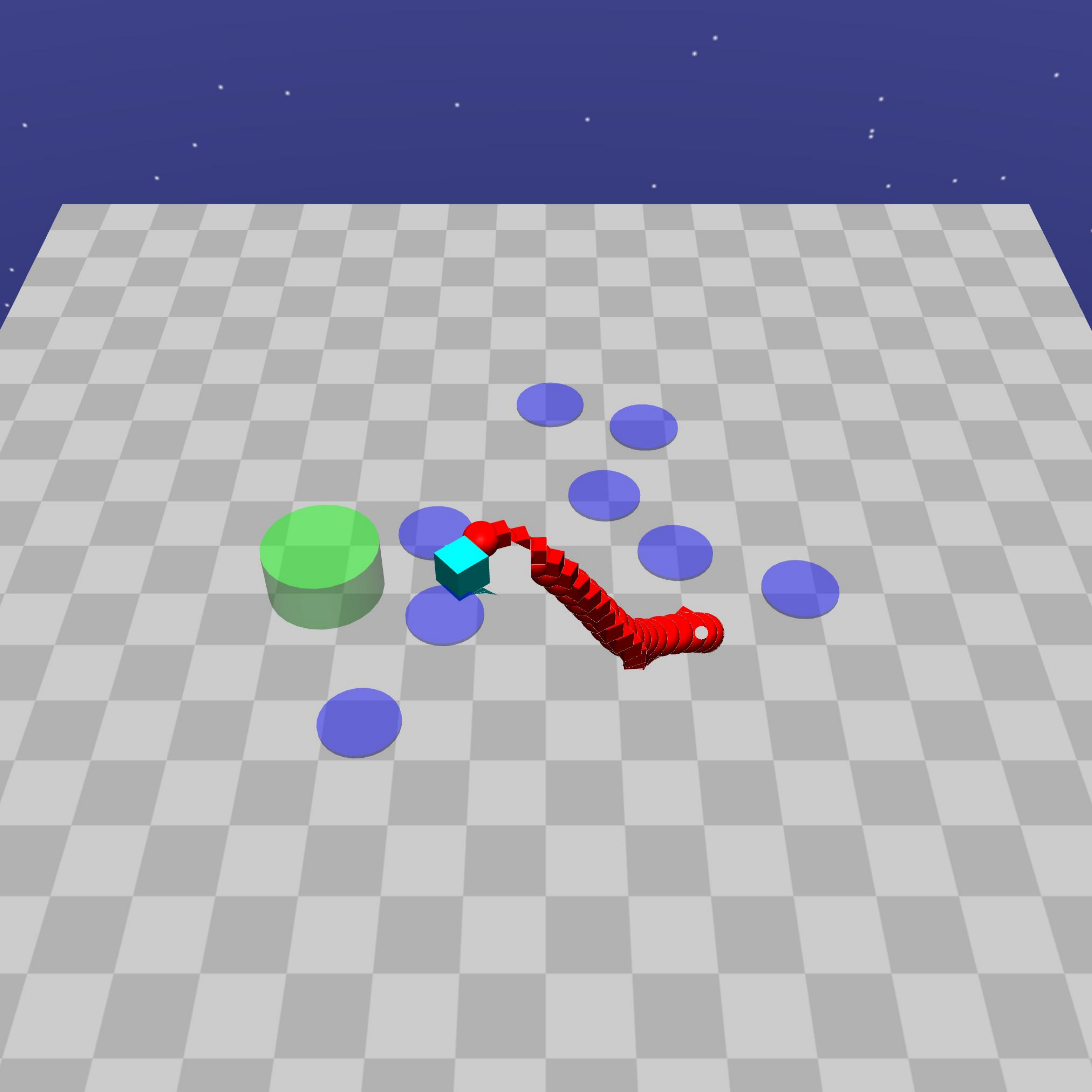}
        \includegraphics[trim=0 300 0 500, clip, width=0.24\linewidth]{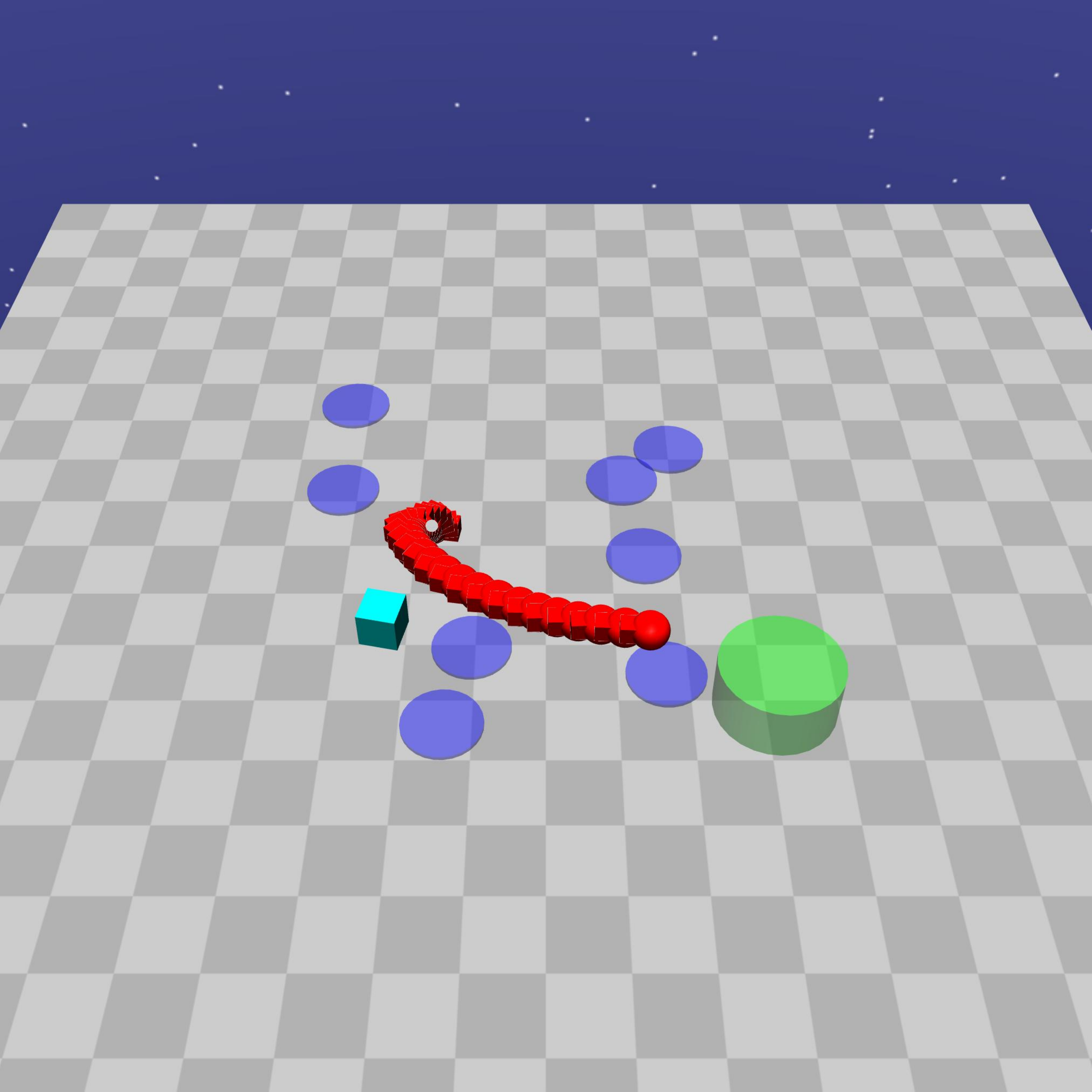}
        \includegraphics[trim=0 300 0 500, clip, width=0.24\linewidth]{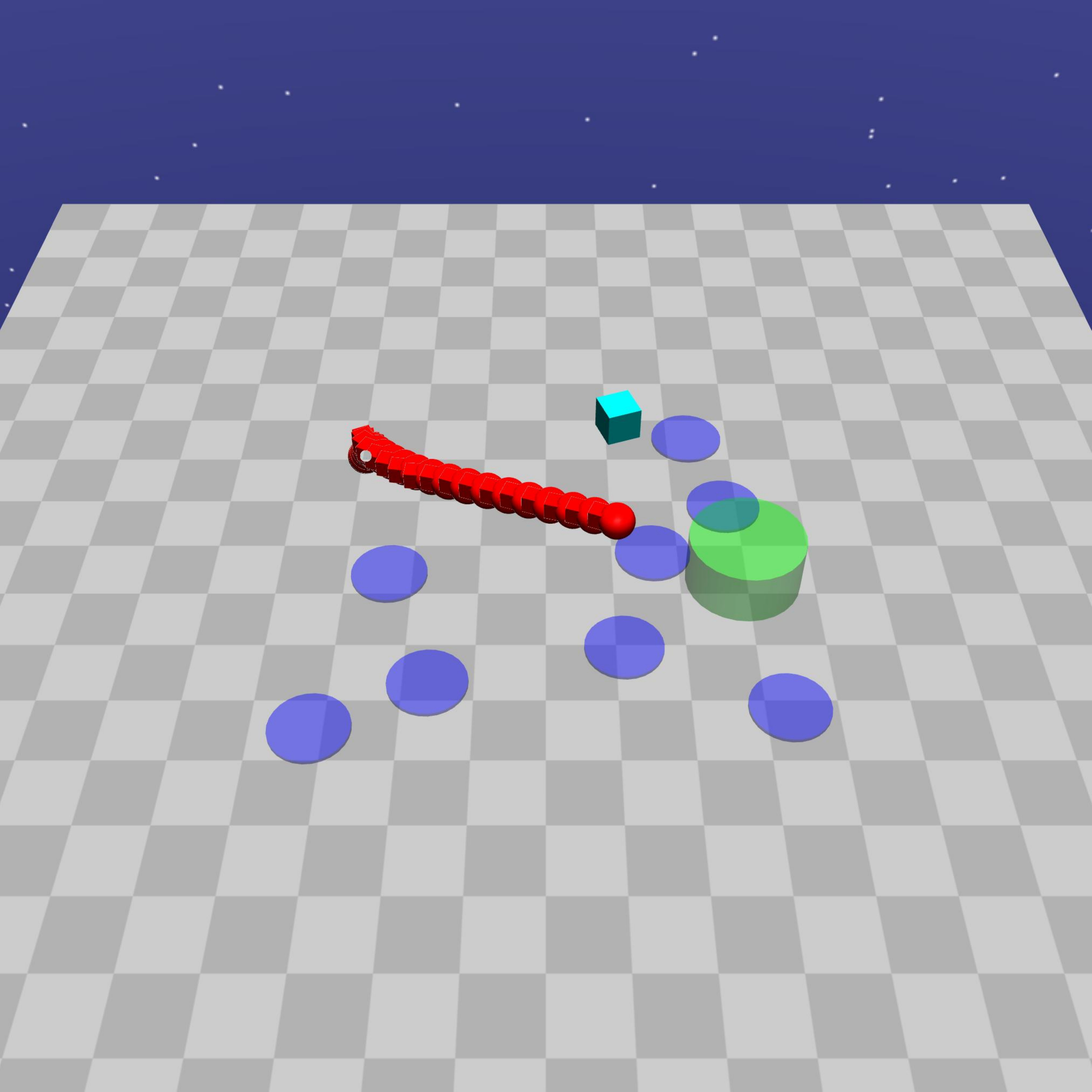}
        \caption{CPO successes (top) and failures (bottom)}
    \end{subfigure}%
    \\
    \begin{subfigure}[t]{\textwidth}
        \centering
        \includegraphics[trim=0 500 0 300, clip, width=0.24\linewidth]{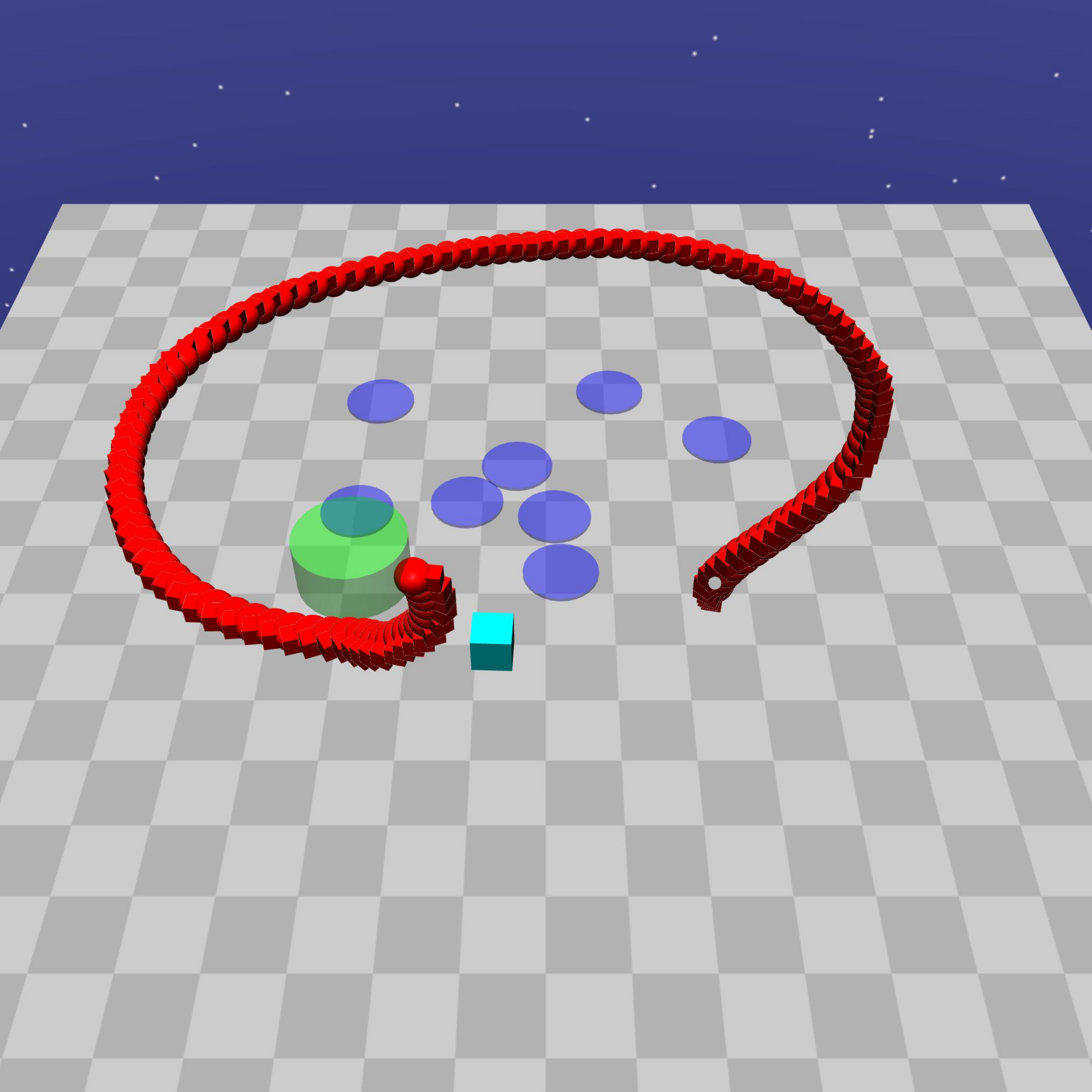}
        \includegraphics[trim=0 500 0 300, clip, width=0.24\linewidth]{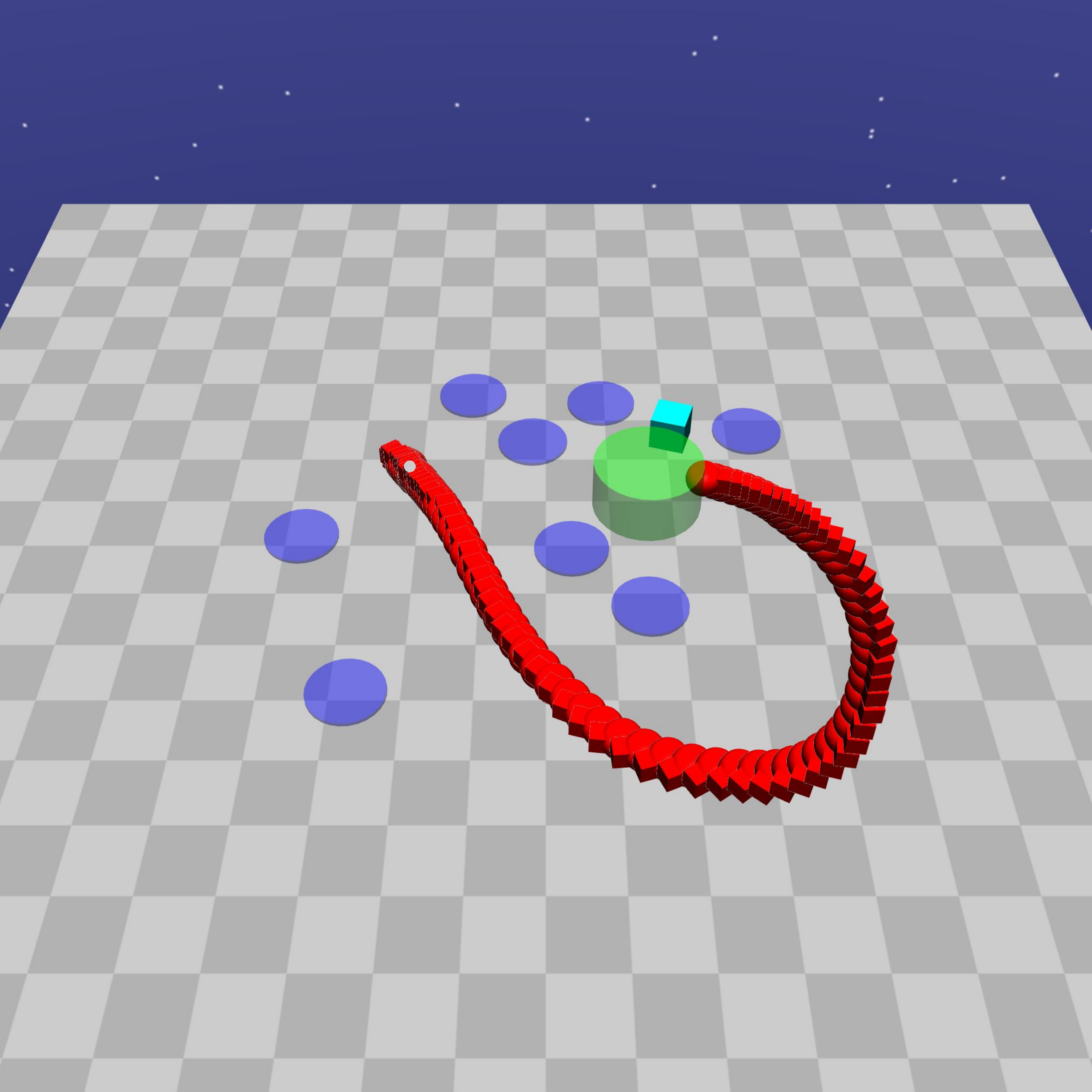}
        \includegraphics[trim=0 500 0 300, clip, width=0.24\linewidth]{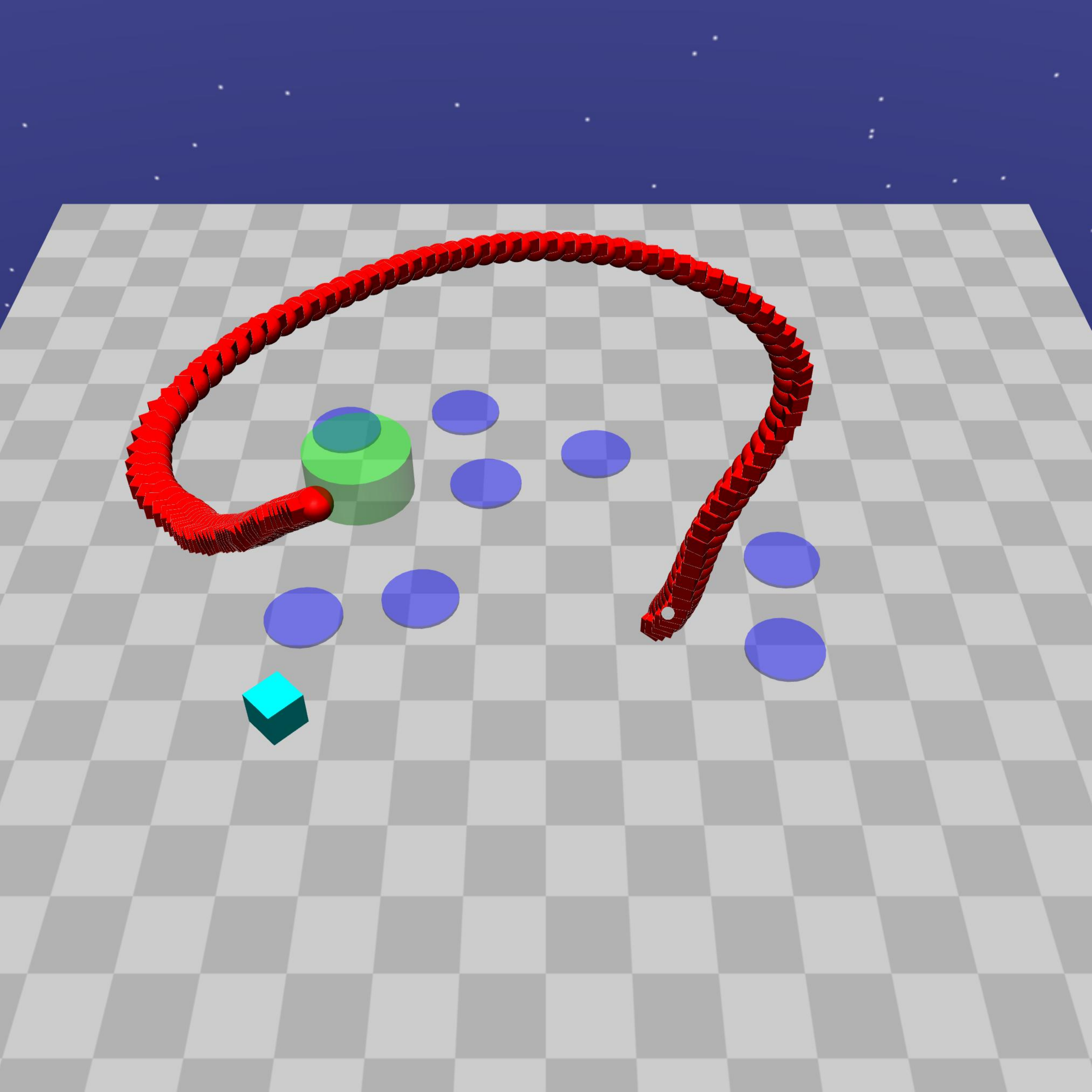}
        \includegraphics[trim=0 500 0 300, clip, width=0.24\linewidth]{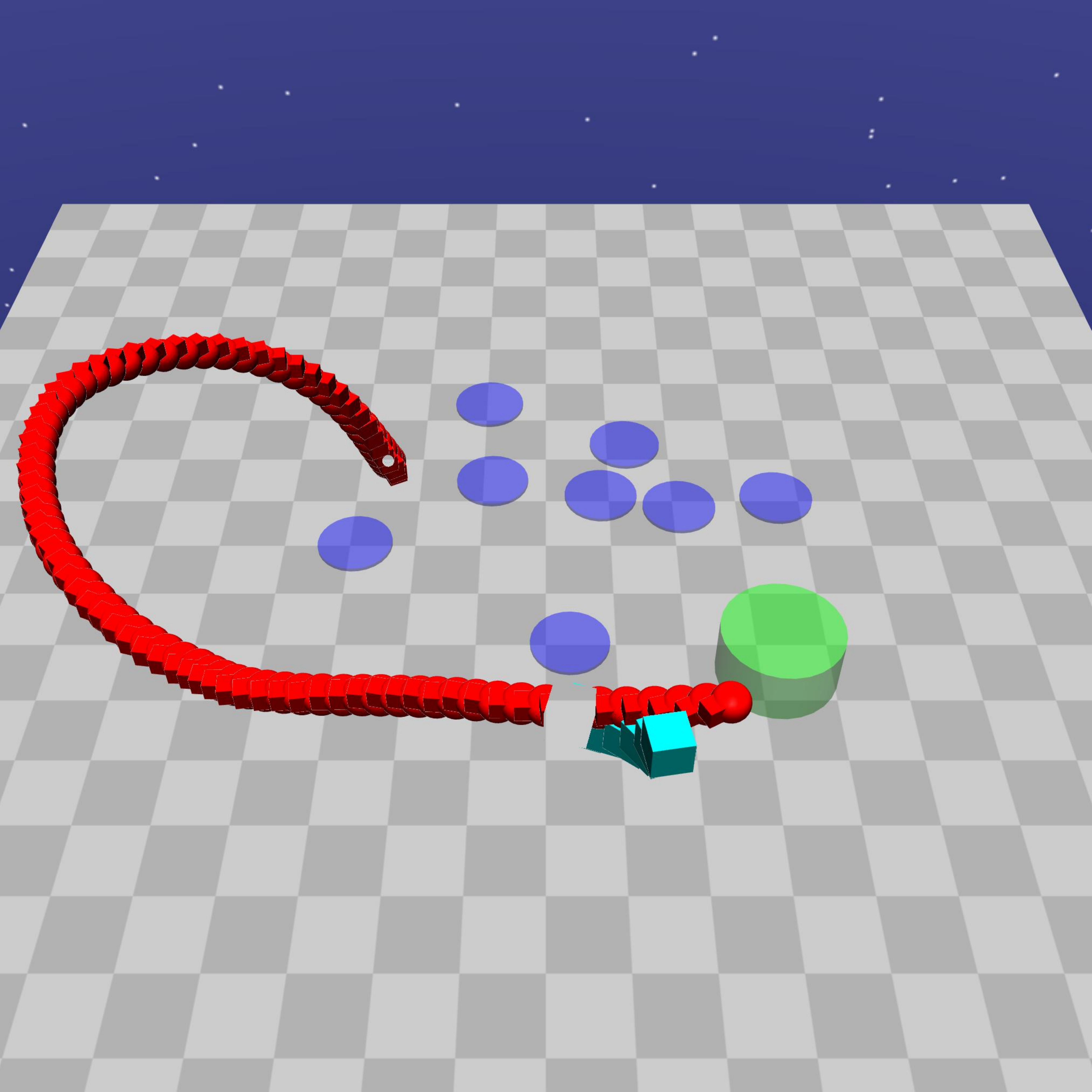}
        \\
        \includegraphics[trim=20 0 20 300, clip, width=0.24\linewidth]{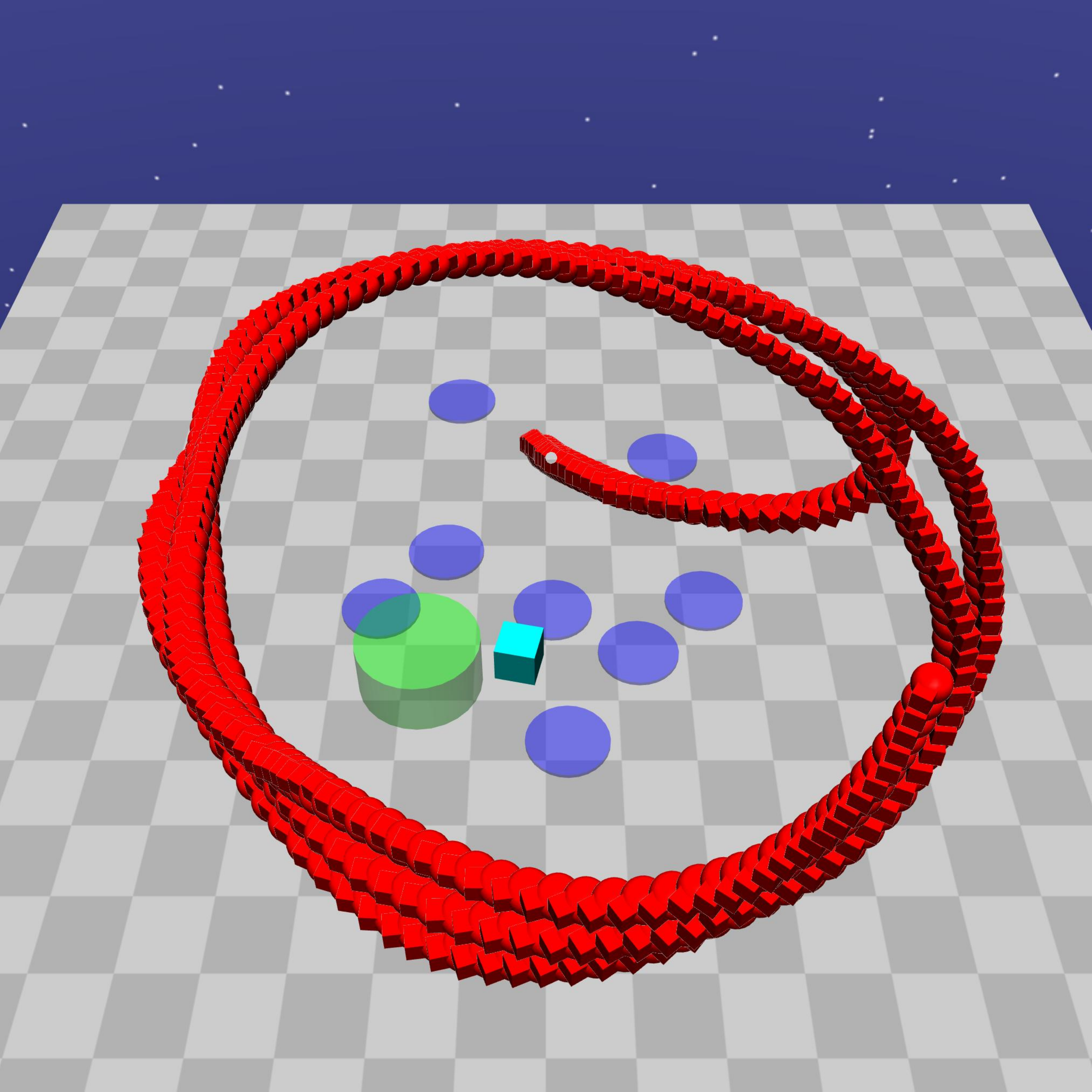}
        \includegraphics[trim=20 0 20 300, clip, width=0.24\linewidth]{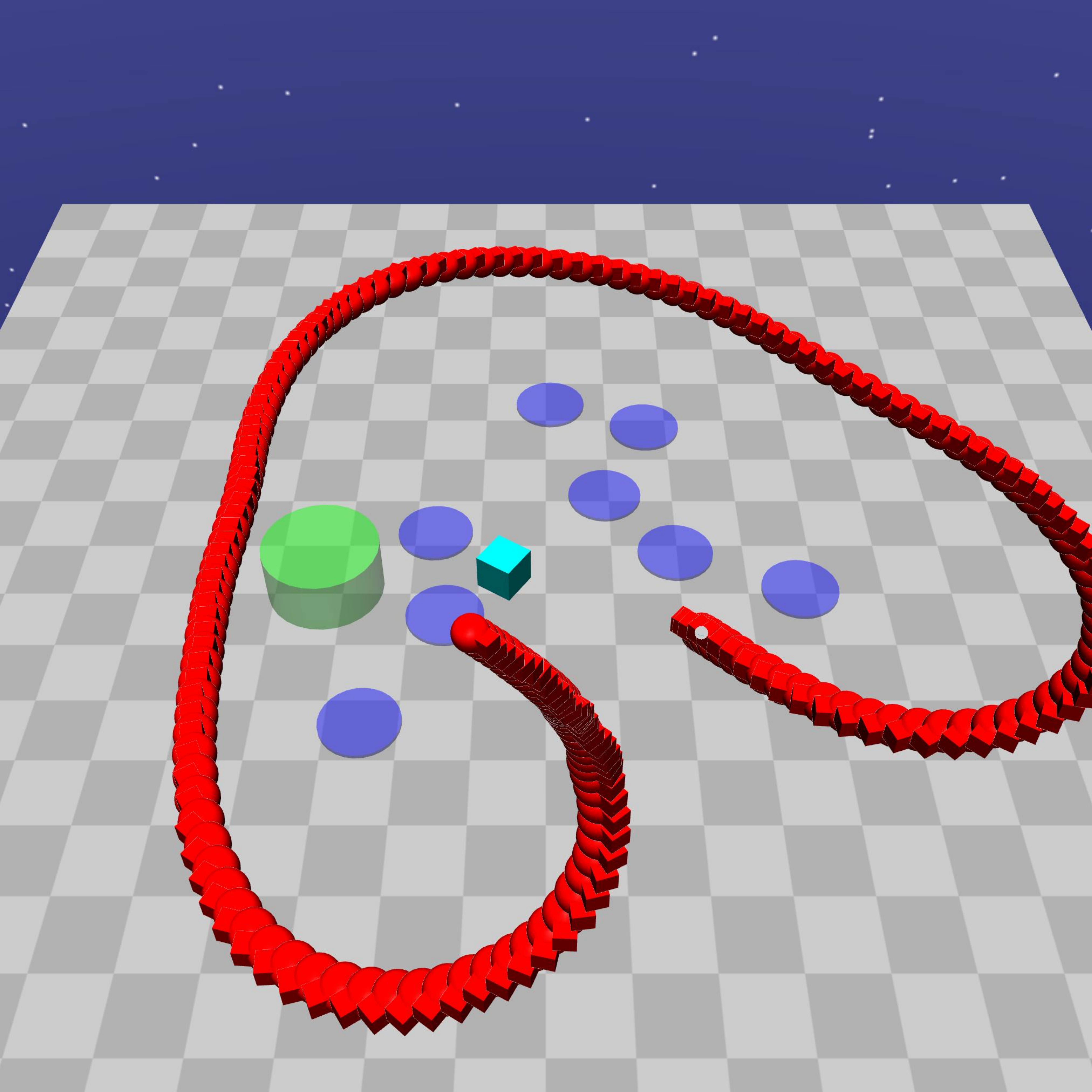}
        \includegraphics[trim=20 0 20 300, clip, width=0.24\linewidth]{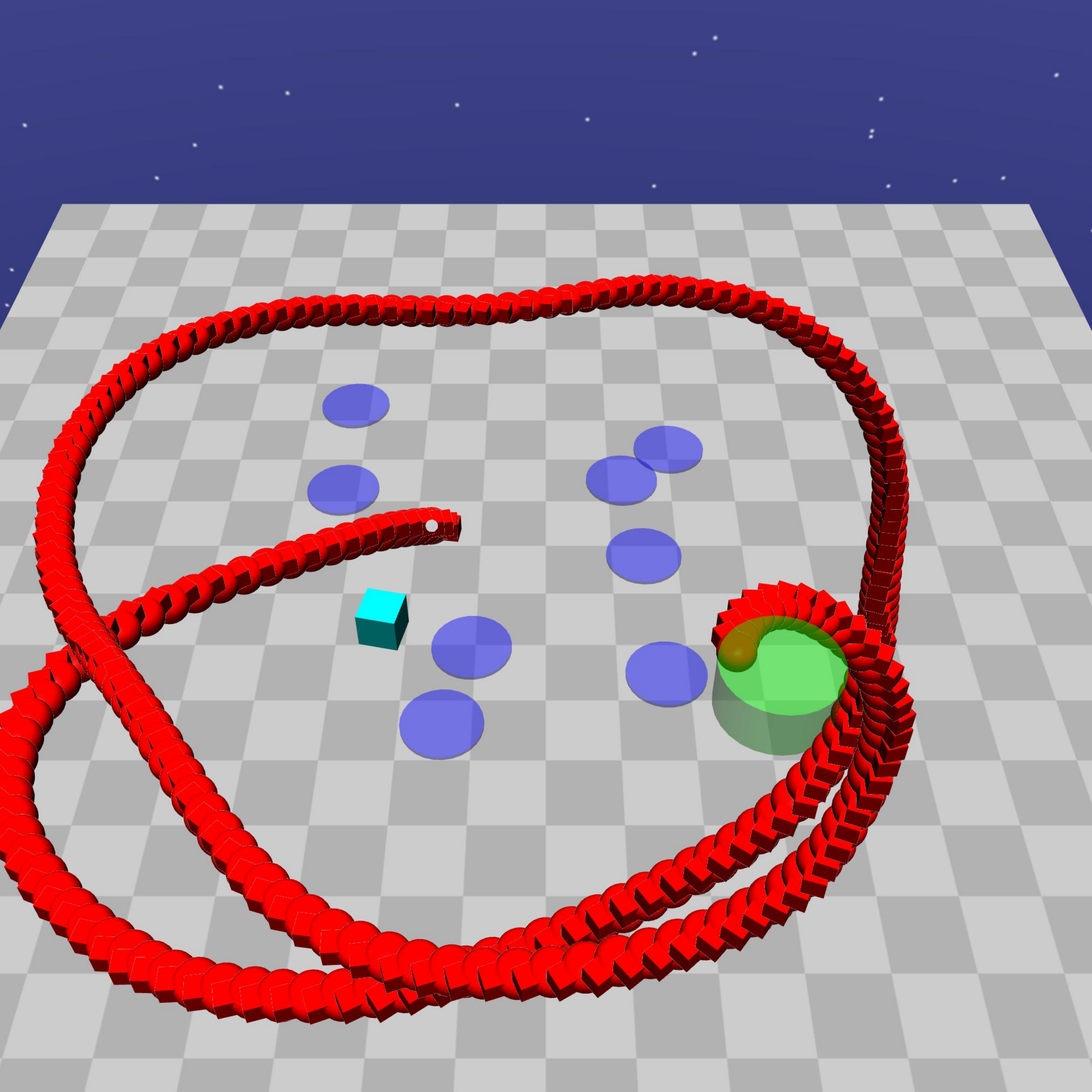}
        \includegraphics[trim=20 0 20 300, clip, width=0.24\linewidth]{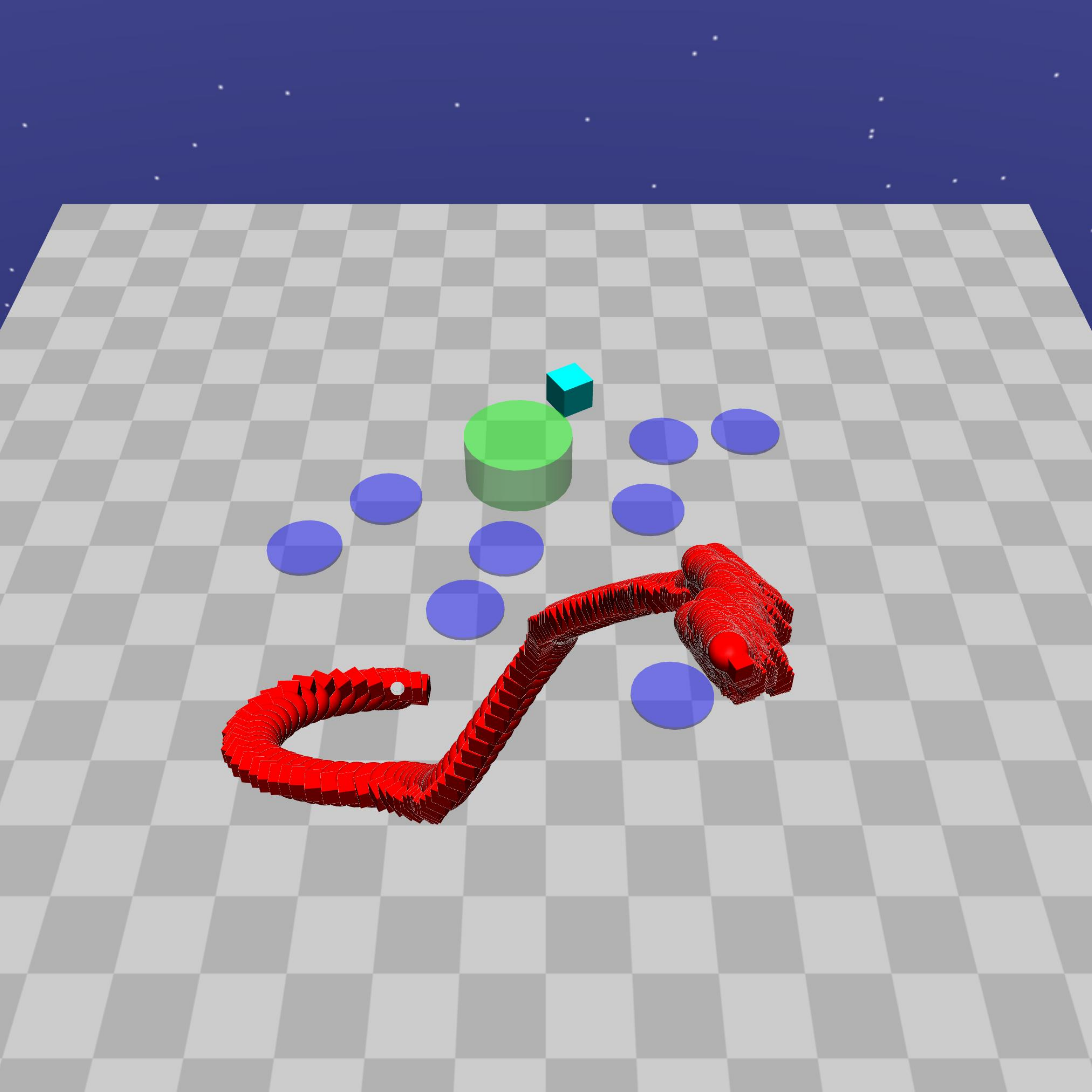}
        \caption{TRPO-Minmax successes (top) and failures (bottom)}
    \end{subfigure}%
    \caption{Sample trajectories of policies learned by each baseline and our Minmax approach in the Safety Gym \pointgoalhard domain, in the experiments of Figure \ref{fig:safety_baselines}.
    Trajectories that hit hazards or take more than 1000 timesteps to reach the goal location are considered failures.   
    }
    \label{fig:all_trajectories}
\end{figure*}%

\begin{figure*}[b!]
    \centering
    \includegraphics[width=0.9\linewidth]{images/_legend.pdf}
    \\
    \begin{subfigure}[t]{0.47\textwidth}
        \centering
        \includegraphics[width=\linewidth]{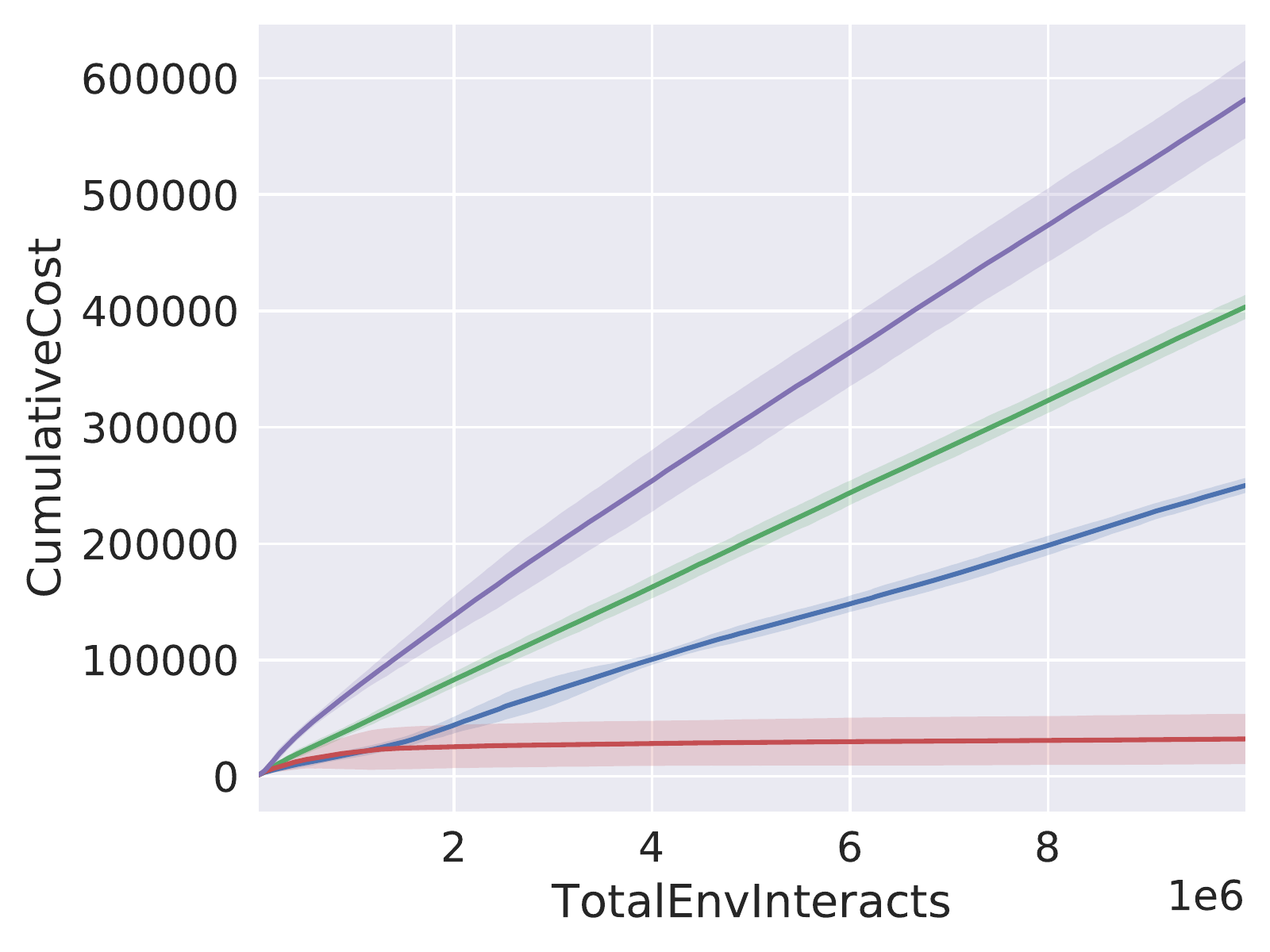}
        \caption{The cumulative cost.}
        \label{fig:cumulative_cost_original}
    \end{subfigure}%
    \quad
    \begin{subfigure}[t]{0.47\textwidth}
        \centering
        \includegraphics[width=\linewidth]{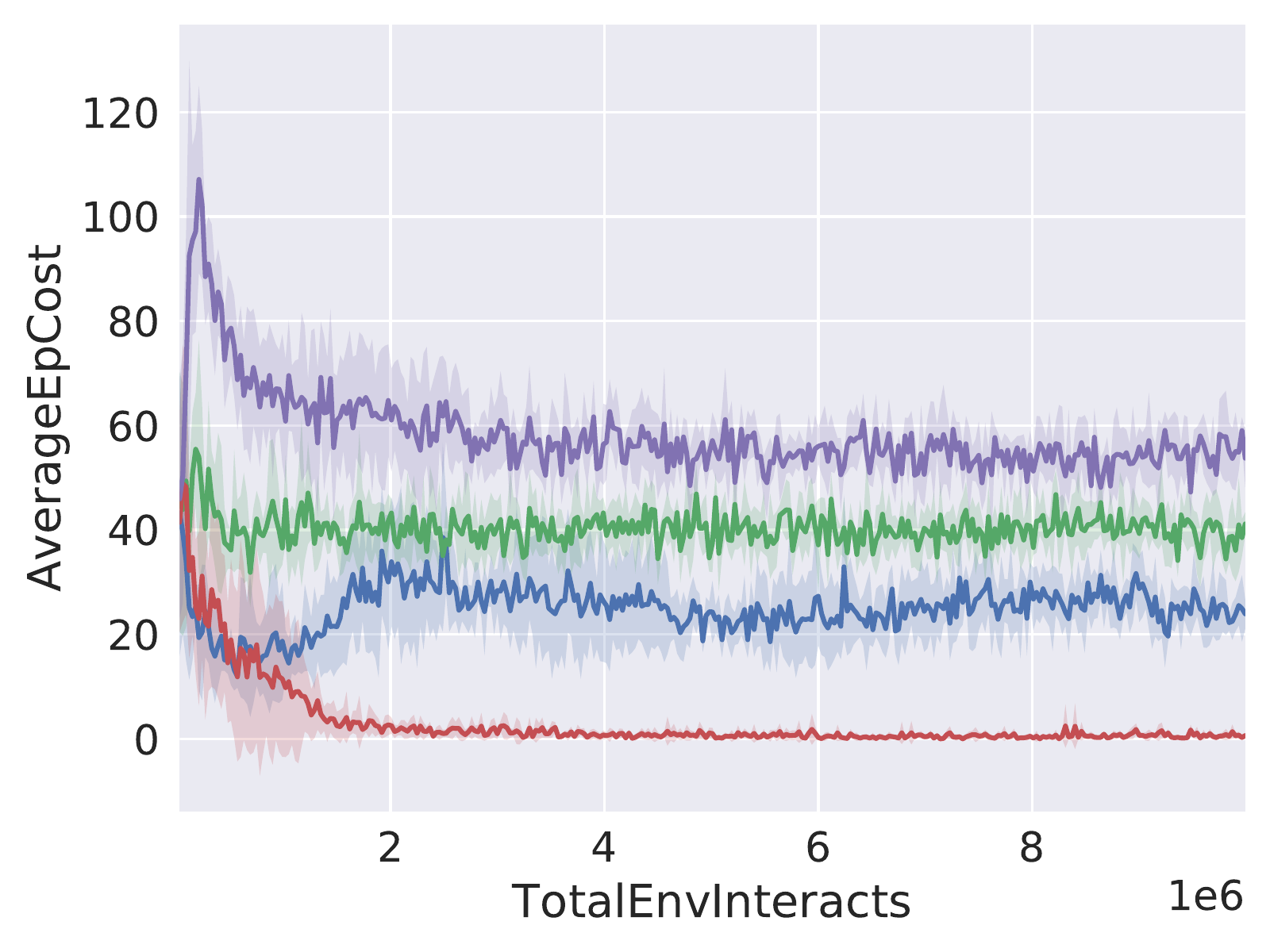}
        \caption{Failure rate}
        \label{fig:baseline_cost_original}
    \end{subfigure}%
    \\
    \begin{subfigure}[t]{0.47\textwidth}
        \centering
        \includegraphics[width=\linewidth]{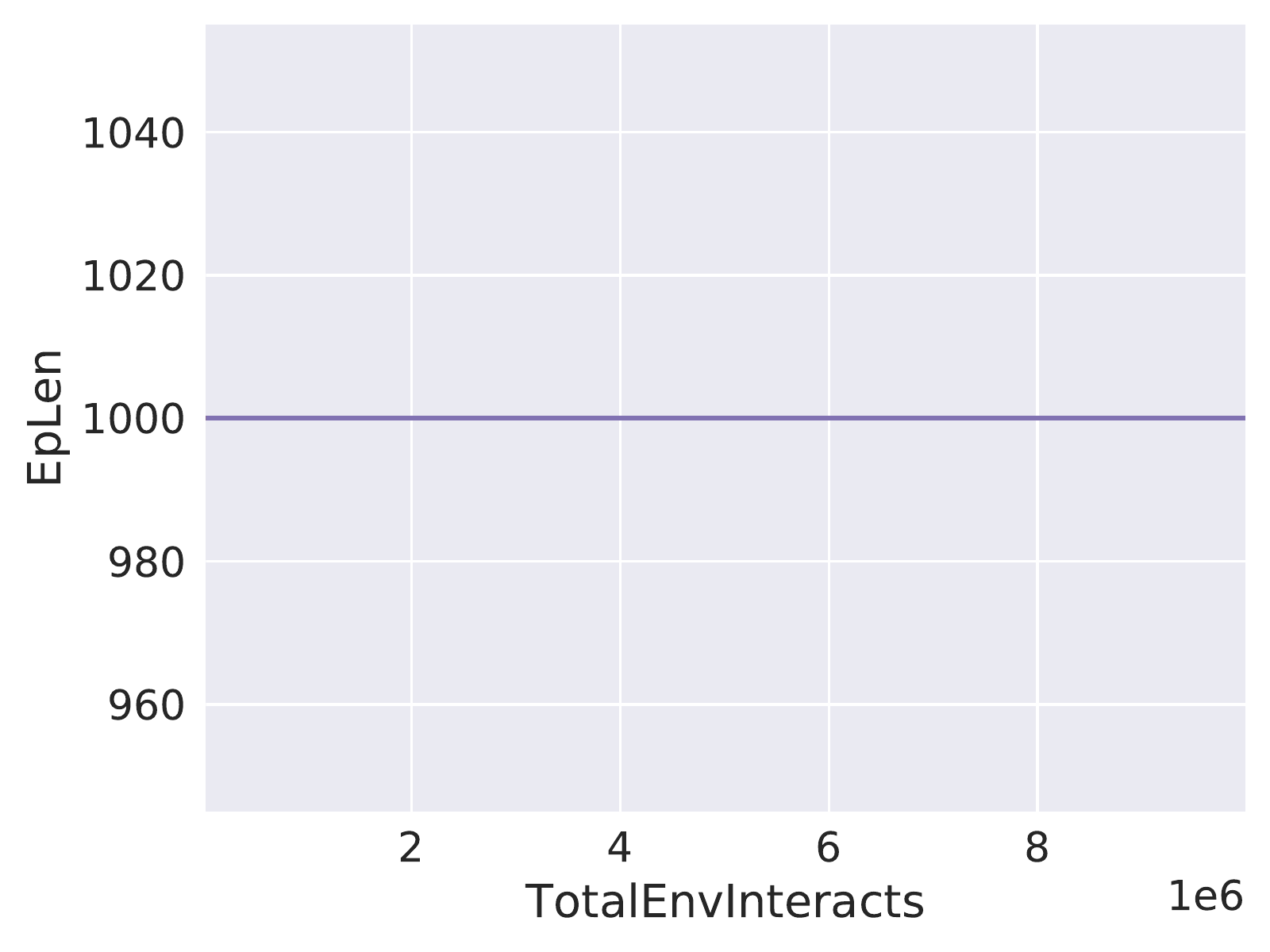}
        \caption{Average episode length}
        \label{fig:baseline_length_original}
    \end{subfigure}%
    \quad
    \begin{subfigure}[t]{0.47\textwidth}
        \centering
        \includegraphics[width=\linewidth]{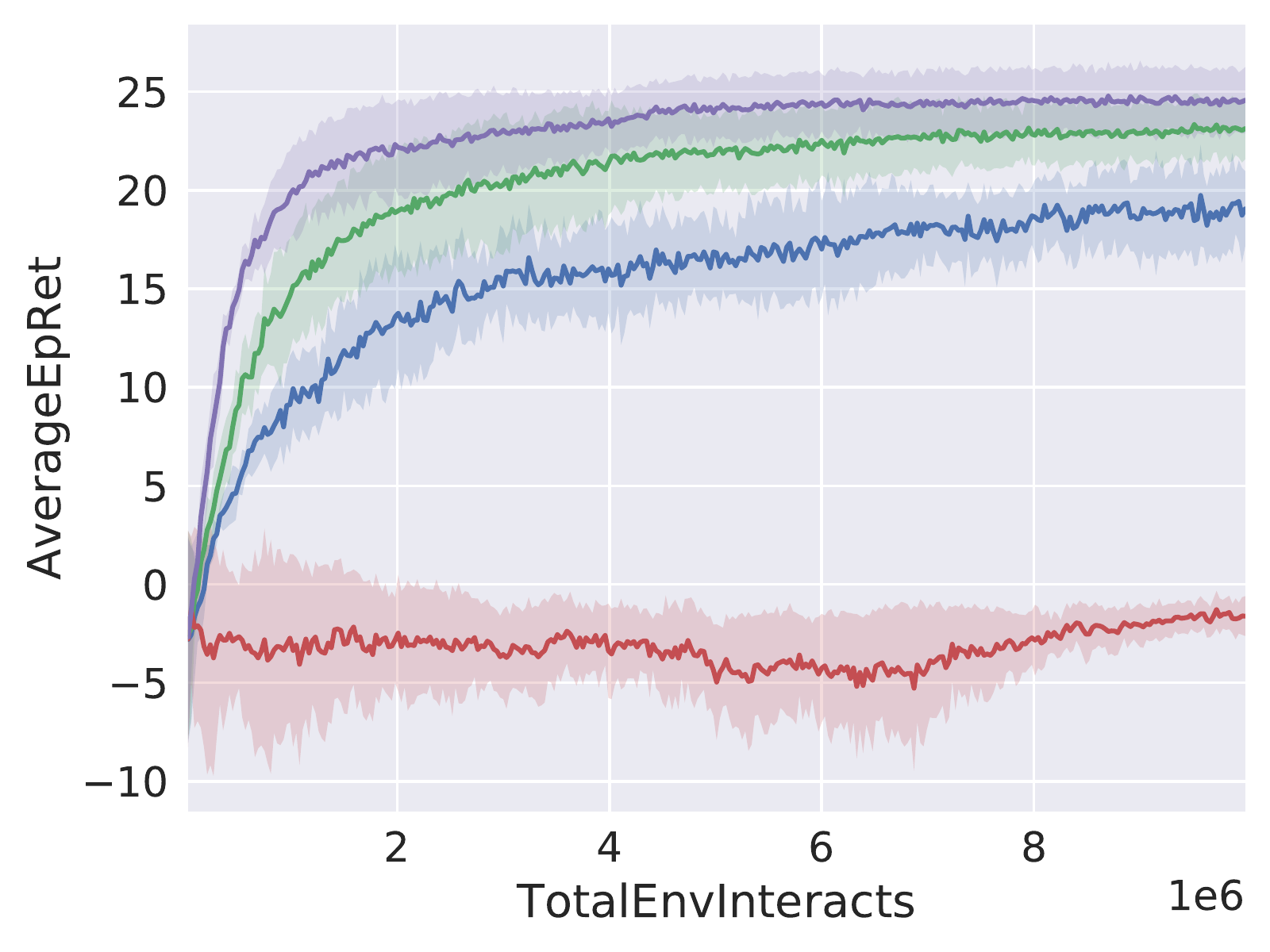}
        \caption{Average returns}
        \label{fig:baseline_return_original}
    \end{subfigure}%
    \caption{Comparison with baselines in the original Safety Gym \pointgoal environment. This domain is the same as \pointgoalhard, except that episodes do not terminate when a hazard is hit (hence every episode only terminates after 1000 steps). We set the cost threshold for TRPO Lagrangian and CPO to $25$ as in \citet{Ray2019}. For TRPO Minmax, we replace the reward with the Minmax penalty every time the agent is in an unsafe state (that is every time the cost is greater than zero), as in previous experiments and as per Algorithm 1. Since unsafe states are not terminal in this domain, our theoretical setting is no longer satisfied. Hence, while TRPO Minmax still beats the baselines in safe exploration (a-b), it struggles to maximise rewards while avoiding unsafe states (d).  }
    \label{fig:safety_baselines_original}
\end{figure*}%

\begin{figure*}[b!]
    \centering
    \begin{subfigure}[t]{\textwidth}
        \centering
        \includegraphics[trim=0 300 0 500, clip, width=0.24\linewidth]{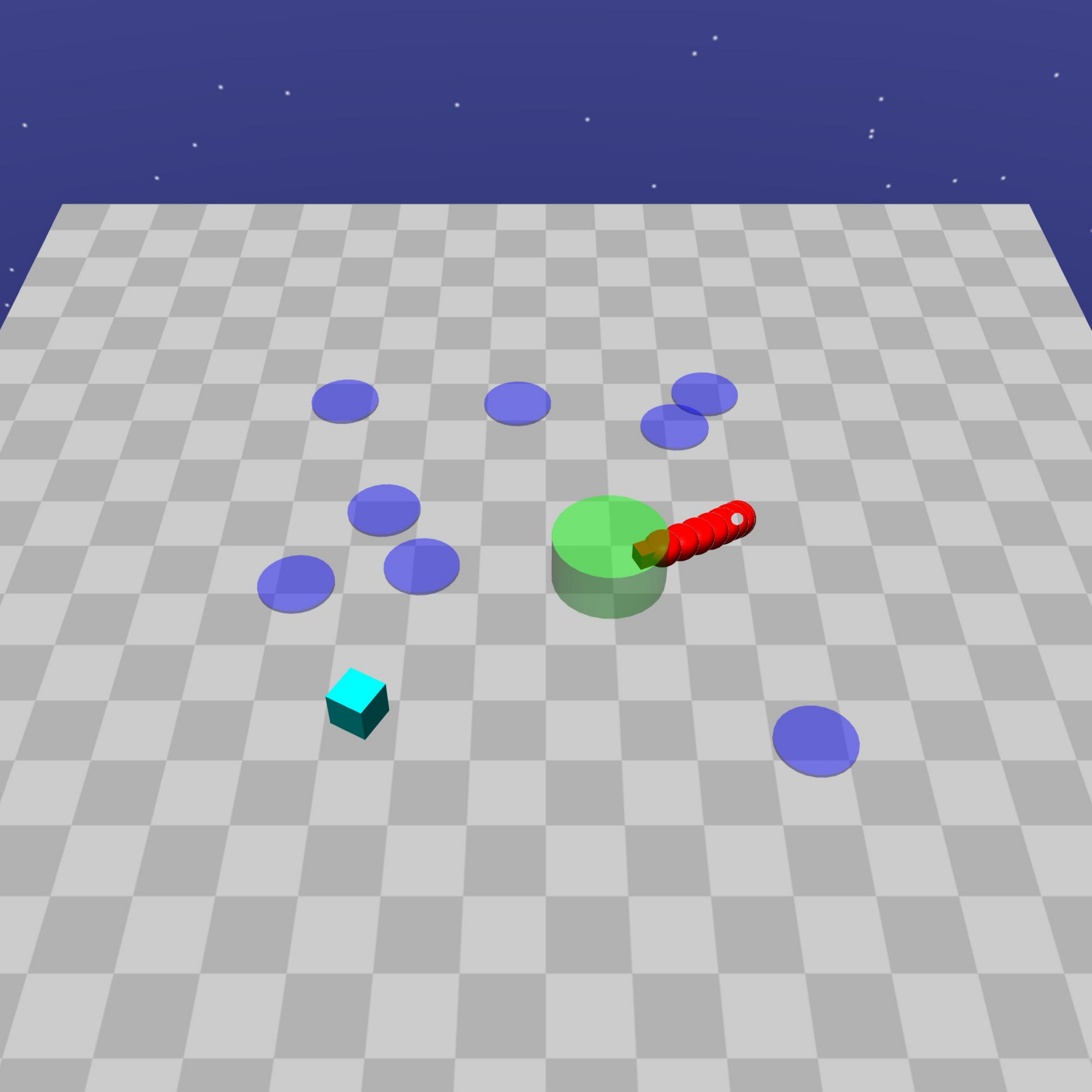}
        \includegraphics[trim=0 300 0 500, clip, width=0.24\linewidth]{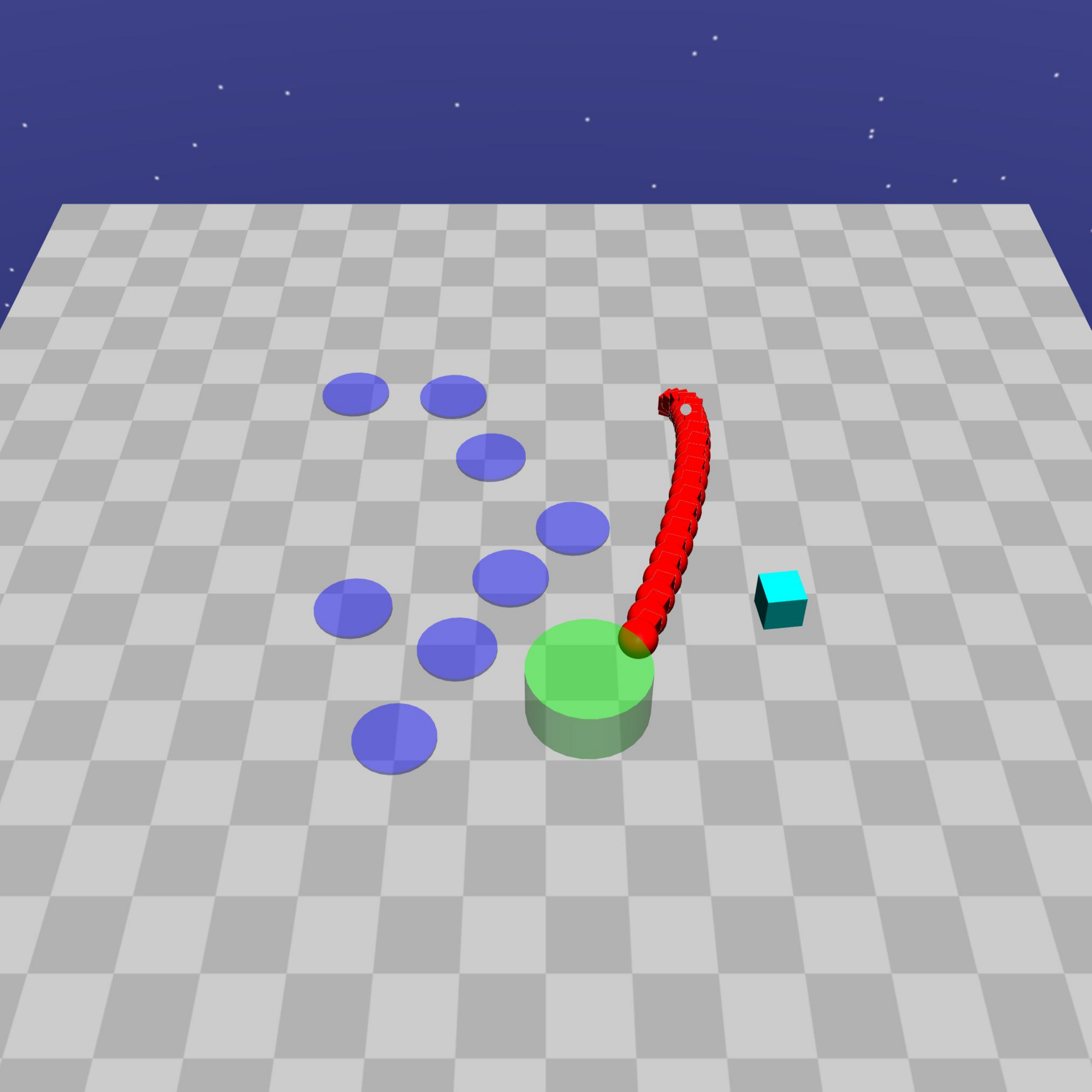}
        \includegraphics[trim=0 300 0 500, clip, width=0.24\linewidth]{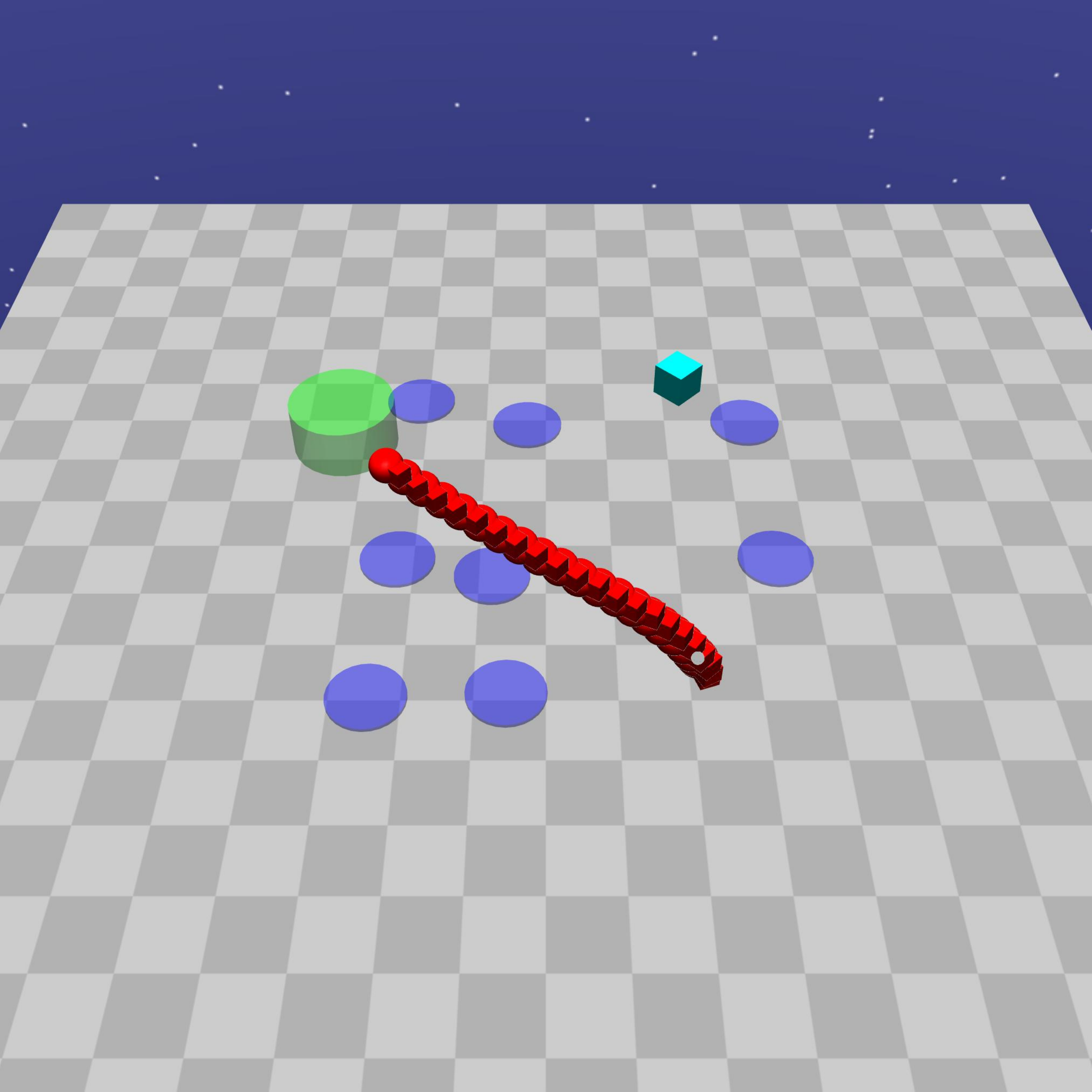}
        \includegraphics[trim=0 300 0 500, clip, width=0.24\linewidth]{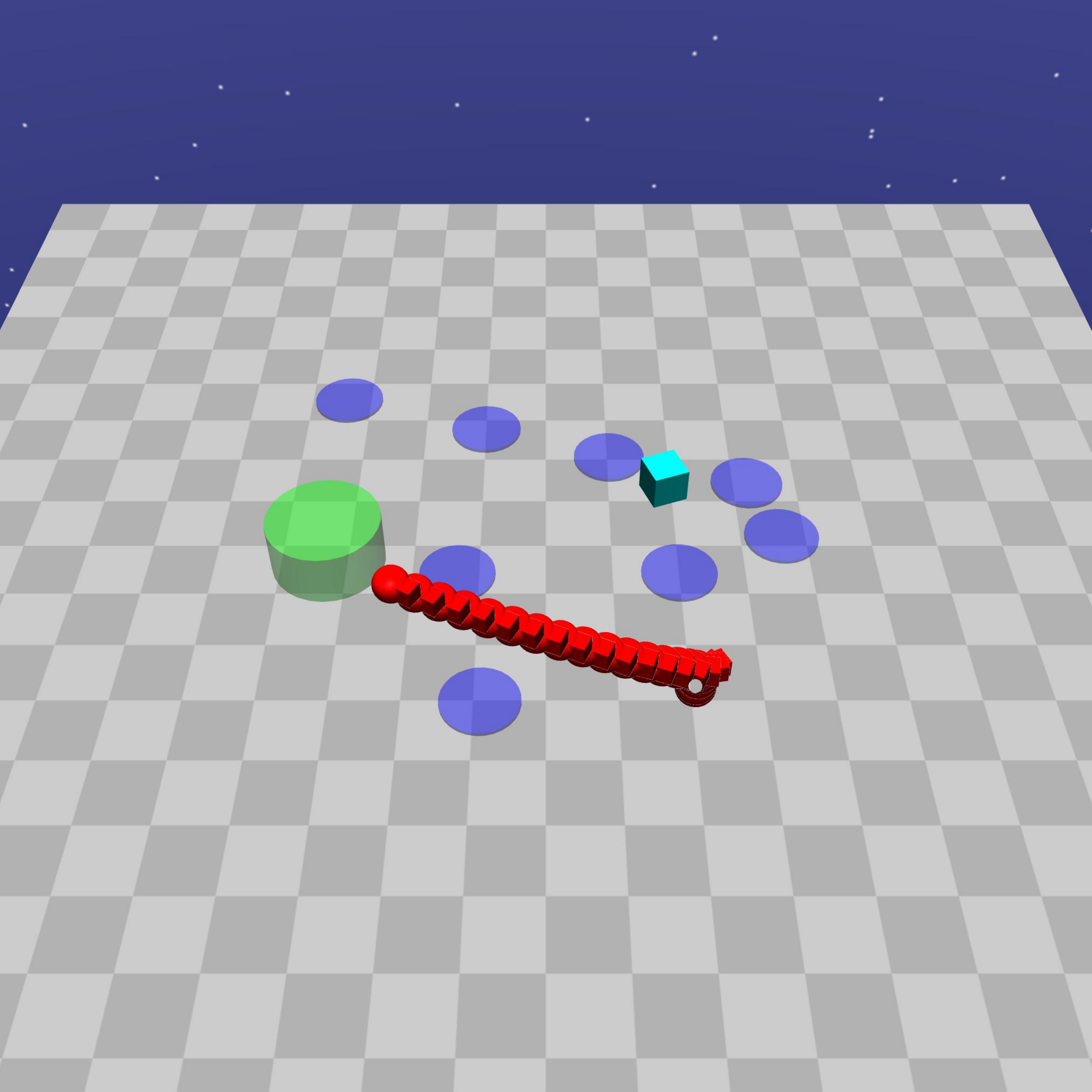}
        \\
        \includegraphics[trim=0 300 0 500, clip, width=0.24\linewidth]{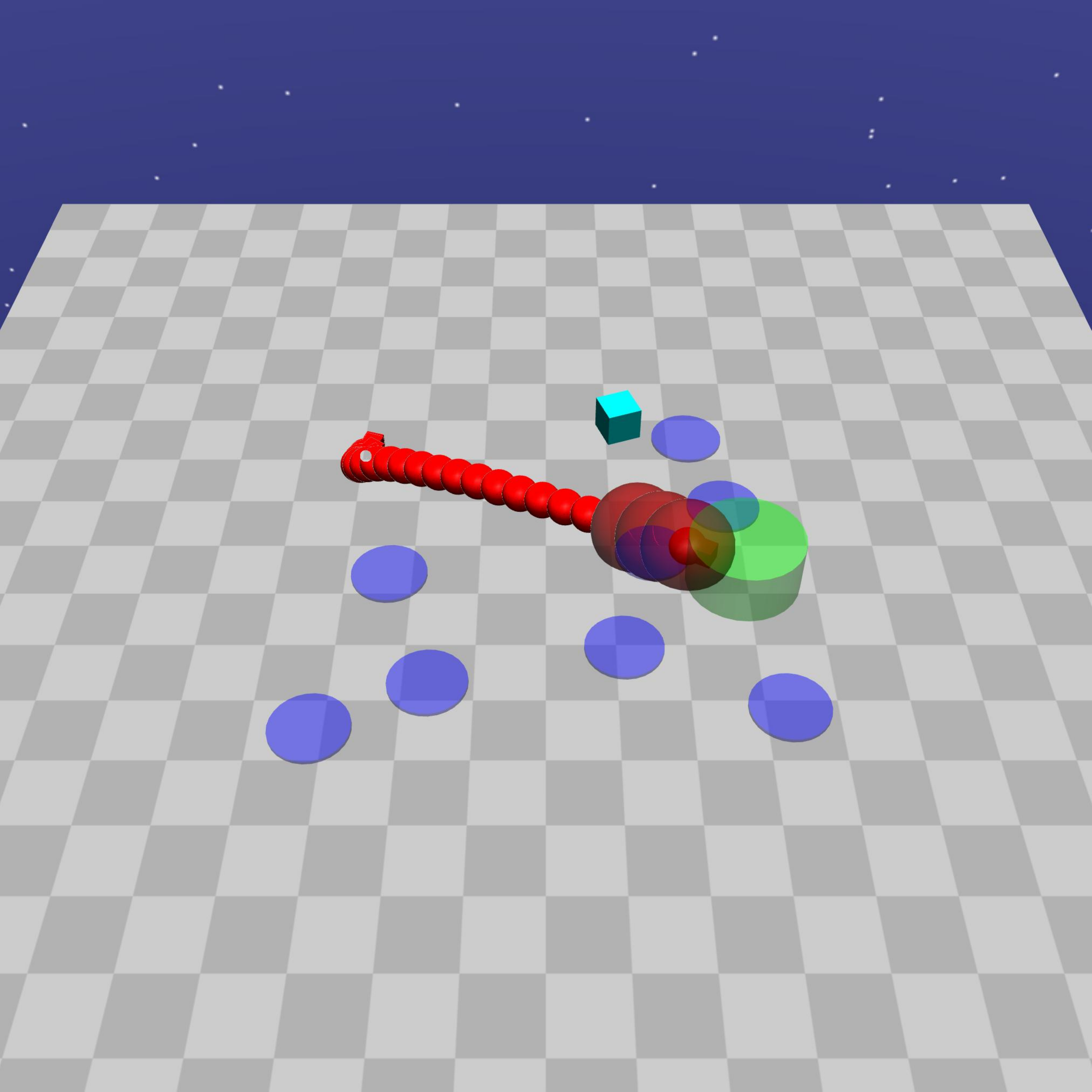}
        \includegraphics[trim=0 300 0 500, clip, width=0.24\linewidth]{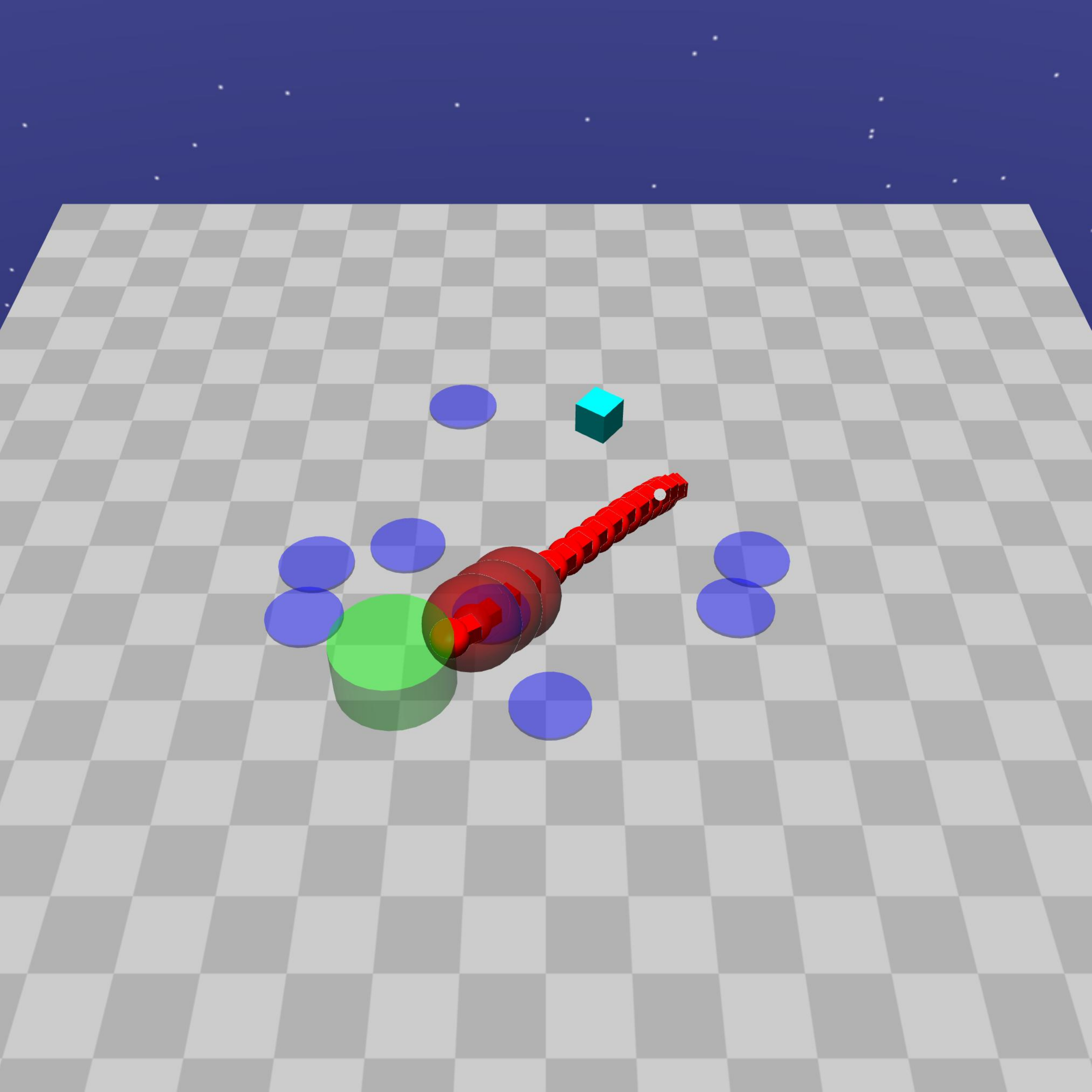}
        \includegraphics[trim=0 300 0 500, clip, width=0.24\linewidth]{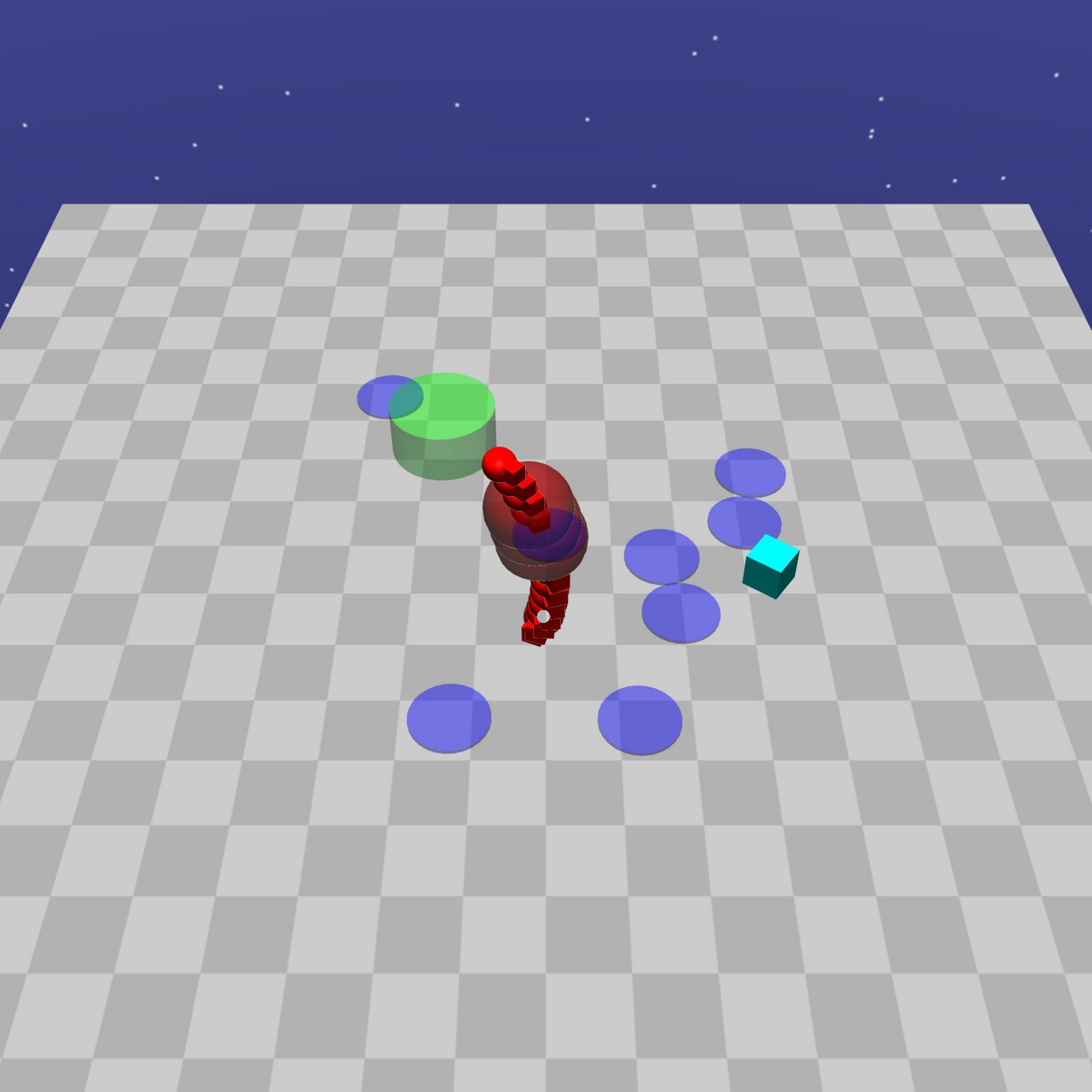}
        \includegraphics[trim=0 300 0 500, clip, width=0.24\linewidth]{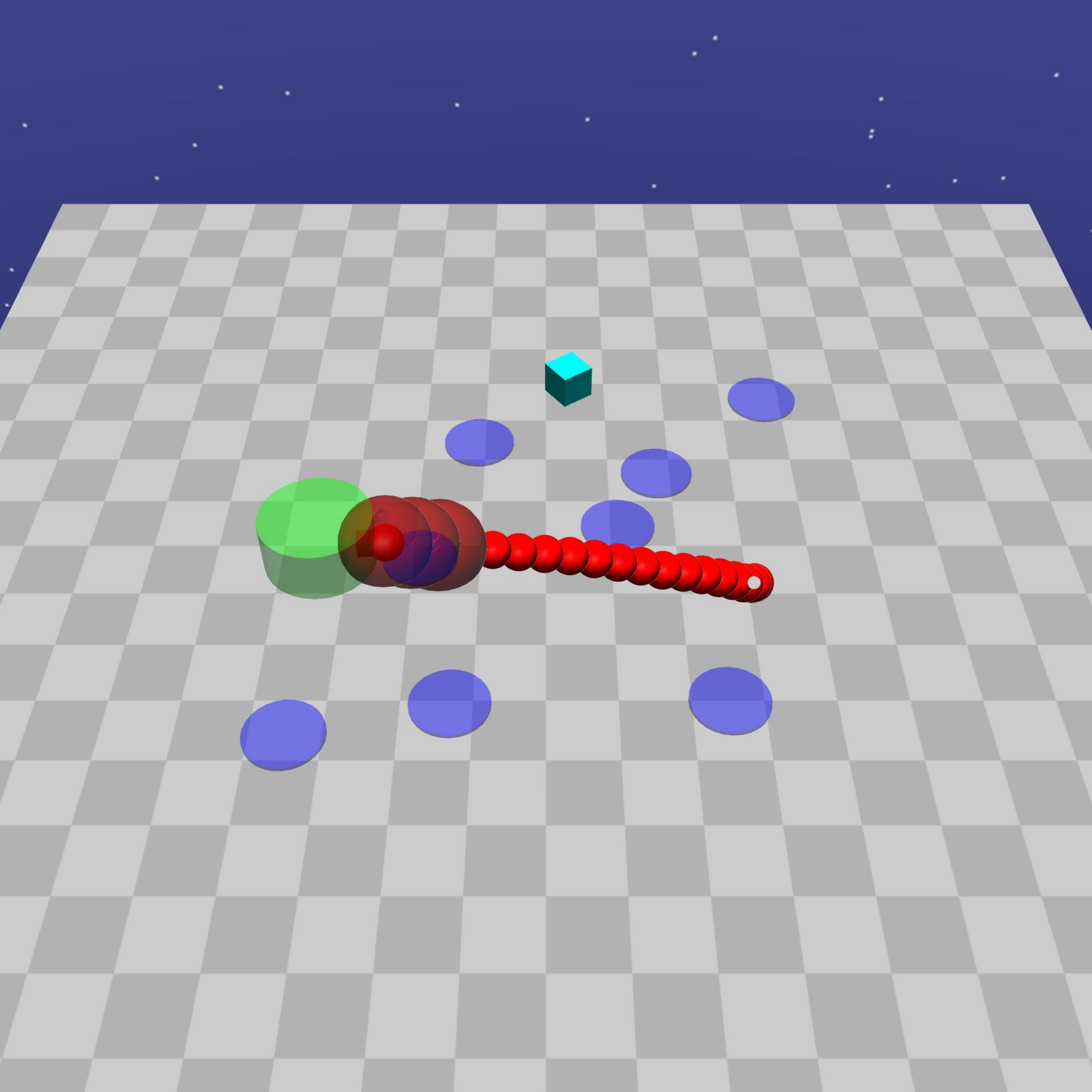}
        \caption{TRPO successes (top) and failures (bottom)}
    \end{subfigure}%
    \\
    \begin{subfigure}[t]{\textwidth}
        \centering
        \includegraphics[trim=0 300 0 500, clip, width=0.24\linewidth]{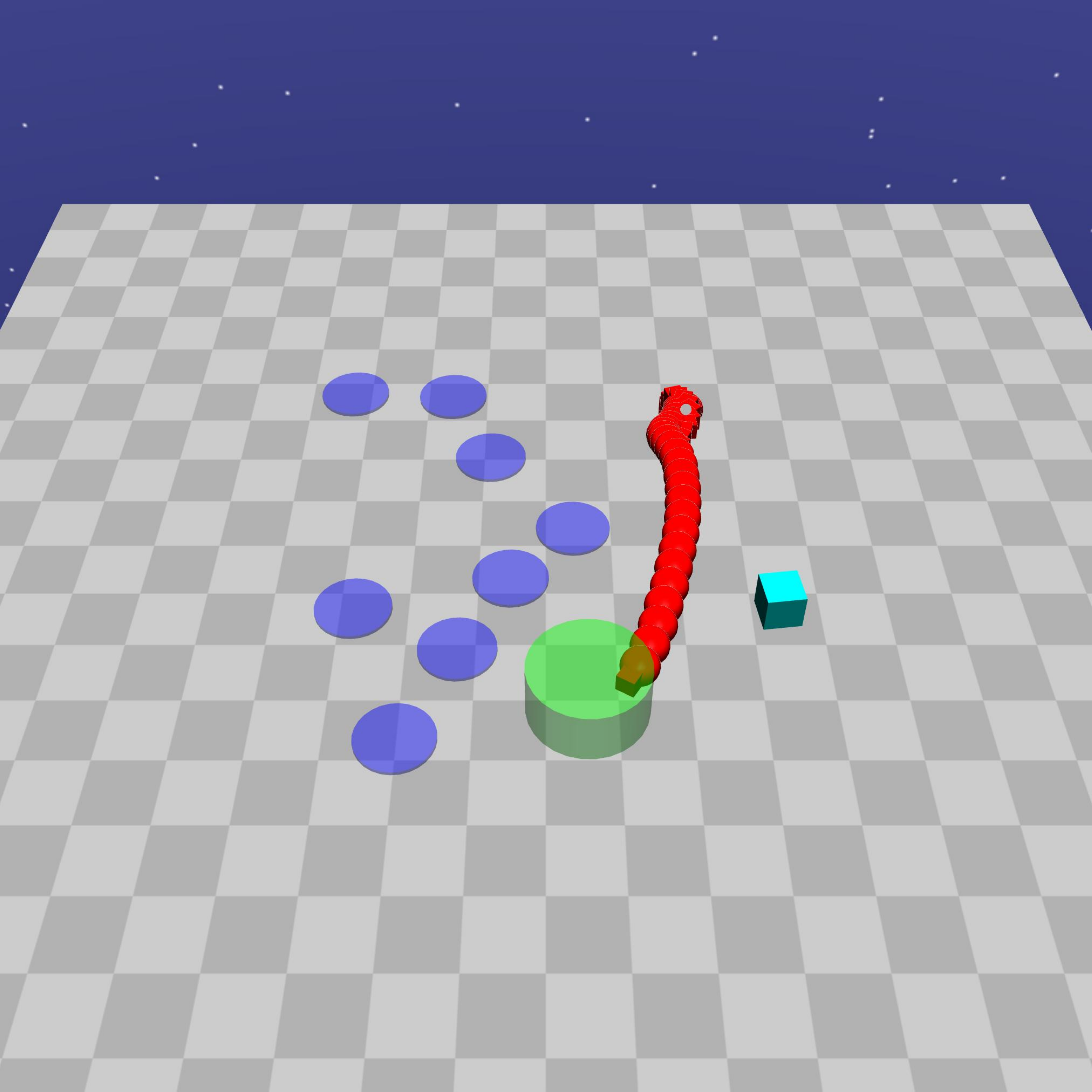}
        \includegraphics[trim=0 300 0 500, clip, width=0.24\linewidth]{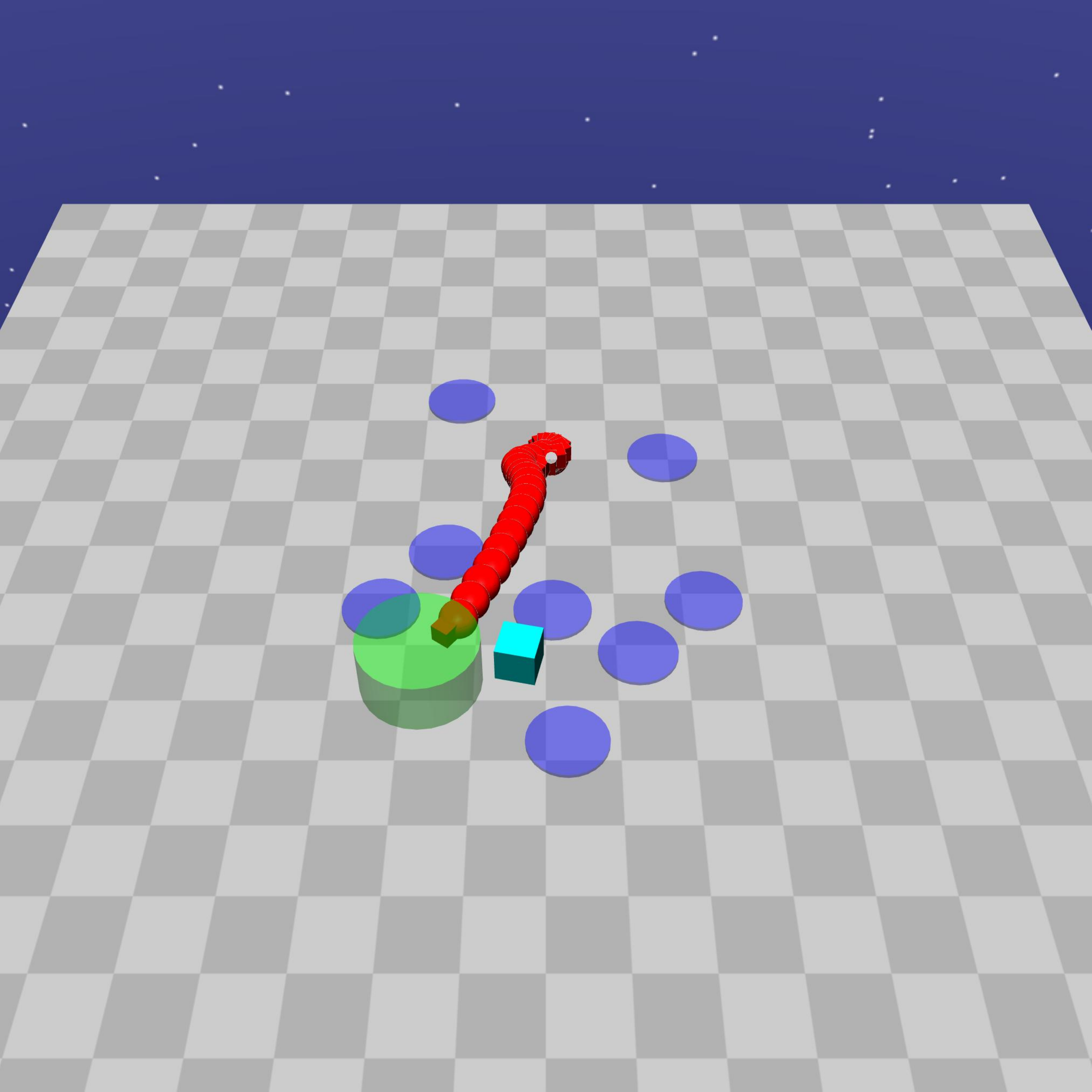}
        \includegraphics[trim=0 300 0 500, clip, width=0.24\linewidth]{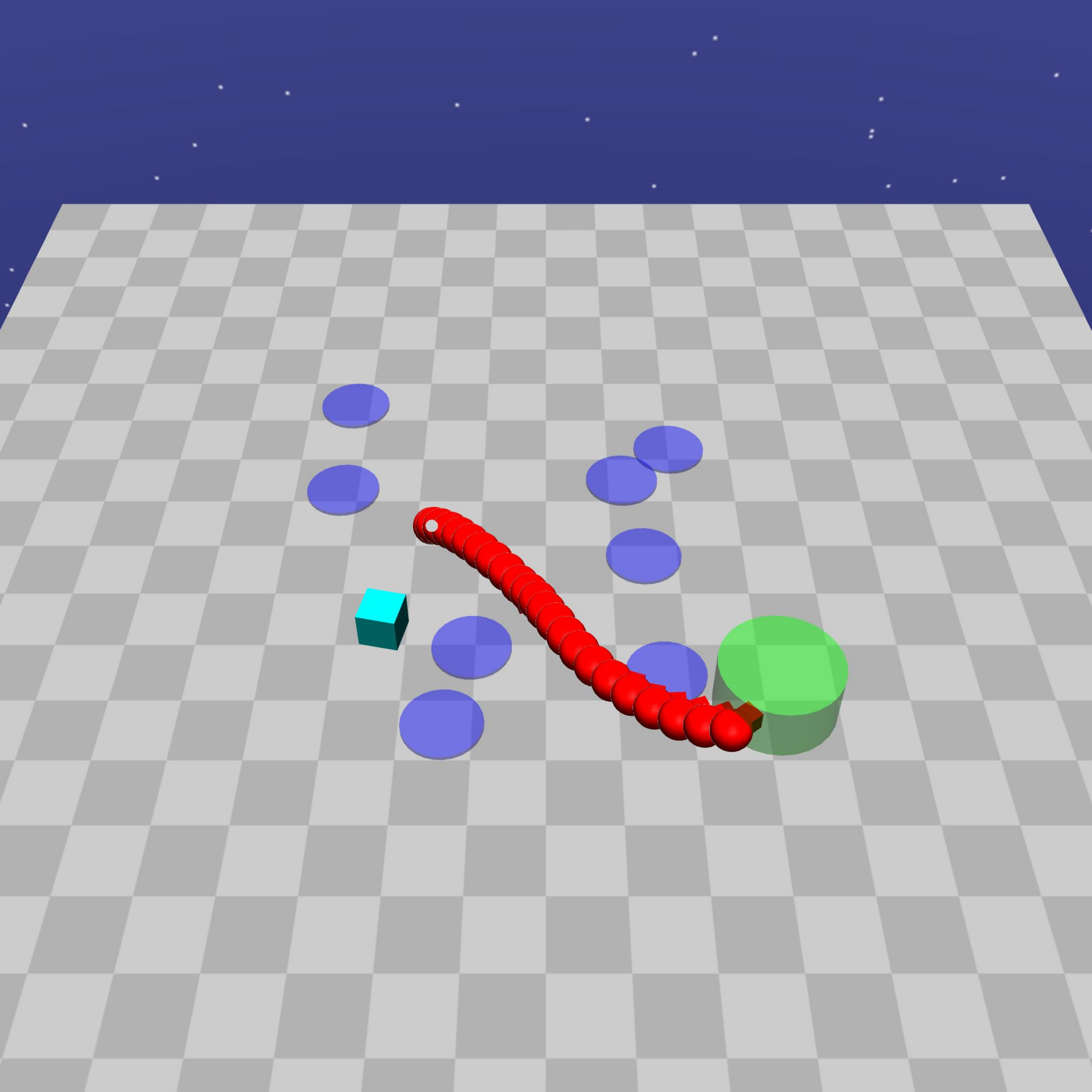}
        \includegraphics[trim=0 300 0 500, clip, width=0.24\linewidth]{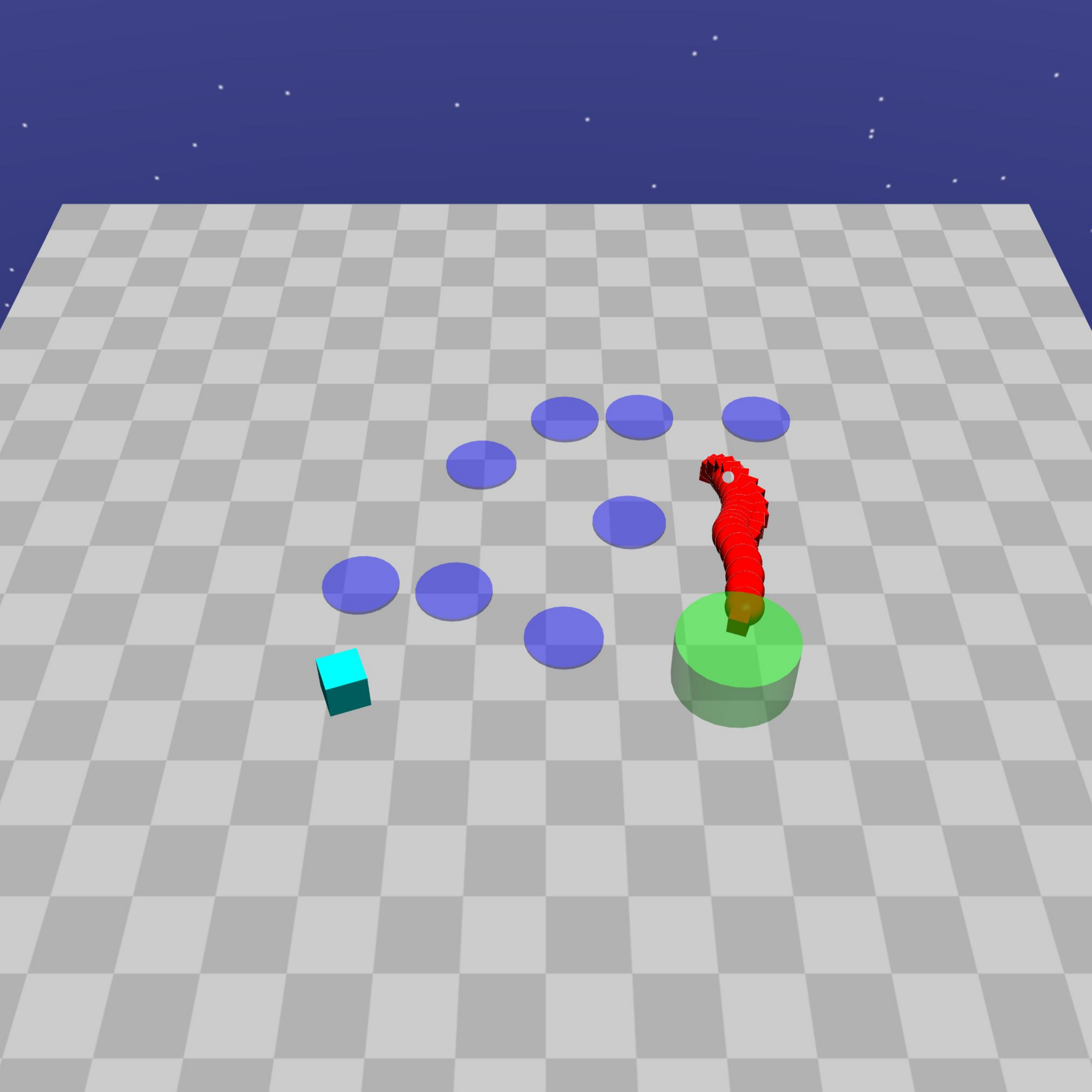}
        \\
        \includegraphics[trim=0 300 0 500, clip, width=0.24\linewidth]{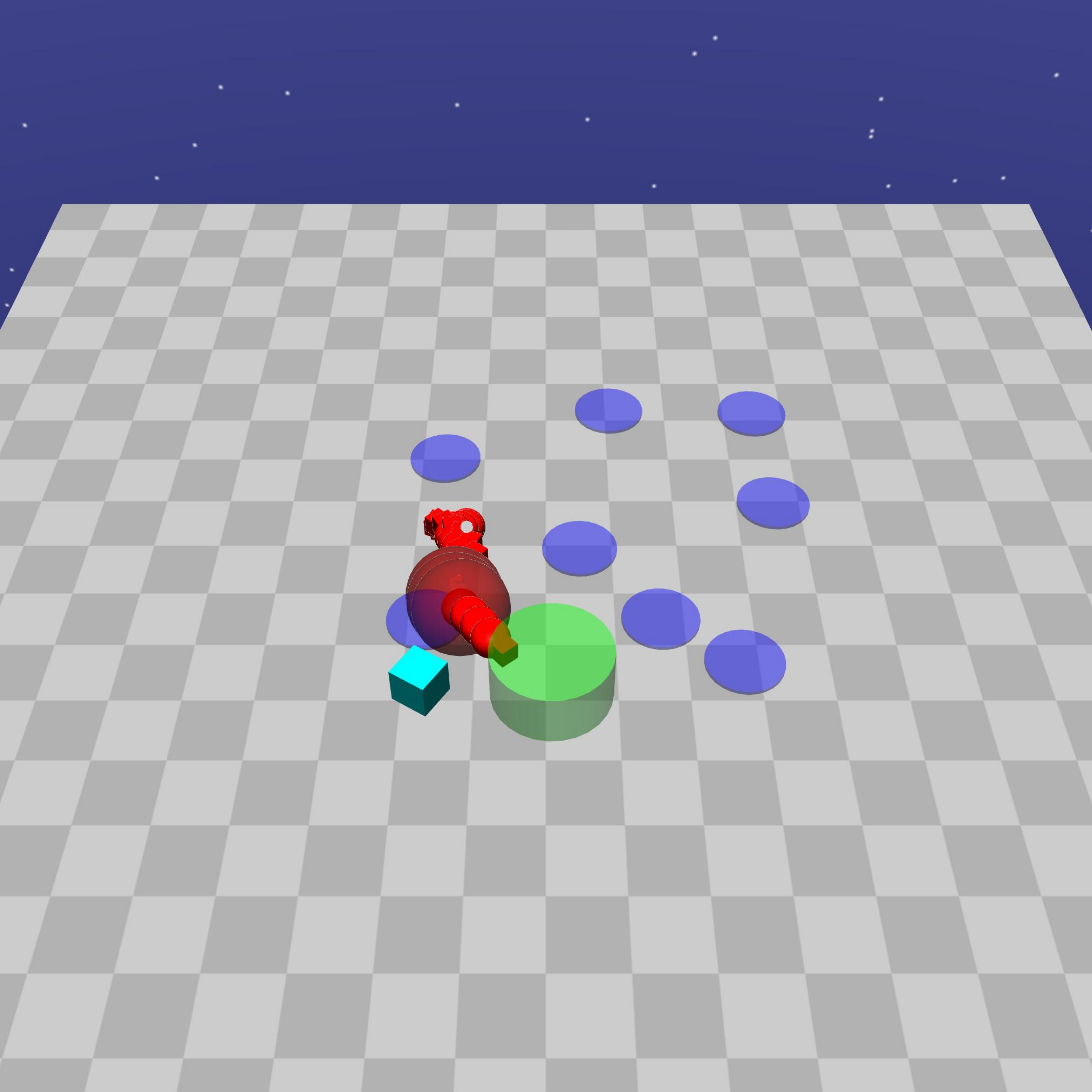}
        \includegraphics[trim=0 300 0 500, clip, width=0.24\linewidth]{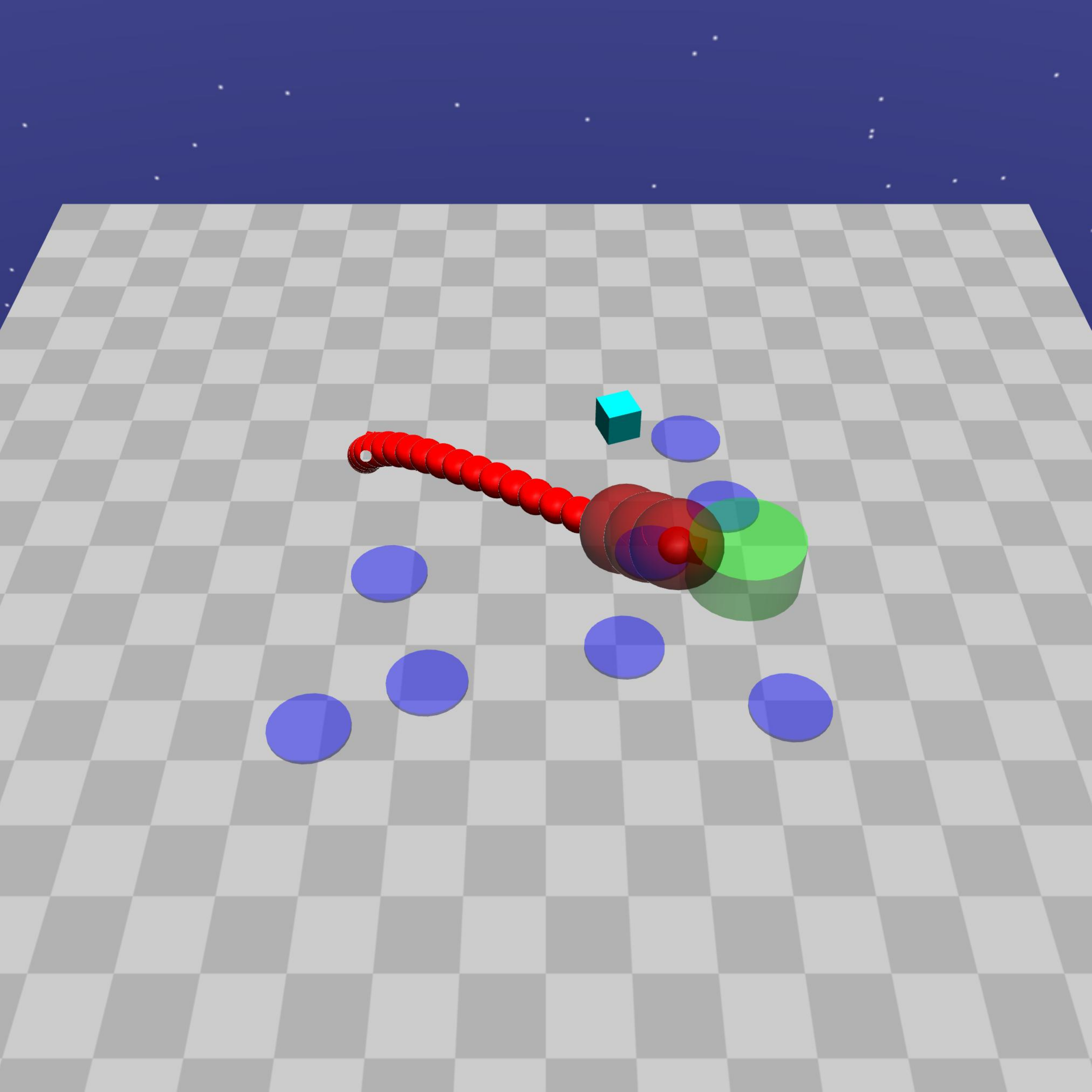}
        \includegraphics[trim=0 300 0 500, clip, width=0.24\linewidth]{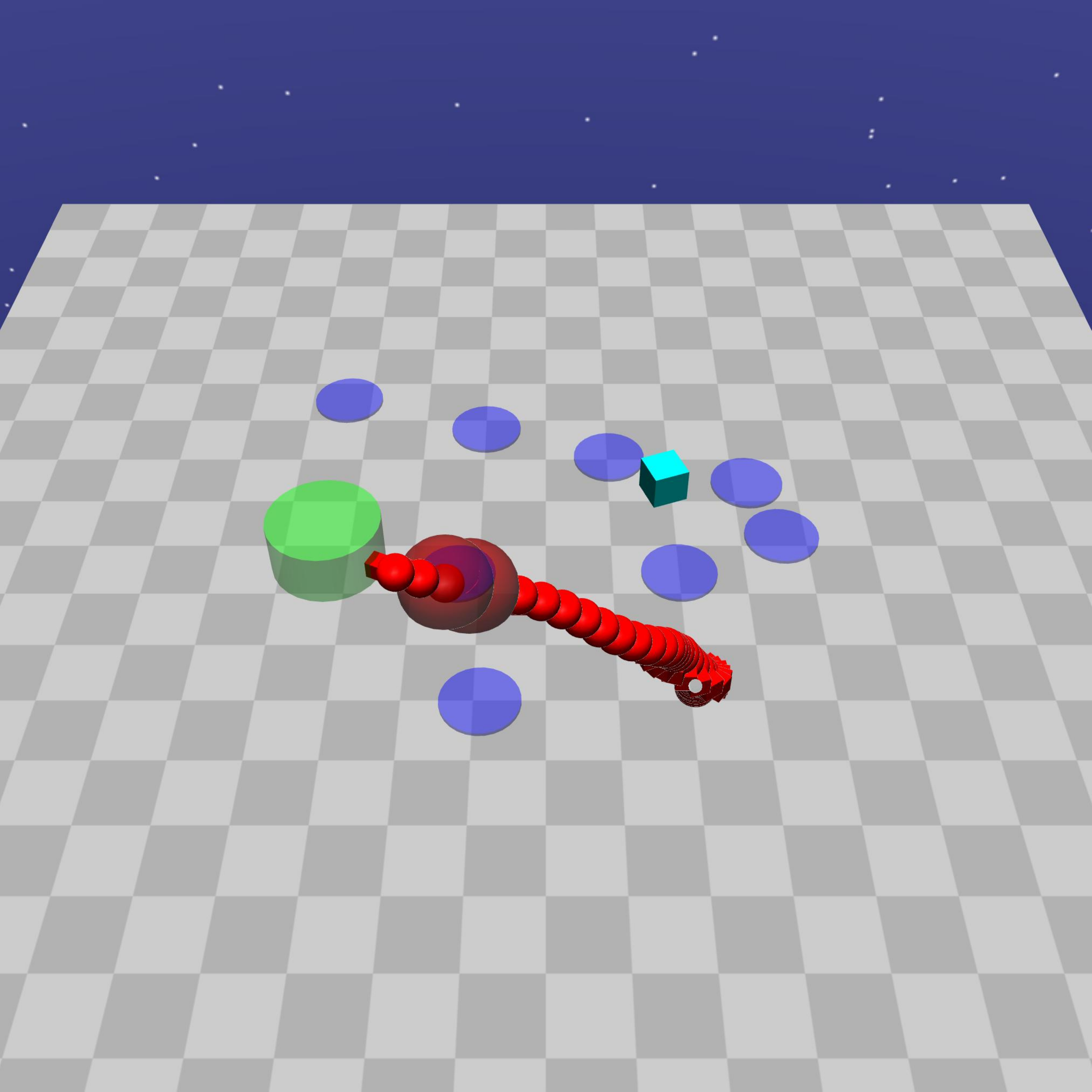}
        \includegraphics[trim=0 300 0 500, clip, width=0.24\linewidth]{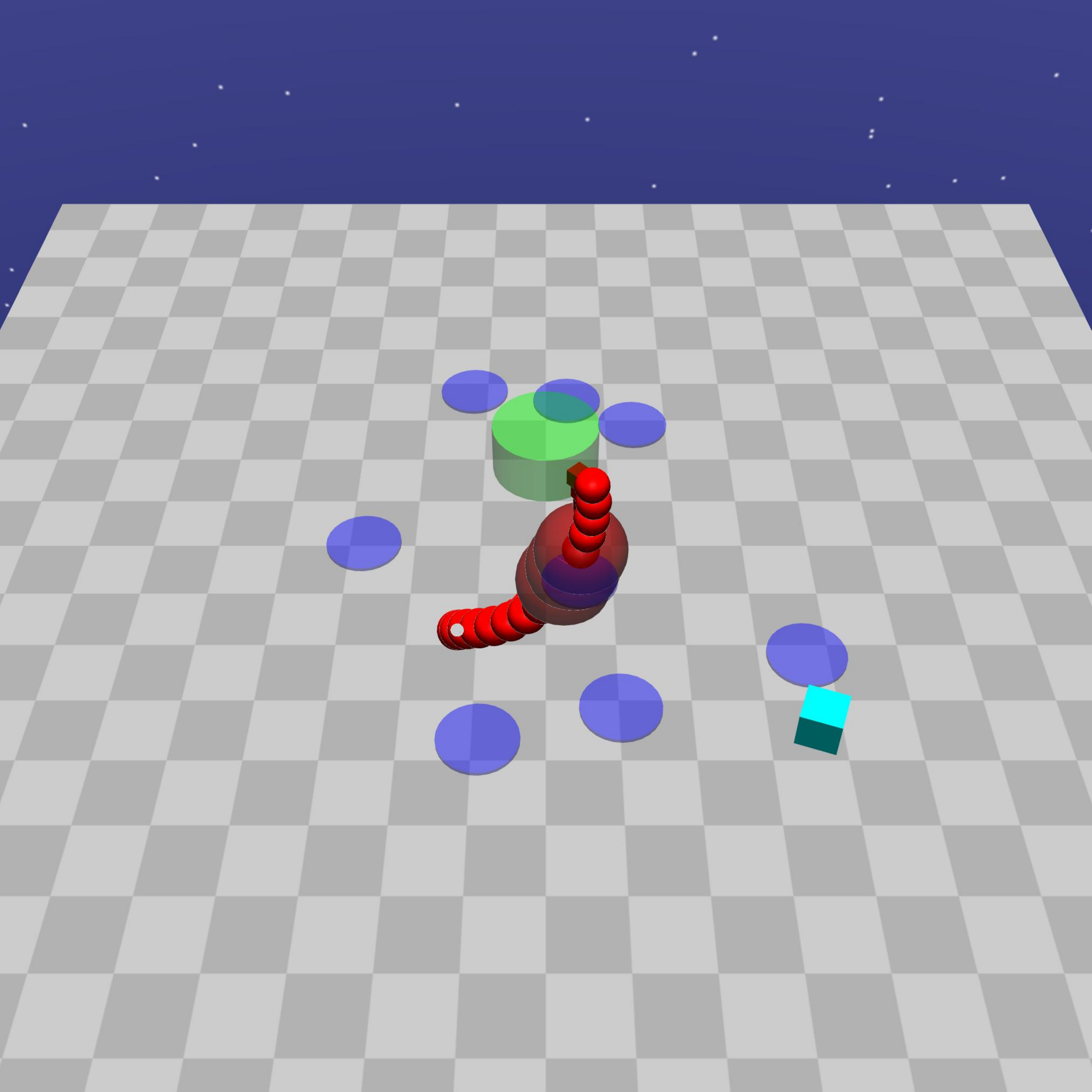}
        \caption{TRPO-Lagrangian successes (top) and failures (bottom)}
    \end{subfigure}%
    \\
    \begin{subfigure}[t]{\textwidth}
        \centering
        \includegraphics[trim=0 300 0 500, clip, width=0.24\linewidth]{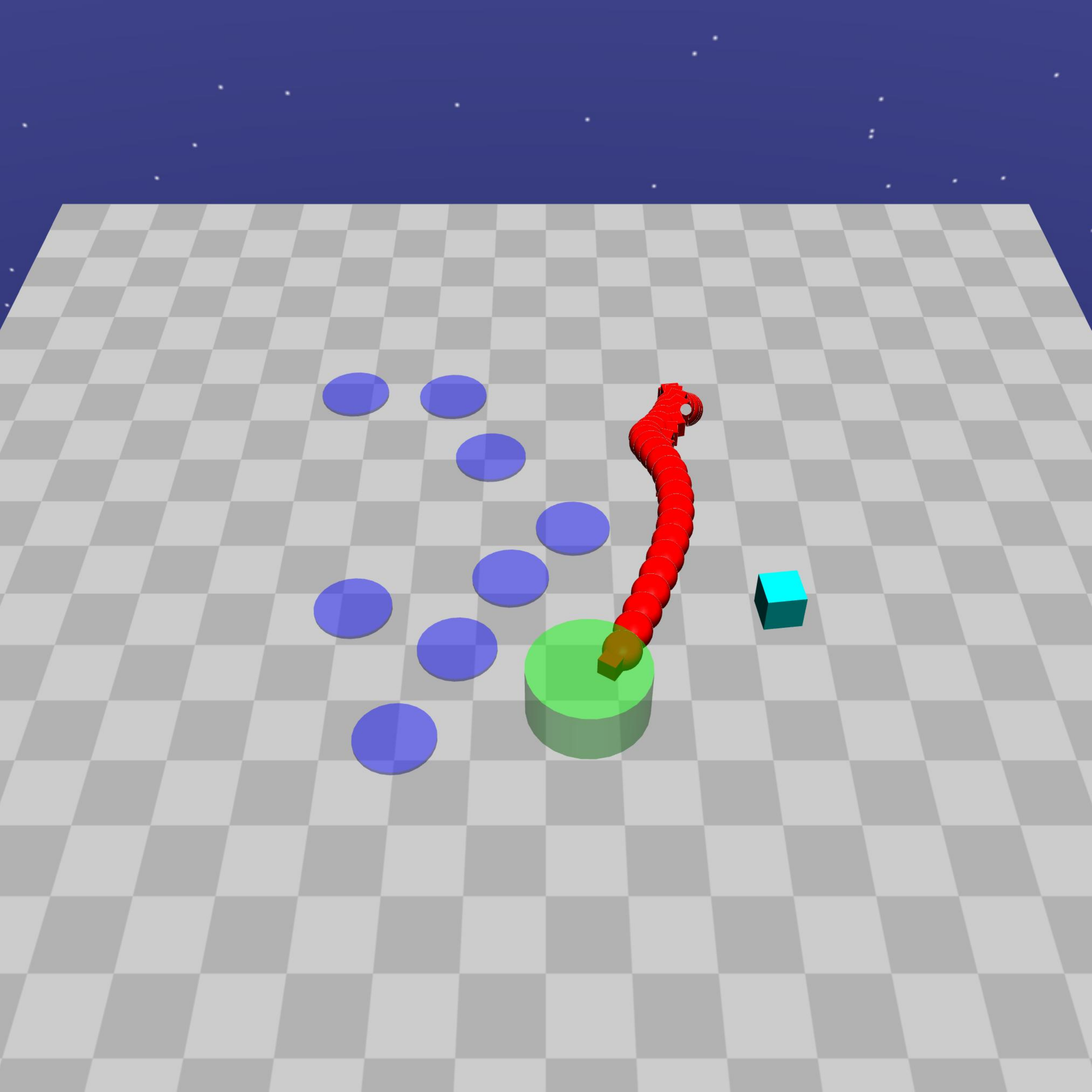}
        \includegraphics[trim=0 300 0 500, clip, width=0.24\linewidth]{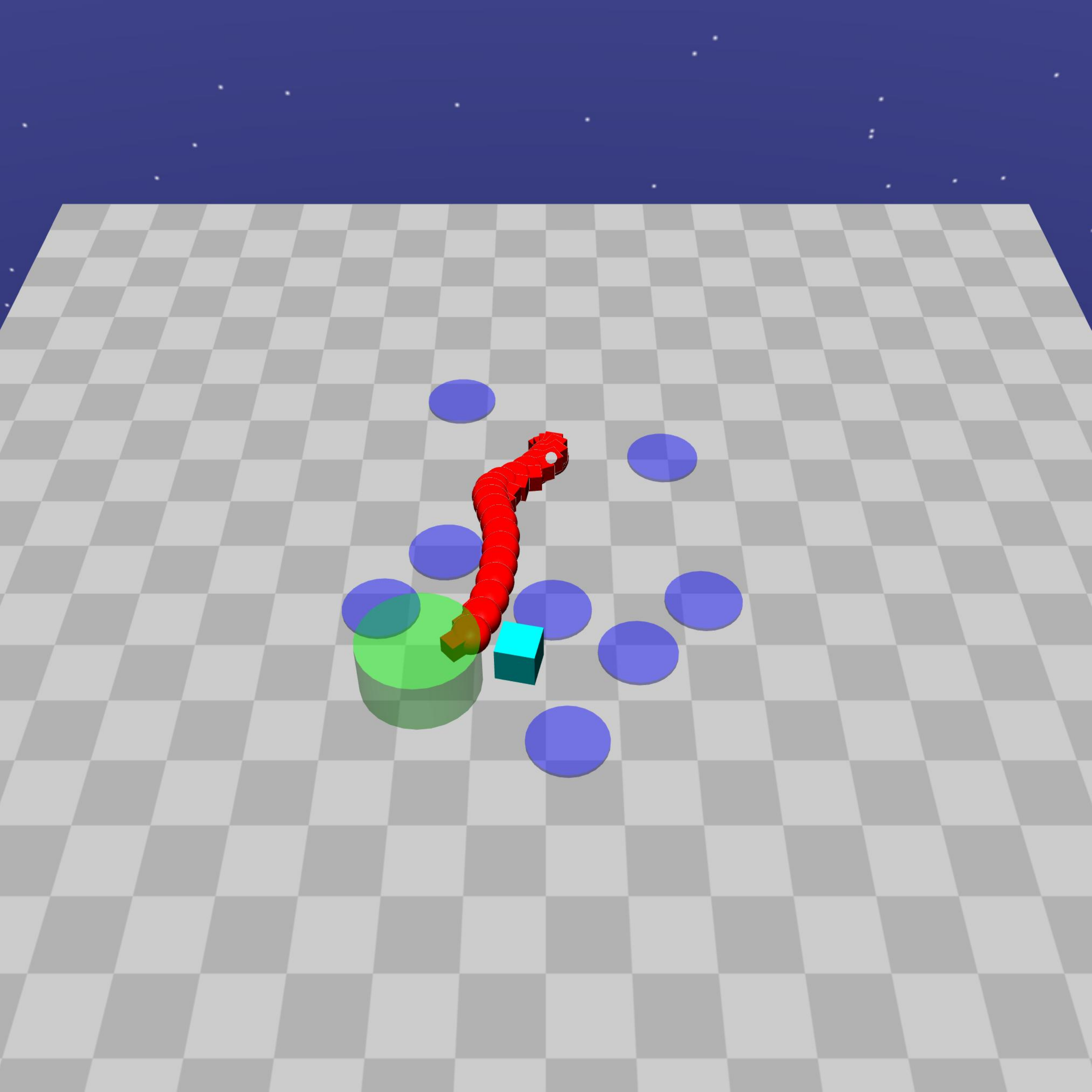}
        \includegraphics[trim=0 300 0 500, clip, width=0.24\linewidth]{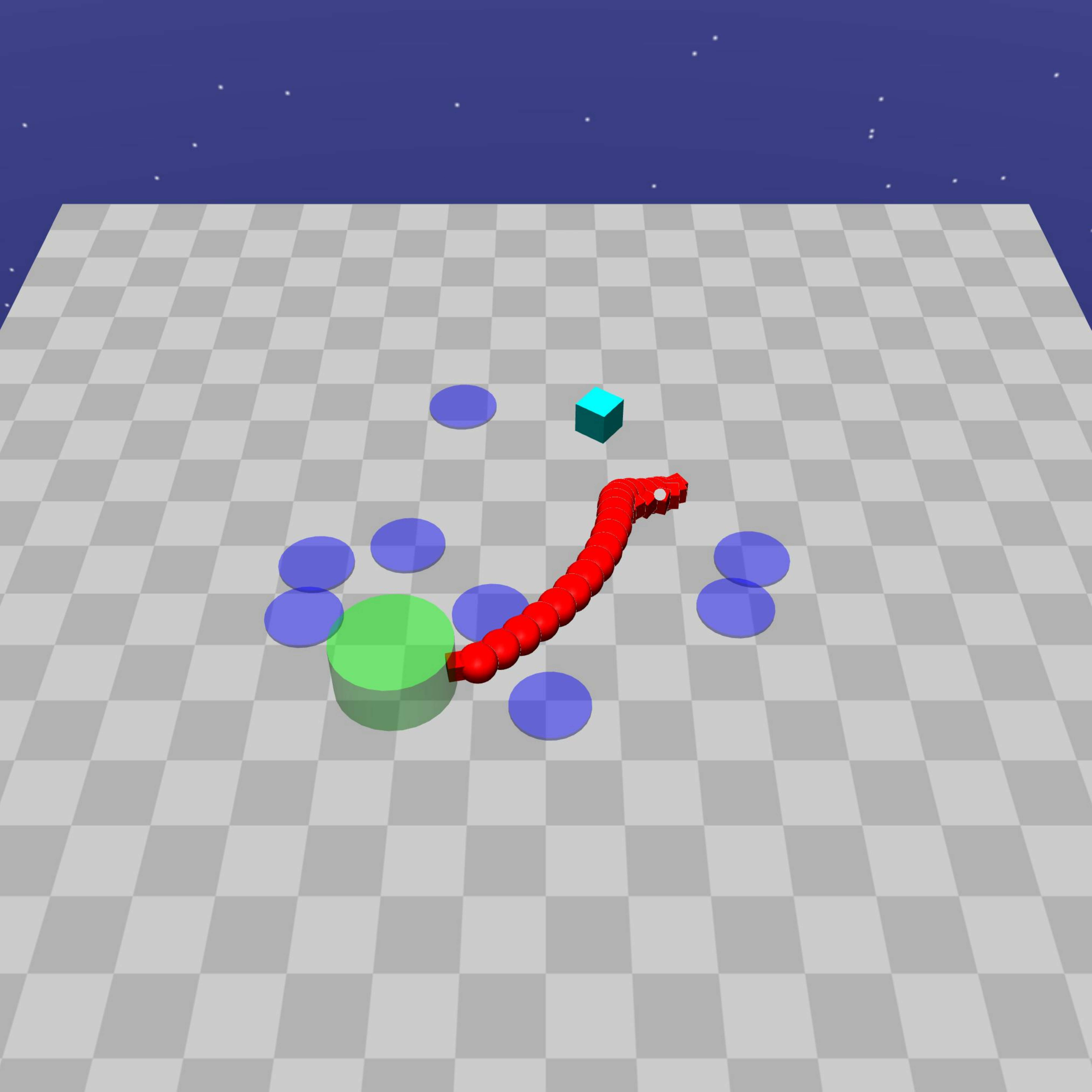}
        \includegraphics[trim=0 300 0 500, clip, width=0.24\linewidth]{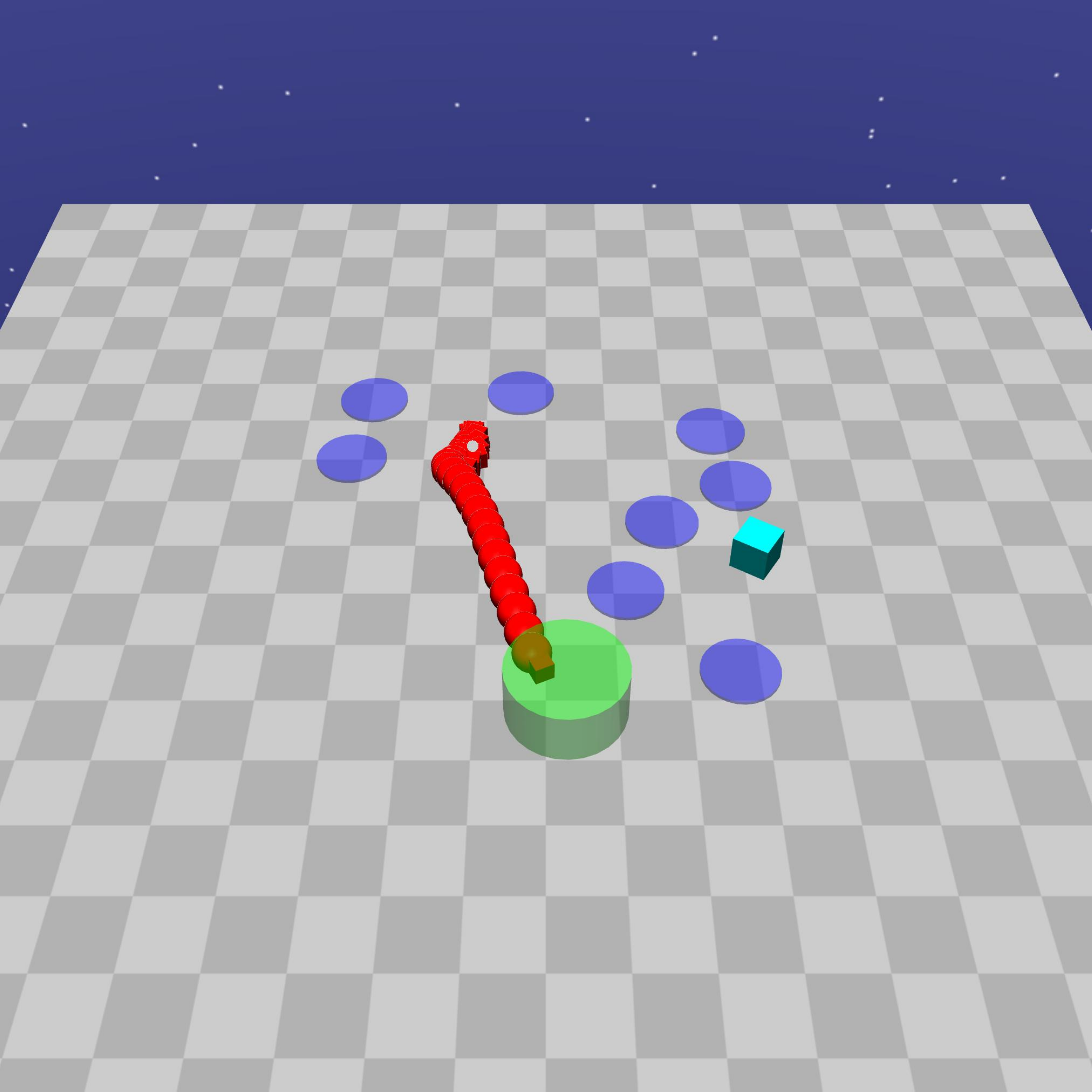}
        \\
        \includegraphics[trim=0 300 0 500, clip, width=0.24\linewidth]{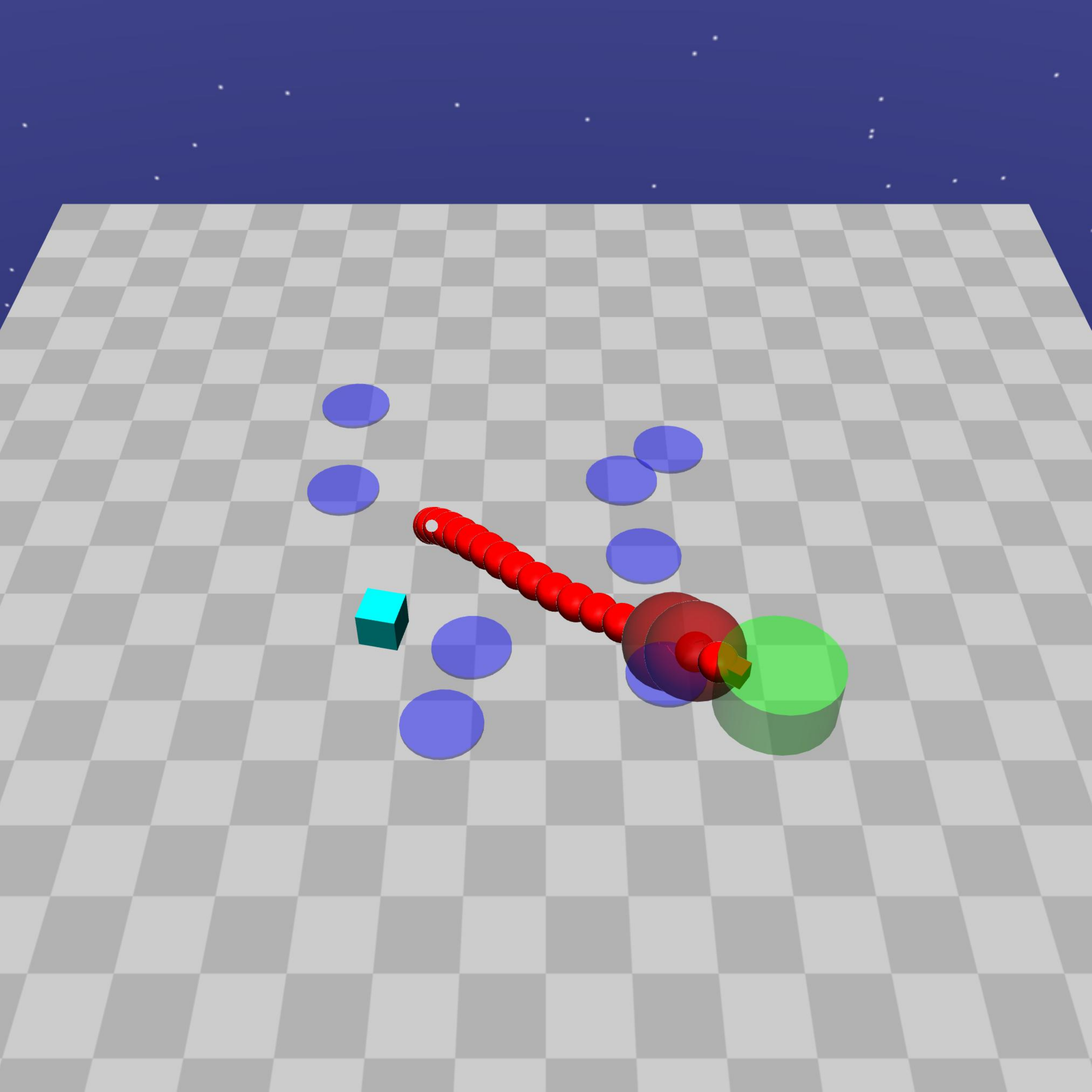}
        \includegraphics[trim=0 300 0 500, clip, width=0.24\linewidth]{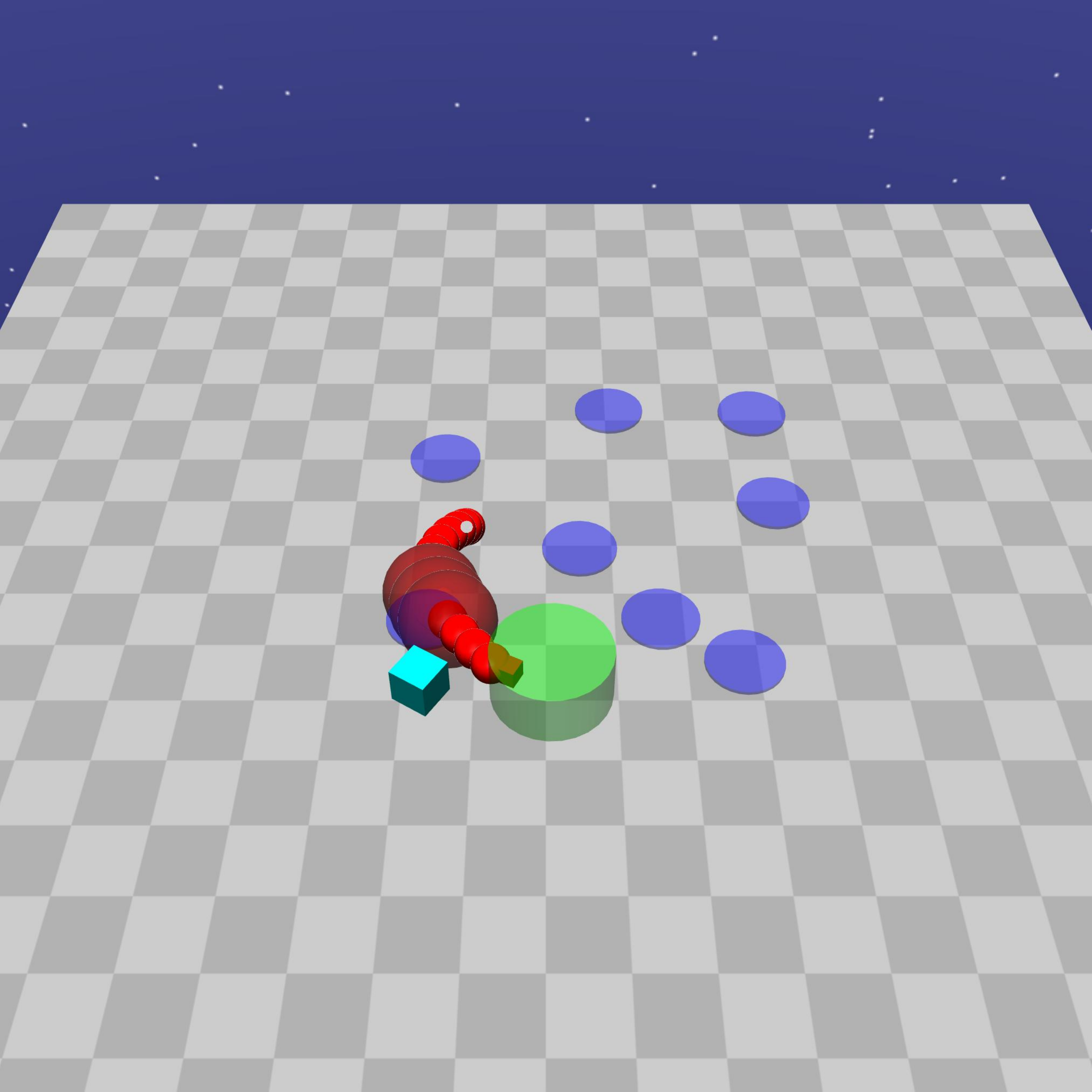}
        \includegraphics[trim=0 300 0 500, clip, width=0.24\linewidth]{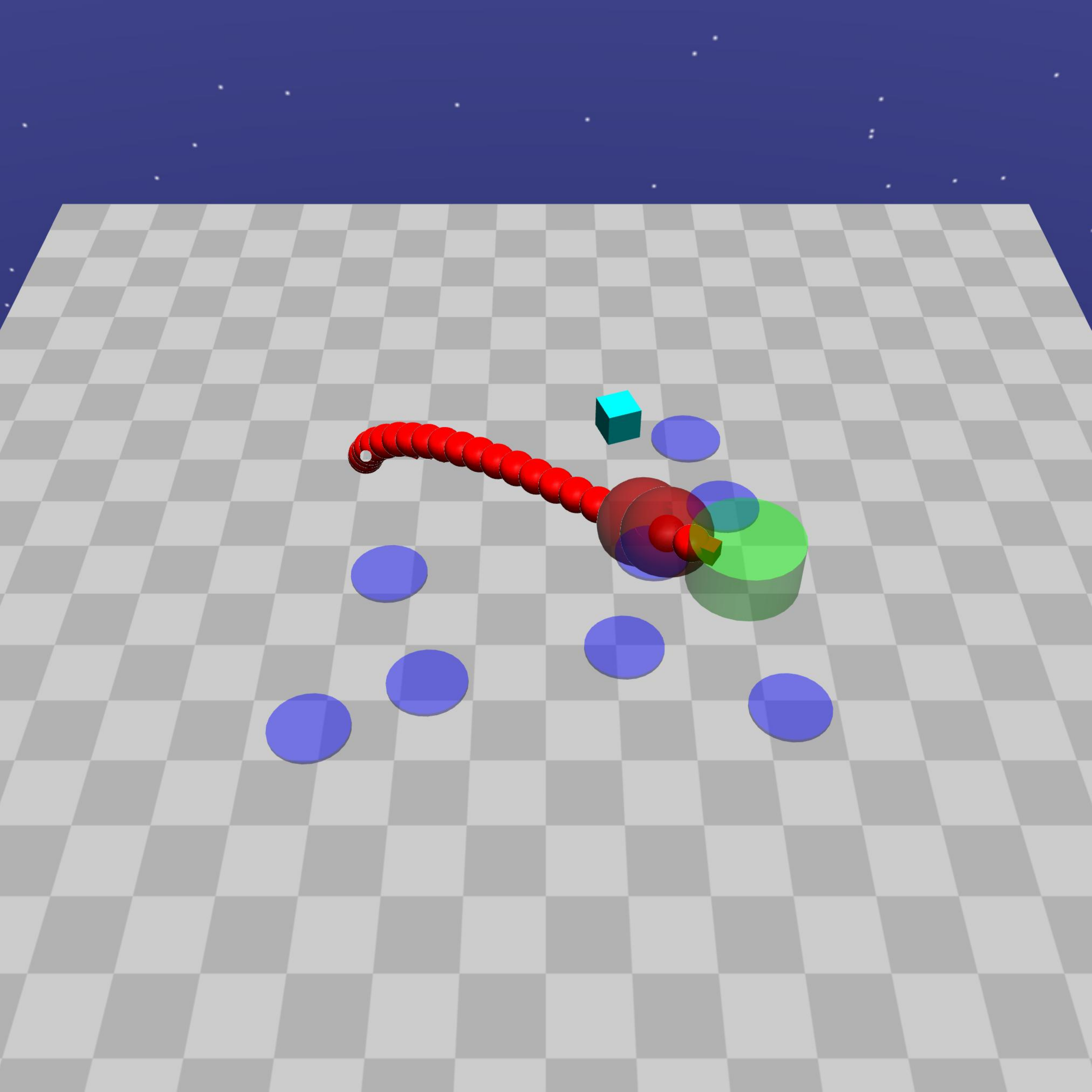}
        \includegraphics[trim=0 300 0 500, clip, width=0.24\linewidth]{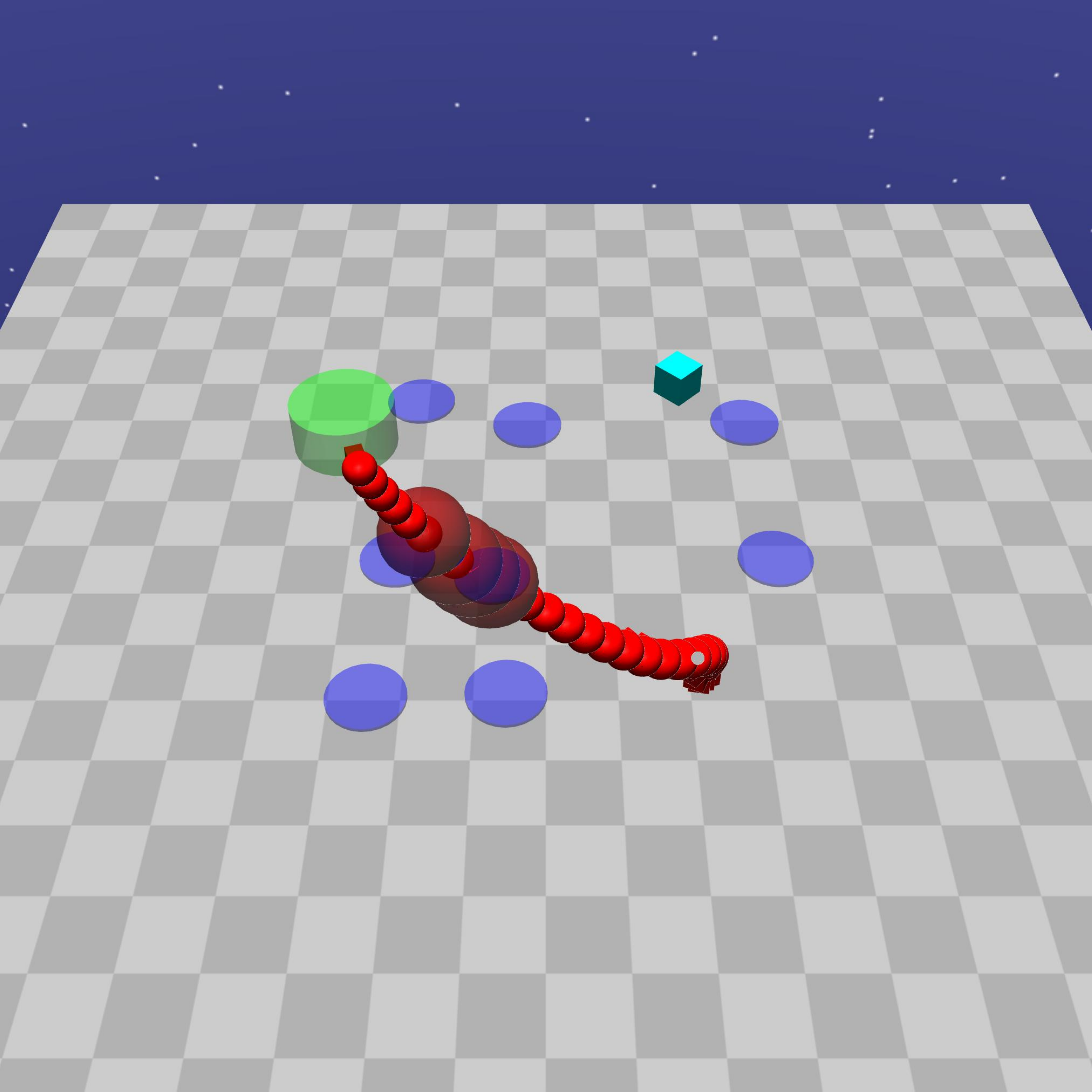}
        \caption{CPO successes (top) and failures (bottom)}
    \end{subfigure}%
    \\
    \begin{subfigure}[t]{\textwidth}
        \centering
        \includegraphics[trim=0 500 0 300, clip, width=0.24\linewidth]{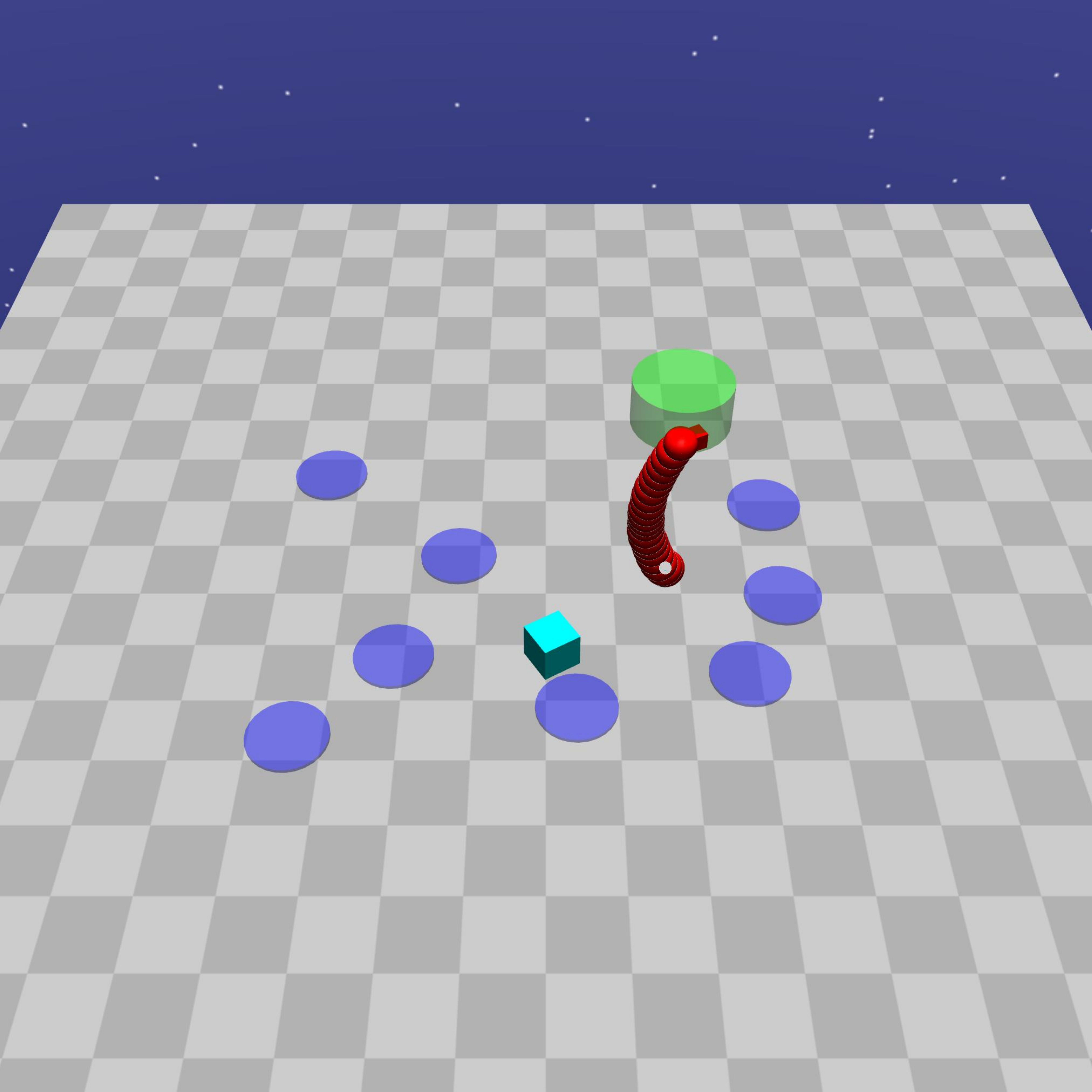}
        \includegraphics[trim=0 500 0 300, clip, width=0.24\linewidth]{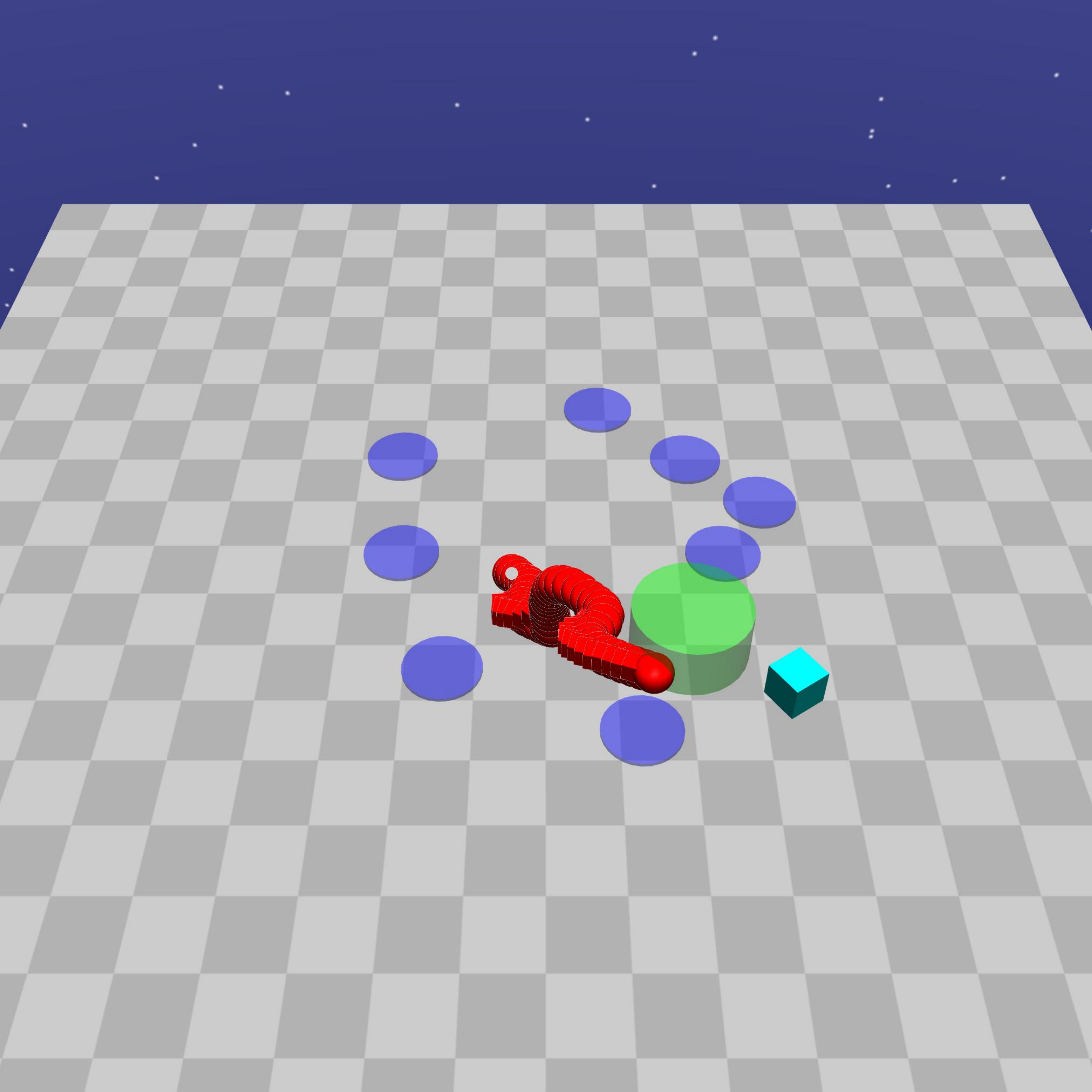}
        \includegraphics[trim=0 500 0 300, clip, width=0.24\linewidth]{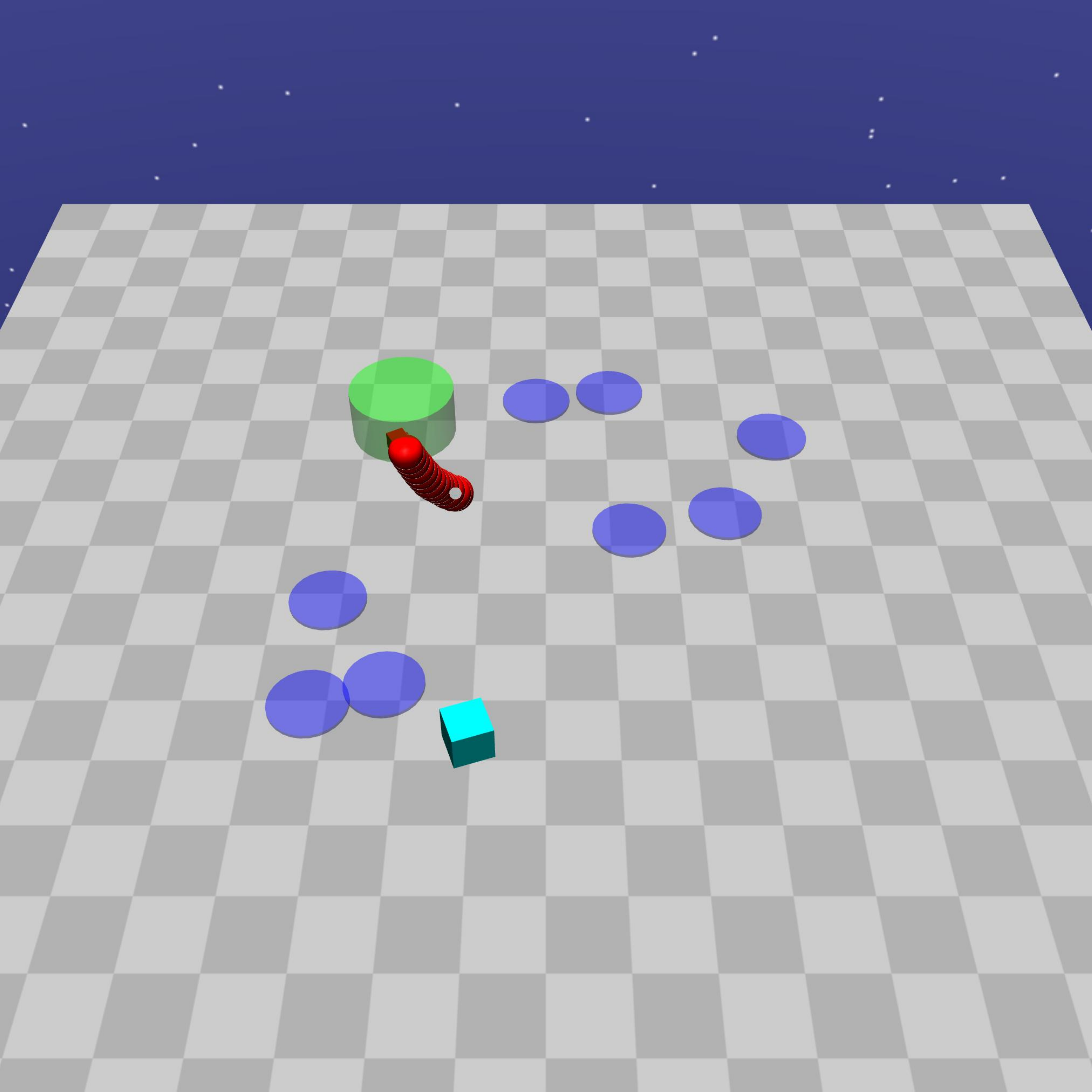}
        \includegraphics[trim=0 500 0 300, clip, width=0.24\linewidth]{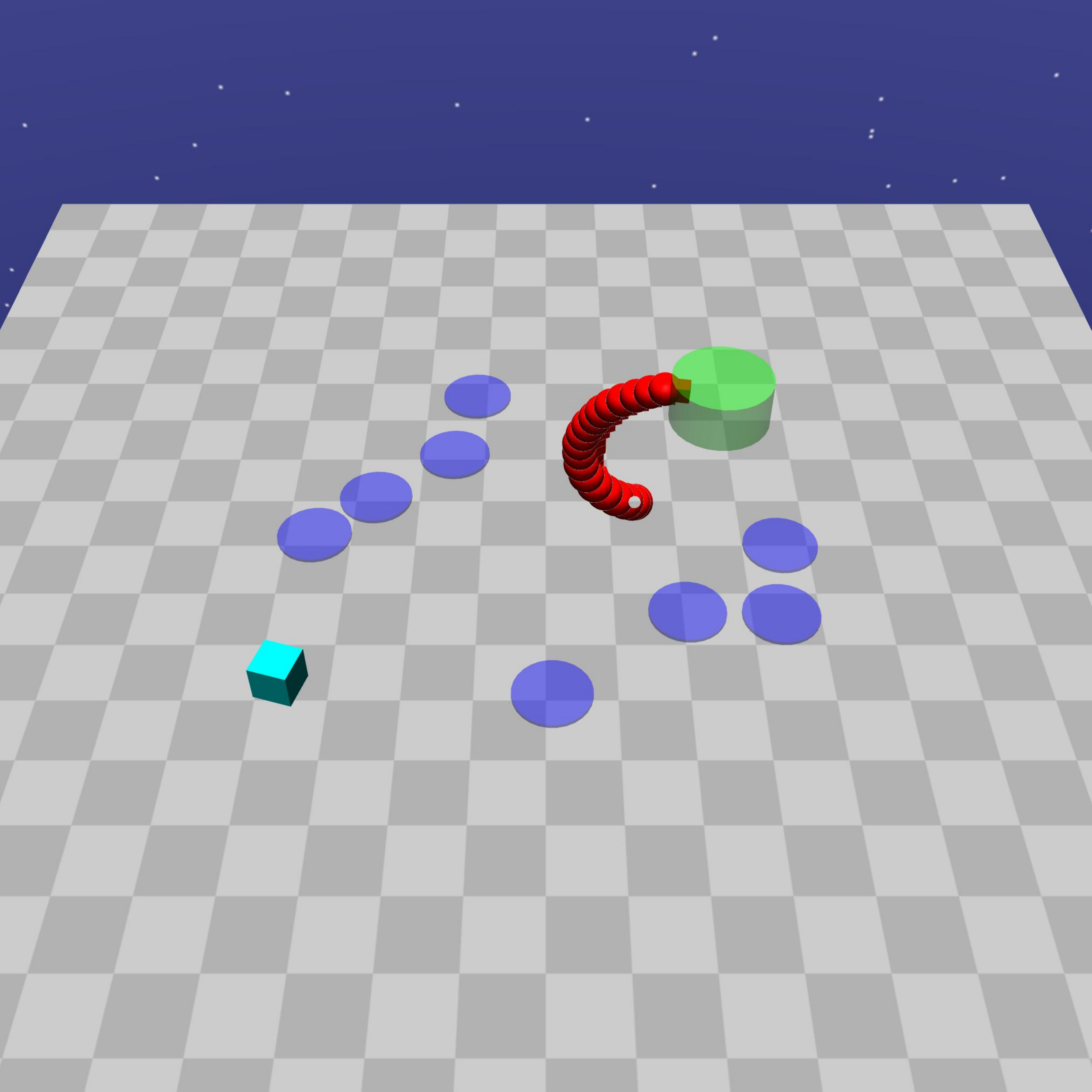}
        \\
        \includegraphics[trim=20 0 20 300, clip, width=0.24\linewidth]{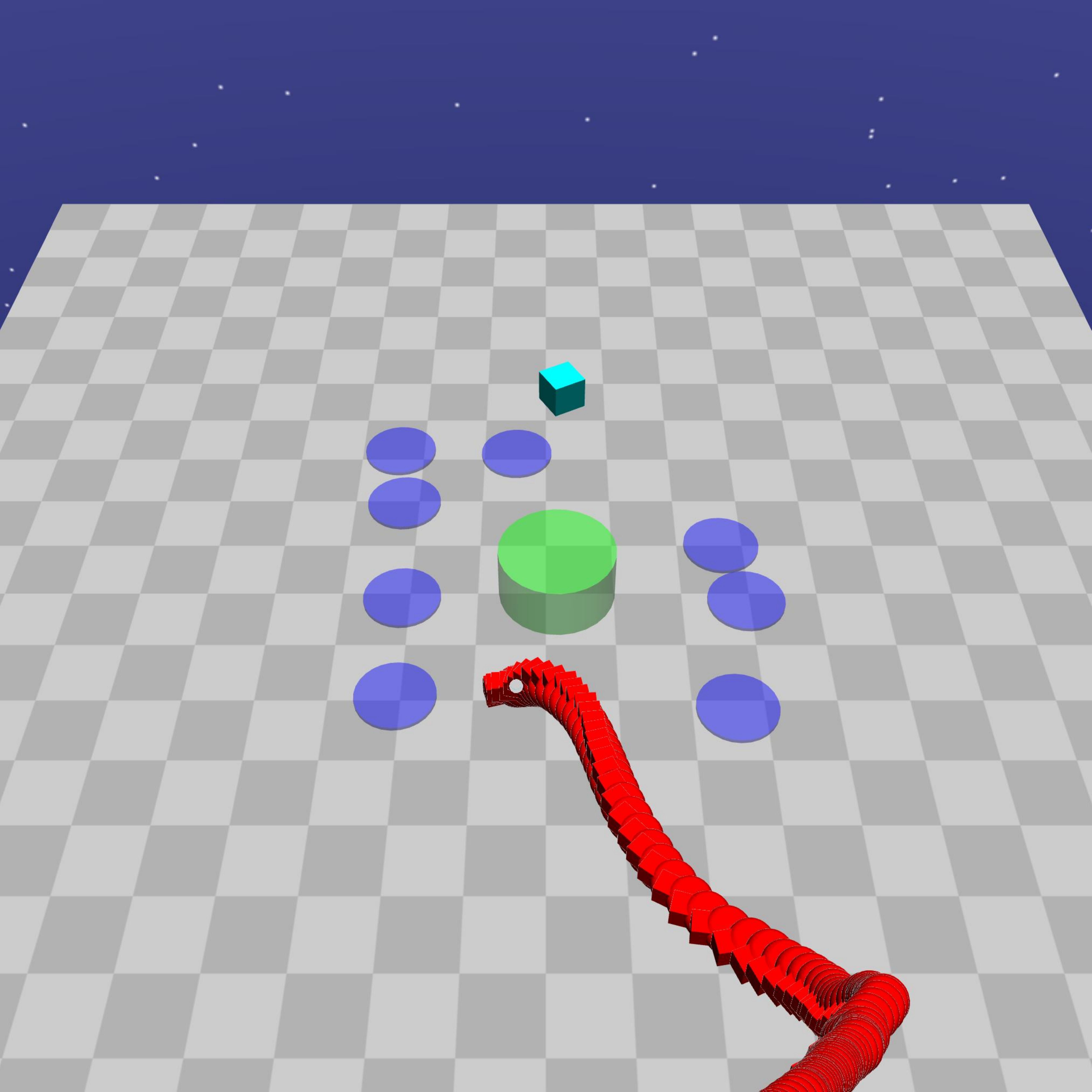}
        \includegraphics[trim=20 0 20 300, clip, width=0.24\linewidth]{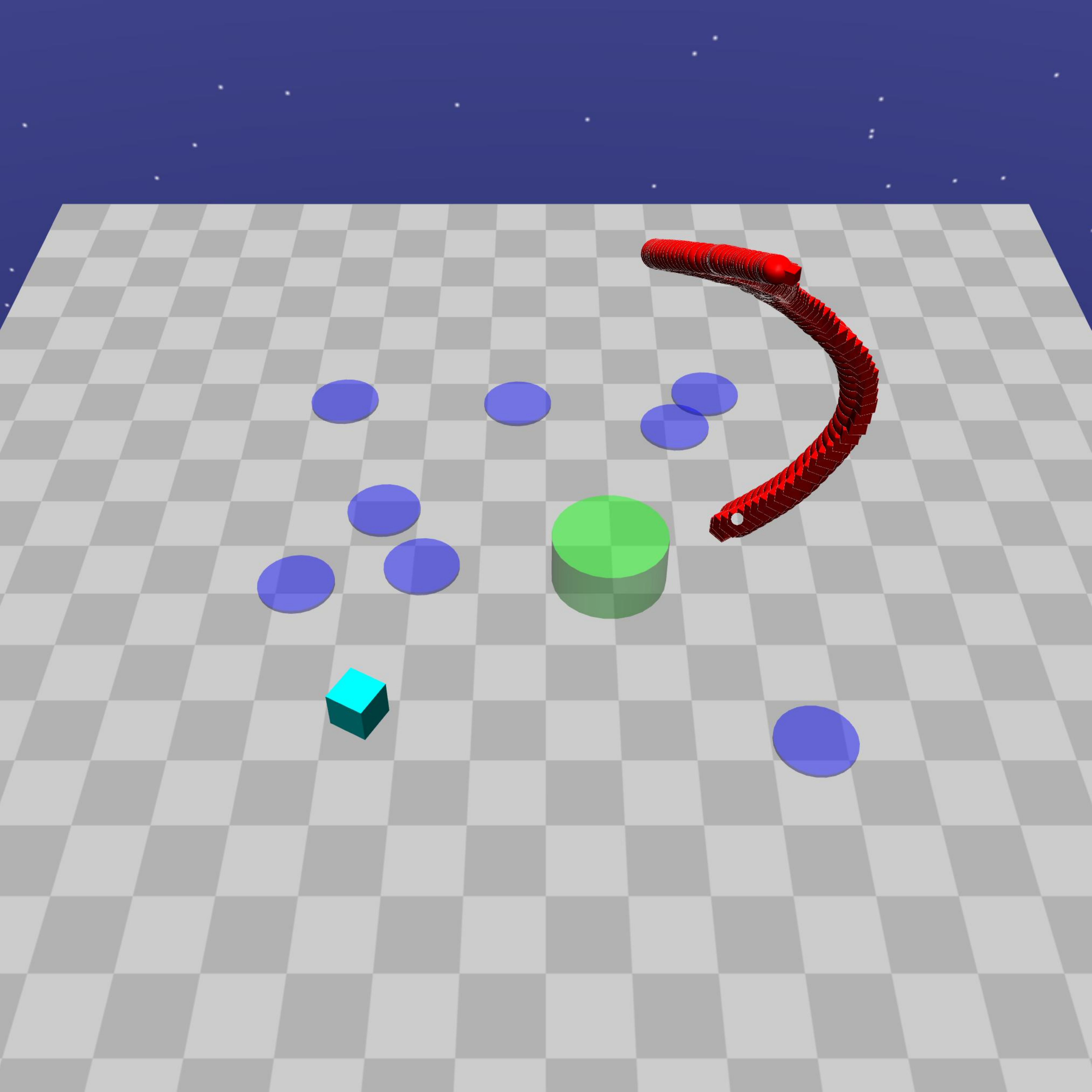}
        \includegraphics[trim=20 0 20 300, clip, width=0.24\linewidth]{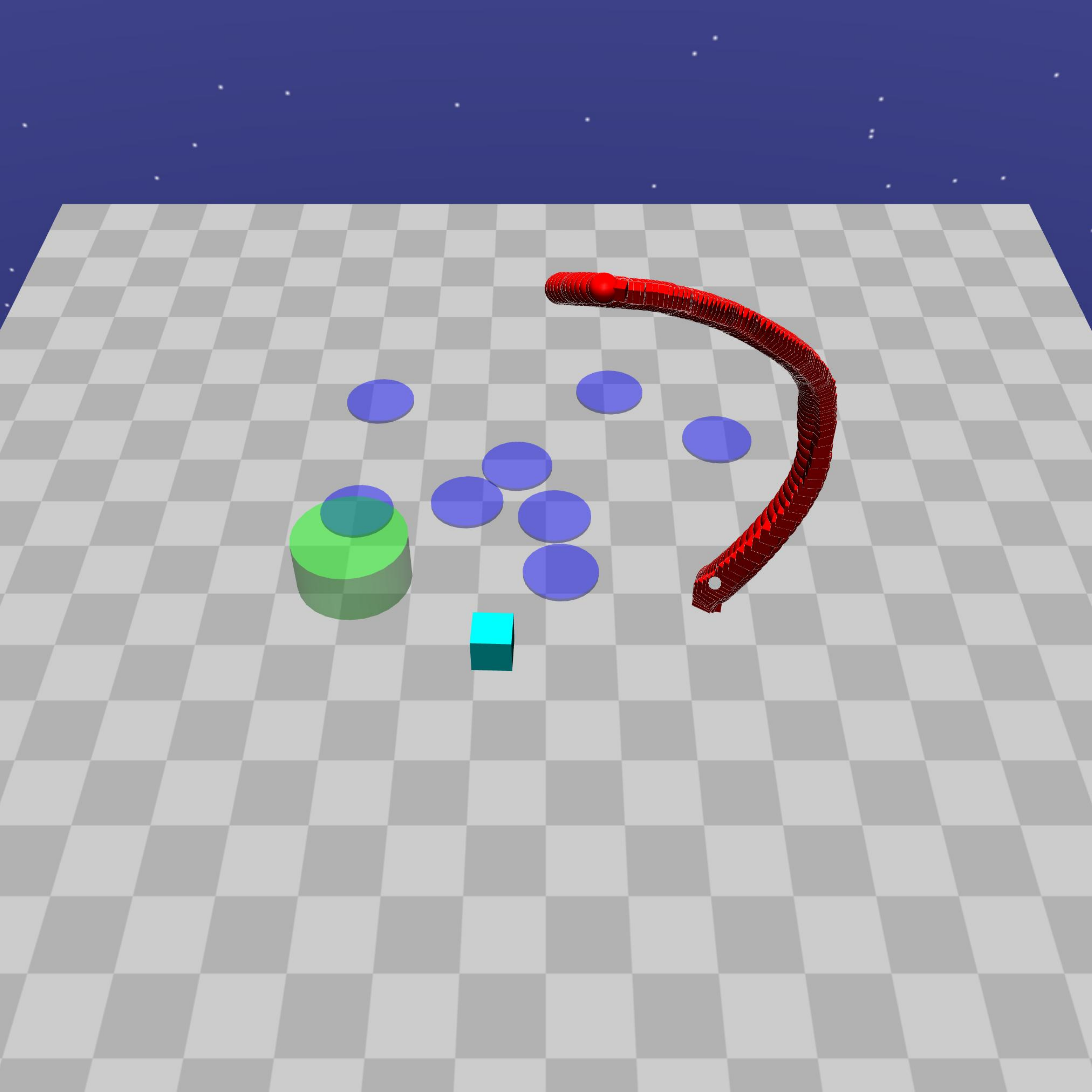}
        \includegraphics[trim=20 0 20 300, clip, width=0.24\linewidth]{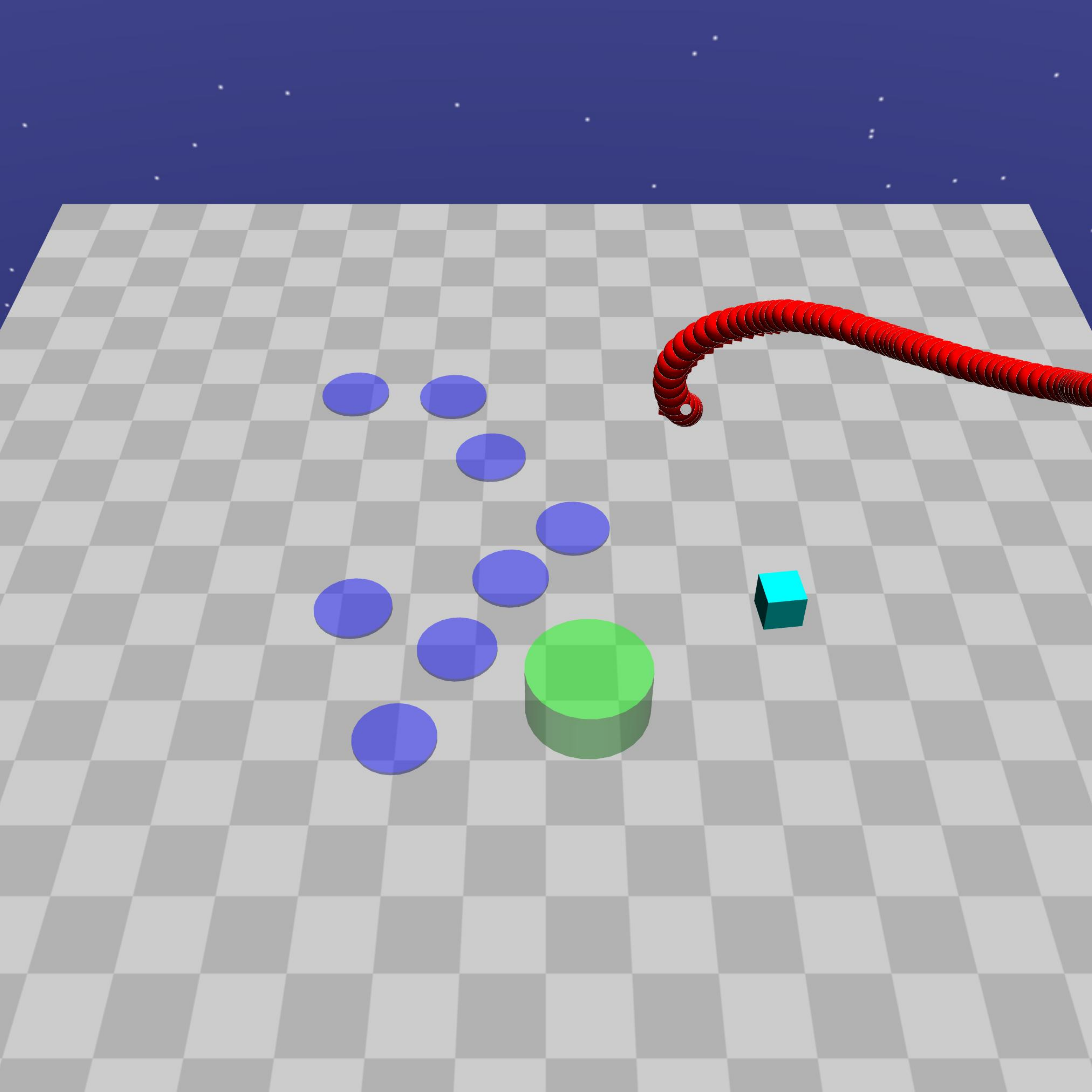}
        \caption{TRPO-Minmax successes (top) and failures (bottom)}
    \end{subfigure}%
    \caption{Sample trajectories of policies learned by each baseline and our Minmax approach in the Safety Gym \pointgoalhard domain, in the experiments of Figure \ref{fig:safety_baselines_original}.
    Trajectories that hit hazards (the hits are highlighted by the red spheres) or take more than 1000 timesteps to reach the goal location are considered failures.   
    }
    \label{fig:all_trajectories_original}
\end{figure*}%

\end{document}